%% file: thesis.tex
\documentclass[a4paper,11pt,oneside]{book}
\pdfoutput=1 
\usepackage[nodisplayskipstretch]{setspace}
\usepackage{etoolbox}
\usepackage{amsthm}
\usepackage{amsmath}
\usepackage{amssymb}
\usepackage{amstext}
\usepackage{amsbsy}
\usepackage{latexsym}
\usepackage{amsfonts}
\usepackage{graphicx,color,subfigure}
\usepackage{algorithm}
\usepackage{algorithmicx}
\usepackage{algpseudocode}
\usepackage[margin=10pt,font=small,labelfont=bf,justification=centering,margin=14pt]{caption}
\usepackage[pdftex=true, a4paper, left=3.7cm,top=1.5cm,right=1.7cm,bottom=2cm,includehead=true]{geometry}
\usepackage{url}
\usepackage{verbatim}
\usepackage{multirow}
\usepackage{color, soul, verbatim}

  \include{definitions}

  \newcommand{\blanknonumber}{\newpage\thispagestyle{empty}}

\allowdisplaybreaks

\begin{document}
	\abovedisplayskip=6pt
	\abovedisplayshortskip=6pt
	\belowdisplayskip=6pt
	\belowdisplayshortskip=6pt

  \frontmatter
		

  \include{titlepage}\blanknonumber
  \include{dedication}\blanknonumber
  \include{declaration}\blanknonumber
  \include{acknowledgements}\blanknonumber
  \include{abstract}\blanknonumber
  \tableofcontents\blanknonumber
	\listoffigures\blanknonumber
	\listofalgorithms\blanknonumber
  \include{notation}\blanknonumber

  \mainmatter

	\include{overview}
  \include{copulae}
  \include{pgm}
  \include{inference}
	\include{sampling}
	\include{learning}
	\include{experiments}
	\include{conclusions}


  \appendix



  \addcontentsline{toc}{chapter}{Bibliography}

	\bibliographystyle{abbrv}
  \bibliography{references}  


\end{document}

%% file: definitions.tex
  \newtheorem{theorem}{Theorem}[chapter]

  \newtheorem{cor}[theorem]{Corollary}

  \newtheorem{defn}[theorem]{Definition}
  \newtheorem{ex}[theorem]{Example}
	\newtheorem{experiment}[theorem]{Experiment}

%% file: titlepage.tex
\begin{titlepage}
\label{titlepage}
\begin{center}

\vspace*{\fill} \Huge
                        Inference, Sampling, and Learning in Copula Cumulative Distribution Networks
\\
\vfill\vfill\Large
                          Stefan Webb
\\
\vfill\vfill
                          May, 2013
\\
\vfill\vfill \normalsize
         Submitted in partial fulfilment of the requirements\\
         for the degree of Bachelor of Statistics with Honours\\
         in Statistics at the Australian National University
\vfill
         \includegraphics[scale=0.45]{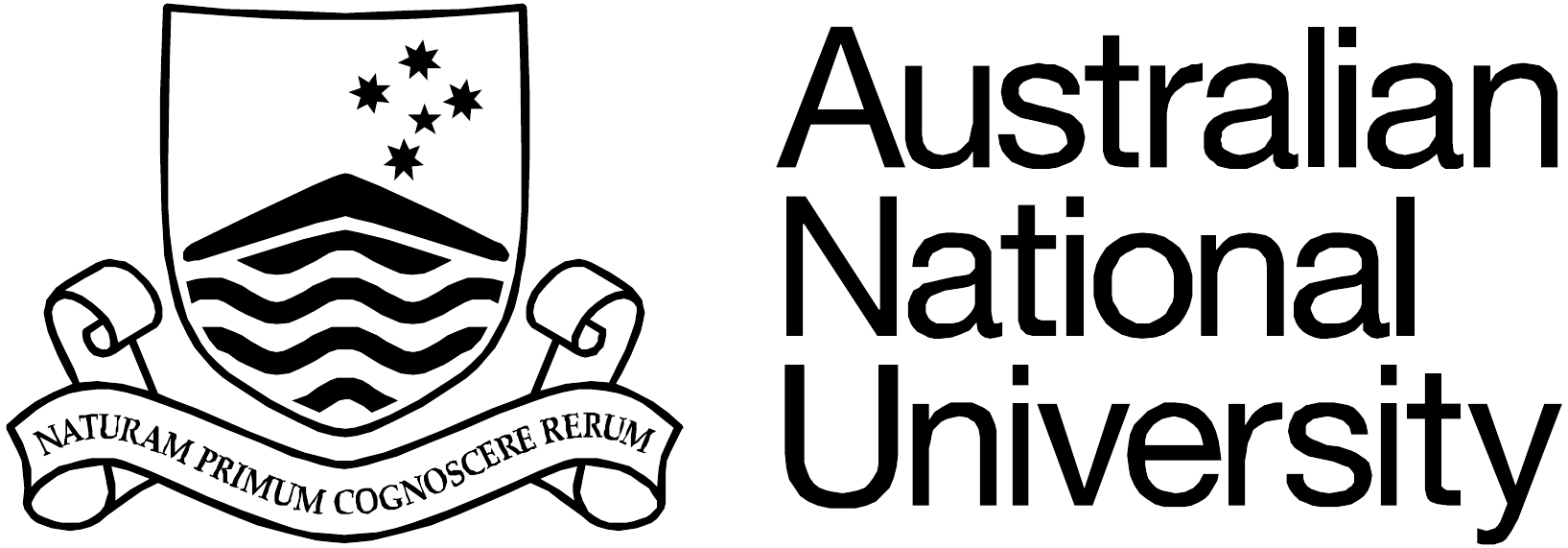}
         
\end{center}

\end{titlepage}

%% file: declaration.tex
\chapter*{Declaration}\label{declaration}
\thispagestyle{empty}
This thesis contains no material that has been accepted for the award of any other degree or diploma in any university, and, to the best of my knowledge and belief, contains no material published or written by another person, except where due reference is made in the thesis.

\vspace{1in}

\hfill\hfill\hfill
Stefan Webb
\hspace*{\fill}

%% file: acknowledgements.tex
\chapter*{Acknowledgements}\label{acknowledgements}
\addcontentsline{toc}{chapter}{Acknowledgements}
Foremost, I would like to express my deep gratitude to Dr Stephen Gould, my primary supervisor, for his thoughtful mentoring, patient guidance, and useful critiques of this research work. I would also like to thank Dr Steven Roberts for his support, and Dr Stephen Sault and Dr Chris Bilson for their excellent management of the Honours program.

I am grateful to Microsoft for providing a complimentary copy of ``Visual Studio 2012 Professional'' under the DreamSpark program. Without its powerful features for debugging and profiling, the software component would have been insufferable. I am also grateful to S. M. Ali Eslami for supplying shape model datasets.

Finally, I wish to thank my parents for their support and encouragement throughout my study.

%% file: abstract.tex
\chapter*{Abstract}\label{abstract}

\addcontentsline{toc}{chapter}{Abstract}

The cumulative distribution network (CDN) \cite{HuangFrey2011} is a recently developed class of probabilistic graphical models (PGMs) permitting a copula factorization, in which the CDF, rather than the density, is factored. Despite there being much recent interest within the machine learning community about copula representations, there has been scarce research into the CDN, its amalgamation with copula theory, and no evaluation of its performance. Algorithms for inference, sampling, and learning in these models are underdeveloped compared those of other PGMs, hindering widerspread use.

One advantage of the CDN is that it allows the factors to be parameterized as copulae \cite{SilvaEtAl2011}, combining the benefits of graphical models with those of copula theory. In brief, the use of a copula parameterization enables greater modelling flexibility by separating representation of the marginals from the dependence structure, permitting more efficient and robust learning. Also, directly modelling the CDF is more appropriate for some tasks. Another advantage is that the CDN permits the representation of implicit latent variables, whose parameterization and connectivity are not required to be specified. Unfortunately, that the model can encode \emph{only} latent relationships between variables severely limits its utility.

In this thesis, we present inference, learning, and sampling for CDNs, and further the state-of-the-art. First, we explain the basics of copula theory and the representation of copula CDNs. Then, we discuss inference in the models, and develop the first sampling algorithm. We explain standard learning methods, propose an algorithm for learning from data missing completely at random (MCAR), and develop a novel algorithm for learning models of arbitrary treewidth and size. Properties of the models and algorithms are investigated through Monte Carlo simulations. We conclude with further discussion of the advantages and limitations of CDNs, using the insight gained from our experiments, and suggest future work.

%% file: overview.tex
\chapter{Overview}
\label{overview}
The cumulative distribution network (CDN) \cite{HuangFrey2011} is a recently developed class of probabilistic graphical models (PGMs) permitting a copula factorization, in which the CDF, rather than the density, is factored. CDNs encode a different set of conditional independence relationships from standard GMs, and as such represents a new class of models. Despite there being much recent interest within the machine learning community about copula representations, there has been scarce research into the CDN, its amalgamation with copula theory, and no evaluation of its performance. Algorithms for inference, sampling and learning in these models are underdeveloped compared those of other PGMs, hindering widerspread use.

Probabilistic Graphical models \cite{KollerFriedman2009} are used as a general framework for compactly representing high-dimensional probability distributions and encoding conditional independence relationships, for which there exist efficient algorithms for inference, sampling, and learning. They have found diverse applications in, to name a few, information extraction, medical diagnosis, speech recognition, and computational biology, and have demonstrated superior performance over earlier techniques.

One advantage of the CDN is that it allows the factors to be parameterized as copulae \cite{SilvaEtAl2011}, combining the benefits of graphical models with those of copula theory. In brief, the use of a copula parameterization enables greater modelling flexibility by separating representation of the marginals from the dependence structure, permitting more efficient and robust learning. For example, the marginals can be learnt nonparametrically by kernel density estimation, while the dependence structure learnt as a parametric copula.

Another advantage is that the CDN permits the representation of implicit latent variables, whose parameterization and connectivity are not required to be specified. Furthermore, for some tasks such as learning to rank \cite{HuangFrey2009}, survival analysis, and data censoring, directly modelling the CDF is more appropriate than modelling the density. Notwithstanding, the utility of pure CDN models is severely limited by that it can \emph{only} represent latent relationships between variables.

In Chapter \ref{copulae}, we present the basics of copula theory, the construction of two common families of copulae, and how to differentiate them with respect to an arbitrary subset of their scope and copula parameter. In Chapter \ref{pgm}, we discuss the representation of the CDN, its copula parameterization, and the copula parameterization of standard PGMs. In Chapter \ref{inference}, we explain a message passing algorithm, akin to the sum-product algorithm in standard PGMs, that efficiently calculates derivatives of the model to perform inference. In Chapter \ref{sampling}, we develop the first algorithm for sampling from CDNs. In Chapter \ref{learning}, we derive algorithms for learning the parameters of a CDN. Properties of the algorithms are evaluated by Monte Carlo simulations in Chapter \ref{experiments}, which also discusses the limitations of the model. We conclude in Chapter \ref{conclusions} with further discussion of the advantages and limitations of CDNs, and suggest future work.

Our main contributions are,
\begin{itemize}
	\item derivation of a new learning algorithm for CDNs that enables learning on higher clique width and high-dimensional models (\S\ref{sec:piecewiseLearning});
	\item development and demonstration of the first sampling algorithm for CDNs (Chapter \ref{sampling});
	\item method for performing gradient based optimization methods on CDNs parameterized with normal copulae (\S\ref{sect:partialDerivNormal} and \S\ref{sec:gradNormal}).
\end{itemize}
Lesser contributions include,
\begin{itemize}
	\item proposal of an algorithm for learning from missing completely at random (MCAR) and censored data (\S\ref{sec:cmarLearning});
	\item investigation of the properties of the algorithms with Monte Carlo simulations (Chapter \ref{experiments});
	\item stable evaluation of the Clayton copula and its partial derivatives for extremal parameters (\S\ref{sec:claytonStability}).
\end{itemize}
A substantial software library was developed in C++ during the course of this project that implements representation, inference, sampling, and learning in CDNs as outlined in this thesis. It also contains the code to construct test networks and run our experiments. Part of the code is reusable outside the CDN context; for example, construction of clique trees, message scheduling, and representation of the copulae. The library comprises over 8000 lines of source code, and was programmed in around 200 commits over August 2012--March 2013. It is available on request to the author.

%% file: copulae.tex
\chapter{Copula Theory}
\label{copulae}
Modelling the association between random variables is of fundamental interest in the practice of Statistics. Characterizing the relationship between variables is required for a deep understanding of stochastic phenomena, and aids accurate prediction and the identification of causal relationships.

In the bivariate case, association possesses two extremities. At one end, the variables are independent; $P(X\ |\ Y)=P(X)$. This means that knowing the value of one variable does not reduce our uncertainty about the other. At the other end, one variable is almost surely a monotonic function of the other; $P(X=f(Y))=1$. In this case, knowing the value of one variable will entirely reduce our uncertainty about the other. Of course, in any interesting case the association will lie ``somewhere in between.''

Inveterate statistics describe limited aspects of the association \cite{HoggEtAl2012}. For example, the Pearson's correlation coefficient is suitable to characterize a linear relationship between variables, and Spearman's rho is suitable for general monotonic relationships. Clearly, single statistics are inadequate to fully characterize the nature of the association. A special distribution, termed the copula \cite{Nelson2010}, will, however, permit us to specify the exact nature of the association, which we will informally term the dependence structure.

The copula function links the marginal distributions to the joint distribution and thus separates the marginals from the dependence structure. This has several advantages. In addition to allowing us to identify an explicit functional relationship of the dependence structure, we can also specify a joint distribution by estimating the marginals. The marginals can be estimated robustly and efficiently, for example, by kernel density estimation with a normal kernel. Importantly, the marginals and dependence structure can be identified separately; first the marginals are estimated, which are then used to estimate the copula.

Moreover, it permits modelling flexibility since we can ``mix and match'' marginals and copulas. Normally, when a parametric multivariate distribution is specified, the marginals are required all to be from the same family.

In this chapter, we introduce the fundamentals of copula theory. We explain precisely how a copula separates the representation of the marginals from the dependence structure. An intuitive understanding is gained from information theory. Several families of copulae, their construction, and behaviour are illustrated. Finally, we derive the formulae to differentiate two classes of copulae with respect to all subsets of their scope and copula parameter.

\emph{An original contribution of this thesis is the formulae for the partial derivatives of the parameter gradient of the normal copula, which are required for learning CDNs parameterized by normal copulae with gradient-based optimization methods. A lesser contribution is the numerically stable evaluation of the partial derivatives of the Clayton copula.}

\section{Transforming the marginals}
Our goal is to obtain an explicit representation of the dependence structure of a set of random variables. Suppose the joint distribution of these variables is given. It will determine both the dependence structure and the marginals, and this suggests that we could develop a method for separating the two.

Consider that we have a cumulative distribution function $F$ over random variables $X_1,\ldots,X_n$. Could there be some transformation of the variables that could ``normalize'' the marginals whilst leaving the dependence structure unchanged? Joint transformations affect the dependence structure; thus we consider only transforming each variable separately. Also, intuitively, the transformation must be monotonically increasing so as to preserve ranks. Normalizing the marginals will make possible comparisons between different joint distributions over the same variables.

One idea is to try the monotonically increasing transformation $U_i=F_i(X_i)$, where $F_i$ is the CDF over variable $X_i$, so that $U_i$ will have the same distribution regardless of the distribution of $X_i$,
\begin{theorem}
	Let $X$ be an arbitrary random variable and $F$ its cumulative distribution function. Then, $F(X)\sim\mathcal{U}[0,1]$.
\end{theorem}
\begin{proof}
	For $u\in[0,1]$,
	\begin{align*}\vspace{-1cm}
		P(F(X)\le u) &= P(X\le F^{-1}(u))\ \ \ \textnormal{(since $F$ is increasing)}\\
		&= F(F^{-1}(u))\\
		&= u.
	\end{align*}
	\end{proof}
The transformation has ``removed'' the information content of the marginals from the joint distribution, so only the dependence structure remains; the marginals have zero entropy,
\begin{cor}
	The differential entropy of $U_i=F_i(X_i)$ is zero.
\end{cor}
\begin{proof}
	\begin{align*}
		H[U_i] &= -\int f(u)\log(f(u))du = -\int1\log(1)du = 0.
	\end{align*}
\end{proof}
We will formalize this idea in the sequel.

Thus, according to our informal reasoning, the distribution $C$ of $\{U_i\}$ contains the dependence structure. This function is given a special name,
\begin{defn}
An $n$-\emph{copula} is an $n$-dimensional CDF for which the marginals are uniformly distributed on $[0,1]$.
\end{defn}
Assuming the CDFs of the marginals are strictly increasing and continuous, the copula takes the form,
\begin{align}
	C(u_1,\ldots,u_n) &= F(F^{-1}_1(u_1),\ldots,F^{-1}_n(u_n))\nonumber \\
	\Rightarrow\ F(x_1,\ldots,x_n) &= C(F(x_1),\ldots,F(x_n)).\label{eqn:copulaDecomposition}
\end{align}
We have thereby decomposed the representation of the joint distribution into the marginals and a function of the marginals encoding the dependence structure. Crucially for modelling, the converse also holds. That is, given $n$ marginals $\{F_i\}$ and a copula $C$, \eqref{eqn:copulaDecomposition} specifies a valid distribution $F$.

Therefore, this method gives us a ``recipe'' for constructing a distribution with given dependence structure and marginals.

\section{Sklar's theorem}\label{sec:sklar}
The informal discussion above is made exact by the following theorem due to Sklar \cite{Sklar1959}.
\begin{theorem}[Sklar]
Let $F$ be an $n$-dimensional distribution function with margins $F_1,\ldots,F_n$. Then there exists a copula $C$ such that,
\begin{align*}
	F(x_1,\ldots,x_n) &= C(F_1(x_1),\ldots,F_n(x_n)),
\end{align*}
for all $x$ in $\overline{\mathbb{R}}^n\equiv(\mathbb{R}\cup\{\pm\infty\})^n$. If $F_1,\ldots,F_n$ are all continuous, then $C$ is unique; otherwise, $C$ is uniquely determined on $\textnormal{Ran}F_1\times\cdots\times\textnormal{Ran}F_n$, where $\textnormal{Ran}(F)$ denotes the range of $F$.

Conversely, if $C$ is an $n$-dimensional copula, and $F_1,\ldots,F_n$ are distribution functions, then the function $F$ defined above is an $n$-dimensional distribution function with margins $F_1,\ldots,F_n$.
\end{theorem}
The copula contains the dependence structure in the sense that it links the marginals to the joint distribution. It suffices to understand the copula as a distribution with uniform marginals.

\section{Frech\'{e}t-Hoeffding bounds}
It is possible to provide a bounds on the copula,
\begin{theorem}
	If $C$ is an $n$-copula, then for every $\mathbf{u}\in\textnormal{Dom}C$,
	\begin{align*}
		W^n(\mathbf{u}) &\le C(\mathbf{u})\ \le\ M^n(\mathbf{u}).
	\end{align*}
	where,
	\begin{align*}
		M^n(\mathbf{u}) &= \textnormal{min}(u_1,\ldots,u_n)\ \ \textnormal{and}\\
		W^n(\mathbf{u}) &= \textnormal{max}(u_1+\cdots+u_n-n+1,0),
	\end{align*}
	are referred to as the upper and lower Frech\'{e}t-Hoeffding bounds, respectively.
\end{theorem}
The upper bound $M^n$ is a copula for all $n$, whereas $W^n$ fails to be a copula for $n>2$. Notwithstanding, $W^n$ is the best possible lower bound in the sense that for $n\ge3$, and all $\mathbf{u}\in[0,1]^n$, there is an $n$-copula $C$ for which $C(\mathbf{u})=W^n(\mathbf{u})$.

Another important copula is the independence copula,
\begin{align*}
	\Pi^n(\mathbf{u}) &= \prod^n_{i=1}u_i.
\end{align*}
It should be clear that this is a valid copula since the margins are uniform on $[0,1]$ and that it expresses independence between the variables; a distribution with the dependence structure of the independence copula can be written,
\begin{align*}
	F(x_1,\ldots,x_n) &= \Pi^n(F_1(x_1),\ldots,F_n(x_n)) \\
	&= F_1(x_1)\cdots F_n(x_n).
\end{align*}
\begin{ex}
	The level curves of the bivariate copulae $M^2$, $W^2$, and $\Pi^2$ are graphed in Figure \ref{fig:copulaBounds} along with samples from the corresponding copulae distributions. One sees that the bounds represent two extremes: the upper bound represents perfectly positive dependence---a comonotonic relation---between the variables, and the lower bound represents perfect negative dependence---a countermonotonic relation. Between these, we have the intermediate relation of independence. That the lower bounds fails to be a copula for $n>2$ is equivalent to the statement that a countermonotonic relation cannot be defined between more than two variables.
\end{ex}
\begin{figure}[t]
  \centering
  \includegraphics[scale=0.5]{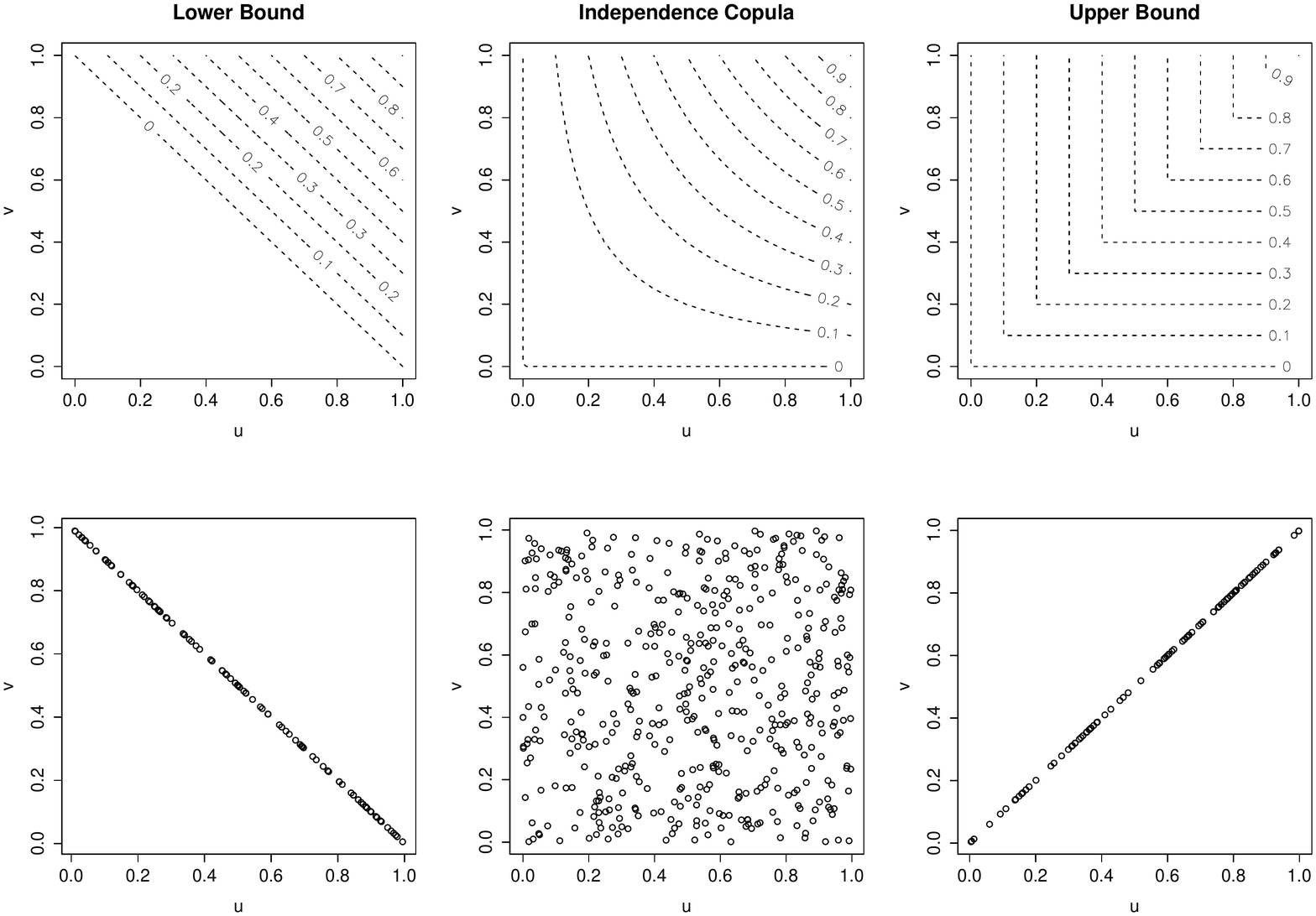}
	\caption[Level curves and samples from three copulae expressing extreme bivariate relationships.]{Level curves and samples from three copulae expressing extreme relationships: the bivariate upper ($M^2$) and lower ($W^2$) bounds, and the independence ($\Pi^2$) copula.}
	\label{fig:copulaBounds}
\end{figure}

\section{Information-theoretic perspective}
It is interesting to examine copulas from an information-theoretic perspective \cite{CalsaveriniVicente2009}. Consider that we are examining the relationship between two random variables $X$ and $Y$. The mutual information can be regarded as the average reduction in the uncertainty in $X$ given knowledge of $Y$ \cite{CoverThomas2006},
\begin{align*}
	I[X,Y] &= H[X]-H[X\ |\ Y],
\end{align*}
where $H[X]$ is the entropy,
\begin{align*}
	H[X] &= -\int f(x)\ln(f(x))dx,
\end{align*}
and $H[X\ |\ Y]$ is the conditional entropy,
\begin{align*}
	H[X\ |\ Y] &= -\iint f(x,y)\ln(f(x\ |\ y))dxdy.
\end{align*}
Alternatively, the mutual information can be viewed as the ``distance'' of the joint distribution to a distribution where the variables are independent,
\begin{align*}
	I[X,Y] &= D[f(x,y)\ ||\ f(x)f(y)]\\
	&= \int f(x,y)\log\left(\frac{f(x,y)}{f(x)f(y)}\right)dxdy,
\end{align*}
where $D[p||q]$ denotes the relative entropy (also known as the KL-divergence) of $p$ from $q$. 

Making the substitutions $u=F_X(x)$ and $v=F_Y(y)$ (so that $du=f_X(x)dx$, etc.) the mutual information is re-expressed,
\begin{align*}
	I[X,Y] &= \iint_{[0,1]^2}c(u,v)\log(c(u,v))dudv\\
	&= -H[U,V]\\
	&\equiv -H[C].
\end{align*}
Thus, the mutual information between $X$ and $Y$ is equal to the negative entropy of the copula distribution, which we have defined as the \emph{copula entropy}, $H[C]$. Importantly, the mutual information is a function only of the copula.

The mutual information will be minimized, or rather, knowledge of $Y$ will provide no reduction in the uncertainty of $X$ when the variables are independent. 

This allows us to decompose the information content of the joint distribution,
\begin{align*}
	H[X,Y]&= H[X]+H[Y]-I[X,Y]\\
	&= H[X]+H[Y]+H[C]\\
	\Rightarrow\ H[C] &= H[X,Y]-H[X]-H[Y].
\end{align*}
We have shown that the information content of the copula is equal to the difference of the information content in the joint distribution from that of the marginals---a result that generalizes to higher dimension---thus validating our intuition that the copula contains only the information content of the dependence structure.

\section{Construction}
Although there are several methods for constructing copulae, including geometric and algebraic ones, we discuss only those that generalize to dimensions higher than two.

\subsection{Inversion Method}
Sklar's Theorem (\S\ref{sec:sklar}) suggests the most obvious way to produce copulae. The dependence structure of a multivariate distribution is extracted by inverting Sklar's Theorem as follows. Given a distribution $F(x_1,\ldots,x_n)$ with invertible marginal CDFs, one forms the copula,
\begin{align*}
	C(u_1,\ldots,u_n) &= F(F_1^{-1}(u_1),\ldots,F_n^{-1}(u_n)),
\end{align*}
which then becomes a standard element to be combined with arbitrary marginals.

Three common copula families produced by this method are those based on the multivariate normal, Student, and Dirichlet distributions \cite{Lewandowski2008}. In this thesis, we consider only the multivariate normal copula parameterized with a single parameter. Without loss of generality, we set $F_i$ to be distributed as $N[0,1]$, so that the covariance matrix, $\Sigma$, has a unit diagonal and the copula parameters are the off diagonal elements of $\Sigma$---the correlations between variables.

\subsection{The normal copula}\label{sec:normalCopula}
Representing a normal copula that has a scope over more than two variables with more than one parameter complicates learning. In this case, the parameters cannot be modified independently, as modifying one parameter whilst holding the others fixed may spoil the positive-definiteness of the covariance matrix. Also, representing and manipulating a model composed of factors with a varying number of parameters complicates our software implementation.

One solution is to reparameterize the covariance matrix. For example, when $n=3$, choose $\alpha_1,\alpha_2,\alpha_3\in(-1, 1)$ and set,
\begin{align*}
	\Sigma &=
		\begin{pmatrix}
			1 & \alpha_1\alpha_2 & \alpha_1\alpha_3 \\
			\alpha_1\alpha_2 & 1 & \alpha_2\alpha_3 \\
			\alpha_1\alpha_3 & \alpha_2\alpha_3 & 1
		\end{pmatrix}.
\end{align*}
It can be shown that $\Sigma$ is positive definite for all choices of the parameters.

For convenience, and to enable comparison with Archimedean copulae, we desire to parameterize the covariance matrix with a single parameter. Our approach is to equate all covariances to $\rho$. We determined the following theorem, which gives the domain for $\rho$,
\begin{theorem}
	Consider the covariance matrix
	\begin{align*}
		\Sigma_\rho &= I_{n\times n} + \rho\left(\mathbf{1}_n\mathbf{1}^T_n-I_{n\times n}\right),
	\end{align*}
	that is, with unit variance for each variable and correlation $\rho$ between each pair of variables.
	
	Then $\Sigma_\rho\succ0$ if and only if $\rho\in(-1/(n-1),1)$.
\end{theorem}
\begin{proof}
	We need to determine for which $\rho$ all eigenvalues of $\Sigma_\rho$ are positive. The eigenvalues $\{\lambda_i\}$ of $\Sigma_\rho$ are the solutions to,
	\begin{align}
		0 &= \det\left(\Sigma_\rho - \lambda I\right)\nonumber \\
		\Rightarrow 0 &= \det\left(\rho\mathbf{1}_n\mathbf{1}_n^T - \left(\lambda + \rho - 1\right)I_{n\times n}\right).\label{eqn:eigenvalue}
	\end{align}
	First, we note that the eigenvalues $\{\lambda'_i\}$ of $\rho\mathbf{1}_n\mathbf{1}^T_n$ are $\rho n$ with multiplicity $1$, and $0$ with multiplicity $n - 1$. This should be clear, since $\textnormal{Rank}\left(\mathbf{1}_n\mathbf{1}^T_n\right)=1$, and $\rho\mathbf{1}_n\mathbf{1}^T_n\mathbf{1}_n=\rho n\mathbf{1}_n$.

Thus, from \eqref{eqn:eigenvalue},
	\begin{align*}
		\lambda'_1 &= \lambda_1 + \rho - 1 = \rho n \\
		\Rightarrow\ \lambda_1 &= (n-1)\rho + 1.
	\end{align*}
Similarly,
	\begin{align*}
		\lambda_2 &= \cdots = \lambda_n  = 1 - \rho,
	\end{align*}
from which the proposition follows.
\end{proof}

\begin{ex}
	Levels curves and samples from the bivariate normal copula are graphed for varying values of the dependence parameter in Figure \ref{fig:copulaNormal}. As $\rho\rightarrow-1$, we see that the copula approaches the lower bound, and similarly as $\rho\rightarrow1$ it approaches the upper bound. The parameter $\rho=0$ corresponds to the independence copula. Thus, the normal copula is able to represent a spectrum of dependencies between the two bounds
\end{ex}
\begin{figure}[p]
  \centering
  \includegraphics[scale=0.52]{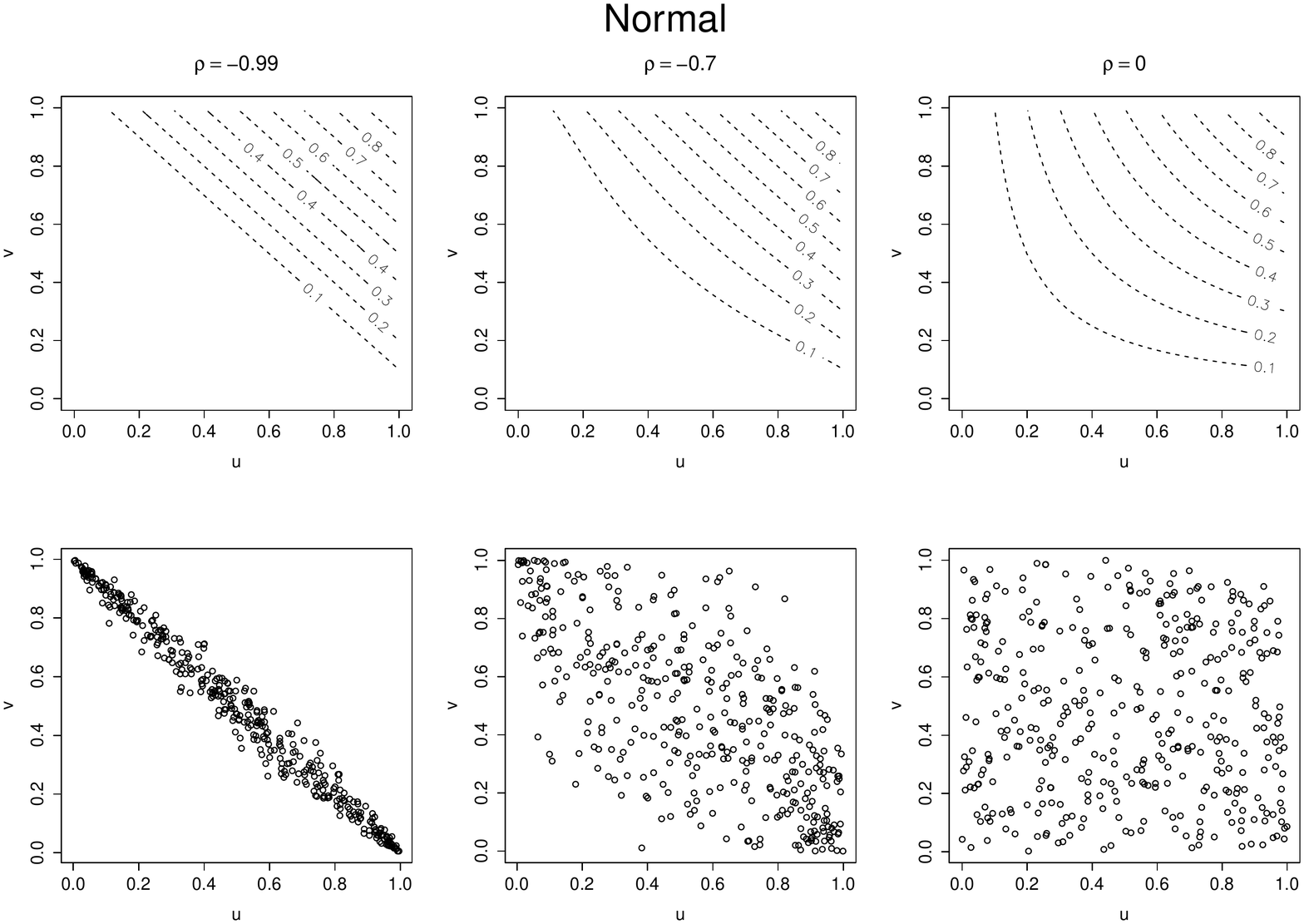} \\[6pt]
	\includegraphics[scale=0.52]{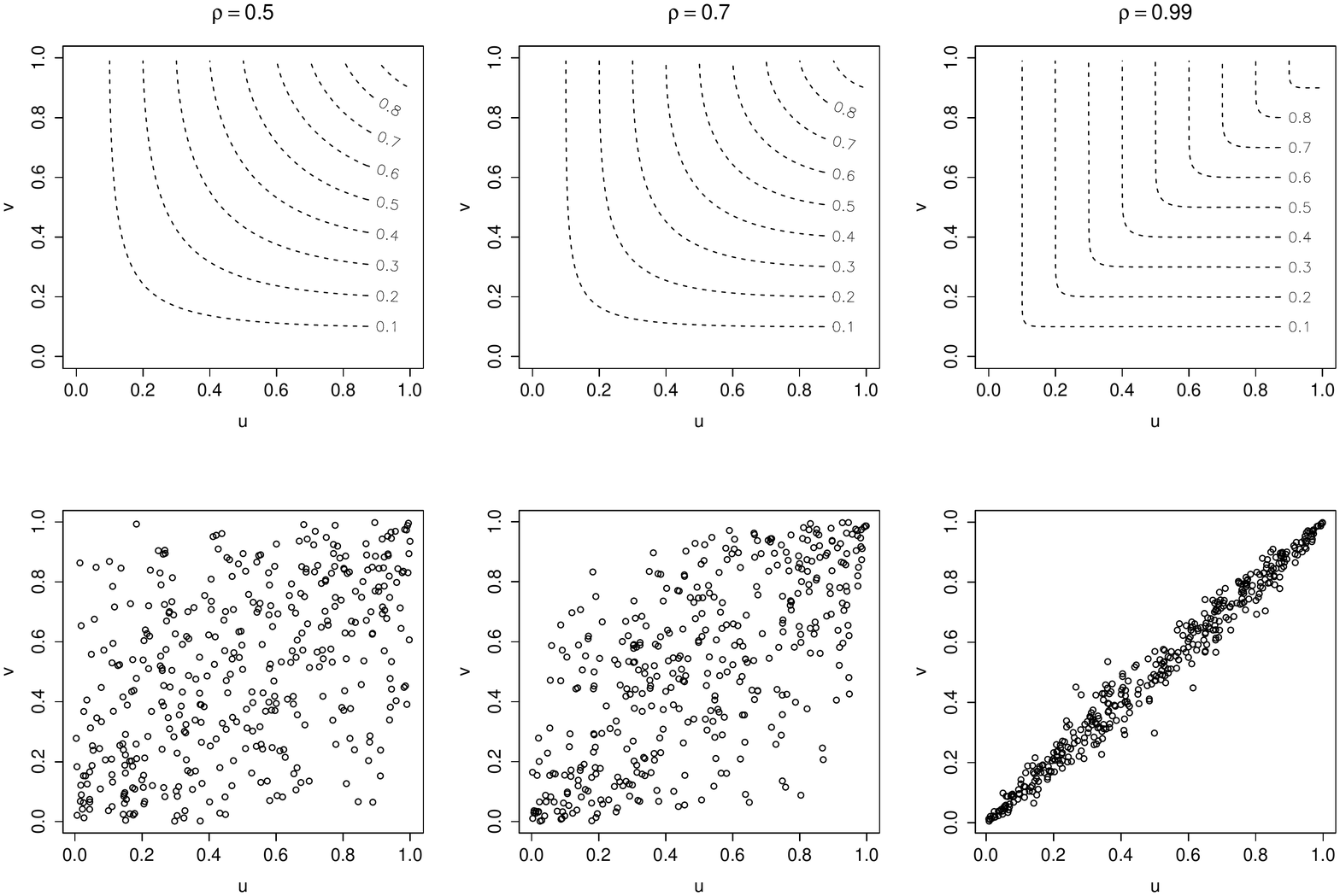}
	\caption[Level curves and samples from the bivariate normal copula.]{Level curves and samples from the bivariate normal copula for varying values of the parameter.}
	\label{fig:copulaNormal}
\end{figure}

\subsection{Archimedean Copulas}
Archimedean copulas are those that admit the representation,
\begin{align*}
	C(u_1,\ldots,u_n) &= \varphi(\varphi^{-1}(u_1) + \cdots + \varphi^{-1}(u_n)),
\end{align*}
for some generator function $\varphi$ and its inverse $\varphi^{-1}$.

A necessary and sufficient condition that a generator function must satisfy to define a valid copula \emph{of arbitrary dimension} is given by the following theorem,
\begin{theorem}[\cite{Kimberling1974}]
	Let $\varphi:[0,1]\rightarrow[0,\infty]$ be a continuous strictly decreasing function such that $\varphi(0)=\infty$ and $\varphi(1)=0$. The function
\begin{align*}
	C^n(\mathbf{u}) &= \varphi^{-1}(\varphi(u_1)+\cdots+\varphi(u_n))
\end{align*}
	defines	a valid $n$-copula for all $n\ge2$ if and only if
\begin{align}\label{eqn:montone}
	0 &\le (-1)^k\frac{d^k}{dt^k}\varphi^{-1}(t)
\end{align}
	for $k=0,1,2,\ldots$ and $t\in(0,\infty)$. That is, if and only if $\varphi$ is \emph{completely monotone}.
\end{theorem}
When condition \eqref{eqn:montone} is satisfied only for $k=0,1,..d-2$ then we say that $\varphi$ is $d$-monotone, and it generates up to a $d$-copula.

Many generators that are completely monotone have been discovered; see \cite[Table 4.1]{Nelson2010}. We discuss three in common use---the Clayton, Frank, and Gumbel--and implement the Clayton copula in our library. Refer to Figure \ref{fig:archcopulae} for their formulae.
\begin{figure}[t]
  \centering
	{\footnotesize
	\begin{tabular}{cccc}
		{\itshape name} & {\itshape generator} $\varphi(t)$ & {\itshape inverse generator} $\varphi^{-1}(t)$ & {\itshape parameter} \\ \hline
		Clayton & $(1+\theta t)^{-1/\theta}$ & $\frac{1}{\theta}(t^{-\theta} - 1)$ & $\theta\in(0,\infty)$\\
		Frank & $-\frac{1}{\theta}\ln\left(1-\left(1-\exp\left(-\theta\right)\right)\exp\left(-t\right)\right)$ & $-\ln\left(\frac{\exp\left(-\theta t\right)-1}{\exp\left(-\theta\right)-1}\right)$ & $\left\{\begin{array}{l}
			\theta\in\left(-\infty,\infty\right)\setminus\{0\},\ n = 2\\
			\theta\in\left(0,\infty\right),\ \textnormal{otherwise}
		\end{array}\right.$\\
		Gumbel & $\exp\left(-t^{1/\theta}\right)$ & $\left(-\ln(t)\right)^\theta$ & $\theta\in[1,\infty)$ 
	\end{tabular}}\\[24pt]	
	
	\begin{tabular}{cc}
		{\itshape name} & {\itshape copula} \\ \hline
		Clayton & $C(u_1,\ldots,u_n) = \left(\sum^n_{i=1}u^{-\theta}_i-n+1\right)^{-1/\theta}$ \\
		Frank & $C(u_1,\ldots,u_n) = -\frac{1}{\theta}\ln\left(1+\frac{\prod^n_{i=1}\left(\exp\left(-\theta u_i\right) - 1\right)}{\left(\exp\left(-\theta\right)-1\right)^{n-1}}\right)$ \\
		Gumbel & $C(u_1,\ldots,u_n) = \exp\left(-\left(\sum^n_{i=1}\left(-\ln\left(u_i\right)\right)^\theta\right)^{1/\theta}\right)$
	\end{tabular}
  \caption[Three common Archimedean copulae.]{Three common Archimedean copulae. All are completely monotone.}
	\label{fig:archcopulae}
\end{figure}
	
\begin{ex}
	The surface of three bivariate Archimedean copulae are graphed for varying values of their dependence parameter along with a scatter plot of samples drawn using our sampling algorithm developed in Chapter \ref{sampling} in Figures \ref{fig:claytonGumbelLevels} and \ref{fig:frankLevels}.
	
	The Clayton and Frank copulae approach $\Pi^2$ as $\theta\rightarrow0$, and the Gumbel copula is identical with $\Pi^2$ when $\theta=1$. As $\theta\rightarrow\infty$, the three copulae approach the upper bound $M^2$, and in the bivariate case the Frank copula approaches the lower bound $W^2$ as $\theta\rightarrow-\infty$.
	
	We note that the Clayton and Gumbel copulae cannot represent a negative association between variables, and the Frank copula can only do so in the bivariate case.
		
	As is visible in the figures, the Clayton copula prohibits upper tail dependence, the Gumbel copula prohibits lower tail dependence, and the Frank copula does not express any tail dependence. The Clayton and Frank copulae are able to represent a bivariate distribution with Kendall's $\tau\in(0,1)$, and the Gumbel, $\tau\in[0,1)$.
\end{ex}
\begin{figure}[p]
  \centering
  \includegraphics[scale=0.5]{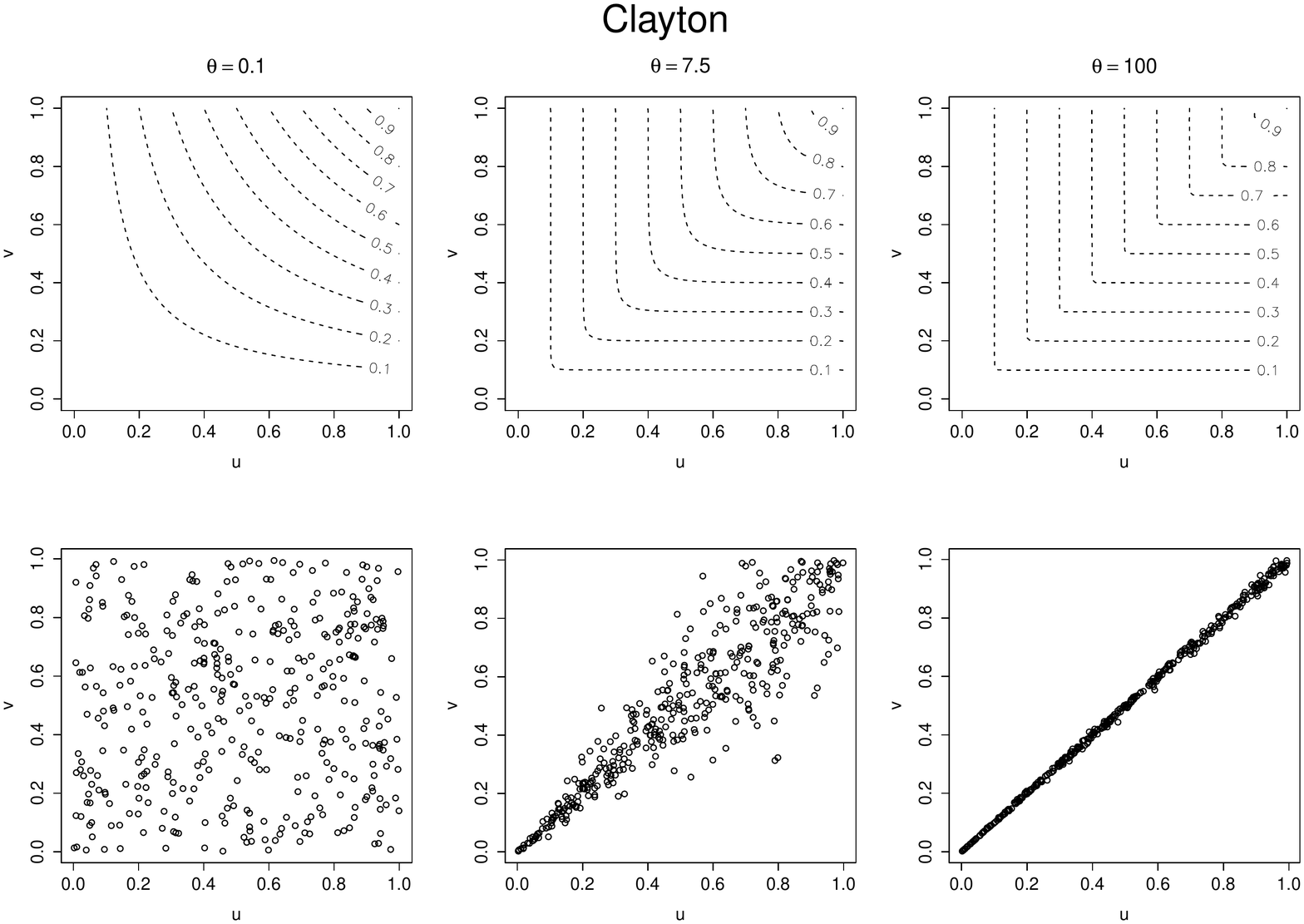} \\[18pt]
	\includegraphics[scale=0.5]{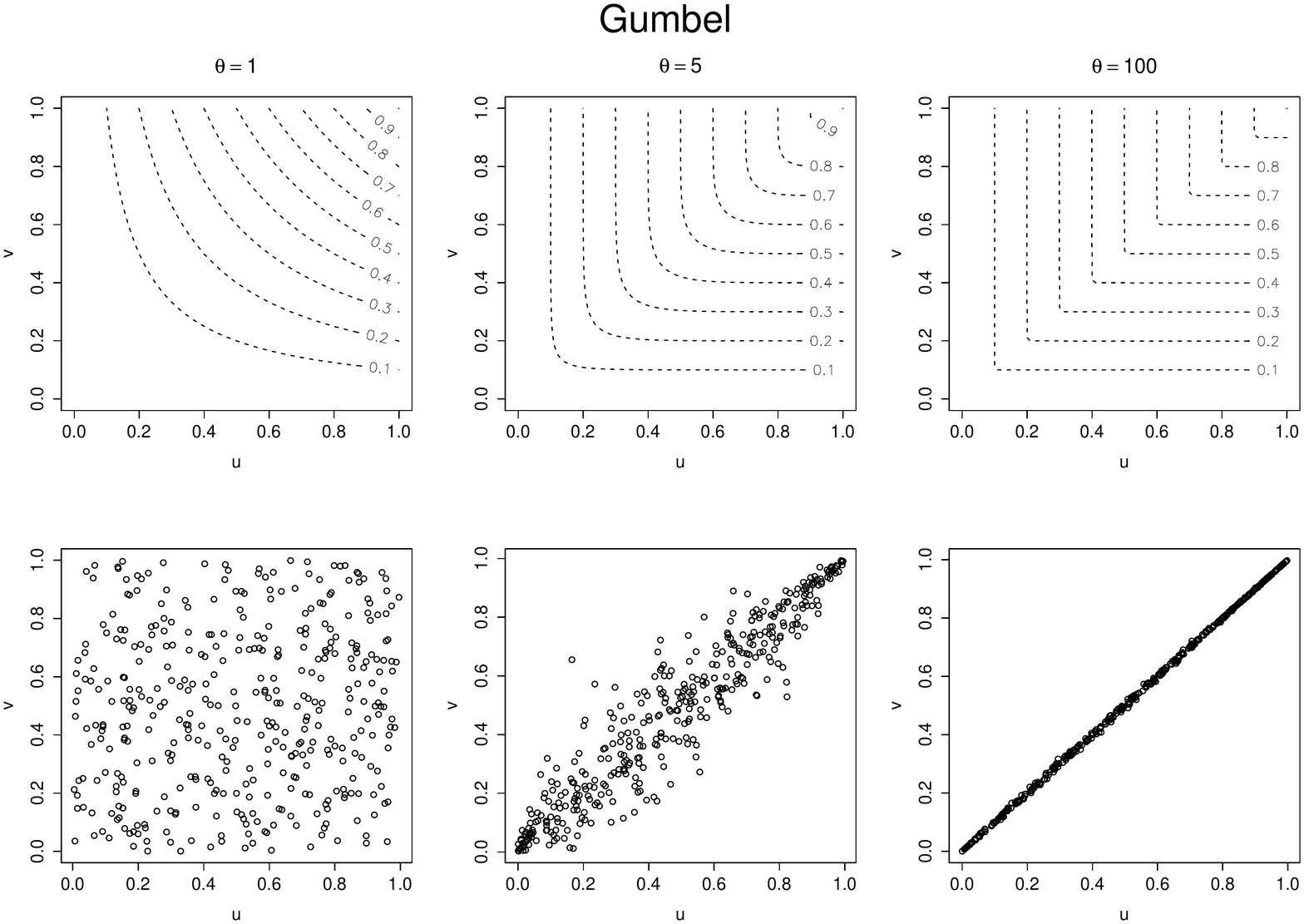}
	\caption[Level curves and samples from the Clayton and Gumbel copulae.]{Level curves and samples from the Clayton and Gumbel copulae for varying values of the parameter.}
	\label{fig:claytonGumbelLevels}
\end{figure}
\begin{figure}[p]
  \centering
  \includegraphics[scale=0.52]{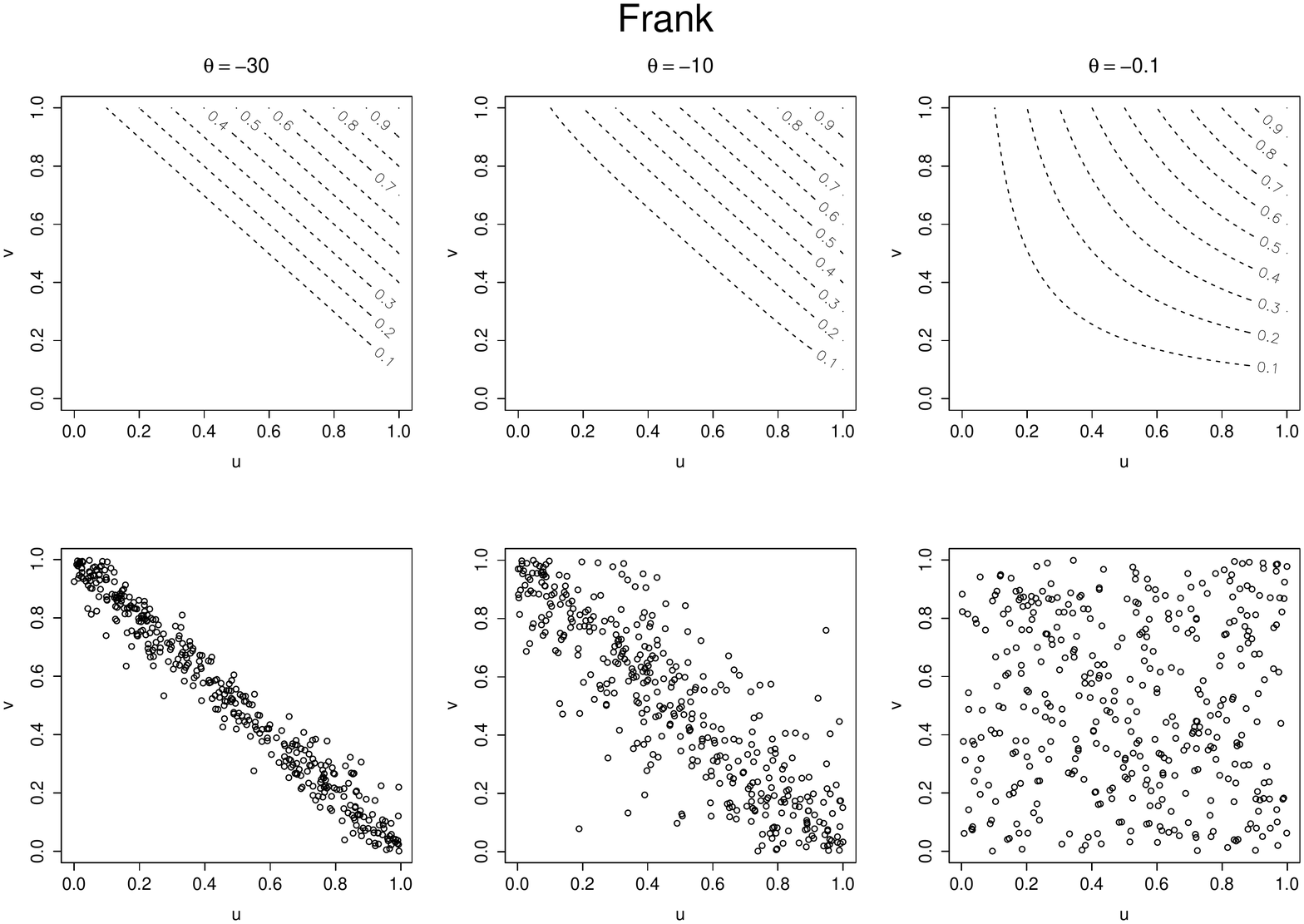} \\[6pt]
	\includegraphics[scale=0.52]{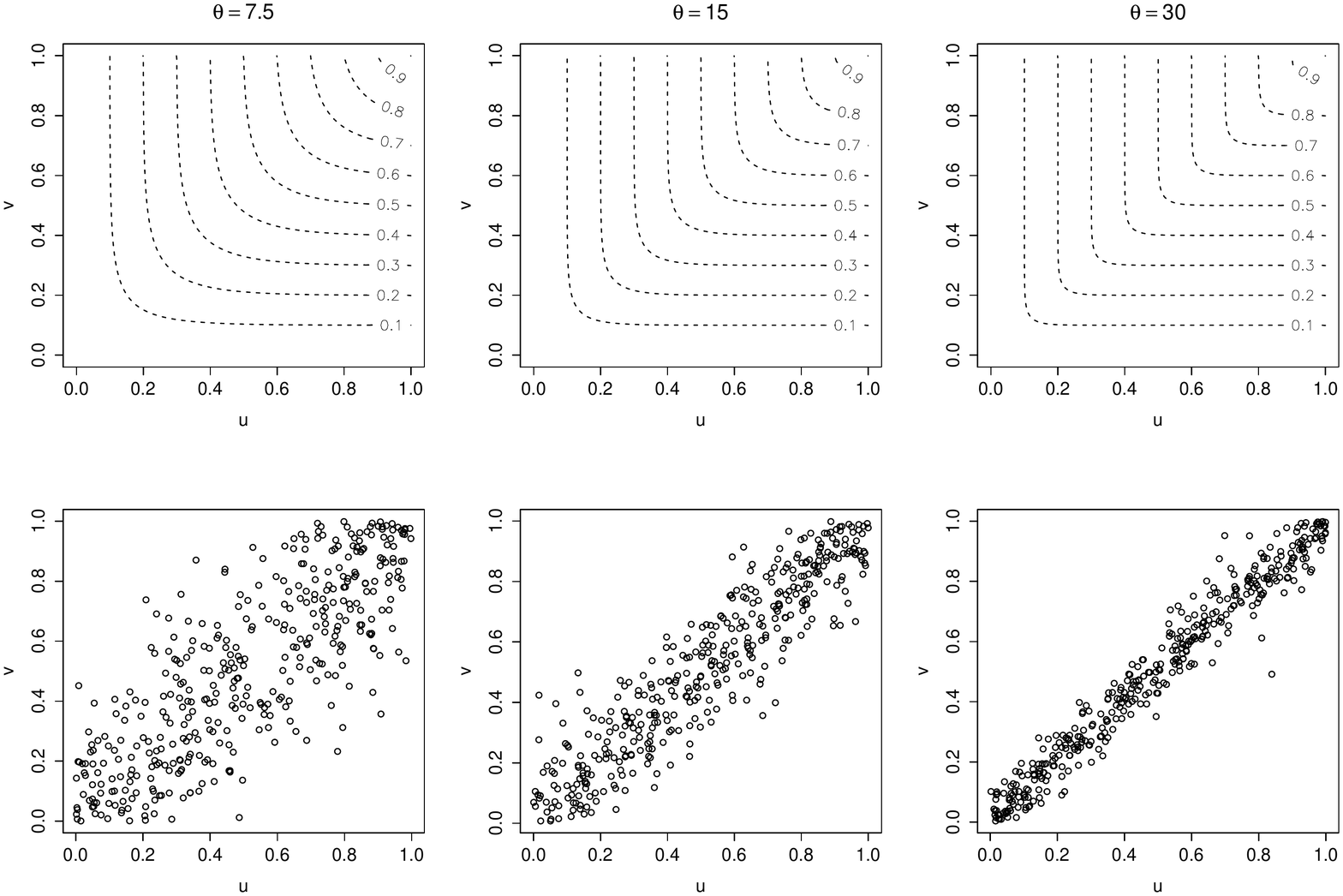}
	\caption[Level curves and samples from the Frank copula.]{Level curves and samples from the Frank copula for varying values of the parameter.}
	\label{fig:frankLevels}
\end{figure}

\section{Evaluating partial derivatives of the factors}\label{sect:partialDerivFactors}
Our model, as will be explained in Chapter \ref{pgm}, represents a CDF as the product of copulae factors. The algorithm to differentiate the model---the derivative-sum-product algorithm---requires evaluation of the partial derivatives of the factors with respect to different subsets of their scopes.

\subsection{The normal copula}\label{sect:partialDerivNormal}
How are we to do this when we have chosen to represent the factors as multivariate normal CDFs?

Let $(\mathbf{X}, \mathbf{Y})$ be distributed as $N(\mu,\Sigma)$, where
\begin{align*}
	\mu &= \begin{pmatrix}\mu_\mathbf{x}\\\mu_\mathbf{y}\end{pmatrix},\\
	\Sigma &= \begin{pmatrix}\Sigma_\mathbf{xx} & \Sigma_\mathbf{xy} \\ \Sigma_\mathbf{xy}^T & \Sigma_\mathbf{yy}\end{pmatrix}.
\end{align*}
The derivative of the CDF with respect to the subset $\mathbf{x}$ is evaluated thus,
\begin{align}
	\frac{\partial}{\partial\mathbf{x}}F(\mathbf{x},\mathbf{y}) &= P(\mathbf{X}=\mathbf{x}, \mathbf{Y}\le\mathbf{y})\nonumber \\ \label{eqn:derivNormal}
	&= f_\mathbf{X}(\mathbf{x})F_{\mathbf{Y}|\mathbf{X}=\mathbf{x}}(\mathbf{y}).
\end{align}
Since $\mathbf{X}$ and $\mathbf{Y}$ are normally distributed, $\mathbf{Y}\ |\ \mathbf{X}=\mathbf{x}$ is normally distributed \cite{HoggEtAl2012}, with
\begin{align*}
	\mu' &=  \mu_\mathbf{y}+\Sigma_{\mathbf{xy}}^T\Sigma^{-1}_{\mathbf{xx}}(\mathbf{x}-\mu_\mathbf{x})\\
	\Sigma' &= \Sigma_{\mathbf{yy}} - \Sigma_{\mathbf{xy}}^T\Sigma^{-1}_{\mathbf{xx}}\Sigma_{\mathbf{xy}}.
\end{align*}
As discussed previously, it suffices to consider multivariate normal distributions with zero mean and variance of each variable equal to unity. Thus, the mean simplifies to
\begin{align*}
	\mu' &= \Sigma_{\mathbf{xy}}^T\Sigma^{-1}_{\mathbf{xx}}\mathbf{x}.
\end{align*}
When $\mathbf{X}$ and $\mathbf{Y}$ are univariate, the parameters simplify to
\begin{align*}
	\mu' &= \rho \mathbf{x}\\
	\Sigma' &= 1 - \rho^2,
\end{align*}
where $\rho$ is the correlation coefficient between $\mathbf{X}$ and $\mathbf{Y}$ (and also the covariance in this case).

\subsection{Archimedean copulae}
It is easiest to derive the form of the derivatives for a general Archimedean copula. Let $C$ be an Archimedean $n$-copula generated by $\varphi$. Then,
\begin{align*}
	\frac{\partial C}{\partial u_i} &= \left(\varphi^{-1}\left(u_i\right)\right)'\varphi'\left(\varphi^{-1}\left(u_1\right)+\cdots+\varphi^{-1}\left(u_n\right)\right)\\
	&= \frac{\varphi'\left(\varphi^{-1}\left(u_1\right)+\cdots+\varphi^{-1}\left(u_n\right)\right)}{\varphi'\left(\varphi^{-1}\left(u_i\right)\right)}
\end{align*}
In general, for a subset $\mathbf{a}\subseteq\mathbf{u}$ of cardinality $m$,
\begin{align*}
	\frac{\partial C}{\partial\mathbf{a}} &= \frac{\varphi^{(m)}\left(\varphi^{-1}\left(u_1\right)+\cdots+\varphi^{-1}\left(u_n\right)\right)}{\prod_{u_i\in\mathbf{a}}\varphi'\left(\varphi^{-1}\left(u_i\right)\right)}.
\end{align*}
Whether the formulae for a particular Archimedean copula have general form thus depends on whether a general form exists for the derivatives of the generator.

Formulae for the $n$th derivative of the generators of five common Archimedean copulae---including the three discussed in this thesis---and the resulting copulae densities are derived in \cite{HofertEtAl2012}. We rederived the formula for the Clayton copula, as the one given in \cite{HofertEtAl2012} was based on a nonstandard form of the generator.

The formulae for the Clayton copula are,
\begin{align}
	\varphi^{(m)}(u) &= (-1)^m\prod^{m-1}_{k=0}\left(1/\theta+k\right)\theta^m\left(1+\theta u\right)^{-(1/\theta+m)},\nonumber\\
	\varphi'\left(\varphi^{-1}\left(u\right)\right) &= -u^{1+\theta}\nonumber\\
	\Rightarrow\ \frac{\partial C}{\partial\mathbf{a}} &= \prod^{m-1}_{k=0}\left(1/\theta+k\right)\theta^m\prod_{u_i\in\mathbf{a}}\left(u_i\right)^{-(1+\theta)}\left(\sum^n_{i=1}u^{-\theta}_i-n+1\right)^{-(1/\theta + m)}\label{eqn:claytonDerivs}
\end{align}

\subsection{Numerical stability}\label{sec:claytonStability}
The partial derivatives of the normal copula were found to be numerically stable at extreme values of the parameter due to the stable implementation of the multivariate normal CDF.

The formulae for the Archimedean copulae in \cite{HofertEtAl2012} were designed to be numerically stable for large $n$. They are not, however, stable for an extreme value of the parameter. For example, suppose we sample from the copula distribution and require to calculate,
\begin{align}
	C(u_2\ |\ U_1=u_1) &= \frac{\phi'(\phi^{-1}(u_1)+\phi^{-1}(u_2)}{\phi'(\phi^{-1}(u_1))}. \label{eqn:condCopCdf}
\end{align}
Firstly, when $u_1$ is very small, $\phi^{-1}(u_1)=\infty$ and \eqref{eqn:condCopCdf} is evaluated as $0/0=\textnormal{NaN}$. Thus, the first part of our solution is to take logs,
\begin{align*}
	\ln\left(\frac{\partial C}{\partial\mathbf{a}}\right) &= \ln\left(\left(-1\right)^m\varphi^{(m)}\left(\varphi^{-1}\left(u_1\right)+\cdots+\varphi^{-1}\left(u_n\right)\right)\right) - \sum_{u_i\in\mathbf{a}}\ln\left(-\varphi'\left(\varphi^{-1}\left(u_i\right)\right)\right).
\end{align*}
For the Clayton copula, by \eqref{eqn:claytonDerivs},
\begin{align*}
	\ln\left(\frac{\partial C}{\partial\mathbf{a}}\right) &= \sum^{m-1}_{k=0}\ln\left(1/\theta+k\right) - \left(1+\theta\right)\sum_{u_i\in\mathbf{a}}\ln\left(u_i\right) + m\ln(\theta)-(1/\theta+m)\ln\left(\sum^n_{i=1}u_i^{-\theta}-n+1\right).
\end{align*}
The exact modification required to make a copula's partial derivative numerically stable depends on its form. In the case of the Clayton copula, it becomes unstable for large values of $\theta$ because $u_i^{-\theta}$ overflows for small $u_i$. Let $u_{\textnormal{min}}=\min_i\{u_i\}$. The solution is to scale the logarithm by this term,
\begin{align}
	\begin{split}
		\ln\left(\frac{\partial C}{\partial\mathbf{a}}\right) &= \sum^{m-1}_{k=0}\ln\left(1/\theta+k\right) - \left(1+\theta\right)\sum_{u_i\in\mathbf{a}}\ln\left(u_i\right) + m\ln(\theta)\\
		&\ \ -(1/\theta+m)\left(\ln\left(\sum^n_{i=1}\left(u_{\textnormal{min}}/u_i\right)^{\theta}+(1-n)u_{\textnormal{min}}\right) - \ln\left(u_{\textnormal{min}}\right)\right).
	\end{split}\label{eqn:stableClaytonDerivs}
\end{align}
Also, we use the function {\ttfamily log1p} to stably evaluate the penultimate logarithm. {\ttfamily log1p}, a numerically stable implementation of $\ln(1+x)$, approximates $\ln(1+x)$ with its Taylor series expansion when $x$ is small. We use the implementation of \cite{website:boost}.

\emph{Our solution has improved the numerical stability over the state-of-the-art R package {\ttfamily copula}.} This is apparent when comparing samples generated from a copula with an extreme parameter. See Figure \ref{fig:stableClayton}

Although we have not implemented the Gumbel copula, we note that it requires a similar modification for numerical stability. The Frank copula, provided the {\ttfamily log1p} function is used, does not require additional modifications.
\begin{figure}[t]
  \centering
  \includegraphics[scale=0.5]{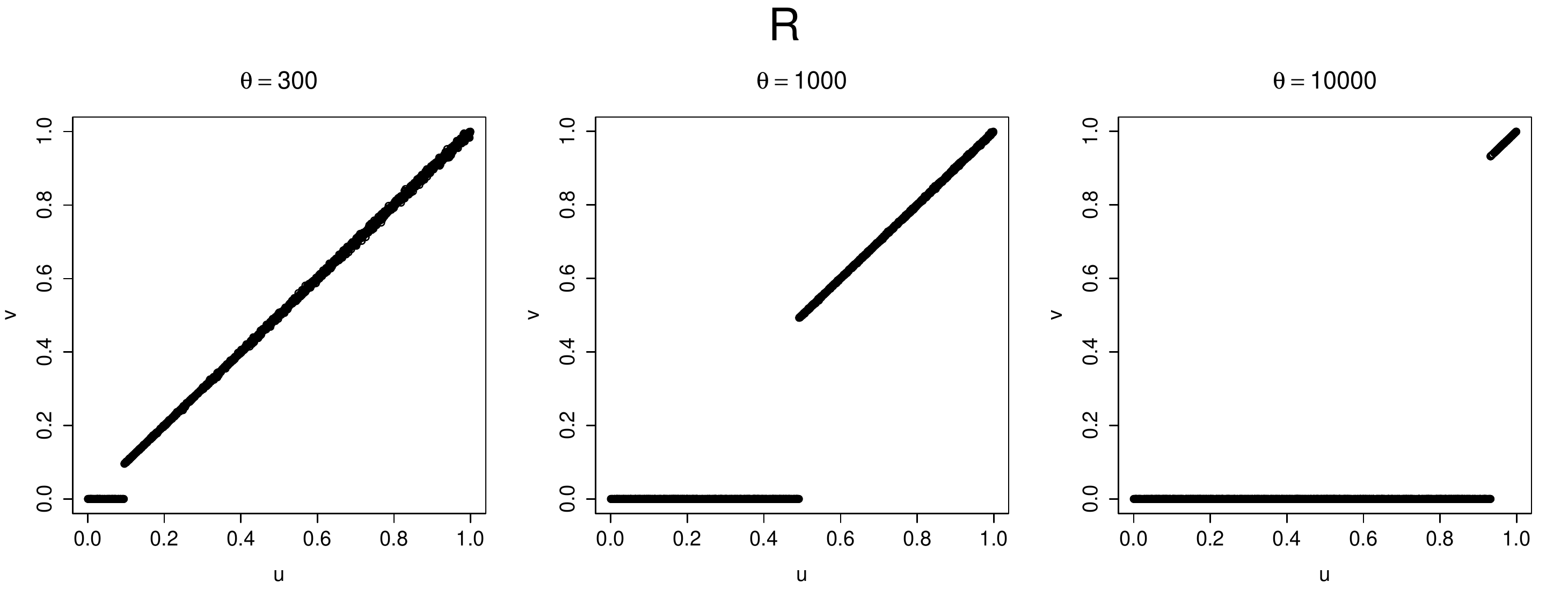} \\[6pt]
	\includegraphics[scale=0.5]{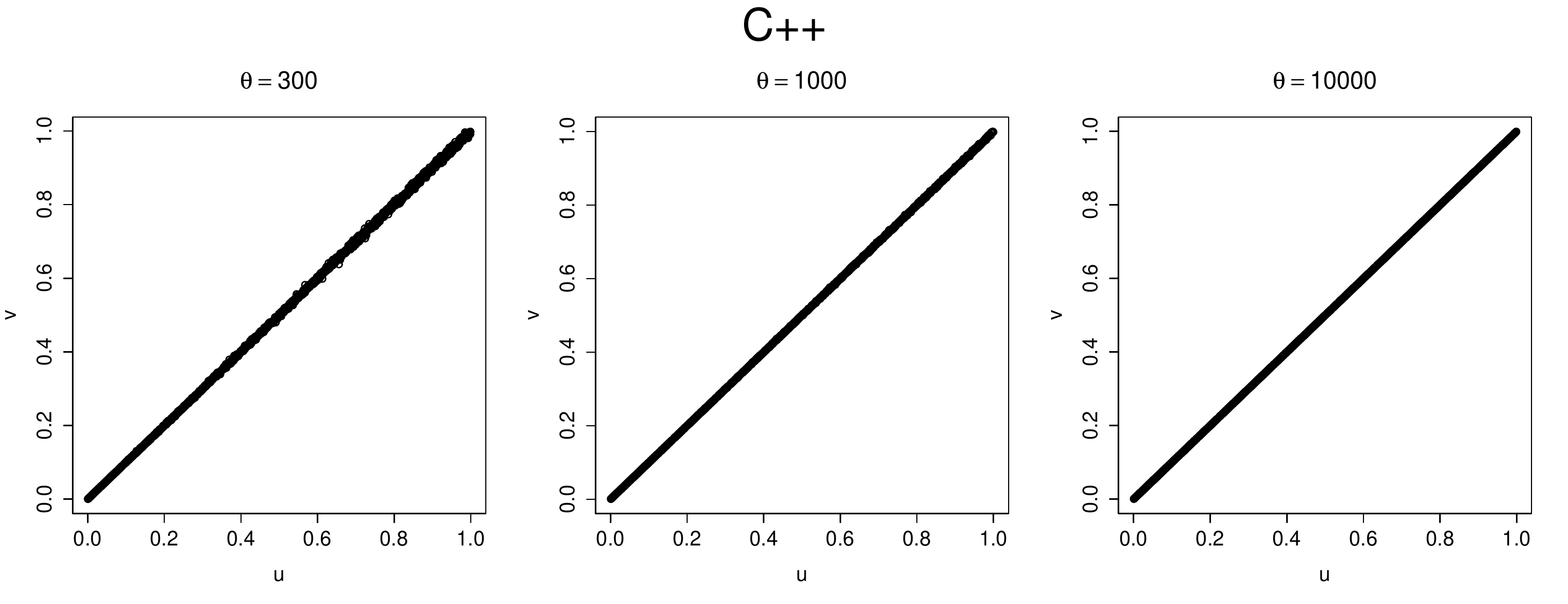}
	\caption[Comparing samples from bivariate Clayton copulae generated by an R package and our library.]{Comparing samples from bivariate Clayton copulae generated by the R package {\ttfamily copula} (top) and our library (bottom), which stably samples the copula even when the parameter is extremal.}
	\label{fig:stableClayton}
\end{figure}

\section{Differentiating with respect to the parameter}
\label{sec:gradCopulae}
Our learning algorithm necessitates calculating the gradient of the sample log-likelihood with respect to the copulae parameters. As shall be explained (see Chapter 5), when the parameters are not shared between factors, it suffices to be able to calculate the derivative of each factor with respect to an arbitrary subset, $\mathbf{x}$, of its scope, and a given parameter $\theta_i$.

\subsection{The normal copula}\label{sec:gradNormal}
For the normal copula, we use the following result due to \cite{Plackett1954},
\begin{align}
	\frac{\partial}{\partial\rho_{ij}}f(\mathbf{x};\mathbf{0},\Sigma) &= \frac{\partial^2}{\partial x_i\partial x_j}f(\mathbf{x};\mathbf{0},\Sigma). \label{eqn:plackett}
\end{align}
(To see why this is true, write the density as the transform of its characteristic function.)

Thus, in the bivariate case,
\begin{align*}
	\frac{\partial}{\partial\rho}F(x_1,x_2;\mathbf{0},\rho) &= f(x_1,x_2;\mathbf{0},\rho).
\end{align*}
To calculate the remaining derivatives, we need to differentiate the density with respect to all nonempty subsets of $\{x_1,x_2\}$. We use the results \cite[eq 325, 326]{PetersenPedersen2008}, which are also true for higher dimensions,
\begin{align}
	\frac{\partial f}{\partial x_i} &= -f(\mathbf{x})\left(\Sigma^{-1}\mathbf{x}\right)_i \label{eqn:normalD1}\\
	\frac{\partial^2 f}{\partial x_i\partial x_j} &= f(\mathbf{x})\left(\Sigma^{-1}\mathbf{x}\mathbf{x}^T\Sigma^{-1}-\Sigma^{-1}\right)_{ij}. \label{eqn:normalD2}
\end{align}
Thus, the remaining derivatives are,
\begin{align*}
	\frac{\partial}{\partial\rho}\frac{\partial}{\partial x_1}F(x_1,x_2;\mathbf{0},\rho) &= \frac{\rho x_2-x_1}{1-\rho^2}f(x_1,x_2;\mathbf{0},\rho),\\
	\frac{\partial}{\partial\rho}\frac{\partial}{\partial x_2}F(x_1,x_2;\mathbf{0},\rho) &= \frac{\rho x_1-x_2}{1-\rho^2}f(x_1,x_2;\mathbf{0},\rho),\\
	\frac{\partial}{\partial\rho}\frac{\partial^2}{\partial x_1\partial x_2}F(x_1,x_2;\mathbf{0},\rho) &= \left(\frac{(\rho x_1-x_2)(\rho x_2-x_1)}{(1-\rho^2)^2}+\frac{\rho}{1-\rho^2}\right)f(x_1,x_2;\mathbf{0},\rho).
\end{align*}
The general case is more complicated and we must distinguish several cases. Suppose we wish to calculate,
\begin{align}
	\frac{\partial}{\partial\rho_{ij}}\frac{\partial}{\partial\mathbf{x}}F(\mathbf{x},\mathbf{y};\mathbf{0},\Sigma) &= \int^{y_1}_{-\infty}\cdots\int^{y_m}_{-\infty}\frac{\partial}{\partial\rho_{ij}}f(x_1,\ldots,x_n,y'_1,\ldots,y'_m)dy'_1\cdots dy'_m. \label{eqn:gencase}
\end{align}
First, suppose (with a slight abuse of notation) that $\rho_{ij}$ is the correlation coefficient between $Y_i$ and $Y_j$. Then, combining \eqref{eqn:plackett} and \eqref{eqn:derivNormal},
\begin{align*}
	\frac{\partial}{\partial\rho_{ij}}\frac{\partial}{\partial\mathbf{x}}F(\mathbf{x},\mathbf{y};\mathbf{0},\Sigma) &= f(\mathbf{x},y_i,y_j)F(\mathbf{y}\setminus\{y_i,y_j\}\ |\ \mathbf{x},y_i,y_j).
\end{align*}
Next, suppose that $\rho_{ij}$ is the correlation coefficient between $X_i$ and $Y_j$, and $\mathbf{Y}\setminus\{Y_j\}\neq\emptyset$, that is, there are variables with respect to which we have not differentiated. Applying \eqref{eqn:plackett} and \eqref{eqn:normalD1} to \eqref{eqn:gencase},
\begin{align*}
	\frac{\partial}{\partial\rho_{ij}}\frac{\partial}{\partial\mathbf{x}}F(\mathbf{x},\mathbf{y};\mathbf{0},\Sigma) &= -f(\mathbf{x},y_j)\int^{\mathbf{y}\setminus\{y_j\}}_{-\mathbf{\infty}}\left(\Sigma^{-1}\mathbf{z}\right)_if(\mathbf{y}\setminus\{y_j\}\ |\ \mathbf{x},y_j)d\mathbf{y}'\setminus\{y'_j\}\\
	&= -f(\mathbf{x},y_j)F(\mathbf{y}\setminus\{y_j\}\ |\ \mathbf{x},y_j)\int^{\mathbf{y}\setminus\{y_j\}}_{-\mathbf{\infty}}\left(\Sigma^{-1}\mathbf{z}\right)_i\frac{f(\mathbf{y}'\setminus\{y'_j\}\ |\ \mathbf{x},y_j)}{F(\mathbf{y}\setminus\{y_j\}\ |\ \mathbf{x},y_j)}d\mathbf{y}\setminus\{y_j\}\\
	&= -f(\mathbf{x},y_j)F(\mathbf{y}\setminus\{y_j\}\ |\ \mathbf{x},y_j)\left(\sum^n_{k=1}\Sigma^{-1}_{i,k}x_k+\Sigma^{-1}_{i,j+n}y_j+\sum_{l\neq j}\left(\Sigma^{-1}_{i,l+n}E[Y_l]\right) \right),
\end{align*}
where $\mathbf{z}=(\mathbf{x},\mathbf{y})$, and the expectation is taken under the \emph{truncated} multivariate normal distribution, truncated from above at $\mathbf{y}\setminus\{y_j\}$, with $\mathbf{\mu}$ and $\Sigma$ given by the parameters of $\mathbf{Y}\setminus\{Y_j\}\ |\ \mathbf{x},y_j$.

The term $\left(\Sigma^{-1}\mathbf{z}\right)_i$ is a linear combination of $\mathbf{z}$. The elements of $\mathbf{z}$ with respect to which we are not integrating are taken outside the integral, which evaluates to $1$. On the contrary, the elements of $\mathbf{z}$ with respect to which we are integrating cannot be taken outside the integral, and form the expectation terms.

When $\mathbf{Y}=\{Y_j\}$, all variables are differentiated at least once, and one variable is differentiated twice. Hence,
\begin{align*}
	\frac{\partial}{\partial\rho_{ij}}\frac{\partial}{\partial\mathbf{x}}F(\mathbf{x},y_j;\mathbf{0},\Sigma) &= -f(\mathbf{x},y_j)\left(\Sigma^{-1}\mathbf{z}\right)_i
\end{align*}

Suppose that $\rho_{ij}$ is the correlation coefficient between $X_i$ and $X_j$, and $\mathbf{Y}\neq\emptyset$. Similarly, applying \eqref{eqn:plackett} and \eqref{eqn:normalD2} to \eqref{eqn:gencase},
\begin{align*}
	\frac{\partial}{\partial\rho_{ij}}\frac{\partial}{\partial\mathbf{x}}F(\mathbf{x},\mathbf{y};\mathbf{0},\Sigma) &= f(\mathbf{x})\int^{\mathbf{y}}_{-\mathbf{\infty}}\left(\Sigma^{-1}\mathbf{z}\mathbf{z}^T\Sigma^{-1}-\Sigma^{-1}\right)_{ij}f(\mathbf{y}'\ |\ \mathbf{x})d\mathbf{y}'\\
	&= f(\mathbf{x})F(\mathbf{y}\ |\ \mathbf{x})\left(\sum^n_{k=1}\sum^n_{l=1}\Sigma^{-1}_{ik}\Sigma^{-1}_{jl}x_kx_l\right. \\
	& \ \ \left. +\sum^m_{l=1}\left(\sum^n_{k=1}\left(\Sigma^{-1}_{ik}\Sigma^{-1}_{j,l+n}+\Sigma^{-1}_{jk}\Sigma^{-1}_{i,l+n}\right)x_k\right)E[Y_l]\right.\\
	& \ \ \left. +\sum^m_{k=1}\sum^m_{l=1}\Sigma^{-1}_{i,k+n}\Sigma^{-1}_{j,l+n}E[Y_kY_l]\right),
\end{align*}
where the expectations are taken under the same distribution as above.

Standard formulae have been derived for the first and second moments of a truncated multivariate distribution \cite{ManjunathWilhelm2009}, \cite{Cartinhour1990}. In this thesis, we only consider copulae having scope between two and four variables. Thus, rather than implementing the (greatly complicated) general case, we use formulae specific to the univariate and bivariate cases \cite{Rosenbaum1961}, \cite{Muthen1990}.

To avoid calculating the expectations twice, and to separate those with a different formula, we use,
\begin{align*}
	\sum^m_{k=1}\sum^m_{l=1}\Sigma^{-1}_{i,k+n}\Sigma^{-1}_{j,l+n}E[Y_kY_l] &= \sum^m_{k=1}\Sigma^{-1}_{i,k+n}\Sigma^{-1}_{j,k+n}E[Y_k^2]\\
	& \ \ +\sum^n_{k=1}\sum_{l>k}\left(\Sigma^{-1}_{i,k+n}\Sigma^{-1}_{j,l+n}+\Sigma^{-1}_{i,l+n}\Sigma^{-1}_{j,k+n}\right)E[Y_kY_l].
\end{align*}

When $\mathbf{Y}=\emptyset$, all variables are differentiated at least once, and two variables are differentiated twice. Hence,
\begin{align*}
	\frac{\partial}{\partial\rho_{ij}}\frac{\partial}{\partial\mathbf{x}}F(\mathbf{x};\mathbf{0},\Sigma) &= f(\mathbf{x})\left(\Sigma^{-1}\mathbf{x}\mathbf{x}^T\Sigma^{-1}-\Sigma^{-1}\right)_{ij}
\end{align*}
Therefore, we see that the result depends on whether the correlation coefficient belongs to variables with respect to which we have already differentiated and whether there are undifferentiated variables.

When the normal copula is parameterized with a single parameter $\rho$, so that $\rho=\rho_{ij}$ for all $i<j$, the chain rule yields,
\begin{align*}
	\frac{\partial}{\partial\rho}\frac{\partial}{\partial\mathbf{x}}F(\mathbf{x},\mathbf{y};\mathbf{0},\Sigma) &= \sum_{i<j}\frac{\partial}{\partial\rho_{ij}}\frac{\partial}{\partial\mathbf{x}}F(\mathbf{x},\mathbf{y};\mathbf{0},\Sigma).
\end{align*}
In this case, we apply our previous procedure and sum the results over all pairs of variables.

When a single parameter is used, the formulae above are greatly simplified by the fact that $\Sigma^{-1}$ contains two unique entries: all diagonal entries are equal, and likewise all off-diagonal entries are equal.

\emph{While we have used some existing results, this is the first time they have been combined and applied to copula theory and gradient-based optimization methods.}

\subsection{Archimedean copulae}
For Archimedean copulae, it is easier to take the derivative of the \emph{log} of the partial derivatives with respect to the copula parameter. Multiplying by the partial derivative then gives the desired result. That is,
\begin{align*}
	\frac{\partial}{\partial\theta}\frac{\partial C}{\partial\mathbf{a}} &= \frac{\partial C}{\partial\mathbf{a}}\frac{\partial}{\partial\theta}\ln\left(\frac{\partial C}{\partial\mathbf{a}}\right).
\end{align*}
For the Clayton copula, by \eqref{eqn:stableClaytonDerivs},
\begin{align*}
	\frac{\partial}{\partial\theta}\ln\left(\frac{\partial C}{\partial\mathbf{a}}\right) &= -\sum^{m-1}_{k=0}\frac{1}{\theta\left(1+\theta k\right)}-\sum_{u_i\in\mathbf{a}}\ln\left(u_i\right)+\frac{m}{\theta}\\
	&\ \ +\frac{1}{\theta^2}\left(\ln\left(\sum^n_{i=1}(u_{\textnormal{min}}/u_i)^{\theta}+(1-n)u_{\textnormal{min}}\right) - \ln(u_{\textnormal{min}})\right)\\
	&\ \ +\left(\frac{1}{\theta} + d\right)\frac{\sum^n_{i=1}\ln\left(u_i\right)(u_{\textnormal{min}}/u_i)^{\theta}}{\sum^n_{i=1}(u_{\textnormal{min}}/u_i)^{\theta}+(1-n)u_{\textnormal{min}}}
\end{align*}
Formulae for the Gumbel and Frank copulae follow the derivation of their score functions in \cite{HofertEtAl2012}. That is, as the score function is the derivative of the \emph{log-density} of a copula with respect to its parameter, a trivial modification is required to differentiate the \emph{log-partial derivative} of a copula with respect its parameter.

\section{Summary}
\begin{itemize}
	\item We explained how the copula enables separation of the representation of the dependence structure from the marginals.
	\item The copula is to be thought of in two equivalent ways:
		\begin{itemize}
			\item as a transformation on the joint distribution that normalizes, or rather, removes the information content, of the marginals;
			\item as a multivariate distribution with standard uniform marginals.
		\end{itemize}
	\item A class of copulae containing the normal copula is constructed by inverting continuous multivariate distributions.
	\item Another class, the Archimedean copulae, is generated by continuous strictly decreasing functions that satisfy complete monotonicity.
	\item Formulae were derived for the partial derivatives of the normal and Clayton copula with respect to an arbitrary subset of their scope and parameter.
	\item Closed form expressions exist for the partial derivatives of the Clayton copula, whereas the normal copula requires the numerical evaluation of the multivariate normal CDF, an expensive operation.
\end{itemize}

\section{A note on implementation}\label{sec:copImplement}
We verified the above formulae using the R packages {\ttfamily mnormt} \cite{website:mnormt} to calculate the multivariate normal density and CDF, {\ttfamily copula} \cite{website:copula} to debug the Archimedean copulae, {\ttfamily numDeriv} \cite{website:numDeriv} to calculate numerical derivatives for comparison, and {\ttfamily tmvtnorm} \cite{website:tmvtnorm} to calculate moments of the truncated multivariate normal distribution to test the formulae for the gradient of normal factors.

In our library, we use the implementation of \cite{Genz1992}, \cite{website:mvnPack}, translating from Fortran to C using {\ttfamily f2c}, to calculate the multivariate normal CDF.

%% file: pgm.tex
\chapter{Probabilistic Graphical Models}
\label{pgm}
Probabilistic graphical models (PGMs) are used as a general framework for representing and reasoning about a wide class of probability distributions. They came to prominence within computer science in the endeavour of artificial intelligence to build an expert system---one that is able to reconcile multiple probabilistic influences to make decisions under uncertainty, performing at the level of a human expert.

In the framework, distributions are represented as the product of factors over subsets of the variables. Associated with the factorization is a graph, wherein the nodes are the random variables of the model, and the edges express the independencies contained in the distribution. The graph is to be thought of in two ways:
\begin{itemize}
	\item as a compact means of representing the probability distribution;
	\item as encoding a set of conditional independence relationships expressed by the distribution.
\end{itemize}
The exact semantics of the graphical representation depend on the type of factorization.

PGMs are a type of declarative representation, that is, they separate the representation of knowledge from the processes that reason with it. This specialization is advantageous; both representation and reasoning can be developed independently. When a new inference algorithm is invented, it is applicable to many existing models without modification, and conversely, improvements to the domain knowledge can be incorporated without having to update the inference engine. General algorithms exist for efficiently performing exact and approximate inference on PGMs, and learning both their structure and parameters.

The applications of PGMs are diverse, since they represent such a general class of distributions. They have found application in, to name a
few, information extraction, medical diagnosis, speech recognition, and computational biology, and have demonstrated superior performance over earlier techniques. Many familiar models, such as Kalman filters, hidden Markov models, and ARCH time series models \cite{ThiessonEtAl2004}, can be expressed in the PGM framework.

Due to space constraints, the technical details of PGMs are beyond the scope of this thesis, and we presuppose the reader's familiarity with the field. Refer to the thorough reference \cite{KollerFriedman2009} if necessary.

In this chapter, we discuss a recently developed class of graphical models---the cumulative distribution network (CDN) \cite{HuangFrey2008}, \cite{HuangFrey2011}. The representation of the model, its basic properties, and a method for parameterization with copulae are explained. Finally, we discuss related work on the copula parameterization of standard PGMs.

\section{Cumulative distribution networks}
Standard PGMs, such as Bayesian networks and Markov fields, factor the model probability density function. In contrast, the CDN is premised on the idea of factoring the cumulative distribution function. As we will explain, this results in a class of models that encodes a set of independencies vastly different from standard PGMs. One obvious application of modelling the CDF is learning to rank \cite{HuangFrey2009}, although we believe the main advantage to be its copula parameterization.

\subsection{Model}
Suppose we have a model over $n$ random variables, $X_1,\ldots,X_n$. Let there be $m$ subsets of the variables, $\mathbf{S}_1,\ldots,\mathbf{S}_m$ so that $\mathbf{S}_i\subseteq\{X_1,\ldots,X_n\}$. Consider the factorization,
\begin{align}\label{eqn:cdnmodel}
	F(x_1,\ldots,x_n) &= \prod_{i=1}^m\phi_i(\mathbf{s}_i),
\end{align}
where $\phi_i$ is an arbitrary function with scope $\mathbf{s}_i$. What condition must we impose on the $\phi_i$ so that \eqref{eqn:cdnmodel} defines a valid CDF?

It can be shown (see \cite{HuangFrey2011}) that a sufficient but not necessary condition is that each factor $\phi_i$ is a CDF over its scope. For example, consider the required property that
\begin{align*}
	F(x_1,\ldots,x_n) &\rightarrow 1\ \ \textnormal{as}\ \ x_1\rightarrow\infty,\ldots,x_n\ \rightarrow\ \infty.
\end{align*} 
Since, by assumption, each factor is a CDF,
\begin{align*}
	\phi_i(\mathbf{s}_i) &\rightarrow 1\ \textnormal{for all}\ i\ \ \textnormal{as}\ \ x_1\rightarrow\infty,\ldots,x_n\ \rightarrow\ \infty,
\end{align*} 
and hence the property holds by the basic properties of limits. The other properties are proven similarly.

We associate with the factorization a bidirected graph that has an edge between two variables when they are in the scope of some factor. More formally,
\begin{defn}
Let $\mathcal{G}$ be a bidirected graph over the variables $X_1,\ldots,X_n$. We say that a distribution $F$ over the same space \emph{factorizes according to $\mathcal{G}$} if $F$ can be expressed as a product,
\begin{align*}
	F(x_1,\ldots,x_n) &= \prod_{i=1}^m\phi(\mathbf{s}_i),
\end{align*}
where there is a bidirected edge $(X_i,X_j)\in\mathcal{G}$ if and only if $\{X_i,X_j\}\subseteq\mathbf{S}_k$ for some $k$.
\end{defn}
Also, we associate a factor graph with the factorization in the same way as for standard PGMs.

Now, we define the CDN.
\begin{defn}
	A \emph{cumulative distribution network} is a pair $\mathcal{C}=(F,\mathcal{G})$ where $F$ factorizes according to $\mathcal{G}$, and where $F$ is specified as a set of CDFs, $\{\phi_i\}$, associated with cliques of $\mathcal{G}$.
\end{defn}
Although, the $\{\phi_i\}$ are not necessarily defined over the \emph{maximal} cliques in $\mathcal{G}$.

\begin{figure}[t]
     \begin{center}
        \subfigure[]{
            \label{fig:studentbn}
            \includegraphics[scale=0.6]{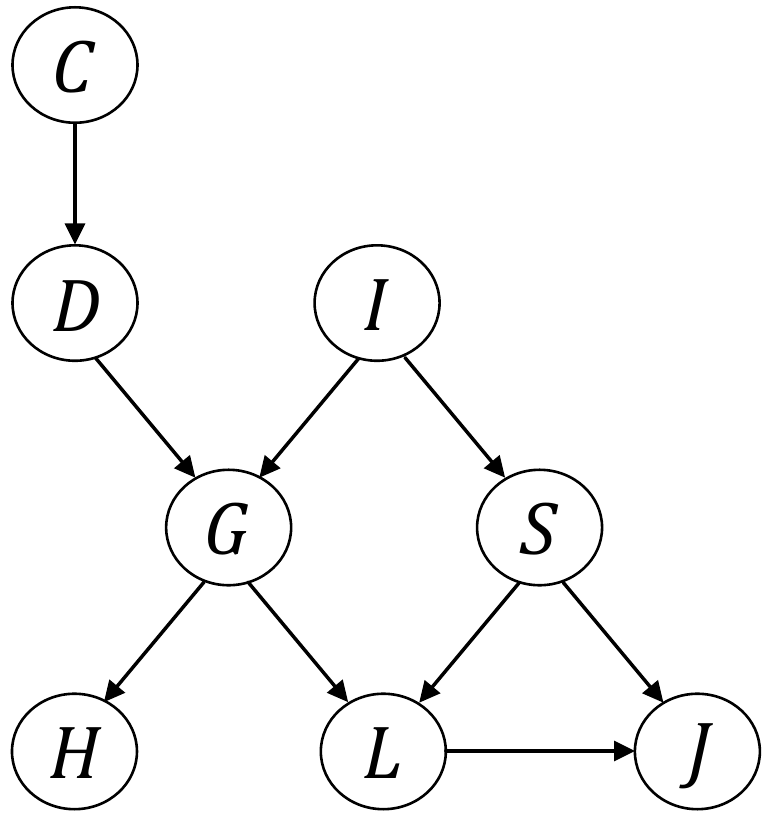}
        }\hspace{1.5cm}
        \subfigure[]{
					\label{fig:studentcdn}
					\includegraphics[scale=0.6]{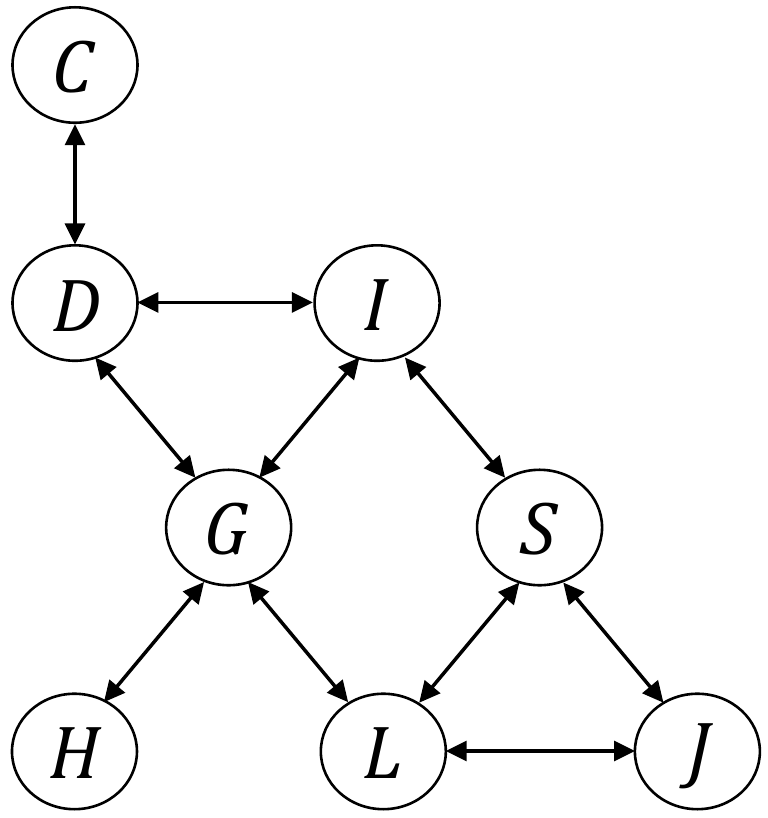}
        }
    \end{center}
    \caption[Student example and an analogous CDN.]{(a) ``Student'' Bayesian network example, adapted from \cite[pg 374]{KollerFriedman2009}. (b) CDN with an analogous structure.}
   \label{fig:studentGraphs}
\end{figure}
\begin{ex}\label{ex:studentcdn}
	We form in Figure \ref{fig:studentcdn} a network structure analogous in terms of edges to the ``Student'' example familiar from \cite{KollerFriedman2009} and shown in Figure \ref{fig:studentbn}. This graph along with the factorization,
	\begin{align*}
		F(c,d,i,g,s,l,j,h) &= \phi_1(c,d)\phi_2(d,i,g)\phi_3(i,s)\phi_4(g,h)\phi_5(g,l)\phi_6(s,l,j),
	\end{align*}
	for some CDFs $\{\phi_i\}$, defines a CDF over the model space.
	
	Note, however, that the the semantics of the graph differ from the Bayesian network. That is, the graph encodes a different set of conditional independence relations.
\end{ex}

\subsection{Parameterization}\label{sec:parameterization}
By parameterizing the factors as copulae, the benefits of copula theory are combined with those of PGMs.

Let $n_i$ denote the number of variables in the scope of the $i$th factor, and let $(i,j)$ index the $j$th variable in the $i$th factor. Our first attempt is to represent the factors as the copulae,
\begin{align*}
	\phi_i(\mathbf{s}_i) &= C_i(u_{(i,1)},\ldots,u_{(i,n_i)}),\ \ i=1,\ldots,n,
\end{align*}
where $U_i=F_i(X_i)$ (which, as mentioned, is distributed as $\mathcal{U}[0,1]$).

Although the product of copulae will always be a valid CDF, it will not, in general, be a copula. For example, consider the lower Frech\'{e}t-Hoeffding bound $W$, which is a copula in two dimensions. Clearly, $W^2$ is not a copula as it is lower than the lower bound. Thus this parameterization, while valid, cannot possess the desired marginals.

We fix this problem by transforming the variables \cite{Liebscher2008}, \cite{SilvaEtAl2011}. Suppose that $u_i$ is in the scope of $k_i$ copula factors. If we raise the $j$th occurrence of $u_i$ to the power $d_{ij}$ for all $i$, where
\begin{align*}
	\sum^{k_i}_{j=1}d_{ij} &= 1,\ \ i=1,\ldots,n
\end{align*}
then it can be shown that the product of the factors on the transformed variables defines a valid copula with the desired marginals.

For convenience, we set $d_{ij}=1/k_i\equiv d_i$. This is a reasonable assumption in the absence domain knowledge. We could, however, make the $d_{ij}$ additional copulae parameters, at the expense of complicating learning.

Let $V_i=U_i^{d_i}$. Then,
\begin{align}\label{eqn:archccdnmodel}
	F(x_1,\ldots,x_n) &= \prod^m_{i=1}C_i(v_{(i,1)},\ldots,v_{(i,n_i)};\ \theta_i),
\end{align}
defines a valid CDF with the desired marginals, where $\theta_i$ is the copula parameter for the $i$th copula. Equation \eqref{eqn:archccdnmodel} is the form of the model when we parameterize with Archimedean copulae.

Suppose we parameterize the factors as normal copulae with a single parameter. Then our model takes the form,
\begin{align}\label{eqn:normalccdnmodel}
	F(x_1,\ldots,x_n) &= \prod^m_{i=1}F_i(w_{(i,1)},\ldots,w_{(i,n_i)};\ \rho_i),
\end{align}
where $W_i=\Phi^{-1}(V_i)$, and $\rho_i$ is the copula parameter for the $i$th normal copula, parameterized with a single parameter as in \S\ref{sec:normalCopula}.

\begin{ex}\label{ex:studentccdn}
	Continuing our previous example, the model takes the form,
	\begin{align*}
		F(c,d,i,g,s,l,j,h) &= F_1(\Phi^{-1}(F_C(c)),\:\Phi^{-1}(F_D(d)^{1/2}))\\
		& \times F_2(\Phi^{-1}(F_D(d)^{1/2}),\:\Phi^{-1}(F_I(i)^{1/2}),\:\Phi^{-1}(F_G(g)^{1/3}))\\
		& \times F_3(\Phi^{-1}(F_I(i)^{1/2}),\:\Phi^{-1}(F_S(s)^{1/2}))\\
		& \times F_4(\Phi^{-1}(F_G(g)^{1/3}),\:\Phi^{-1}(F_H(h)))\\
		& \times F_5(\Phi^{-1}(F_G(g)^{1/3}),\:\Phi^{-1}(F_L(l)^{1/2}))\\
		& \times F_6(\Phi^{-1}(F_S(s)^{1/2}),\:\Phi^{-1}(F_L(l)^{1/2}),\:\Phi^{-1}(F_J(j))),
	\end{align*}
	where we have omitted the dependence on the copulae parameters for clarity.
\end{ex}

\subsection{Independencies}\label{sec:independencies}
The conditional independencies encoded by the factorization are easily read from the associated graph. It can be shown that a bidirected edge between two variables encodes equivalent independencies to a directed acyclic graph (DAG) wherein those two variables have a common parent that has been marginalized. For this reason, one can understand a bidirected edge between variables as modelling an association between those variables due to a latent variable. By replacing the bidirected edges in this manner, one forms the associated \emph{canonical DAG}. 

It is easiest to understand the semantics of the bidirected graph in terms of its corresponding DAG, as we illustrate with an example.
\begin{ex}
The canonical DAG associated with the mode of Example \ref{ex:studentcdn} is shown in Figure \ref{fig:canonicaldag}. Both graph encode the same conditional independencies, which can be read from the DAG using the $d$-separation criterion.

For example, $C\not\perp D$ (unless they are connected by the independence copula), since the path $C\rightarrow Z_1\leftarrow D$ is active. Also, $H\perp L$, since there is no active path between these nodes. They are all blocked by $v$-structures, such as $Z_6\rightarrow G\leftarrow Z_9$. Similarly, we expect $H\not\perp L\ |\ G$, as the $v$-structure $Z_6\rightarrow G\leftarrow Z_9$ is active.
\end{ex}
\begin{figure}[t]
  \centering
  \includegraphics[scale=0.7]{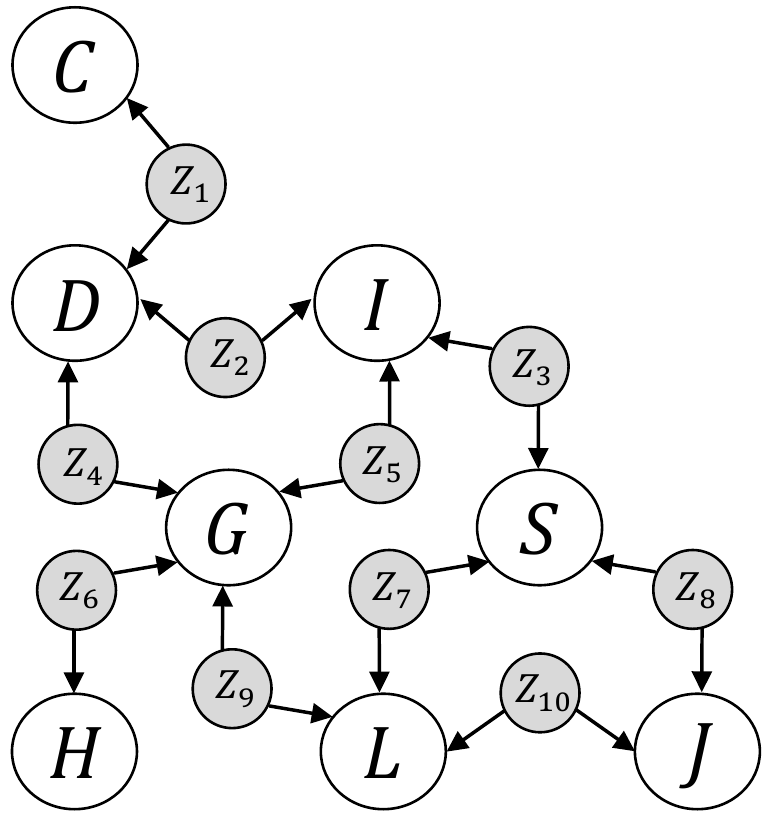}
  \caption[Canonical DAG for Student CDN]{Canonical DAG for Student CDN. The shaded variables are latent. The independencies are read using the $d$-separation criterion.}
	\label{fig:canonicaldag}
\end{figure}
As in the example, in general, variables not connected by a bidirected edge are marginally independent, and may be dependent otherwise. Variables not directly connected by a bidirected edge are independent conditioned on a subset only if no path between them consists wholly of observed variables variables. They may be conditionally dependent otherwise.

Expressing these properties in their most general forms,
\begin{itemize}
	\item $\mathbf{X}\perp\mathbf{Y}$ when $X\leftrightarrow Y\not\in\mathcal{G}$ for all $X\in\mathbf{X}$ and $Y\in\mathbf{Y}$.
	\item Typically, $\mathbf{X}\not\perp\mathbf{Y}$ when $X\leftrightarrow Y\in\mathcal{G}$ for all $X\in\mathbf{X}$ and $Y\in\mathbf{Y}$.
	\item $\mathbf{X}\perp\mathbf{Y}\ |\ \mathbf{Z}$ when there is no path composed of observed variables between any $X\in\mathbf{X}$ and $Y\in\mathbf{Y}$.
	\item Typically, $\mathbf{X}\not\perp\mathbf{Y}\ |\ \mathbf{Z}$ when all paths between each $X\in\mathbf{X}$ and $Y\in\mathbf{Y}$ consist entirely of observed variables.
\end{itemize}
``Typically'' is meant in the sense of occurring for \emph{almost all} parameterizations, as for the converse of $d$-separation in BNs. That is, the set of parameterizations for a given structure for which the second and fourth conditions do not hold has measure zero (see \cite[Theorem 3.5]{KollerFriedman2009}).

We see that a CDN encodes an unusual set of independencies; the use of \emph{bi}directed edges was to distinguish these as separate from those of BNs and MRFs. These unusual independencies entail modifications to standard inference, sampling, and learning algorithms. Indeed, we are studying a new class of models.

That the CDN is equivalent to a DAG with marginalized variables cannot be exploited for inference and sampling, as it is unclear how to go from the CDN parameterization to an equivalent DAG parameterization. Indeed, if it were easy to go from one parameterization to the other, inference and learning in BNs with latent variables would not prove a difficult task.

\subsection{Latent variables in CDNs}\label{sec:latentVariables}
In the previous section, it was explained how a bidirected edge implicitly represents a latent parent between two variables. It does not make sense in this model to explicitly include latent variables, as we demonstrate with a simple example.
\begin{ex}
Consider connecting two variables by a latent variable, as in Figure \ref{fig:latentexa}. The canonical DAG is given in Figure \ref{fig:latentexb}. The $v$-structure $Z_1\rightarrow Z\leftarrow Z_2$ is always blocked, as $Z$ is never observed, and $Z$ cannot be made to have descendents in the canonical DAG formed after introducing additional variables. (It is only possible to give $Z$ additional parents by introducing new variables.)

Therefore the addition of the latent variable has not encoded any conditional dependencies, or rather, $A$ and $B$ are independent regardless of the inclusion of $Z$.
\end{ex}
\begin{figure}[t]
     \begin{center}
        \subfigure[]{
            \label{fig:latentexa}
            \includegraphics[scale=0.7]{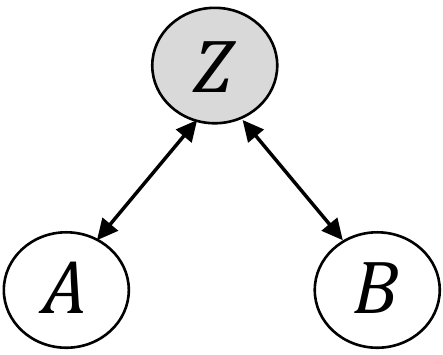}
        }\hspace{1cm}
        \subfigure[]{
					\label{fig:latentexb}
					\includegraphics[scale=0.7]{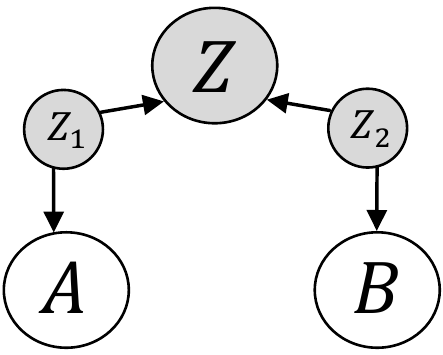}
        }
    \end{center}
    \caption[Example of adding a latent variable to a CDN.]{(a) CDN with latent variable $Z$. (b) Canonical DAG associated with this model.}
   \label{fig:latentex}
\end{figure}

A similar argument holds when the latent variable connects more than two variables or is connected to another latent variable.

From another perspective, latent variables, when marginalized out, simply reduce the dimension of any copula for whose scope they are in. This has the same effect as if they had not been included in the first instance. Therefore, there is nothing to be gained by the inclusion of explicit latent variables.

\section{Markov networks}
Undirected models, such as the Markov network, are an important class of graphical model that can encode symmetric conditional independence relations between variables. They possess, however, no obvious copula parameterization.

A Markov network factors the density as a normalized product of nonnegative factors. Suppose we factor the copula density as,
\begin{align*}
	c(\mathbf{u}) &= \frac{1}{Z}\prod_i\phi_i(\mathbf{c}_i).
\end{align*}
where $\mathbf{u}=(u_1=F_1(x_1),\ldots,u_n=F_n(x_n))$ and the $\{\mathbf{c}_i\}$ are subsets of the variables $\{u_1,\ldots,u_n\}$.
  
In order for $c$ to be a valid copula density, it must have uniform marginals. That is,
\begin{align}\label{eqn:copulaMarkov}
	1 &= \int c(\mathbf{u})d({\mathbf{u}\setminus\{u_i\}}),\ i=1,\ldots,n
\end{align}
for all $u_i$. It is unclear what sufficient condition on the factors enforces \eqref{eqn:copulaMarkov}.

\section{Copula Bayesian networks}
In contrast, Bayesian networks permit an obvious copula parameterization \cite{Elidan2010a}. Recall that in a Bayesian Network, the distribution is factored over the graph as,
\begin{align*}
	f(\mathbf{x}) &= \prod_if(x_i\ |\ \textnormal{Pa}_\mathcal{G}(x_i)),
\end{align*}
where $\textnormal{Pa}_\mathcal{G}(x_i)$ denotes the parents of $x_i$ in the graph $\mathcal{G}$.

One combines this framework with copula theory as follows. For a joint distribution $F$ that factors over a graph $\mathcal{G}$, suppose the marginals of each variable $F_i$ and a copula over each family $C_i(u_i, \textnormal{Pa}_\mathcal{G}(u_i))$ in $\mathcal{G}$ are given. Any conditional independence relations that hold in the distribution also hold in the copula distribution. Therefore,
\begin{align*}
	c(\mathbf{u}) &= \prod_ic(u_i\ |\ \textnormal{Pa}_\mathcal{G}(u_i))\\
	&= \prod_i\frac{c(u_i, \textnormal{Pa}_\mathcal{G}(u_i))}{c(\textnormal{Pa}_\mathcal{G}(u_i))}.
\end{align*}
The denominator is calculated by taking the derivative of $C(\textnormal{Pa}_\mathcal{G}(u_i))=C(1,\textnormal{Pa}_\mathcal{G}(u_i))$, which is analytically tractable in many useful cases.

Thus, the distribution factors as,
\begin{align*}
	f(\mathbf{x}) &= \prod_ic(u_i\ |\ \textnormal{Pa}_\mathcal{G}(u_i))f(x_i),
\end{align*}
having separately specified the marginals and the dependence structure over subsets of the variables. The copulae parameters of each family are learnt by the standard methods, for example by maximizing a pseudo-likelihood \cite{GenestFavre2007}.

The model has several advantages. Most obviously, modelling is made more flexible by separating the specification of marginals their association, which also permits more robust and efficient learning, and faster learning with missing-at-random data \cite{Elidan2010b}. When learning the structure of a Bayesian Network, evaluating an expensive entropy term is required to evaluate the change in score of adding, removing, or reversing an edge. Spearman's rho, however, can be used as a proxy for the entropy in this model, drastically speeding up structure learning \cite{Elidan2012a}.

There is a straightforward extension allowing this model to be used for classification that performs better than the state-of-the-art support vector machine (SVM) in some domains \cite{Elidan2012b}. Moreover, the model is drastically faster to train than SVM. We speculate that the superior performance is due to the precise representation of the marginals with non-parametric kernel density estimation, which may be more important for classification than the dependence structure.

\section{Summary}
\begin{itemize}
	\item A CDN factors a CDF as a product of CDFs over subsets of the variables.
	\item The CDN permits a copula parameterization in which each of the factors is a copula with arguments that have been transformed.
	\item The independencies encoded by a CDN are read by applying the $d$-separation criterion to its canonical DAG.
	\item An edge in the model is equivalent to a latent parent between those variables that has been marginalized.
	\item Latent variables cannot be modelled explicitly with CDNs.
	\item Bayesian networks permit an obvious copula parameterization, in contrast to Markov networks.
	\item Copula BNs are advantageous for representation, learning, and classification.
\end{itemize}

%% file: inference.tex
\chapter{Inference}
\label{inference}
Inference, in the context of machine learning, refers to deriving conclusions that are valid with respect to one's model, such as the probability an event occurs conditioned on another event, or the most likely outcome for some subset of the model conditioned on the remaining variables. It is not to be confused with ``statistical inference,'' which is deriving valid conclusions about the model parameters from a sample, and is typically referred to in machine learning as ``parameter estimation.''

In this chapter, we discuss the types of inference, or queries, that can be performed on a CDN, and the requisite algorithms. We will see that CDNs admit a wider variety of queries relative to standard PGMs. The derivative-sum-product algorithm for efficiently differentiating the model is explained in fine detail.

\section{Types of queries}
All queries are comprised of three basic operations: evaluation of the factors, taking the limit of a variable, and differentiation. 

\subsection{Evaluating the full CDF}
Suppose our model is over variables $\mathbf{X}$ and we have two subsets $\mathbf{A}\subseteq\mathbf{X}$ and $\mathbf{B}\subseteq\mathbf{X}$ such that $\mathbf{A}\cap\mathbf{B}=\emptyset$. Recall the form of the model given by \eqref{eqn:cdnmodel}.

The most basic query is to calculate the full model CDF. We do this simply by evaluating each factor on its scope,
\begin{align*}
	P(\mathbf{X}\preceq\mathbf{x}') &= \prod_{i=1}^n\phi_i\left(\mathbf{s}_i[\mathbf{x}=\mathbf{x}']\right).
\end{align*}
This operation is analogous to calculating the full density in standard PGMs.

\subsection{Marginalization}
By definition of the CDF, variables are marginalized from the model CDF by taking the limit of \eqref{eqn:cdnmodel} as those variables approach infinity. Or rather, working in the extended reals, we set those variables to infinity. For example,
\begin{align*}
	P(\mathbf{A}\preceq\mathbf{a}') &= \prod_{i=1}^n\phi_i\left(\mathbf{s}_i[\mathbf{a}=\mathbf{a}',\mathbf{x}\setminus\mathbf{a}=\mathbf{\infty}]\right).
\end{align*}
To marginalize a variable in a copula distribution we set it to $1$, since the marginals are distributed as $\mathcal{U}[0,1]$.

It is noteworthy that marginalization is a trivial operation in CDNs, in contrast to standard PGMs where one must sum or integrate over those variables.

\subsection{The derivatives of the CDF}
Taking the derivative of the CDF with respect to all variables gives the density, a quantity of prime interest; for example, it is required to calculate the log-likelihood of a sample. With regards to our model,
\begin{align*}
	f(\mathbf{X}=\mathbf{x}') &= \left(\frac{\partial}{\partial\mathbf{x}}\left.\prod_{i=1}^n\phi_i\left(\mathbf{s}_i\right)\right)\right|_{\mathbf{x}=\mathbf{x}'}.
\end{align*}

Taking the partial derivative of a CDF with respect to a subset of the variables produces a mixed density/CDF. For example,
\begin{align*}
	P(\mathbf{A}=\mathbf{a}', \mathbf{B}\preceq\mathbf{b}') &= \left(\frac{\partial}{\partial\mathbf{a}}\left.\prod_{i=1}^n\phi_i\left(\mathbf{s}_i[\mathbf{b}=\mathbf{b}']\right)\right)\right|_{\mathbf{a}=\mathbf{a}'},
\end{align*}
assuming that $\mathbf{A}\cup\mathbf{B}=\mathbf{X}$.

\subsection{More complicated queries}
One performs a range of queries by combining marginalization and differentiation. Refer to Figure \ref{fig:cdnqueries} for a full listing.
\begin{figure}[t]
	\begin{tabular}{ll}
		Full CDF & $P(\mathbf{X}\preceq\mathbf{x}') = \prod_{i=1}^n\phi_i\left(\mathbf{s}_i[\mathbf{x}=\mathbf{x}']\right)$ \\[16pt]
		Marginal CDF & $P(\mathbf{A}\preceq\mathbf{a}') = \prod_{i=1}^n\phi_i\left(\mathbf{s}_i[\mathbf{a}=\mathbf{a}',\mathbf{z}=\infty]\right)$ \\[16pt]
		Marginal CDF conditioned & \multirow{2}{*}{$P(\mathbf{A}\preceq\mathbf{a}'\ |\ \mathbf{B}\preceq\mathbf{b}') = \frac{P(\mathbf{A}\preceq\mathbf{a}', \mathbf{B}\preceq\mathbf{b}')}{P(\mathbf{B}\preceq\mathbf{b}')}$} \\
		on a cumulative event & \\[16pt]
		Full density & $f(\mathbf{X}=\mathbf{x}') = \left(\frac{\partial}{\partial\mathbf{x}}\left.\prod_{i=1}^n\phi_i\left(\mathbf{s}_i\right)\right)\right|_{\mathbf{x}=\mathbf{x}'}$ \\[16pt]
		Marginal density & $f(\mathbf{A}=\mathbf{a}') = \left(\frac{\partial}{\partial\mathbf{a}}\left.\prod_{i=1}^n\phi_i\left(\mathbf{s}_i[\mathbf{z}=\infty]\right)\right)\right|_{\mathbf{a}=\mathbf{a}'}$ \\[16pt]
		Mixed density and CDF & $P(\mathbf{A}=\mathbf{a}', \mathbf{B}\preceq\mathbf{b}') = \left(\frac{\partial}{\partial\mathbf{a}}\left.\prod_{i=1}^n\phi_i\left(\mathbf{s}_i[\mathbf{b}=\mathbf{b}',\mathbf{z}=\infty]\right)\right)\right|_{\mathbf{a}=\mathbf{a}'}$ \\[16pt]
	Density conditioned & \multirow{2}{*}{$f(\mathbf{A}=\mathbf{a}'\ |\ \mathbf{B}\preceq\mathbf{b}') = \frac{P(\mathbf{A}=\mathbf{a}', \mathbf{B}\preceq\mathbf{b}')}{P(\mathbf{B}\preceq\mathbf{b}')}$} \\
	on a cumulative event & \\[16pt]
	Conditional density & $f(\mathbf{A}=\mathbf{a}'\ |\ \mathbf{B}=\mathbf{b}') = \frac{P(\mathbf{A}=\mathbf{a}', \mathbf{B}=\mathbf{b}')}{P(\mathbf{B}=\mathbf{b}')}$\\[8pt]
	\end{tabular}
  \caption[Queries that can be performed on a CDN.]{Complete list of queries that can be performed on a CDN. The vector $\mathbf{z}$ is to be understood as those variables in the model that do not appear as arguments.}
	\label{fig:cdnqueries}
\end{figure}

\section{Differentiating the model}
Without loss of generality, we consider in detail calculating the full model density. Calculating a marginal density or mixed density/CDF is accomplished by reducing the factors as explained previously, then performing differentiation over a subset of the model.

\subsection{Exchanging differentiation and multiplication}\label{sec:diffcdn}
Consider again the form of the model with normal copulae and arbitrary marginals given by \eqref{eqn:normalccdnmodel}. In this case, the density is,
\begin{align}
	f(x_1,\ldots,x_n) &= \frac{\partial}{\partial\mathbf{x}}\prod^m_{i=1}F_i(w_{(i,1)},\ldots,w_{(i,n_i)};\ \rho_i) \nonumber \\
	&= \kappa\frac{\partial}{\partial\mathbf{w}}\prod^m_{i=1}F_i(w_{(i,1)},\ldots,w_{(i,n_i)};\ \rho_i). \label{eqn:normalDensity}
\end{align}
where,
\begin{align*}
	\kappa &= \prod^m_{i=1}\frac{\partial u_i}{\partial x_i}\frac{\partial v_i}{\partial u_i}\frac{\partial w_i}{\partial v_i}\\
		&= \prod^n_{i=1}f_i(x_i)d_iu_i^{d_i-1}\frac{d\Phi^{-1}_i}{dv_i}.
\end{align*}
The derivative of the probit function is calculated by,
\begin{align*}
	x &= F(\Phi^{-1}(x)) \\
	\Rightarrow\ 1 &= f(\Phi^{-1}(x))\frac{d\Phi^{-1}}{dx} \\
	\Rightarrow\ \frac{d\Phi^{-1}}{dx} &= 1 / f(\Phi^{-1}(x)).
\end{align*}
Similarly, when the model is parameterized with Archimedean copulae, the density is,
\begin{align}
	f(x_1,\ldots,x_n) &= \kappa'\frac{\partial}{\partial\mathbf{v}}\prod^m_{i=1}C_i(v_{(i,1)},\ldots,v_{(i,n_i)};\ \theta_i), \label{eqn:archDensity}
\end{align}
where,
\begin{align*}
	\kappa' &= \prod^n_{i=1}f_i(x_i)d_iu_i^{d_i-1}.
\end{align*}
The problem is, of course, how to calculate the derivative of the product of factors. Consider the simplest example of two factors. The naive approach uses the product rule,
\begin{align*}
	\frac{\partial}{\partial\mathbf{x}}\phi_1(\mathbf{s}_1)\phi_2(\mathbf{s}_2) &= \sum_{\mathbf{a}\subseteq\mathbf{x}}\frac{\partial\phi_1}{\partial\mathbf{a}}\frac{\partial\phi_2}{\partial(\mathbf{x}\setminus\mathbf{a})}.
\end{align*}
The summation has $2^n$ terms, where $n$ is the dimension of $\mathbf{x}$, and there will be an even greater number of terms as the number of factors increases. Clearly, this approach is intractable.

However, by exploiting the structure of the factorization---a recurrent theme of PGMs---the derivative is efficiently calculated. The method is based on the observation that one may exchange the order of multiplication and differentiation. All factors that do not have a given variable in their scope can be taken outside the differentiation with respect to that variable. (The procedure is analogous to variable elimination and the sum-product algorithm in standard PGMs.)

Each step of differentiation produces a factor over a smaller scope. By performing the differentiation on a smaller set of the factors at each step, differentiation is no longer exponential in the size of the model. This is best illustrated with an example.

\begin{ex}\label{ex:cdnvarelim}
	Refer to Example \ref{ex:studentcdn}. We choose the differentiation order $c,d,h,g,j,i,s,l$. Then the intermediate factors and density are calculated as,
	\begin{align*}
		\psi_1(d) &= \frac{\partial}{\partial c}\phi_1(c,d) \\
		\psi_2(i,g) &= \frac{\partial}{\partial d}\left(\phi_2(d,i,g)\psi_1(d)\right) \\
		\psi_3(g) &= \frac{\partial}{\partial h}\phi_4(g,h) \\
		\psi_4(i,l) &= \frac{\partial}{\partial g}\left(\phi_5(g,l)\psi_2(i,g)\psi_3(g)\right) \\
		\psi_5(s,l) &= \frac{\partial}{\partial j}\phi_6(s,l,j) \\
		f(c,\ldots,h) &= \frac{\partial^3}{\partial i\partial s\partial l}\left(\phi_3(i,s)\psi_4(i,l)\psi_5(s,l)\right).
	\end{align*}
	Substituting the intermediate factors into the final answer clarifies how the two operations were interchanged,
	\begin{align*}
		f(c,\ldots,h) &= \frac{\partial^3}{\partial i\partial s\partial l}\left(\phi_3\frac{\partial}{\partial g}\left(\phi_5\frac{\partial}{\partial d}\left(\phi_2\frac{\partial}{\partial c}\phi_1\right)\frac{\partial}{\partial h}\phi_4\right)\frac{\partial}{\partial j}\phi_6\right).
	\end{align*}
\end{ex}

The efficiency of this procedure depends on the differentiation order, as this effects the size of the intermediate factors at each step. It can be shown that choosing an optimal differentiation order is an $NP$-hard problem, analogous to variable elimination, although the same proven heuristics for variable elimination are applicable.

\subsection{Message passing}
The procedure of \S \ref{sec:diffcdn} is equivalent to an algorithm that passes messages over a data structure known as a clique tree. A clique tree is a tree for which the nodes, or \emph{cliques}, are subsets of the model variables and each factor is associated with a node. It must also satisfy a connectedness condition. Specifically,
\begin{defn}
A \emph{clique tree} for a set of factors $\Phi$ over $\mathcal{X}$ is an undirected tree, each of whose nodes $i$ is associated with a subset $\mathbf{C}_i\subset\mathcal{X}$, that satisfies the following two properties,
	\begin{enumerate}
		\item family preservation: each factor $\phi_i\in\Phi$ must be associated with a node $\mathbf{C}_{\alpha(i)}$ such that $\textnormal{Scope}[\phi_i]\subset\mathbf{C}_{\alpha(i)}$.
		\item running intersection: whenever there is a variable $X$ such that $X\in\mathbf{C}_i$ and $X\in\mathbf{C}_j$, then $X$ is also in every node in the (unique) path between $\mathbf{C}_i$ and $\mathbf{C}_j$.
	\end{enumerate}
	A set $\mathbf{S}_{i,j}=\mathbf{C}_i\cap\mathbf{C}_j$, known as the \emph{sepset}, is associated with the edge between adjacent nodes $\mathbf{C}_i$ and $\mathbf{C}_j$. In a slight abuse of notation, we use $\mathbf{C}_i$ to denote both the set of variables in the $i$-th clique and the set of factors associated with that clique.
\end{defn}
At a high level, a clique tree can be thought of as a data structure that exposes the structure of the factorization, so that it may be exploited for efficient inference, learning, sampling, and so on.

In synchronous message passing, for each node, starting from the leaves and progressing towards an arbitrary node that is chosen as the root, a message is formed by performing some operation to combine the factors associated with a clique with its incoming messages. The message is sent downstream towards the root. A node must wait for all of its upstream messages before it is ready to calculate and transmit its own downstream, necessitating an algorithm for scheduling message passing. The message from clique $i$ to clique $j$ is denoted $\psi_{i\rightarrow j}$.
\begin{algorithm}[t]
\caption{Pass messages in the default topological ordering given an arbitrary message function}
\label{alg:passMessages}
\begin{algorithmic}[1]
	\Procedure{PassMessages}{$\mathbf{u}, \mathcal{C}, \textnormal{\scshape CalcMsgs}$}
		\State $\textnormal{result} = 1$\Comment{there may be more than one root if ``clique forest''}
		\For{$C_i\in\mathcal{C}$ (taken in the topological ordering)}
			\If{$\textnormal{child}(i)\neq-1$}
				\State $\textnormal{\scshape CalcMsgs}(i, \textnormal{child}(i))$
			\Else
				\State $\textnormal{result} = \textnormal{result}\times\textnormal{\scshape CalcMsgs}(i, \textnormal{child}(i))$
			\EndIf
		\EndFor
		\State \textbf{return} $\textnormal{result}$
	\EndProcedure
\end{algorithmic}
\end{algorithm}

In our case, the incoming messages and the associated factors are multiplied together, and differentiated with respect to the variables in the clique that are not in its downstream sepset. This is analogous to the sum-product algorithm, for which those variables are summed out. As discussed previously, this operation is exponential in the size of the clique due to the correspondence between the scopes of the intermediate factors and the cliques. The \emph{treewidth} is defined as the size of the largest clique minus one. Thus, calculation of the messages is exponential in the treewidth.

Multiple clique trees exist for a distribution with differing treewidths. Finding the optimal clique tree---the one with the lowest treewidth---like finding the optimal differentiation ordering, is an {\itshape NP}-hard problem. Nonetheless, a satisfactory clique tree is constructed with the same greedy heuristics as for determining a good elimination ordering. A discussion of clique tree construction is incidental to our purpose; we simply state that in our implementation clique trees are constructed by simulating variable elimination, successive variables chosen by the min fill-edge criterion (see \cite[\S10.4.1]{KollerFriedman2009}).

A topological ordering that is used to schedule messages in a default ordering is determined during construction. In this ordering, the final clique is the root, and for each clique, all cliques with a lesser index are upstream and all cliques with a greater index are downstream. Therefore, to schedule the messages, one simply passes the message from clique $i$ to its downstream clique in the order $i=1,2,3,\ldots$. When notating clique trees, we direct the edges to indicate the default direction that messages are passed. Subsequently, it makes sense to refer to a clique's parents and (unique) child. See Algorithm \ref{alg:passMessages} for a summary of default message scheduling.

\begin{ex}
	A clique tree for Example \ref{ex:studentcdn} is given in Figure \ref{fig:cliquetree}. In fact, since the tree was constructed using the same variable order as Example \ref{ex:cdnvarelim}, the messages correspond to those intermediate factors. That is, $\psi_{1\rightarrow2}=\psi_1$, $\psi_{2\rightarrow5}=\psi_2$, etc. It is clear we have constructed a valid topological ordering: all cliques with a lesser index are upstream relative to that clique.
\end{ex}
\begin{figure}[t]
  \centering
  \includegraphics[scale=0.6]{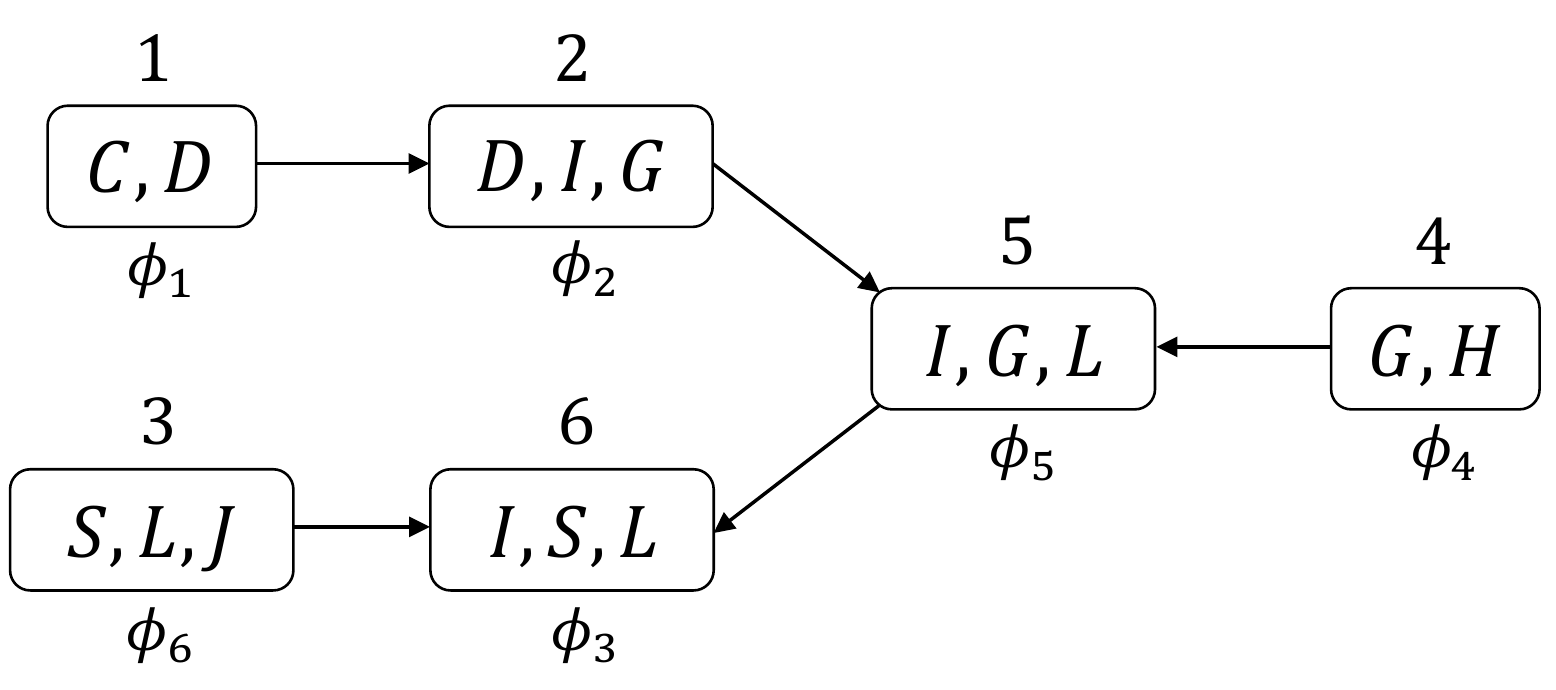}
  \caption[A clique tree for the Student example]{A clique tree for the Student example outputted by our implementation of clique tree construction. The number above the cliques indicate a topological ordering, and the direction of the arrows indicates the direction of message passing under this ordering. The factors in each clique are notated beneath.}
	\label{fig:cliquetree}
\end{figure}

\section{The derivative-sum-product algorithm}\label{sec:dspAlgorithm}
In practice, however, representing and manipulating symbolic factors is complicated and inefficient, and we are only interested in the result evaluated at a given number. In this section, we extend the previous algorithm to evaluate the model derivative for a given value. The resulting algorithm is known as the \emph{derivative-sum-product algorithm} \cite{HuangEtAl2010}, \cite{HuangFrey2008}, \cite{HuangFrey2011}.

The intuition behind this algorithm was presented in the previous section. Instead of, however, passing functions $\psi_{i\rightarrow j}$ from clique $i$ to clique $j$, we pass a set of partial derivatives evaluated at a fixed value. Define for each non-root clique $C_i$ that passes to $C_j$,
\begin{align*}
	\delta_{i\rightarrow j}(A) &:= \left.\left(\frac{\partial}{\partial A}\psi_{i\rightarrow j}\right)\right|_{\mathbf{x}=\mathbf{x}'} \\
	&=
	\left.\left(
		\frac{\partial}{\partial(A\cup (C_i\setminus S_{i,j}))}
		\left(
			\prod_{\phi_k\in C_i}\phi_k\prod_{l\in\textnormal{Pa}(C_i)}\psi_{l\rightarrow i}
		\right)
	\right)\right|_{\mathbf{x}=\mathbf{x}'},
\end{align*}
for all $A\subseteq S_{i,j}$. The set of messages passed from $i$ to $j$ is,
\begin{align*}
	\Delta_{i\rightarrow j} &:= \{\delta_{i\rightarrow j}(A)\ |\ A\subseteq S_{i,j}\}. 
\end{align*}

To calculate the final answer, the root clique is differentiated with respect to its entire scope, requiring the derivatives with respect to every subset of its scope. Define for the root clique $i$,
\begin{align*}
	\delta_{i}(A) &:= \left.\left(\frac{\partial}{\partial A}\left(\prod_{\phi_k\in C_i}\phi_k\prod_{l\in\textnormal{Pa}(C_i)}\psi_{l\rightarrow i}\right)\right)\right|_{\mathbf{x}=\mathbf{x}'},
\end{align*}
for all $A\subseteq C_i$. We think of these values as ``messages'' from clique $i$ to itself. The result of message passing is $\delta_{i}(C_i)$. (The dependence of $\delta_{i\rightarrow j}$ and $\delta_{i}$ on $\mathbf{x}'$ is omitted for clarity.)

It will become clear why the additional derivatives are required to evaluate the density numerically when we explain how the messages are calculated.

Assume that the parents of each clique are ordered, and denote by $(i,p)$ the $p$th parent of $C_i$. Also, assume the factors associated with each clique are ordered and denote by $(i,q)$ the $q$th factor in the $C_i$ clique. (It will be clear from the context which meaning is intended for the parentheses.) Let $p_i$ denote the number of parents of $C_i$, and $q_i$ the number of its factors.

Each message is efficiently calculated using the product rule and dynamic programming. To this end, define for each non-root clique $C_i$ that passes to $C_j$ the \emph{partial messages}, 
\begin{align*}
	\delta_i(A, p) &:= \left.\left(\frac{\partial}{\partial A}\left(\prod_{k=1}^{q_i}\phi_{(i,k)}\prod_{l=1}^p\psi_{(i,l)\rightarrow i}\right)\right)\right|_{\mathbf{x}=\mathbf{x}'}
\end{align*}
for $p=1,\ldots,p_i$ and $A\subseteq C_i$.

Also, define the \emph{partial factor derivatives},
\begin{align*}
	\mu_i(A, q) &:= \left.\left(\frac{\partial}{\partial A}\prod^q_{k=1}\phi_{(i,k)}\right)\right|_{\mathbf{x}=\mathbf{x}'}
\end{align*}
for $q=1,\ldots,q_i$ and $A\subseteq C_i$. Define the set of final partial factor derivatives as,
\begin{align*}
	M_i &:= \{\mu_i(A,q_i)\ |\ A\subseteq C_i\}.
\end{align*}
Note that for non-root cliques $C_i$,
\begin{align}
	\delta_{i\rightarrow j}(A) &= \delta_i(A\cup(C_i\setminus S_{i,j}), p_i), \label{eqn:partialFinalMsg} \\
	\mu_i(A,q_i) &= \delta_i(A, 0), \label{eqn:partialFinalFactors}
\end{align}
and similarly for root cliques. Thus, the partial messages can be thought of as to what the message from $i$ to $j$ would have evaluated had we omitted some of its parents.

Applying the product rule,
\begin{align}
	\delta_{i\rightarrow j}(A,p) &= \left.\left(\frac{\partial}{\partial A}\left(\prod_{k=1}^{q_i}\phi_{(i,k)}\prod_{l=1}^p\psi_{(i,l)\rightarrow i}\right)\right)\right|_{\mathbf{x}=\mathbf{x}'} \nonumber \\
	&= \sum_{B\subseteq A}\left.\left(\frac{\partial}{\partial B}\psi_{(i,p)\rightarrow i}\right)\right|_{\mathbf{x}=\mathbf{x}'}\left.\left(\frac{\partial}{\partial(A\setminus B)}\prod_{k=1}^{q_i}\phi_{(i,k)}\prod_{l=1}^{p-1}\psi_{(i,l)\rightarrow i}\right)\right|_{\mathbf{x}=\mathbf{x}'} \nonumber \\
	&= \sum_{B\subseteq A\cap S_{(i,p),i}}\delta_{(i,p)\rightarrow i}(B)\delta_{i\rightarrow j}(A\setminus B,p-1). \label{eqn:calcMsg}
\end{align}
In the last line, we used the fact that $\partial\psi/\partial B=0$ when $B$ is not a subset of the scope of $\psi$.

Equation \eqref{eqn:calcMsg} characterizes the structure of the overlapping subproblems. Importantly, each term in this equation is a value rather than a function, and thus the partial messages can be calculated without symbolic manipulation.

To calculate the final messages, we use a top-down approach. For each $A\subseteq S_{i,j}$, we start from \eqref{eqn:partialFinalMsg} and apply \eqref{eqn:calcMsg}. The contributions in the sum from all partial messages that have been calculated are added in to this partial message, and those that have not been calculated are pushed to a stack. Then, the top element of the stack is examined, and we apply \eqref{eqn:calcMsg} to this partial message, again adding in the contributions of calculated partial messages and pushing all further uncalculated partial messages to the stack. 

When an element from the top of the stack is examined a second time, all of the terms of \eqref{eqn:calcMsg} for this partial message have been, perforce, added in, and thus it is popped from the stack. This procedure is repeated until the stack is empty and the calculation of $\delta_{i\rightarrow j}(A)$ completed. The recursion is guaranteed to terminate by \eqref{eqn:partialFinalFactors} and the fact that the derivatives of the product of the factors have been computed in advance.

Alternatively, one can adopt a bottom-up algorithm, as in the original paper \cite{HuangEtAl2010}, and calculate $\delta_i(A, p)$ for all $A\subseteq C_i$, for increasing $p$. Usually, some partial messages are not required to calculate the final messages, and it is difficult to determine in advance which messages these are. On the other hand, our top-down approach only calculates as many partial messages as are required.

The authors of \cite{HuangEtAl2010} leave unspecified the details of calculating the partial derivatives of the product of all factors associated with a clique. We devised a top-down dynamic programming algorithm similar to the one previously described. The structure of the subproblems is characterized by,
\begin{align}
	\mu_i(A,q) &= \left.\left(\frac{\partial}{\partial A}\prod^q_{k=1}\phi_{(i,k)}\right)\right|_{\mathbf{x}=\mathbf{x}'} \nonumber \\
	&= \sum_{B\subset A}\left.\left(\frac{\partial}{\partial B}\phi_{(i,q)}\right)\right|_{\mathbf{x}=\mathbf{x}'}\left.\left(\frac{\partial}{\partial(A\setminus B)}\prod^{q-1}_{k=1}\phi_{(i,k)}\right)\right|_{\mathbf{x}=\mathbf{x}'} \nonumber \\
	&= \sum_{\substack{B\subseteq A\cap\textnormal{Scope}[\phi_{(i,q)}] \\ A\setminus B\subseteq\textnormal{Scope}[\prod^{q-1}_{k=1}\phi_{(i,k)}]}} \left.\left(\frac{\partial}{\partial B}\phi_{(i,q)}\right)\right|_{\mathbf{x}=\mathbf{x}'}\mu_i(A\setminus B,q-1). \label{eqn:calcProdFactors}
\end{align}
First, $\partial\phi_k/\partial B$ is evaluated and stored for every $\phi_k\in C_i$ and $B\subseteq\textnormal{Scope}[\phi_k]$. (See \S\ref{sect:partialDerivFactors}.)

Then, for each $A\subseteq C_i$, \eqref{eqn:calcProdFactors} is recursively applied to \eqref{eqn:partialFinalFactors}, calculating the partial factor derivatives in a manner analogous to the calculation of the partial messages.

In our implementation, calculation of the partial factor derivatives and the partial messages are broken into two steps for simplicity of implementation. Although, it \emph{is} possible to combine them into a single top-down dynamic programming algorithm, which would avoid the calculation of any unnecessary value. Again, since each term in \eqref{eqn:calcProdFactors} is clamped to the values $\mathbf{x}=\mathbf{x}'$, no symbolic manipulation is required.

The correctness of the derivative-sum-product algorithm should be clear from the intuition given in the previous section. Essentially, it follows from the properties of a clique tree, and its proof is very similar to that for the sum-product algorithm.

Refer to Algorithm \ref{alg:calcDspMessages} for a summary of calculating the derivative-sum-product messages. We use the convention that calculating the result at the root clique is denoted by passing to ``clique $-1$.'' It is assumed that the helper functions memoize their values, which are accessible to all methods. Also, it is assumed that the parent messages have been calculated and can be indexed by their subset. The reduction of terms in the summations, as given in \eqref{eqn:calcMsg} and \eqref{eqn:calcProdFactors}, is omitted in lines $13$ and $20$ for clarity.

\emph{We emphasize that implementation of the derivative-sum-product algorithm is complex.} Issues such as the enumeration and indexing of subsets and messages must be addressed. Performing a stack-based dynamic programming procedure is intricate. Several optimizations that can be found in our library, such as only allocating as much memory as is required for the messages and temporary values (in advance), avoiding the evaluation of any term guaranteed to be zero, and caching the values of the factors, complicate the implementation further. The author of \cite{Silva2012} abandons it for these reasons.

See Experiment \ref{exp:inference} for an empirical analysis of how the running time of d-s-p varies with the model structure.

\begin{algorithm}[p]
\caption{Calculate the derivative-sum-product messages from clique $i$ to clique $j$}
\label{alg:calcDspMessages}
\begin{algorithmic}[1]
	\Procedure{DspMessages}{$i, j$}
		\If{$j\neq-1$}
			\For{$A\subseteq S_{i,j}$}
				\State \textbf{calculate} $\delta_{i\rightarrow j}(A)=\textnormal{\scshape PartialMessage}(A\cup(C_i\setminus S_{i,j}), p_i)$
			\EndFor
			\State \textbf{return} $\Delta_{i\rightarrow j} := \{\delta_{i\rightarrow j}(A)\ |\ A\subseteq S_{i,j}\}$
		\Else
			\State \textbf{return} $\{\textnormal{\scshape PartialMessage}(C_i, p_i)\}$
		\EndIf
	\EndProcedure\vspace{6pt}
	\Procedure{PartialMessage}{$A, p$}
		\If{$p>0$}
			\State \textbf{return} $\sum_{B\subseteq A}\textnormal{\scshape DspMessages}\left(\left(i,p\right), i\right)[B]\times\textnormal{\scshape PartialMessage}(A\setminus B, p-1)$
		\Else
			\State \textbf{return} {\scshape PartialFactorDerivative}($A, q_i$)
		\EndIf
	\EndProcedure\vspace{6pt}
	\Procedure{PartialFactorDerivative}{$A, q$}
		\If{$q>0$}
			\State \textbf{return} $\sum_{B\subseteq A}\left.\left(\frac{\partial}{\partial B}\phi_{(i,q)}\right)\right|_{\mathbf{x}=\mathbf{x}'}\times\textnormal{\scshape PartialFactorDerivative}(A\setminus B, q-1)$
		\Else
			\State \textbf{return} $\left.\left(\frac{\partial}{\partial A}\phi_{(i,0)}\right)\right|_{\mathbf{x}=\mathbf{x}'}$
		\EndIf
	\EndProcedure
\end{algorithmic}
\end{algorithm}

\section{Calculating the density}
Putting together the previous steps, we obtain Algorithm \ref{alg:density} for calculating the density of a CDN. Refer to \S\ref{sec:parameterization} and \S\ref{sec:diffcdn} for details about the notation. We have assumed that for transformation copulae, the margins of the transformed distribution are identical, as they are for the normal copula. Of course, for the normal copula, $f$ and $F$ in lines 10-11 are the density and CDF of the standard normal distribution, $\phi$ and $\Phi$.

\begin{algorithm}[p]
\caption{Calculate the density of a CDN}
\label{alg:density}
\begin{algorithmic}[1]
	\Procedure{Density}{$\mathbf{x}, \mathcal{C}$}
		\State $u_i := F_i(x_i)$,\ $i=1,\ldots,n$
		\State $\alpha := \prod^n_{i=1}f_i(x_i)$\Comment{product of marginals}
		\State \textbf{return} $\alpha\times\textnormal{\scshape CopulaDensity}(\mathbf{u},\mathcal{C})$
	\EndProcedure\vspace{6pt}
	\Procedure{CopulaDensity}{$\mathbf{u},\mathcal{C}$}
		\State $v_i := u_i^{d_i}$,\ $i=1,\ldots,n$
		\State $\beta := \prod^n_{i=1}d_iu_i^{d_i-1}$
		\If{model comprised of transformation copulae}
			\State $w_i := F^{-1}(v_i)$,\ $i=1,\ldots,n$\Comment{$F$ is marginal CDF of transformed distribution}
			\State $\gamma := \prod^n_{i=1}f(F^{-1}(v_i))^{-1}$\Comment{$f$ is corresponding density}
			\State \textbf{return} $\beta\gamma\times\textnormal{\scshape PassMessages}(\mathbf{w}, \mathcal{C}, \textnormal{\scshape DspMessages})$
		\Else
			\State \textbf{return} $\beta\times\textnormal{\scshape PassMessages}(\mathbf{v}, \mathcal{C}, \textnormal{\scshape DspMessages})$
		\EndIf
	\EndProcedure
\end{algorithmic}
\end{algorithm}

\section{Numerical stability}\label{sec:numericalStability}
A numerical instability encountered in this project occurred during evaluation of Archimedean copulae for extreme values of the parameters. For a network composed of any such factors, message passing would result in underflow, and it is conceivable that it could do so for other cases as well.

All message passing can be run in log-space to avoid underflow using the following numerical tricks. First, we make the copulae cache the log of their partial derivatives.

To multiple numbers in log-space, we add their logs,
\begin{align*}
	\ln\left(\prod^n_{i=1}x_i\right) &= \sum^n_{i=1}\ln\left(x_i\right).
\end{align*}
Adding numbers in log-space uses the following trick,
\begin{align*}
	\ln\left(\sum^n_{i=1}x_i\right) &= \ln\left(\sum^n_{i=1}e^{\ln(x_i)}\right)	\\
	&= \max_i\{\ln(x_i)\} + \ln\left(\sum^n_{i=1}e^{\ln\left(x_i\right) - \max_i\{\ln\left(x_1\right)\}}\right)
\end{align*}
The final $\ln$ is evaluated using the C++ function {\ttfamily log1p}. This function stably evaluates $\ln(1+x)$ when $x$ is small.

The reason this rearrangement is more numerically stable is that it makes the exponential terms much less likely to underflow. Also, we can understand it in terms of the bounds on the log-sum-exp function,
\begin{align*}
	\max\{x_1,\ldots,x_n\} &\le \ln\left(\sum_{i=1}^ne^{x_i}\right) \le \max\{x_1,\ldots,x_n\} + \ln(n).
\end{align*}
Thus, even if the exponential terms underflow, one still has a lower bound on the sum.

\section{Inference in discrete CDNs}\label{sec:discreteInference}
Although this thesis has focused on representing and manipulating continuous distributions with CDNs, they are capable of representing discrete distributions also. Suppose we have a discrete distribution over $X_1,\ldots,X_n$ for which the CDF is, for some copula $C$,
\begin{align*}
	F(X_1\le x'_1,\ldots,X_n\le x'_n) &= C(u'_1,\ldots,u'_n)
\end{align*}
using the notation of Chapter \ref{pgm}. Then, the probability mass function (pmf) is (\cite{SmithKhaled2012}),
\begin{align*}
	p(x'_1,\ldots,x'_n) &= \Delta^{a_1}_{b_1}\cdots\Delta^{a_n}_{b_n}C(u_1,\ldots,u_n),
\end{align*}
where
\begin{align*}
	\Delta^{a_k}_{b_k}C(u_1,\ldots,u_n) &:= C(u_1,\ldots,u_{k-1},a_k,u_{k+1},\ldots,u_n) \\
	&\ \ - C(u_1,\ldots,u_{k-1},b_k,u_{k+1},\ldots,u_n),\ \ \textnormal{and}\\
	a_k &= u'_k := F_k(x'_k),\\
	b_k &= F_k(x'_k - 1).
\end{align*}
The pmf is calculable using a variation of the derivative-sum-product algorithm that applies the calculus of finite differences. The modifications required are,
\begin{itemize}
	\item replacing the continuous differential operator with the backwards finite difference operator;
	\item omitting the product of the terms that result from the chain rule;
	\item substituting the familiar product rule with its analogue for backwards finite differences,
		\begin{align*}
			\Delta\left(f(x)g(x)\right) &= g(x)\Delta f(x)+f(x-1)\Delta g(x).
		\end{align*}
\end{itemize}

\section{Summary}
\begin{itemize}
	\item A combination of marginalization and differentiation permits one to perform a myriad of queries on CDNs; more than can be performed on regular PGMs.
	\item Variables are marginalized by taking limits.
	\item Differentiation of the model is efficiently performed by the derivative-sum-product algorithm, a message-passing algorithm analogous to the sum-product algorithm for standard PGMs. The implementation of the algorithm is involved.
	\item The running time of derivative-sum-product is exponential in the size of the largest clique. Consequently, the speed of inference depends on which subset of the model is differentiated and the structure of its factorization.
	\item An adequate differentiation ordering is obtained using the same heuristics as for acquiring an elimination ordering in standard PGMs.
	\item The numerical stability of inference is improved by running the message passing in log-space.
	\item A discrete version of the derivative-sum-product is obtained by replacing the differential calculus with the calculus of finite differences.
\end{itemize}

%% file: sampling.tex
\chapter{Sampling}
\label{sampling}
As copula CDNs are capable of representing a broad class of probability distributions, it is desirable to develop an algorithm for sampling in these models. Sampling allows us to study properties of the models and algorithms---for example, the agreement of asymptotic results with small sample ones, and the comparison of novel algorithms with existing methods---through Monte Carlo studies, as we do in Chapter \ref{experiments}. Other applications of sampling include approximate inference, and determining high probability outcomes, an approximation to maximum a posteriori probability (MAP) inference.

In this chapter we develop an algorithm to sample from the copula distribution. It suffices to be able to sample from the copula distribution, for if $U_1,\ldots,U_n$ is distributed as $C(U_1,\ldots,U_n)$, then $X_1=F_1^{-1}(U_1),\ldots,X_n=F_n^{-1}(U_n)$ is distributed as $F(X_1,\ldots,X_n)$, by definition. That is, once we have generated a sample from the copula, we apply the quantiles of each marginal variable to obtain a sample from the joint distribution.

\emph{To the best of our knowledge, this thesis presents the first development and demonstration of a sampling algorithm for CDNs.}

\section{Conditional method}
The method we will use to sample from CDNs is based on the following theorem,
\begin{theorem}
	Suppose that $X_1,\ldots,X_n$ are jointly distributed as $F$. First, sample $X_1$. Next, sample $X_{i+1}\ |\ X_i,\ldots,X_1$ for $i=1,\ldots,n-1$. Then, the sample thus produced will be jointly distributed as $F$.
\end{theorem}
This is known as the \emph{conditional method} and follows from the fact that the distribution decomposes using the chain rule as,
\begin{align*}
	F(x_1,\ldots,x_n) &= F(x_1)\prod^{n-1}_{i=1}F(x_{i+1}\ |\ x_{i},\ldots,x_1).
\end{align*}
Forward sampling in Bayesian networks is based on the same idea, but the difference in the semantics of the graphical representation, that is, the different conditional independence relationships encoded in the factorization, results in a different algorithm.

The conditional method has the advantage of producing independent samples from the exact distribution, unlike MCMC methods, which require a period of burn-in and produce correlated samples. It does suffer, however, from requiring inference over the full model clique tree, and is thus infeasible for high treewidth models.

\section{Numerical inverse transform sampling}
Recall that the marginals of the copula are uniformly distributed on $[0,1]$. Indeed, this is the defining property of a copula distribution. For the first variable in some arbitrary sampling order, $U_1$, we simply sample from $\mathcal{U}[0,1]$ using standard algorithms implemented in the Boost Random library \cite{website:boost}.

For subsequent variables we must sample from $U_{i}\ |\ U_{i-1}=u'_{i-1},\ldots,U_1=u'_1$ for $i=1,\ldots,n$. We use inverse transform sampling to perform each of these steps. The method is based on the following theorem \cite{HoggEtAl2012},
\begin{theorem}
	Let $X\sim F$ and $K\sim\mathcal{U}[0,1]$. Then $F^{-1}(K)\sim F$.
\end{theorem}
\begin{proof}
	$P(F^{-1}(K)\le x) = P(K\le F(x)) = F(x)$.
\end{proof}
Consequently, when we sample $k$ from $\mathcal{U}[0,1]$ using standard methods, $F^{-1}(k)$ is a sample from the desired distribution.

For simple parametric distributions, the CDF is invertible in closed form. However, the CDN is a product of arbitrary factors, and the derived CDFs cannot, in general, be inverted algebraically. The solution can, nevertheless, be found numerically. We formulate it as a root finding problem:

For $i=1,\ldots,n-1$, solve,
\begin{align*}
	\textnormal{find} &\ u_i\\
	\textnormal{such that} &\ g(u_i) := P(U_i\le u_i\ |\ U_{i-1}=u'_{i-1},\ldots,U_1=u'_1) - k_i\\
	&\ \ \ \ \ \ \ \ \ \ \ = 0,
\end{align*}
where $k_i$ has been sampled from $\mathcal{U}[0,1]$, and $u'_i$ is the solution for the value of $U_i$.

Evaluating the objective $g(u_i)$ requires a clique tree over the variables $U_1,\ldots,U_{i-1}$ with respect to which the function has been differentiated. We term the collection of these clique trees the \emph{sampling clique trees}. There may be factors that have $U_i$ in their scope but none of the variables $U_1,\ldots,U_{i-1}$. These factors will clearly not be associated with the clique tree, and thus must be recorded so that they can be evaluated and multiplied into the result of message passing over the clique tree. We term these factors the \emph{extra factors}.

Preliminary experiments revealed that a combination of the bisection method and Newton's method is typically unstable for solving the root finding problems. Instead, we chose a well known algorithm called Brent's method. The exact details are immaterial to this thesis, although it suffices to understand that the method combines the bisection method, the secant method, and inverse quadratic interpolation to produce an algorithm that has the robustness of the bisection method but with superlinear convergence. At each iteration, one of the three methods is chosen and applied according to the current values. We adapt the implementation given in \cite[\S 9.3--4]{PressEtAl2002}.




\section{An efficient ordering for sampling}
The independencies encoded by the CDN can be exploited to simplify sampling. Recall from \S\ref{sec:independencies} that $\mathbf{X}\perp\mathbf{Y}\ |\ \mathbf{Z}$ when there is no path composed of observed variables between any $X\in\mathbf{X}$ and $Y\in\mathbf{Y}$. Therefore, when we sample a variable $X$, we only have to condition on those previously sampled variables that are directly connected to $X$, or indirectly through a path of sampled variables connected to $X$.

Also, the order in which we sample the variables effects the sizes of the sampling clique trees. Consider the model in Figure \ref{fig:buildingSamplingCliques}. The shaded variables have been sampled and we are considering whether to sample $A$ or $B$ next. According to our discussion in the previous paragraph, sampling $A$ first requires conditioning on the two left observed branches, of size $4$ each. Thus, the sampling clique tree for $A$ when sampling $A$ first is constructed over $4+4=8$ variables. On the other hand, suppose we sampling $B$ first. We must condition on the two right observed branches; the two left branches are not connected to $B$. Thus, the sampling clique tree for $B$ is constructed over just $2$ variables. The sampling clique tree for the remaining variable is constructed over remaining 11 variables.

In summary, if we sample $A$ first, the sampling clique trees for $A$ and $B$ are constructed over 19 variables, whereas if $B$ is sampled first they are constructed over $13$. Inference is faster over smaller clique trees both because the running time is linear in the number of variables, and because reducing the size could possibly reduce the treewidth. The cumulative savings from choosing an efficient sampling ordering can be substantial.
\begin{figure}[t]
  \centering
  \includegraphics[scale=0.6]{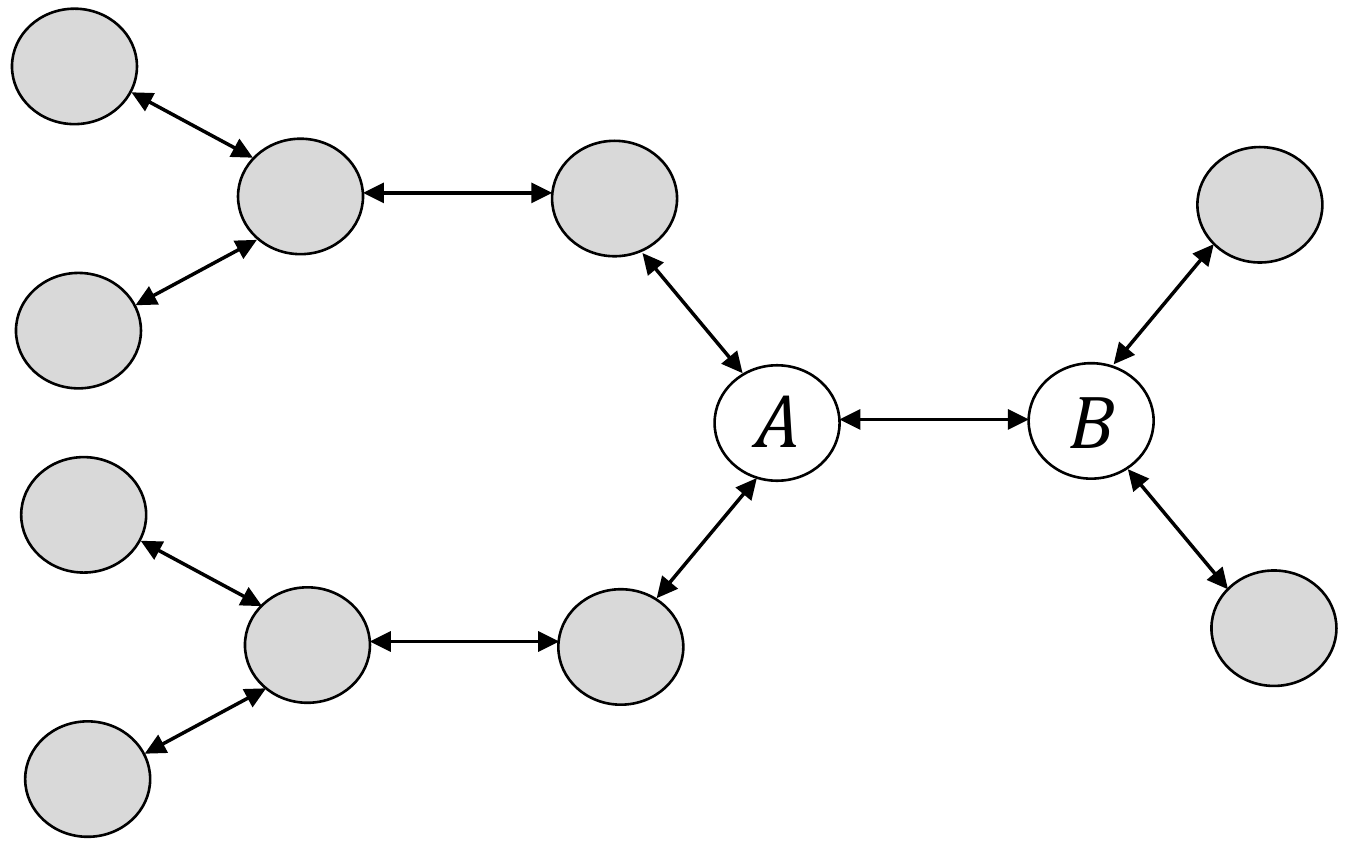}
  \caption[Demonstrating that the sampling order effects the total sampling clique tree size.]{Demonstrating that the sampling order effects the total sampling clique tree size. If A is chosen first, then the sum of the sizes of the sampling clique trees of A and B is $19$. If B is chosen first, then it is $13$.}
	\label{fig:buildingSamplingCliques}
\end{figure}

We have constructed a greedy algorithm using this intuition to determine an efficient sampling order. First, the leaves are added to a ``to-do'' list of variables to consider for sampling. When there are no leaves (for example, in a grid), a number of variables are chosen at random and added to the to-do list. For each variable in the to-do list, the size of the sampling clique tree for that variable conditioned on the previous variables is calculated. This is done by keeping track of the connected branches of sampled variables in the model. For each variable, a record is kept of whether it has been sampled and if so to which branch it belongs. When calculating the size of a sampling clique tree for adding a variable, we simply add the sizes of the unique branches to which its sampled neighbours belong.

The variable in the to-do list with the minimum sampling clique tree size is greedily added, removed from the to-do list, its unsampled neighbours added to the list, and the branches on its sampled neighbours updated. When multiple variables in the to-do list have identically sized sampling clique trees, one is chosen at random. This is repeated until the list is empty and the complete sampling order has been determined. The sampling clique trees for each variable are created and cached at each step so that multiple samples can be drawn efficiently. Our discussion is formalized in Algorithm \ref{alg:calcSamplingOrder}.

\begin{algorithm}[p]
\caption{Determine a good variable ordering for sampling and create corresponding clique trees}
\label{alg:calcSamplingOrder}
\begin{algorithmic}[1]
	\Procedure{MakeSamplingCliques}{$\mathcal{G}$}
		\State $\mathcal{V} := \textnormal{\scshape Leaves}(\mathcal{G})$, $\mathcal{B} := \emptyset$ \Comment{$\mathcal{V}$ is the to-do list, $\mathcal{B}$ is the set of branches}
		\State $\mathcal{U} := \mathcal{G}$, $i := 1$\Comment{$\mathcal{U}$ is the set of unused variables, $i$ indexes the sampling order}
		\If{$\mathcal{V} = \emptyset$}\Comment{in case $\mathcal{G}$ is, for example, a loop or a grid}
			\State \textbf{choose} $v\in\mathcal{G}$ randomly
			\State $\mathcal{V} := \{v\}$
		\EndIf
		\While{$|\mathcal{U}| > 0$}\Comment{in case $\mathcal{G}$ is not connected}
			\While{$\mathcal{V} \neq \emptyset$}
				\State $W_{min} := \infty$
				\For{$v \in \mathcal{V}$}
					\State $B :=\textnormal{\scshape Branches}(\mathcal{B}, v)$\Comment{determine branches adjacent to $v^*$}
					\State $W := \sum_{b\in B}|b|$\Comment{calculate number of variables in adjacent branches}
					\If{$W < W_{min}$}
						\State $V_{min} := \{v\}$
					\ElsIf{$W = W_{min}$}
						\State $V_{min} := V_{min}\cup\{v\}$
					\EndIf
				\EndFor
				\State \textbf{choose} $v^*\in V_{min}$ randomly
				\State $B := \textnormal{\scshape Branches}(\mathcal{B}, v^*)$
				\State $\mathcal{V} := \mathcal{V}\cup(\textnormal{\scshape Neighbours}(v^*) \cap \mathcal{U})$\Comment{add unused neighbours of $v^*$ to to-do list}
				\If{$B=\emptyset$}
					\State $\mathcal{B} := \mathcal{B}\cup\{v^*\}$\Comment{create new branch for this variable}
					\State $\textnormal{\ttfamily SampleCt}[v^*] := \emptyset$\Comment{no sampling clique tree}
				\ElsIf{$|B|>1$}
					\State $\mathcal{B} := \mathcal{B}\setminus B$\Comment{remove branches from $\mathcal{B}$}
					\State $b^* := \bigcup_{b\in B}b$\Comment{merge branches}
					\State $\mathcal{B} := \mathcal{B}\cup\{b^*\} $ \Comment{insert merged branch}
					\State $\textnormal{\ttfamily SampleCt}[v^*] := \textnormal{\scshape MakeCliqueTree}(\mathcal{G},b^*)$
				\Else
					\State $b^* := b$ for the unique $b\in B$
					\State $\textnormal{\ttfamily SampleCt}[v^*] := \textnormal{\scshape MakeCliqueTree}(\mathcal{G},b^*)$\Comment{this clique tree already exists}
				\EndIf
				\State $\textnormal{\ttfamily Extras}[v^*] := \textnormal{missing factors for $\textnormal{\ttfamily SampleCt}[v^*]$}$\Comment{provided $B\neq\emptyset$}
				\State $b^* := b^*\cup\{v^*\}$\Comment{this implicitly updates $\mathcal{B}$ also}
				\State $\mathcal{V} := \mathcal{V}\setminus\{v^*\}$\Comment{remove from to-do list}
				\State $\mathcal{U} := \mathcal{U}\setminus\{v^*\}$\Comment{and mark as used}
				\State $\textnormal{\ttfamily Order}[i] := v^*$, $i:=i+1$\Comment{record sampling order}
			\EndWhile
			\If{$|\mathcal{U}|>0$}
				\State \textbf{choose} $v\in\mathcal{U}$ randomly
				\State $\mathcal{V} := \mathcal{V}\cup\{v\}$
			\EndIf
		\EndWhile
		\State \textbf{return} $\{\textnormal{\ttfamily Order}, \textnormal{\ttfamily SampleCt}, \textnormal{\ttfamily Extras}\}$
	\EndProcedure
\end{algorithmic}
\end{algorithm}

\begin{ex}
We exemplify the determination of a sampling order on the model of Example \ref{ex:studentccdn}. The reader may wish to follow the ordering and the development of the branches in Figure \ref{fig:samplingOrder}. Initially, the leaves $\{C,H\}$ are added to the to-do list. The first four steps of the algorithm are as follows:
\begin{enumerate}
	\item Adding either $C$ or $H$ would create a sampling clique tree of size 0. $C$ is chosen randomly from the two, removed from the to-do list, its sampling clique tree constructed, and $D$ added to the to-do list.
	\item Adding $D$ would create a sampling clique tree of size $1$, whereas adding $H$ would create a sampling clique tree of size $0$. Thus, $H$ is chosen, removed from the to-do list, its sampling clique tree constructed, and $G$ added to the to-do list.
	\item Adding either $D$ or $G$ would create a sampling clique tree of size $1$. $G$ is randomly chosen, removed from the to-do list, its sampling clique tree constructed, and $\{I,S,L\}$ are added to the to-do list.
	\item Adding $D$ would create a sampling clique tree of size $3$, whereas adding any of $\{I,S,L\}$ would create a clique tree of size $2$. $I$ is chosen at random from the three, and the structures updated as before.
\end{enumerate}
The remainder of the steps proceed similarly. The sampling clique trees created during determination of the sampling order are illustrated in Figure \ref{fig:samplingCliqueTreesSecondary}.
\end{ex}

\begin{figure}[p]
  \centering
	\renewcommand{\arraystretch}{4}
	{
	\begin{tabular}{c@{\hskip 1cm}c@{\hskip 1cm}|@{\hskip 1cm}c@{\hskip 1cm}c}
		$1$ & \raisebox{-.5\height}{\includegraphics[scale=0.6]{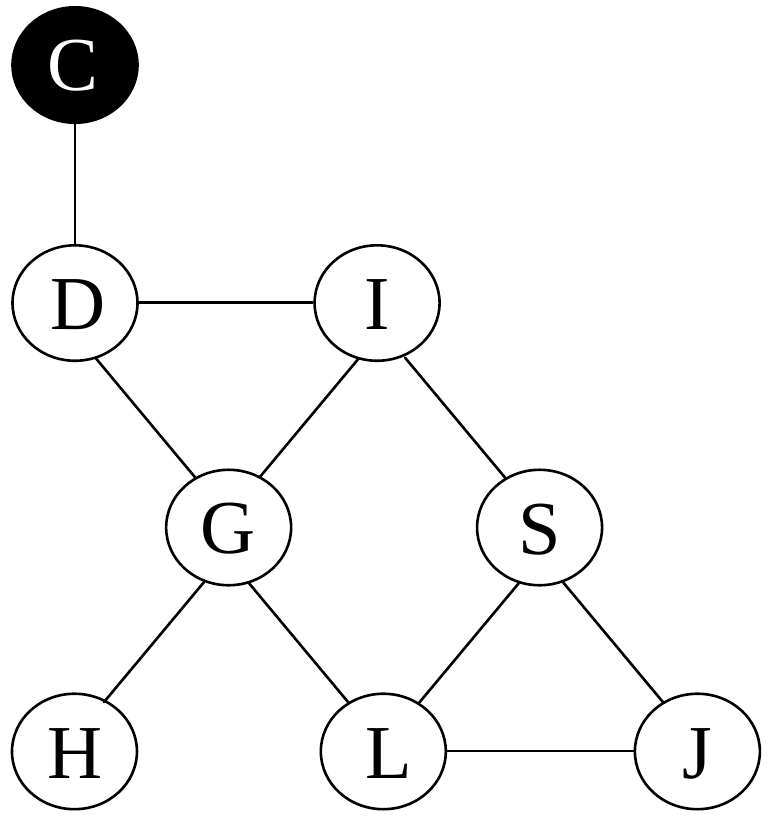}} & $5$ & \raisebox{-.5\height}{\includegraphics[scale=0.6]{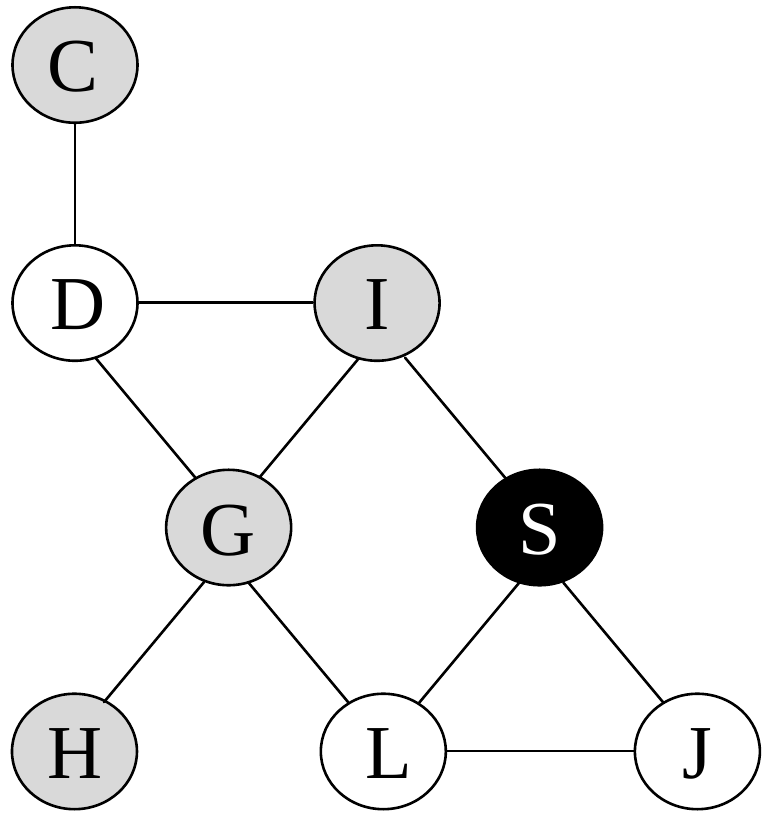}} \\
		& &\\[-5ex]
		$2$ & \raisebox{-.5\height}{\includegraphics[scale=0.6]{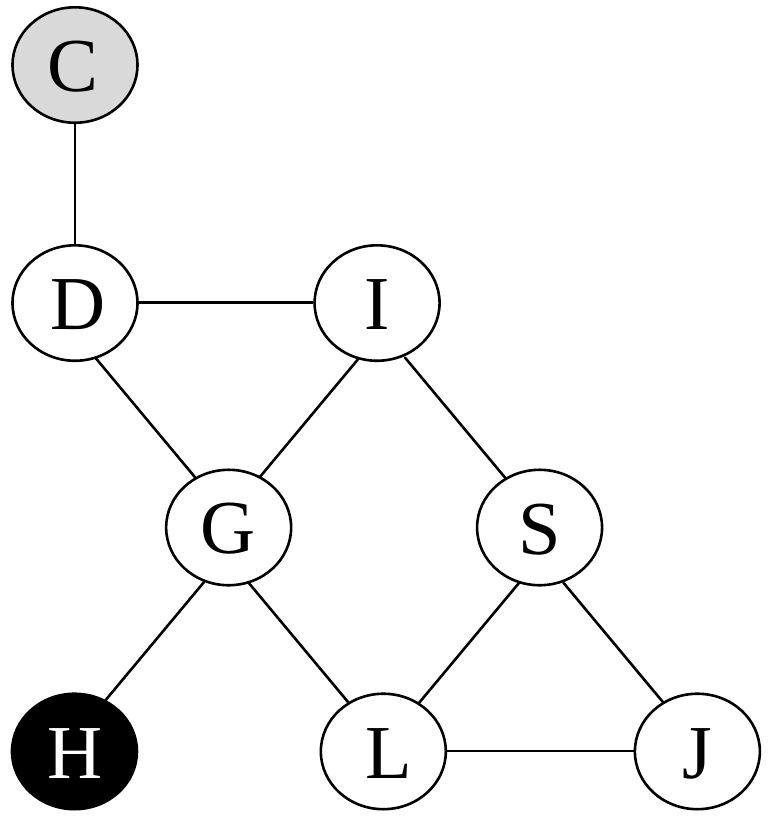}} & $6$ & \raisebox{-.5\height}{\includegraphics[scale=0.6]{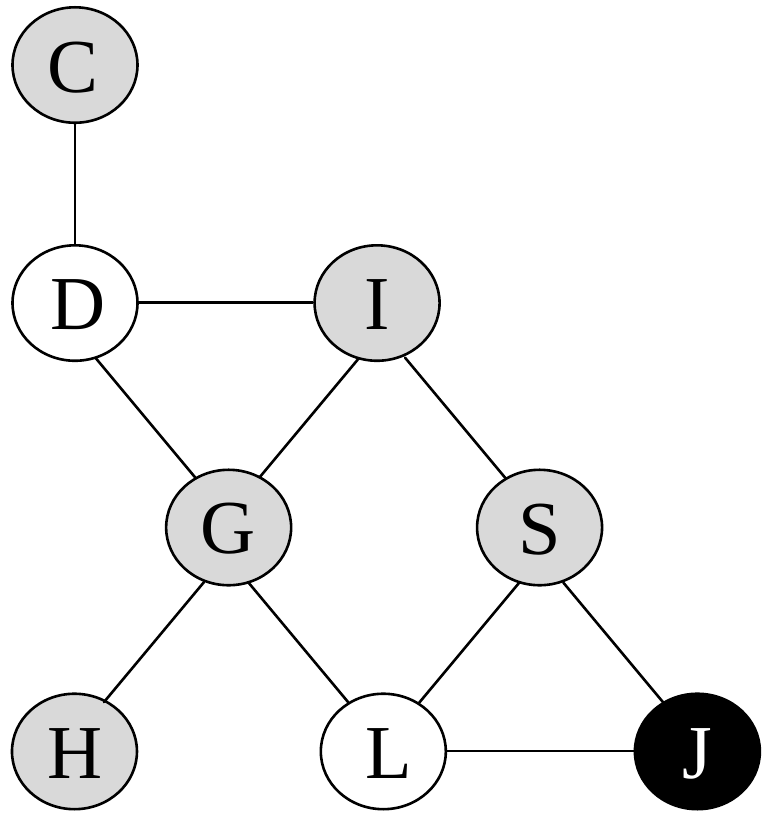}}\\
		& &\\[-5ex]
		$3$ & \raisebox{-.5\height}{\includegraphics[scale=0.6]{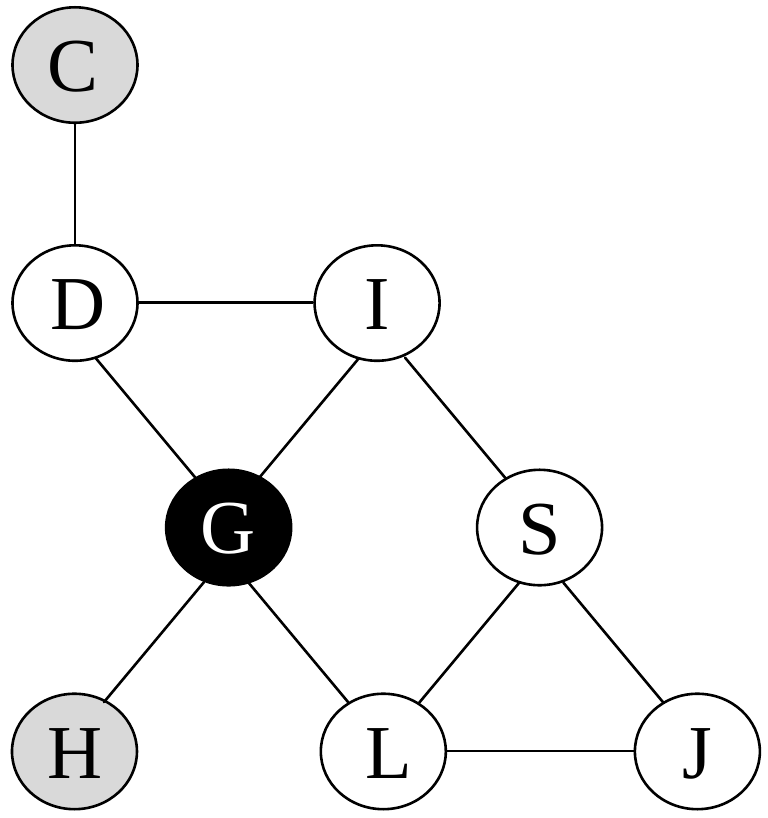}} & $7$ & \raisebox{-.5\height}{\includegraphics[scale=0.6]{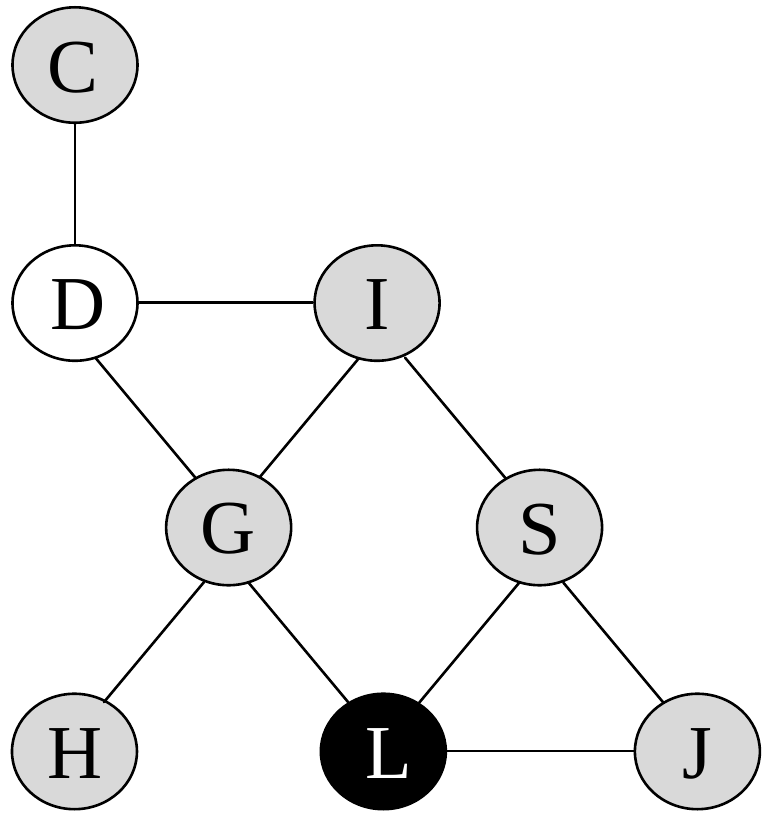}} \\
		& &\\[-5ex]
		$4$ & \raisebox{-.5\height}{\includegraphics[scale=0.6]{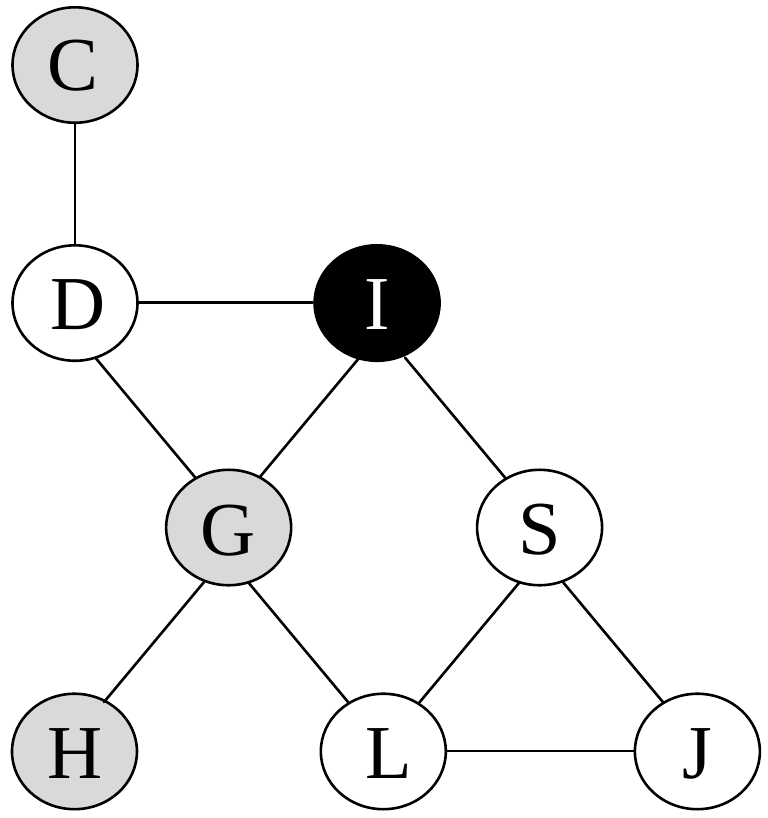}} & $8$ & \raisebox{-.5\height}{\includegraphics[scale=0.6]{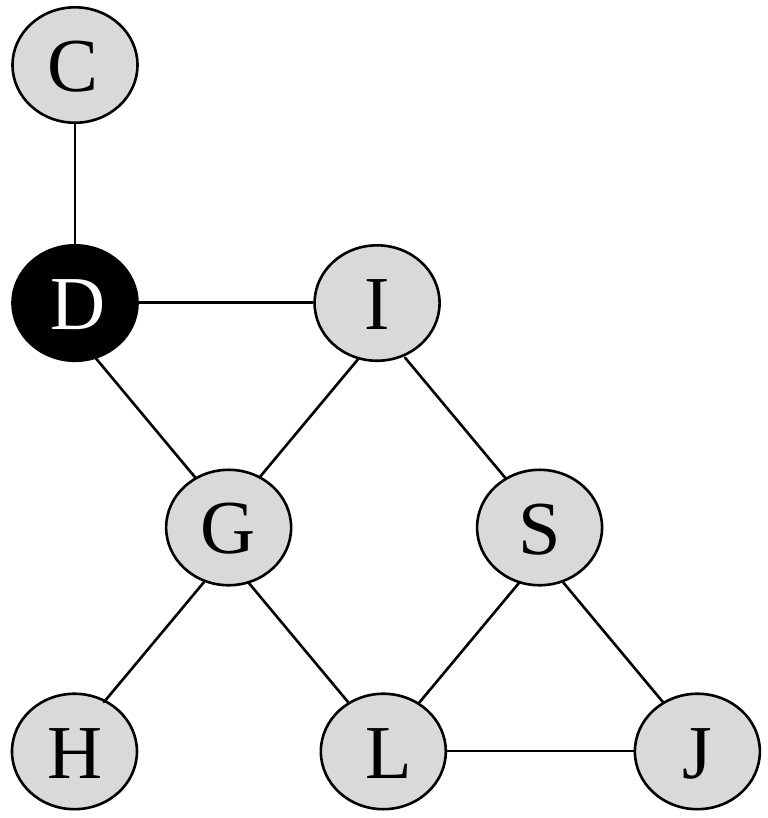}}\vspace{6pt}
	\end{tabular}}
	\caption[Illustration of the sampling order for the Student example]{Illustration of how branches grow during the determination of the sampling order for the Student example. The black node indicates the variable added at that step of the sampling order, and the shaded nodes indicate variables that have already been added. The arrows on the edges are omitted.}
	\label{fig:samplingOrder}
\end{figure}

\begin{figure}[p]
  \centering
	\renewcommand{\arraystretch}{4}
	{
	\begin{tabular}{ccc}
		{\itshape variable} & $P(U_n \le u_n, U_{n-1}=u'_{n-1},\ldots,U_1=u'_1)$ & {\itshape extra factors} \\ \hline
		$C$ & --- & ---\\
		$H$ & --- & --- \\
		$G$ & \raisebox{-.5\height}{\includegraphics[scale=0.6]{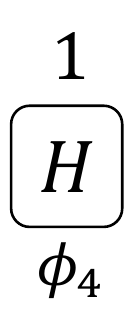}} & $\phi_2,\phi_5$ \\
		$I$ & \raisebox{-.5\height}{\includegraphics[scale=0.6]{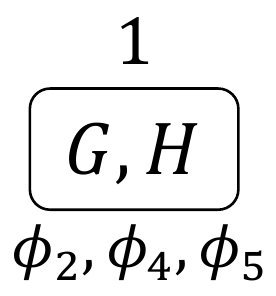}} & $\phi_3$ \\
		$S$ & \raisebox{-.5\height}{\includegraphics[scale=0.6]{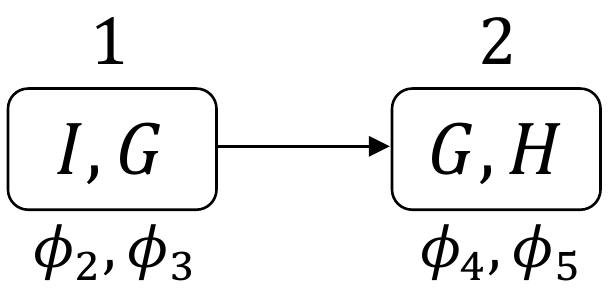}} & $\phi_6$ \\
		$J$ & \raisebox{-.5\height}{\includegraphics[scale=0.6]{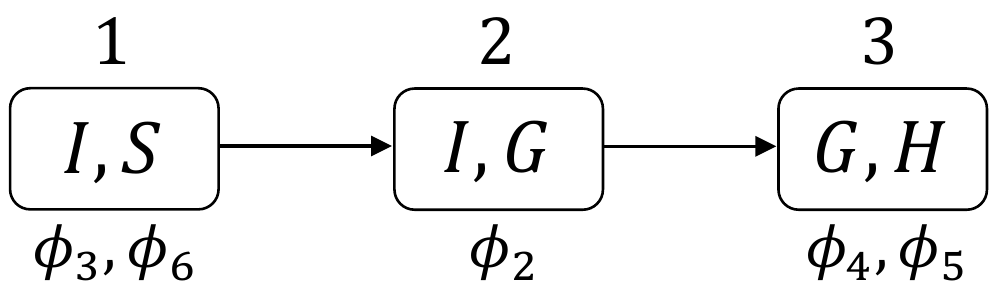}} & --- \\
		$L$ & \raisebox{-.5\height}{\includegraphics[scale=0.6]{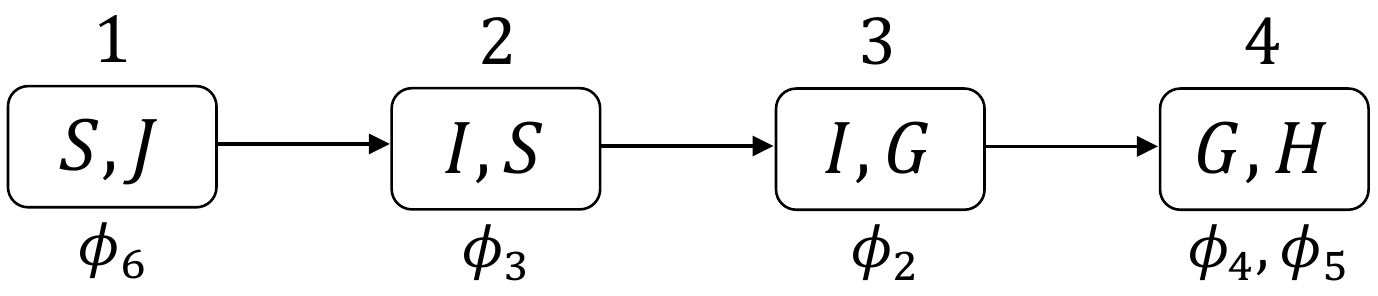}} & --- \\
		$D$ & \raisebox{-.5\height}{\includegraphics[scale=0.6]{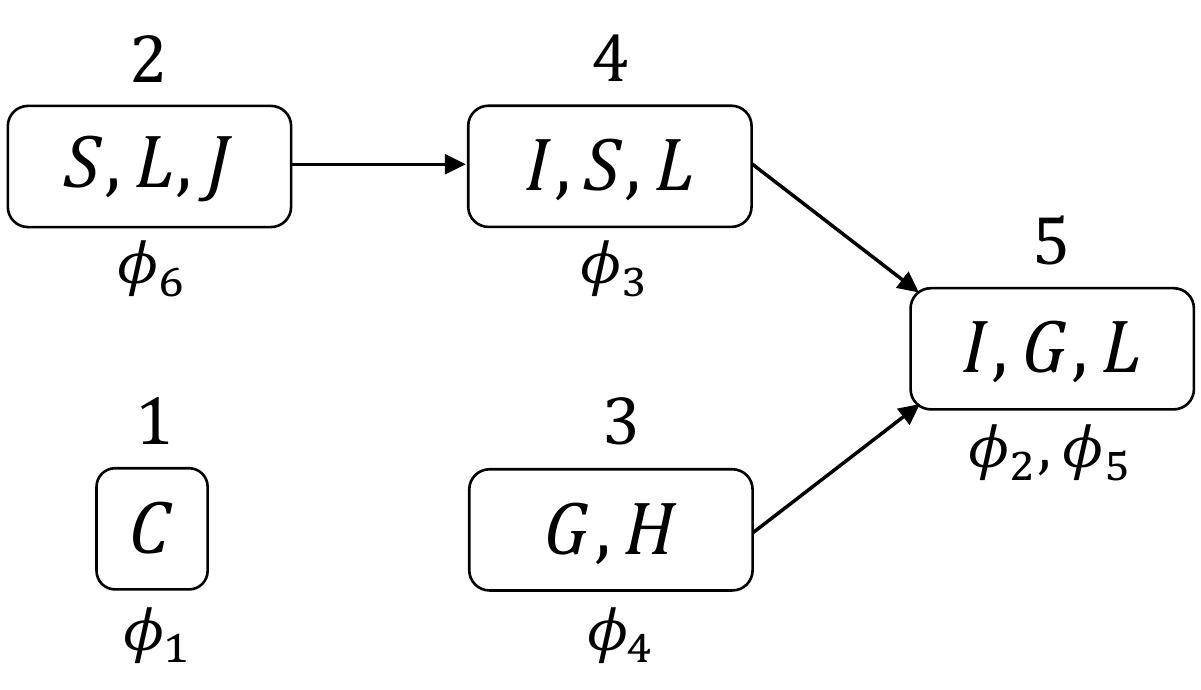}} & ---
	\end{tabular}}
	\caption[Array of sampling cliques]{Array of sampling clique trees and extra factors for the Student example.}
	\label{fig:samplingCliqueTreesSecondary}
\end{figure}

\section{Putting it together}
Combining the steps in our previous discussion produces an algorithm to sample from CDNs, as formalized in Algorithm \ref{alg:calcSampleCDN}. First, the sampling order and clique trees are calculated and cached. Then, the sequence of root finding problems are solved in order using Brent's method, making use of the derivative-sum-product algorithm on the cached sampling clique trees. Finally, the samples from the copula distribution are transformed back to the joint distribution by applying to each variables its quantile function. Examples of our algorithm on bivariate models are found in Figures \ref{fig:copulaNormal}, \ref{fig:claytonGumbelLevels}, and \ref{fig:frankLevels}.

The root finding problem can be solved in log-space to improve numerical stability. That is, we solve instead the following for $i=1,\ldots,n-1$, 
\begin{align*}
	\textnormal{find} &\ u_i\\
	\textnormal{such that} &\ g(u_i) := \ln\left(P\left(U_i\le u_i, U_{i-1}=u'_{i-1},\ldots,U_1=u'_1\right)\right)\\
	&\ \ \ \ \ \ \ \ \ \ \ \ \ \ \ \ - \ln\left(P\left(U_{i-1}=u'_{i-1},\ldots,U_1=u'_1\right)\right) - \ln(k_i)\\
	&\ \ \ \ \ \ \ \ \ \ \ = 0,
\end{align*}
where $k_i$ has been sampled from $\mathcal{U}[0,1]$, and $u'_i$ is the solution for the value of $U_i$. For additional stability, the inference is performed in log-space, as described in \S\ref{sec:numericalStability}. We found that sampling succeeds for models parameterized by Clayton copulae with more than three or four variables \emph{only} when the log-space version of root finding is used.

\begin{algorithm}[t]
\caption{Generate samples from CDN}
\label{alg:calcSampleCDN}
\begin{algorithmic}[1]
	\Procedure{SampleCDN}{$n, \mathcal{G}$}
		\State $\{\textnormal{\ttfamily Order}, \textnormal{\ttfamily SampleCt}, \textnormal{\ttfamily Extras}\}=\textnormal{\scshape MakeSamplingCliques}(\mathcal{G})$
		\For{$i=1,\ldots,n$}\Comment{sample $n$ times}
			\For{$j=1,\ldots,|G|$}
				\State $\textnormal{\ttfamily Samples}[i,j]=1$\Comment{initialize sample by marginalizing variables}
			\EndFor
			\State \textbf{sample} $u\sim\mathcal{U}[0,1]$\Comment{first variable is uniform on $\mathcal{U}[0,1]$}
			\State $\textnormal{\ttfamily Samples}[i,\textnormal{\ttfamily Order}[1]]=u$
			\For{$j=2,\ldots,|G|$}\Comment{sample remaining variables}
				\State \textbf{sample} $k\sim\mathcal{U}[0,1]$\Comment{random variable for inverse transform method}
				\State \textbf{find} $u=u^*$ such that $\textnormal{\scshape CondCDF}(u,i,j)-k=0$\Comment{invert CDF numerically}
				\State $\textnormal{\ttfamily Samples}[i,\textnormal{\ttfamily Order}[j]]=u^*$
			\EndFor
			\For{$j=1,\ldots,|G|$}\Comment{transform from copula to original distribution}
				\State $\textnormal{\ttfamily Samples}[i,j] = F^{-1}_j(\textnormal{\ttfamily Samples}[i,j])$
			\EndFor
		\EndFor
		\State \textbf{return} {\ttfamily samples}
	\EndProcedure\vspace{6pt}
	\Procedure{CondCDF}{$u, i, j$}
		\State $v=\textnormal{\ttfamily Order}[j]$
		\State $d=\textnormal{\scshape CopulaDensity}(\textnormal{\ttfamily SampleCt}[v], \textnormal{\ttfamily Samples}[i,:])$
		\State $\textnormal{\ttfamily Samples}[i,v] = u$
		\State \textbf{return} $\textnormal{\scshape CopulaDensity}(\textnormal{\ttfamily SampleCt}[v], \textnormal{\ttfamily Samples}[i,:], \textnormal{\ttfamily Extras}[v])/d$
	\EndProcedure
\end{algorithmic}
\end{algorithm}

\section{Verification of sampling algorithm}
It is easy to check that the marginals of a sample generated by our algorithm have the desired distribution. The best approach is to examine samples from the copula distribution rather than after we have transformed them to the true marginals. One can visually inspect the corresponding histograms or estimations of the densities (which should show a uniform distribution on $[0,1]$). Formally, one can perform a goodness-of-fit test.

Even when the sample marginals appear to be correctly distributed, it is possible that they do not have the correct joint distribution. For networks with a single factor, correctness of sampling was verified by comparing the model statistics. For example, when sampling from a bivariate normal distribution we compared the sample correlation to the theoretical correlation. However, for more complex networks, such as chains of bivariate normal factors, the theoretical statistics do not exist in closed form.

We have not determined a general procedure to verify correctness of the joint distribution, although the fact that our learning algorithm correctly recovers the parameters from samples generated with known parameters provides strong evidence that it is correct.

\section{Sampling a subset of the model given observed variables}
Sampling a subset of the model conditioning on another subset is a simple extension of the previously described sampling algorithm. Suppose we desire to sample $\{X_{k+1},\ldots,X_n\}$ given $\{X_1=x'_1,\ldots,X_k=x'_k\}$. Applying the conditional method, first we sample $X_{k+1}\ |\ X_1=x'_1,\ldots,X_k=x'_k$, then we sample $X_{k+2}\ |\ X_1=x'_1,\ldots,X_{k+1}=x'_{k+1}$, and so on.

Thus, sampling a subset of the model conditioning on another subset can be performed by starting Algorithm \ref{alg:calcSamplingOrder} from a state where it were as if we had already sampled $\{X_1=x'_1,\ldots,X_k=x'_k\}$. That is, we simulate the process of sampling the variables to be conditioned on and update the data structures of Algorithm \ref{alg:calcSamplingOrder} accordingly, then continue as normal from line 8. Also, trivial modifications to Algorithm \ref{alg:calcSampleCDN} are required.


\section{Why Gibbs sampling is inefficient}\label{sec:gibbsSampling}
Gibbs sampling is a type of Markov chain Monte Carlo (MCMC) sampling method. MCMC methods construct a Markov chain that has the same asymptotic distribution as the desired sampling distribution. Samples are drawn from the Markov chain, discarding the initial samples (the burn-in period) until the sample statistics indicate that convergence is adequate.

Gibbs sampling, in its simplest form, is performed as follows. Initial values for the variables, $x^{(0)}_1,\ldots,x^{(0)}_n$, are given. For each sample $j=1,\ldots,m$, each variable $\{X_i\}$ is sampled in order from
\begin{align*}
	X_i\ &|\ x^{(j-1)}_1,\ldots, x^{(j-1)}_{i-1}, x^{(j)}_{i+1},\ldots,x^{(j)}_n,
\end{align*}
and set as $x^{(j)}_i$.

When it is easier to sample from the conditional distributions than the joint model, Gibbs sampling is preferable to direct sampling. Whereas this is the case for BNs, it is not for CDNs---a consequence of the different independencies that the model encodes.

In BNs, a variable is independent from the rest of the variables in the network given its Markov blanket, which is the set of its parents, children, and parents of its children.
\begin{ex}
	Refer to the Student BN of Figure \ref{fig:studentbn}. The Markov blanket of $S$ is shown in Figure \ref{fig:markovblanket}. From $d$-separation,
	\begin{align*}
		S &\perp C, D, H\ |\ I, G, L, J.
	\end{align*}
	In general, it is necessary to condition on the parents of children of a node. For example, suppose we did not condition on $G$. Then,
	\begin{align*}
		S &\not\perp C, D, H, G\ |\ I, L, J,
	\end{align*}
	since observing $L$ activates the v-structure $G\rightarrow L\leftarrow S$. Alternatively, the parents of a variable's children are informative about the variable since they can be used to explain away the variable in question.
\end{ex}
\begin{figure}[t]
  \centering
  \includegraphics[scale=0.6]{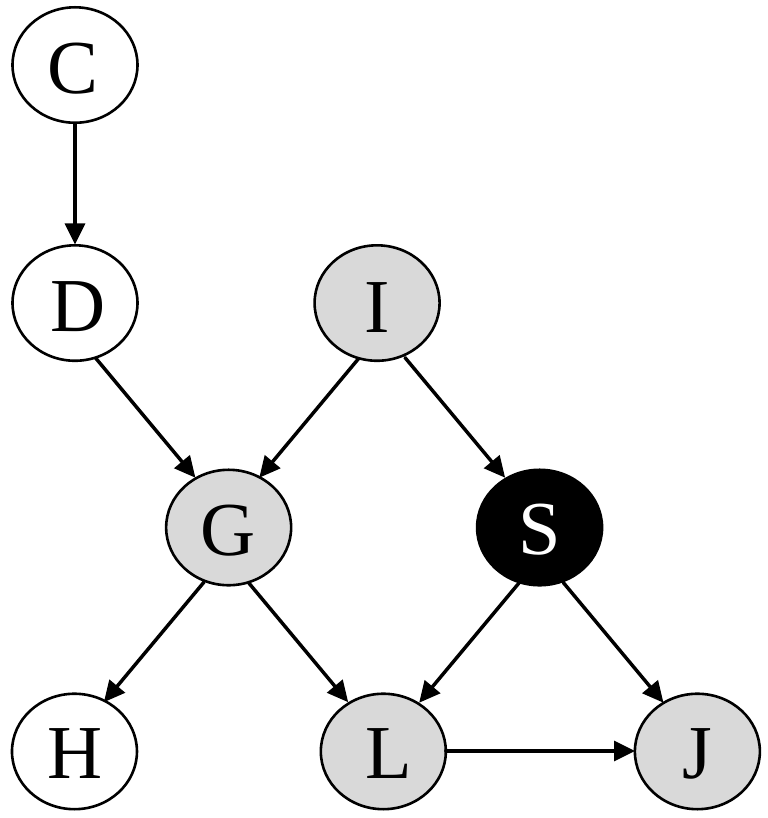}
  \caption[Markov blanket example]{The Markov blanket of $S$ in the student BN is shaded.}
	\label{fig:markovblanket}
\end{figure}
Therefore, in a BN, sampling a variable having conditioned on all others is equivalent to sampling having conditioned solely on its Markov blanket. This makes Gibbs sampling efficient in BNs, relative to sampling the full distribution.
\begin{ex}
	Continuing our example, suppose one wishes to sample $S$ given all other variables. The conditional distribution is,
	\begin{align}
		f(s\ |\ \mathbf{x}\setminus\{s\}) &= f(s\ |\ i,g,l,j) \nonumber \\ 
		&\propto f(s,i,g,l,j) \nonumber \\ 
		&= \sum_{c,d,h}f_C(c)f_D(d|c)f_G(g|d,i)f_I(i)f_S(s|i)f_H(h|g)f_L(l|g,s)f_J(j|l,s) \nonumber \\
		&= f_G(g|d,i)f_I(i)f_S(s|i)f_L(l|g,s)f_J(j|l,s) \nonumber \\
		&= f_S(s\ |\ \textnormal{Pa}_S)\prod_{x\in\textnormal{Markov}_S}f_X(x\ |\ \textnormal{Pa}_X). \label{eqn:markovblanket}
	\end{align}
	That is, the conditional distribution is obtained by multiplying together the factor for $S$ and the variables in its Markov blanket and renormalizing. In general, \eqref{eqn:markovblanket} holds, replacing $S$ with the variable of interest.
\end{ex}
CDNs encode a different set of conditional independencies, and for this reason an efficient Gibbs sampling algorithm cannot be developed.
\begin{ex}
	Refer to the analogous CDN in Example \ref{ex:studentcdn}. Again, consider sampling $S$ conditioned on the rest of the variables. What is the minimal set of variables such that, conditioned on this set, the variable is independent of the rest? $S$ is dependent on its neighbours, so the set must include $I$, $G$, $L$, and $J$.
	
	But then, having conditioned on its neighbours, $S$ is dependent on $D$ and $H$. So the set must include those variables. Similarly, conditioning its neighbours, and neighbours of neighbours, $S$ is dependent on $C$. Clearly, the set must be $\mathbf{X}\setminus\{S\}$!
	
	In general, in a CDN, a variable conditioned on a subset is dependent on at least some of the remaining variables. In this sense, all other variables are informative about a given variable.
\end{ex}
	Thus, no simplification of the conditional distribution is possible and Gibbs sampling requires inference over the full model to sample each variable. Direct sampling is better, because sampling each variable requires inference over only a subset of the variables (except for the last variable), and the sample thus produced is from the true distribution.
	
	We note that other MCMC methods may be applicable to CDNs.

\section{Summary}
\begin{itemize}
	\item Sampling in CDNs is performed by the conditional method.
	\item The inverse transformation method is used to sample each variable given the previously sampled variables, and the conditional CDF is inverted numerically using a robust root-finding method.
	\item The sequence of root-finding problems requires inference, and thus an array of clique trees, over increasing subsets of the model.
	\item The sampling order effects the total size of the sampling cliques, necessitating a simple greedy algorithm for producing an efficient ordering.
	\item Sampling a subset of the model conditioned on another subset is a simple extension of our sampling algorithm.
	\item An efficient algorithm for Gibbs sampling from CDNs does not exist, although other MCMC methods may be applicable.
\end{itemize}

%% file: learning.tex
\chapter{Learning}
\label{learning}
Learning, most generally, refers to the process of choosing a ``good'' statistical model for a given purpose based on information about the task, which is either elicited from domain-experts or formed from direct observations of the variables. The purpose of the model may be prediction, forecasting, description, discovery of correlations or causal relationships, to name a few. The measure of goodness is task specific, and usually encoded by a loss function; for example, a predictive model can be judged by a bootstrapping estimate of its ability to generalize to unseen instances.

Learning decomposes into a number of subtasks such as feature selection and extraction, model selection, and parameter estimation, for which manual and automated methods exist. Typically, much domain knowledge (or arbitrariness!) is implicit---for example, the choice of features under consideration, the values of the metaparameters, and when the class of models is restricted to those that are linear.

In this chapter, we consider the supervised learning problem. The model variables, the structure of the CDN, and the form of the (parametric) marginals and copula factors are given. We are presented with a data set of samples over the complete model, and the task is to estimate the value of the parameters that best fits the data according to some criterion.

Many criteria attempt to prevent overfit of the model to the specific data set in order to improve the model's ability to generalize. However, we consider the simplest measure of model fit---the sample log-likelihood---and regularize by restricting the network structure. Several descent methods for optimizing this criterion are described and an efficient algorithm for calculating the gradient.

\emph{An original contribution of this thesis is the development of a novel learning algorithm for CDNs that is able to learn high-dimensional and large treewidth models, at the expense of additional message passing and a loss in efficiency. Also, we propose an algorithm for learning from data missing completely at random (MCAR).}

Structure learning of CDNs, whilst being an interesting and scarcely studied problem, falls outside the scope of this thesis. 

\section{Learning in copula CDNs}
In copula CDNs, there are two types of parameters to estimate: those of the marginals and those of the copulae factors. The desired parameters minimize the negative log-likelihood,
\begin{align}
	E(\theta ; \mathcal{D}) &= -\sum^m_{i=1}\ln\left(f(x^i_1,\ldots,x^i_n\right) \nonumber \\
	&= -\sum^m_{i=1}\ln\left(\frac{\partial^n}{\partial x_1\cdots\partial x_n}F(x^i_1,\ldots,x^i_n)\right) \nonumber \\
	&= -\sum^m_{i=1}\ln\left(\prod^n_{j=1}\frac{\partial u_j}{\partial x_j}\frac{\partial^n}{\partial u_1\cdots\partial u_n}C(u^i_1,\ldots,u^i_n; \boldsymbol\theta)\right) \nonumber \\
	&= -\sum^m_{i=1}\sum^n_{j=1}\ln\left(f_j\left(x^i_j\right)\right)-\sum^m_{i=1}\ln\left(\frac{\partial^n}{\partial u_1\cdots\partial u_n}C(u^i_1,\ldots,u^i_n ; \boldsymbol\theta)\right), \label{eqn:negloglike}
\end{align}
where $x^i_j$ is the value of the $j$th variable in the $i$th sample, $\mathcal{D}=\{\mathbf{x}^1,\ldots,\mathbf{x}^m\}$ is the data set of $m$ samples, and $C$ is the model copula. The function $E$ is termed the \emph{energy}.

From \eqref{eqn:negloglike}, one may misconstrue that the energy decomposes into the sum of two terms where one term depends only on the marginal parameters and the other solely on the copula parameters. The second term, however, also depends on the marginal parameters through each variable $u_i=F(x_i;\boldsymbol\xi_i)$, where $\boldsymbol\xi_i$ are the marginal parameters  of the $i$th variable. Moreover, that a marginal parameter could be present in multiple copulae, those that have intersecting scope (as the model copula is formed from the product of copulae over subsets of the variables), greatly complicates learning.

Minimizing the second term with respect to both parameters simultaneously is infeasible. Instead, we use the method of ``inference functions for margins'' \cite{JoeXu1996}. In this approach, the marginal parameters are estimated first, and then the copula parameters are estimated having fixed the marginal parameters. Thereby, estimation becomes tractable at the cost of some accuracy.

Estimation of the marginal parameters is trivial. We use simple parametric forms and estimate the parameters by MLE. Therefore, we redesignate the energy function as,
\begin{align*}
	E(\boldsymbol\theta ; \mathcal{D}) &= -\frac{1}{m}\sum^m_{i=1}\ln\left(\frac{\partial^n}{\partial u_1\cdots\partial u_n}C(u^i_1,\ldots,u^i_n ; \boldsymbol\theta)\right)
\end{align*}
assuming that the margins have been learnt. The average is taken over the number of samples to avoid overflow. Kernel density estimation, which is non-parametric, can be used instead \cite{Elidan2010b}.

In contrast, estimation of the copulae parameters is nontrivial, and the focus of this chapter is two gradient-based optimization methods for this purpose that minimize the energy. Throughout, we assume each factor has a single parameter (but the method can be extended in a straightforward way to vector valued parameters).


\section{Gradient descent}
Let $f(\mathbf{x})$ be an arbitrary continuous real-valued multidimensional function. The negative gradient $-\nabla f(\mathbf{x})$ gives the direction of maximum decrease of $f$ at $\mathbf{x}$. Therefore, $f(\mathbf{x}-\eta\,\nabla f(\mathbf{x}))<f(\mathbf{x})$, provided $\eta>0$ is small enough and $\nabla f(\mathbf{x})\neq\mathbf{0}$.

Gradient descent is an iterative algorithm for unconstrained optimization that uses this idea. Given an initial guess $\mathbf{x}^0$ for the maximum of $f$, we improve our estimate by moving in the direction of the negative gradient,
\begin{align*}
	\mathbf{x}^{t+1} &= \mathbf{x}^{t}-\eta\,\nabla f(\mathbf{x}^{t}).
\end{align*}
This is repeated until some termination criterion is satisfied. In our implementation, the algorithm terminates when any of the following is sufficiently small:
\begin{itemize}
	\item the relative change in the objective,
	\begin{align*}
		\frac{f(\mathbf{x}^{t}) - f(\mathbf{x}^{t+1})}{|f(\mathbf{x}^{t})|} &< \epsilon_1;
	\end{align*}
	\item the length of the step,
	\begin{align*}
		||\mathbf{x}^{t}-\mathbf{x}^{t+1}|| &< \epsilon_2;
	\end{align*}
	\item the length of the gradient,
	\begin{align*}
		||\nabla f(\mathbf{x})^{t+1}|| &< \epsilon_3.
	\end{align*}
\end{itemize}
For simplicity, we set $\epsilon=\epsilon_1=\epsilon_2=\epsilon_3$.

The parameter $\eta$ is known as the learning rate or step size. Ideally, we would like to choose
\begin{align*}
	\eta &= \arg\min_{s\ge0}f(\mathbf{x}-s\nabla f(\mathbf{x}))
\end{align*}
This is known as \emph{exact line search}. In our case, however, the problem is insolvable in closed form. Instead, we use the method of \emph{backtracking line search}. An initial step size, $\eta^0$, is assumed. The step size is reduced by a ratio $\beta\in(0,1)$ until the objective has ``sufficiently decreased,'' that is, until the first Wolfe condition,
\begin{align*}
	f(\mathbf{x}^{t+1}) &\le f(\mathbf{x}^{t}) - \alpha\eta^{t}\nabla f(\mathbf{x}^{t})^T\nabla f(\mathbf{x}^{t}),
\end{align*}
is satisfied, where $\eta^{t+1} = \beta\eta^{t}$ and $\alpha\in(0,0.5)$. A large $\beta$ corresponds to a fine line search, and a small one to a crude search. See Figure \ref{fig:backtrackingLineSearch} for a geometric interpretation of this criterion.
\begin{figure}[t]
  \centering
  \includegraphics[scale=1.2]{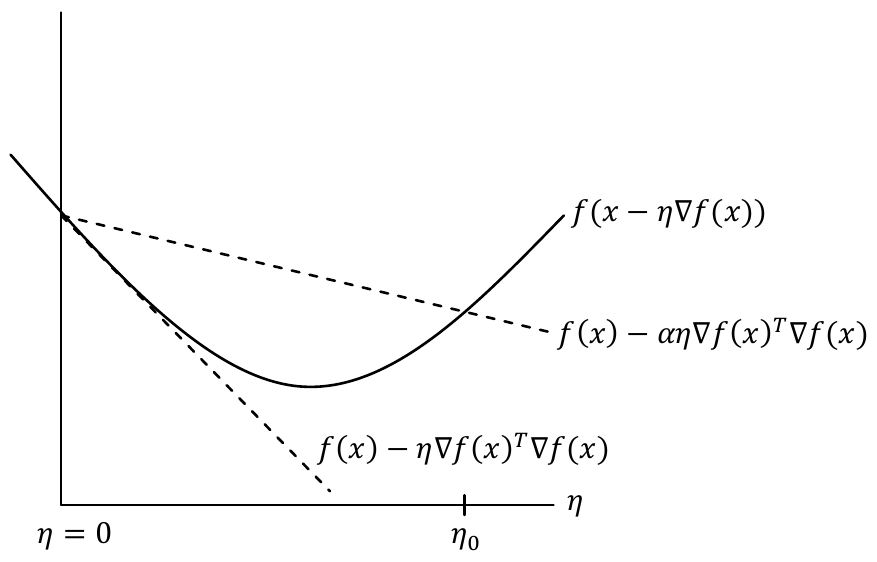}
	\caption[Geometric interpretation of backtracking line search condition.]{The solid line is the objective restricted to the line over which we search. The dashed lines are the tangent of $f$ at the origin, and that line flatter by a factor $\alpha$. The ``sufficient decrease'' condition means that the step size, $\eta$, is reduced until $\eta<\eta_0$. (Adapted from \cite[Figure 9.1]{BoydVandenberghe2004}.)}
	\label{fig:backtrackingLineSearch}
\end{figure}

With regards to our problem, the objective is the energy function, $E$. The initial parameters are uniformly sampled from their domain. The constraint that the parameters are restricted to a subset of $\mathbb{R}$ is enforced by returning $+\infty$ when a parameter strays outside its bounds. The algorithm can be run several times with different initial parameters, known as random restarts, to see whether the solution is sensitive to the initial values (which is necessary since our objective is unlikely to be convex).

Pseudo-code for gradient descent is given in Algorithm \ref{alg:graddescent}. The metaparameters $\alpha=0.001$ and $\beta=0.9$ were chosen experimentally. Initially, we set $\alpha=0.1$, however, we discovered that this resulted in convergence just short of the solution.
\begin{algorithm}[t]
\caption{Gradient descent with backtracking line search}
\label{alg:graddescent}
\begin{algorithmic}[1]
	\Procedure{GradientDescent}{$f$, $\nabla f$, $\mathbf{x}^0$, $\epsilon$}
		\State $\mathbf{x}^{\textnormal{old}} := \mathbf{x}^0$
		\Repeat
			\State $\eta := 1$
			\Repeat
				\State $\mathbf{x}^{\textnormal{new}} := \mathbf{x}^{\textnormal{old}} - \eta\nabla f(\mathbf{x}^{\textnormal{old}})$
				\State $\eta := \beta\eta$
			\Until{$f(\mathbf{x}^{\textnormal{new}})\le f(\mathbf{x}^{\textnormal{old}})-\alpha\eta||\nabla f(\mathbf{x}^{\textnormal{old}})||^2$}
			\If{$(f(\mathbf{x}^{\textnormal{old}}) - f(\mathbf{x}^{\textnormal{new}}))/|f(\mathbf{x}^{\textnormal{old}})|<\epsilon$, or $||\mathbf{x}^{\textnormal{new}}-\mathbf{x}^{\textnormal{old}}||<\epsilon$, or $||\nabla f(\mathbf{x}^{\textnormal{new}})||<\epsilon$}
				\State \textbf{return} $\mathbf{x}^{\textnormal{new}}$
			\Else
				\State $\boldsymbol\theta^{\textnormal{old}} := \boldsymbol\theta^{\textnormal{new}}$
			\EndIf
		\Until{$I$ iterations performed}
		\State \textbf{error} ``Did not converge within $I$ iterations.''
	\EndProcedure
\end{algorithmic}
\end{algorithm}

\section{Newton and Pseudo-Newton methods}\label{sec:newtonMethods}
Gradient descent is an instance of a \emph{descent method}, a general class of iterative optimization techniques. See Algorithm \ref{alg:descentmethod} for the general form these methods take.
\begin{algorithm}[t]
\caption{General descent method}
\label{alg:descentmethod}
\begin{algorithmic}[1]
	\State given a starting $x\in\textnormal{dom}f$
	\Repeat
		\State determine a descent direction $\Delta x$
		\State choose a step size $\eta>0$
		\State update $x := x+\eta\Delta x$
	\Until{stopping criterion is satisfied}
\end{algorithmic}
\end{algorithm}

A descent direction for $f$ at $x$ is one for which $\nabla f(x)^T\Delta x<0$. For example, in the gradient descent algorithm, $\Delta x=-\nabla f(x)$.

It can be shown that gradient descent requires a large number of iterations when the Hessian, or, equivalently, the sublevel sets of $f$ near the solution are ill-conditioned \cite{BoydVandenberghe2004}. Moreover, it only offers linear convergence.

A solution is to choose $\Delta x=-\nabla^2f(x)^{-1}\nabla f(x)$; the \emph{Newton step}. The descent method with this choice of search direction is known as \emph{Newton's method}. It can be motivated in several ways. For example, the Newton step minimizes a second-order approximation to $f$ at $x$. Or one can think of it as solving a first-order approximation to the optimality condition $\nabla f(x^\star)=0$.

Newton's method is insensitive to the condition number of the sublevel sets of the objective, and convergence is rapid in general, and quadratic near a local minimum. Moreover, its performance is similar in low and high dimensions alike. This comes at the cost of calculating, storing, and inverting the Hessian matrix.

In our problem, calculating the Hessian is feasible but requires more message passing relative to calculating the gradient. The dimension of our problem is the same as the number of factors, so storing and inverting the Hessian matrix is prohibitive when there are many factors.

Many so-called \emph{quasi-Newton} methods have been devised to ameliorate these issues. We consider the L-BFGS algorithm \cite{LiuNocedal1989}, and use the implementation of \cite{website:Darwin}, which has the same termination criteria as gradient descent. The details of the algorithm are immaterial to this thesis, although it suffices to understand that the algorithm uses as the descent direction a low-rank approximation to the inverse Hessian. The approximation is not represented directly as a dense matrix, but implicitly by a number of the most recent gradients. 

Thus, the L-BFGS algorithm proffers the quadratic convergence of Newton's method whilst foregoing the calculation of any derivative of an order higher than the gradient. For these reasons, L-BFGS is established as an excellent general purpose unconstrained optimization algorithm.

We tried enforcing the constraints on the parameters by evaluating the objective as $+\infty$ when the parameters were out of their domains. Convergence, however, failed for many combinations of initial values and true parameters. On the iteration of failure, the search direction points away from the solution (so that there is no step size that reduces the objective), and we suspect this is due to the algorithm passing over an indefinite area of the search space, one for which the sublevel sets are nonconvex.

We devised two solutions. First, the algorithm is simply restarted from where it fails, repeating if necessary until success is reported. Restarting the algorithm discards the previous gradients and begins building up a fresh approximation to the inverse Hessian. We call this method \emph{L-BFGS with restart}. The approach is familiar from the optimization literature; for example, as a part of the conjugate gradient method.

Second, the constraints are enforced by the barrier method \cite[\S11.3]{BoydVandenberghe2004}. Consider the general constraint, $f_i(\theta_i) < 0$. Our initial approach essentially adds to the objective a hard constraint,
\begin{align*}
	I_-(\theta_i) &= \left\{\begin{array}{ll} 0, & f_i(\theta_i) < 0 \\ \infty, & \textnormal{otherwise} \end{array}\right..
\end{align*}
The barrier method softens the constraint with the differentiable function,
\begin{align*}
	\widehat{I}_-(\theta_i) &= \left\{\begin{array}{ll} -(1/t)\ln(-f_i(\theta_i)), & f_i(\theta_i) < 0 \\ \infty, & \textnormal{otherwise} \end{array}\right..
\end{align*}
As $t$ increases, the approximation becomes more accurate. On the other hand, it can be shown that when $t$ is large, the Hessian is large near the boundary of the feasible set and the L-BFGS algorithm becomes unstable.

As a compromise, a sequence of problems are solved. Initially, the objective plus the barrier function is minimized by the L-BFGS algorithm with a small value of $t$. Then, $t$ is increased by a fixed ratio $\mu$. The L-BFGS algorithm is restarted from the previous solution, the new solution being a better approximation due to the increase in $t$. This is repeated until the convergence criterion (see \cite[\S11.2.2]{BoydVandenberghe2004}) is attained. We call this method \emph{L-BFGS with barrier method}.

Another solution is to use a variant of the L-BFGS algorithm for box constraints \cite{ByrdEtAl1995}, although a convenient implementation proved elusive.

\section{Calculating the energy function}
In Chapter \ref{inference}, calculating the density, or equivalently, the likelihood, of a CDN using the derivative-sum-algorithm was explained in some detail. Thus, evaluating the energy function, which is the sum of the negative copula log-likelihood of the samples, is trivial. The log copula density of each sample (or rather, this term scaled by a constant) is calculated at the same time as the gradient of the model evaluated at the sample.

\section{Calculating the gradient}
In general, the $i$th component of the log-gradient is,
\begin{align}
	\frac{\partial}{\partial\theta_i}\ln f(x) &= \frac{1}{f(x)}\frac{\partial}{\partial\theta_i}f(x). \label{eqn:gradComponent}
\end{align}
Combining \eqref{eqn:gradComponent} and \eqref{eqn:normalDensity}, and using the fact that the parameters are not shared between copulae, when our model factors are normal copulae with a single parameter,
\begin{align}
	\frac{\partial}{\partial\rho_i}\ln f(x'_1,\ldots,x'_n) &= \frac{\frac{\partial^n}{\partial w_1\cdots\partial w_n}\frac{\partial F_i}{\partial\rho_i}\prod_{j\neq i}F_j(w'_{(j,1)},\ldots,w'_{(j,n_j)})}{\frac{\partial^n}{\partial w_1\cdots\partial w_n}\prod^m_{j=1}F_j(w'_{(j,1)},\ldots,w'_{(j,n_j)})}, \label{eqn:normalGradComponent}
\end{align}
where $\rho_i$ is the parameter of the $i$th copula. Note the terms of the chain rule that have cancelled.

Similarly, using \eqref{eqn:archDensity}, when our model is parameterized with Archimedean copulae,
\begin{align}
	\frac{\partial}{\partial\theta_i}\ln f(x'_1,\ldots,x'_n) &= \frac{\frac{\partial^n}{\partial v_1\cdots\partial v_n}\frac{\partial C_i}{\partial\theta_i}\prod_{j\neq i}C_j(v'_{(j,1)},\ldots,v'_{(j,n_j)})}{\frac{\partial^n}{\partial v_1\cdots\partial v_n}\prod^m_{j=1}C_j(v'_{(j,1)},\ldots,v'_{(j,n_j)})}. \label{eqn:archGradComponent}
\end{align}
Both the numerator and denominator in \eqref{eqn:normalGradComponent} and \eqref{eqn:archGradComponent} are calculated using the derivative-sum-product algorithm (\S\ref{sec:dspAlgorithm}). The denominator is proportional to the result of running the message passing algorithm before scaling by the additional terms of the chain rule.

To calculate the numerator, regular message passing is performed replacing the values of the partial derivatives of $F_i$ or $C_i$ with the partial derivatives of $\partial F_i/\partial\rho_i$ or $\partial C_i/\partial\theta_i$. Calculating the partial derivatives of copulae with respect to their parameter and subsets of their scope was discussed in \S\ref{sec:gradCopulae}.

This is, however, an inefficient method for calculating the full gradient. For each component of the gradient, messages are passed over the full clique tree in calculating the numerator. When, for each component, the node associated with the copulae containing the differentiated parameter is designated as the root, corresponding messages are identical across components.

To eliminate this repetition, an alternative dynamic programming procedure was adopted. The clique tree is ``calibrated'' by calculating all messages---not just those in the direction of an arbitrary root. First, the messages are passed towards the default root in the topological ordering. Then, the remaining messages are calculated by passing messages in the reverse topological ordering. That is, starting from the node before the root in the topological ordering and proceeding to the first node, a message is passed from each node's child to itself.

After the clique tree has been calibrated, the components of the gradient are calculated by evaluating the result of message passing at the clique associated with the copula containing the differentiated parameter, replacing that copula with its derivative with respect to its parameter. Any clique can be used to calculate the density, although it is preferable to use the clique with the smallest scope. See Algorithm \ref{alg:calculateGradient}.
\begin{algorithm}[t]
\caption{Calculate the log-gradient and copula log-likelihood of a CDN}
\label{alg:calculateGradient}
\begin{algorithmic}[1]
	\Procedure{GradientLoglikelihood}{$\mathbf{u}, \mathcal{C}$}
		\State $\textnormal{\scshape CalibrateCliqueTree}(\mathbf{u}, \mathcal{C})$
		\State $c(\mathbf{u}) := \textnormal{\scshape DspMessages}(1, -1)$\Comment{chain rule term of copula density is irrelevant}
		\For{$i$ in $1$ to $\#\{\phi_j\}$}
			\State \textbf{replace} $\phi_i$ with $\partial\phi_i/\partial\theta_i$ temporarily
			\State $\left(\nabla\ln f\left(\mathbf{x}\right)\right)_i := \textnormal{\scshape DspMessages}(\alpha(i), -1) / c(\mathbf{u})$
		\EndFor
		\State \textbf{return} $\{\ln(c(\mathbf{u})), \nabla\ln f(\mathbf{x})\}$
	\EndProcedure\vspace{6pt}
	\Procedure{CalibrateCliqueTree}{$\mathbf{u}, \mathcal{C}, \textnormal{\scshape CalcMsgs}$}	
		\For{$C_i\in\mathcal{C}$ taken in the topological ordering}
			\If{$\textnormal{child}(i)\neq-1$}
				\State $\textnormal{\scshape CalcMsgs}(i, \textnormal{child}(i))$
			\EndIf
		\EndFor
		\For{$C_i\in\mathcal{C}$ taken in reverse topological ordering}
			\If{$\textnormal{child}(i)\neq-1$}
				\State $\textnormal{\scshape CalcMsgs}(\textnormal{child}(i), i)$
			\EndIf
		\EndFor
	\EndProcedure
\end{algorithmic}
\end{algorithm}

An alternative algorithm, called the gradient-derivative-product algorithm, for calculating the gradient is given in \cite{HuangJojic2010}. Essentially, it works by passing a separate set of messages for each component of the gradient where the parameter has been introduced as an additional variable to be differentiated. The algorithm is more general than ours in the sense that it handles the, albeit uncommon, case of shared copulae parameters. Notwithstanding, our approach is quicker, and easier to implement.

\section{Examples}
Learning is illustrated with two simple examples.
\begin{ex}
	Consider the CDN parameterized with normal copulae and margins having the factor graph of Figure \ref{fig:learningExGraph1}. The parameters were initialized to $\rho_1=0.6$ and $\rho_2=-0.2$, and ten thousand samples drawn. Then, we applied gradient descent and L-BFGS with restart from three random restarts to determine if the parameters could be recovered.
	
	The results are given in Figure \ref{fig:learningEx1Parameters}. Notice that for gradient descent, the iterations move perpendicular to the level surfaces (the gradient at a point is normal to the level surface intersecting that point). The large movements for the initial iterations are not a problem since the likelihood is still decreasing sufficiently. The L-BFGS algorithm converges in fewer iterations and approaches the optimal solution more smoothly.
\end{ex}

\begin{ex}
	Consider the CDN parameterized with normal copulae and margins having the factor graph of Figure \ref{fig:learningExGraph2}. The parameters were initialized to $\rho_1=0.5$ and $\rho_2=-0.5$, and ten thousand samples drawn. Note that the parameters in this model are not identifiable, as exchanging them does not change the log-likelihood. Again, we applied both learning methods from a number of random restarts.
	
	A record was made for each random restart whether the solution converged to $\rho_1=0.5$ and $\rho_2=-0.5$, or vice versa. The region of initial values that converge to a given solution is apparent from Figure \ref{fig:learningEx2}, wherein the endpoint of convergence is indicated along with the initial values.
	
	Learning by gradient descent exhibits clear regions of convergence, and there are only a few starting points for which convergence fails. These are near the boundary, where the gradient is very steep. Learning by L-BFGS with barrier method also has clear regions of convergence, although they are different to those of gradient descent, which we can understand in terms of the force field interpretation of the barrier function. There are more starting values that fail to converge, suggesting that it is not as robust as gradient descent.
\end{ex}
See Experiment \ref{exp:learning} for a thorough comparison of the learning methods.

\begin{figure}[p]
     \begin{center}
        \subfigure[]{
            \label{fig:learningEx1overview}
            \includegraphics[scale=0.5]{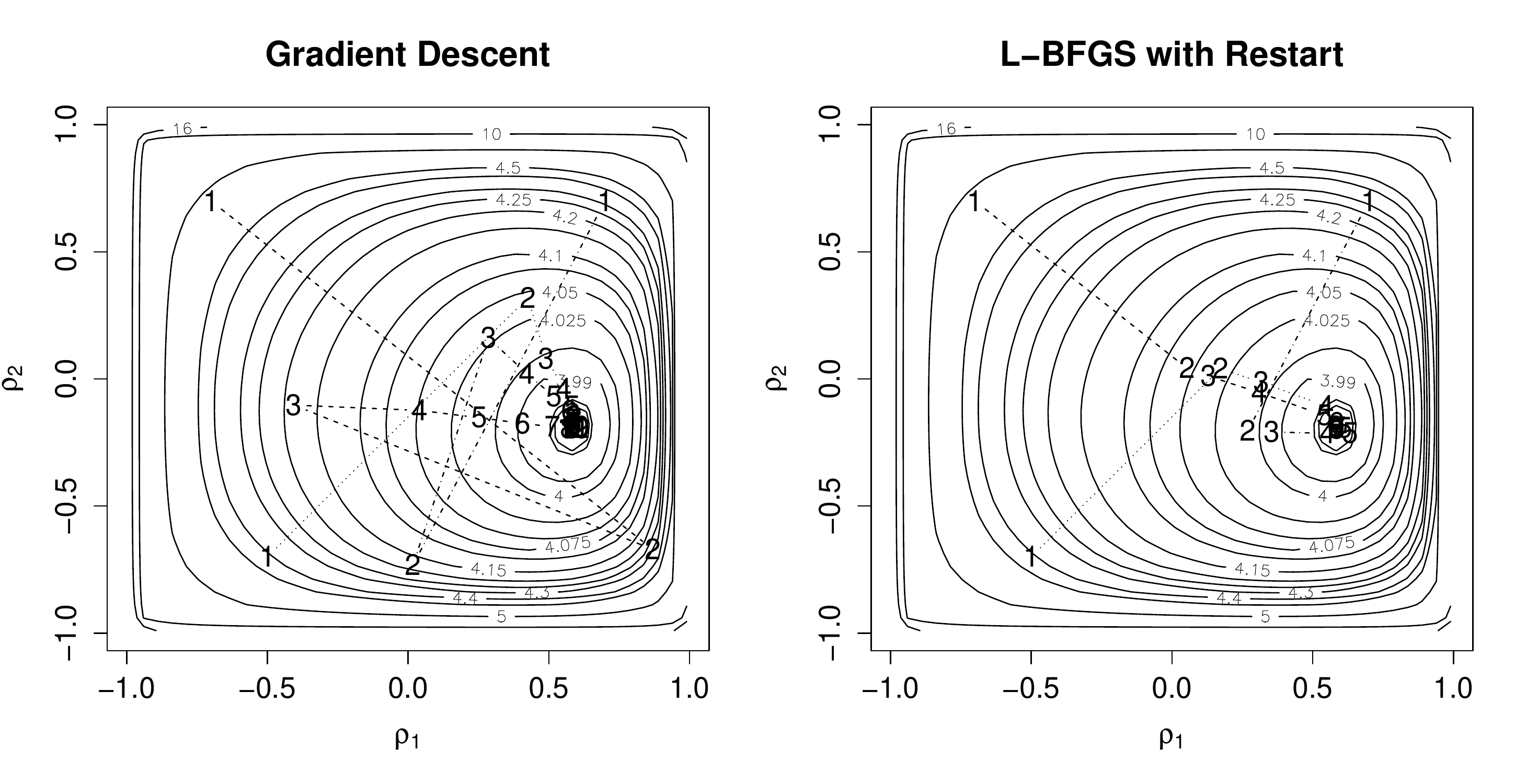}
        }\vspace{1cm}
        \subfigure[]{
           \label{fig:learningEx1zoomed}
					
					\begin{tabular}{c}
						\includegraphics[scale=0.5]{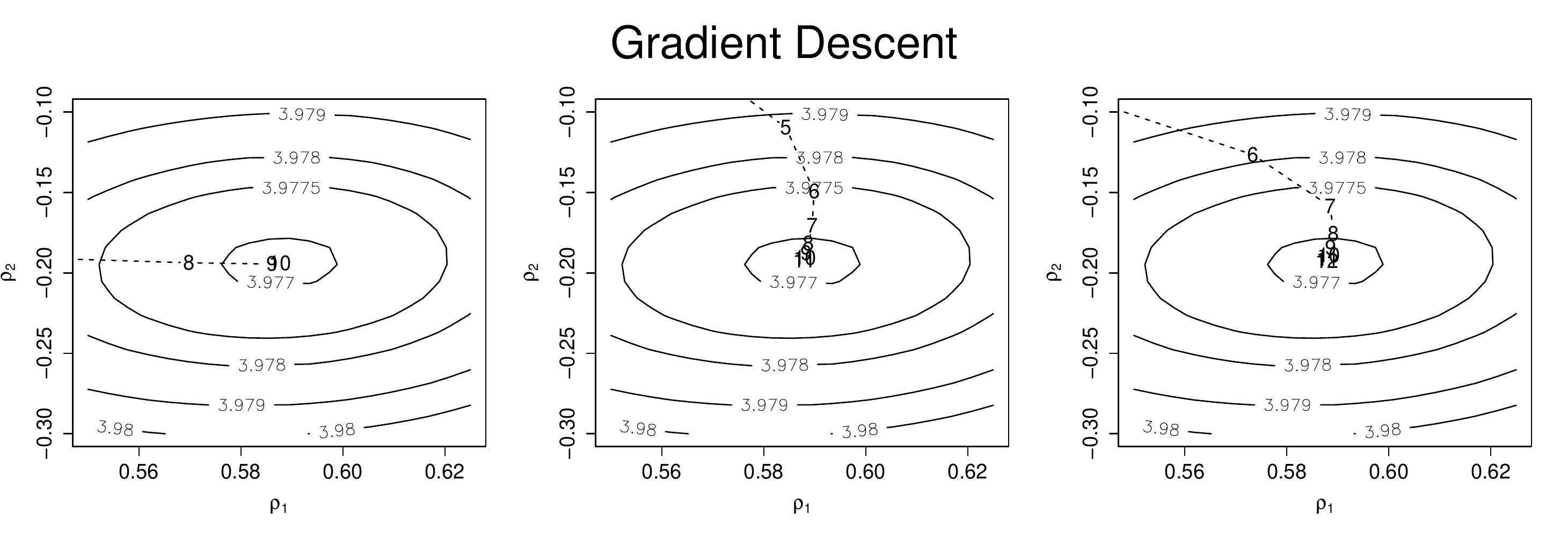}\\
						\includegraphics[scale=0.5]{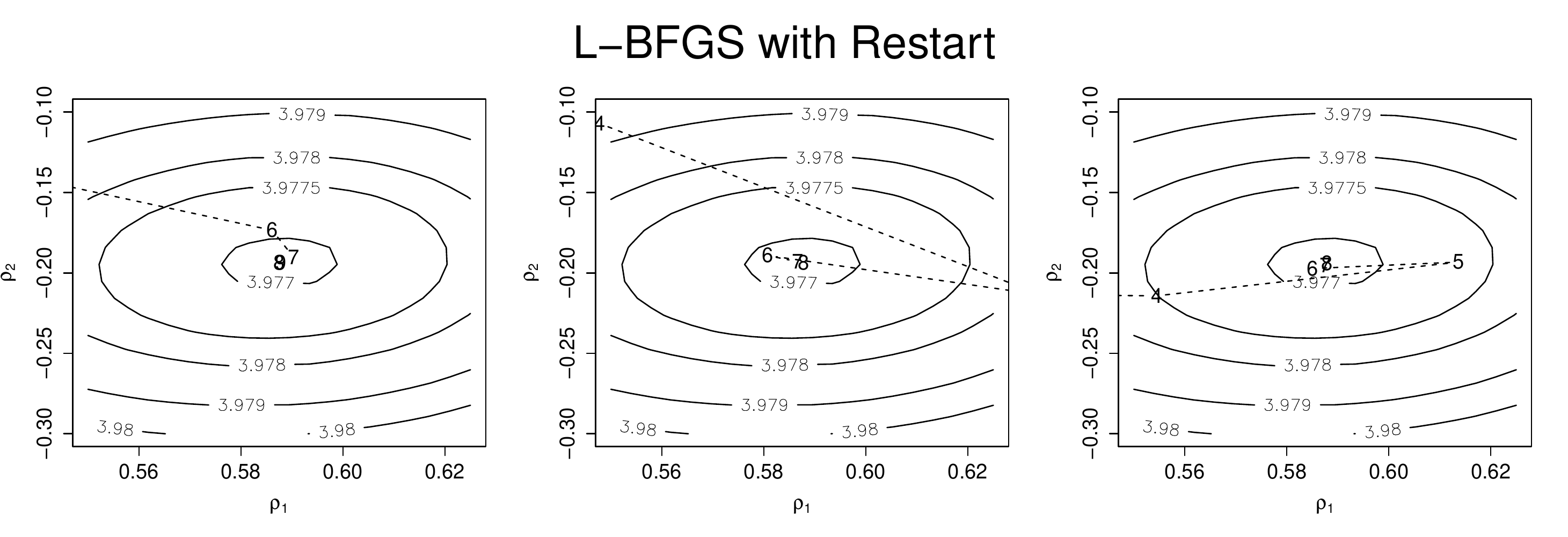}
					\end{tabular}
        }
    \end{center}
    \caption[Visualization of the iterations of learning.]{(a) Overview of the parameters at each iteration of learning for three random restarts. (b) Close-up near optimal point.}
   \label{fig:learningEx1Parameters}
\end{figure}

\begin{figure}[t]
     \begin{center}
        \subfigure[]{
            \label{fig:learningExGraph1}
            \includegraphics[scale=0.6]{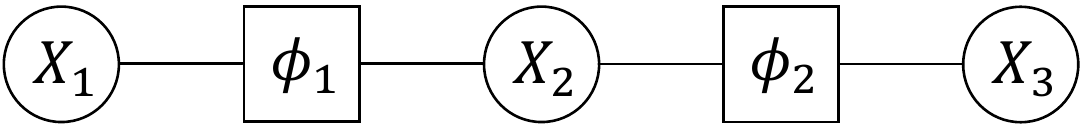}
        }\hspace{1cm}
        \subfigure[]{
           \label{fig:learningExGraph2}
           \includegraphics[scale=0.6]{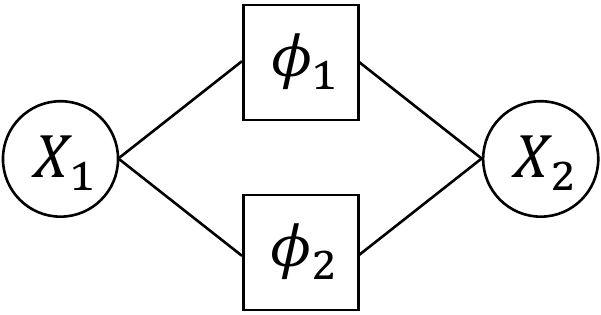}
        }
    \end{center}
    \caption[Factor graph for learning examples.]{(a) Factor graph for a CDN that has three variables connected by two bivariate normal copulae. (b) Factor graph for a CDN over two variables connected by two bivariate normal copulae. Note that the parameters are not identifiable.}
   \label{fig:learningExGraphs}
\end{figure}
\begin{figure}[t]
	\begin{center}\includegraphics[scale=0.5]{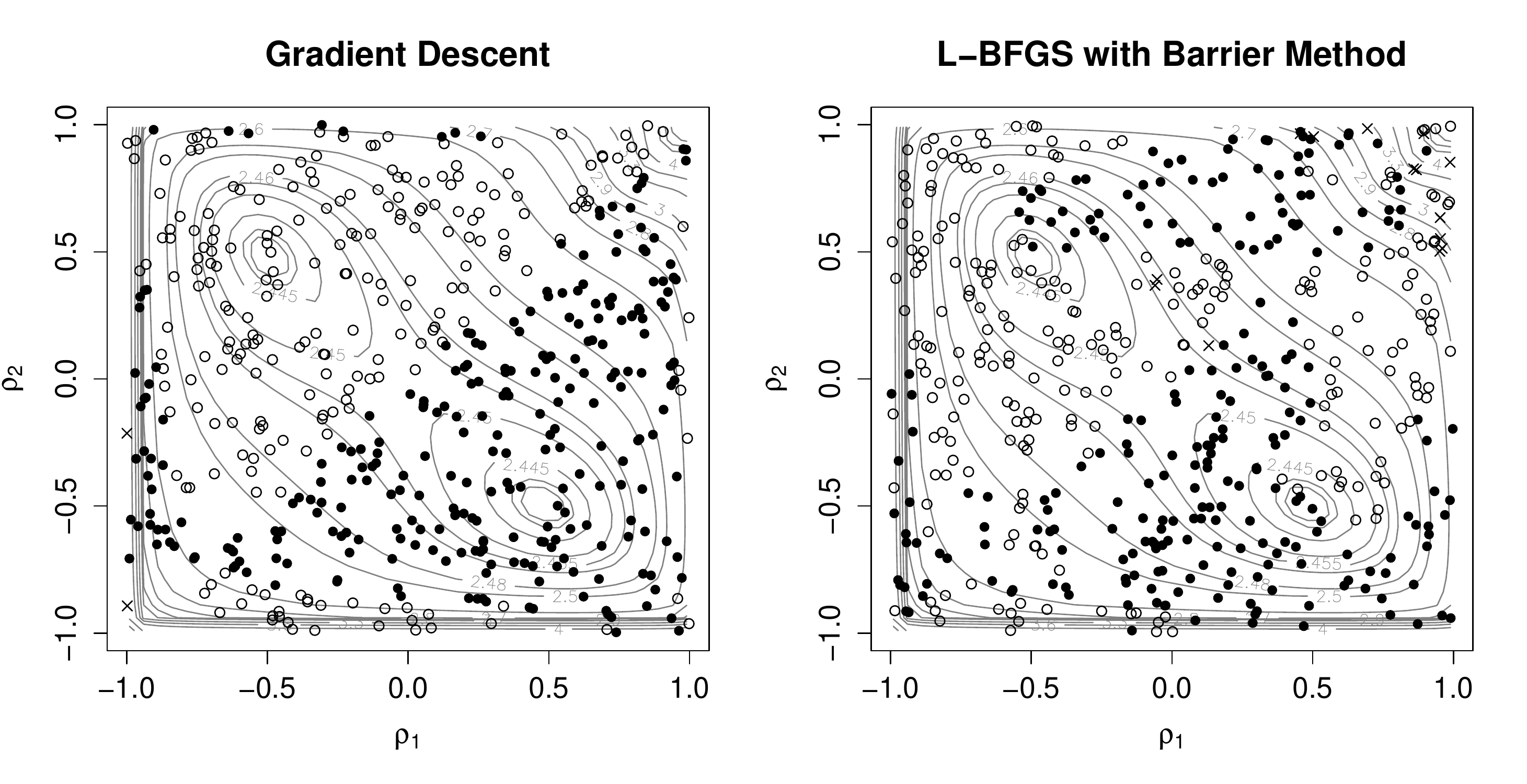}\end{center}\caption[Regions of convergence for learning non-identifiable parameters.]{Filled circles indicate the initial values that converge to $\rho_1=0.5$, $\rho_2=-0.5$, hollow circles indicate those converging to $\rho_1=-0.5$, $\rho_2=0.5$, and crosses indicates values that do not converge to an optimal value.}\label{fig:learningEx2}
\end{figure}

\section{Learning from data missing completely at random}\label{sec:cmarLearning}
A simple algorithm for learning from missing completely at random (MCAR) data ensues from the triviality of marginalization. Data values are said to be MCAR when, roughly speaking, their observational pattern is independent of both the observed variables and the parameters.

Suppose that data for $\mathbf{v}^i\subseteq\mathbf{u}$ is missing from the $i$th sample. When calculating the energy function, the likelihood of a sample with missing data is equal to the likelihood of its observed variables having marginalized the missing variables. Recall that variables are marginalized from a copula by setting them to unity. Consequently, learning from MCAR data can be performed by minimizing the modified energy,
\begin{align*}
	E(\boldsymbol\theta ; \mathcal{D}) &= -\frac{1}{m}\sum^m_{i=1}\ln\left(\frac{\partial}{\partial(\mathbf{u}\setminus\mathbf{v}^i)}C(\mathbf{u}^i\setminus\mathbf{v}^i,\mathbf{v}^i=\mathbf{1} ; \boldsymbol\theta)\right).
\end{align*}
The likelihood of a subset of the model is calculated by the derivative-sum-product algorithm having set the missing variables to unity. In a naive approach, a clique tree is constructed over each combination of observed variables. It is, however, possible to run inference on a subset of the variables using the full model clique tree, which is preferable as it obviates determining all combinations of missing variables and constructing and storing the resulting clique trees.

Provided partial derivatives with respect to subsets that include a marginalized variable are zeroed, and we do not differentiate missing variables in $C_i\setminus S_{i,j}$ when sending a message from $i$ to $j$, the standard learning procedure applies.

A similar idea produces an algorithm for learning from censored data. For example, suppose for one sample that while the exact value of a variable $X_k$ is unknown, we observe that $X_k\le x'_k$. To calculate the likelihood for this sample, we differentiate with respect to all variables save $x_k$, and evaluate the resulting mixed density and CDF at $x'_k$.

See Experiment \ref{exp:cmar} for a study of the robustness of learning in CDNs from missing data.

\section{Piecewise composite likelihood learning}\label{sec:piecewiseLearning}
Standard gradient-based learning methods require calculating the full model likelihood and gradient. A recurrent theme of PGMs, however, is that the factorization of the model allows one to operate on smaller parts. In this section, we develop a novel algorithm for learning CDNs that adopts this approach, and discuss its advantages and disadvantages.

To visualize the information flow throughout the network during learning, we form what we term the \emph{parameter graph}. The nodes in the parameter graph represent the copula parameters, or equivalently the factors since there is one parameter per factor, and an edge connects two nodes if the factors have intersecting scope. The parameter graph can be thought of as the factor graph with the variable nodes removed and the paths between factors merged. See Figure \ref{fig:parameterGraph} for the parameter graph of Example \ref{ex:studentcdn}.
\begin{figure}[t]
     \begin{center}
        \includegraphics[scale=0.7]{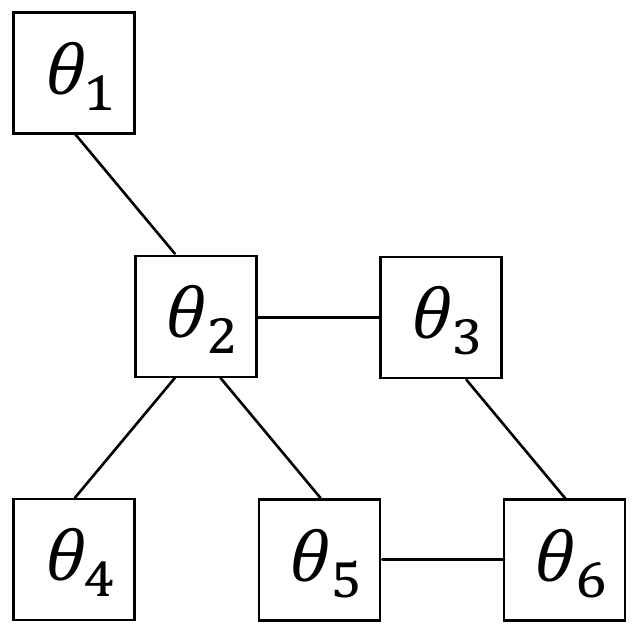}
    \end{center}
    \caption[Parameter graph for the Student CDN.]{Parameter graph for the Student CDN.}
   \label{fig:parameterGraph}
\end{figure}

A novel perspective is revealed when this concept is applied to standard learning methods. In learning methods based on MLE, a likelihood is formed from the factors at all nodes, and is optimized once. Put another way, there is one subproblem over the entire parameter graph that is solved once. Coordinate descent, in its simplest form, is viewed in terms of the graph as solving a succession of subproblems, each over a single node, incorporating information from the remaining nodes.

Composite likelihood learning \cite{VarinEtAl2011} is also comprehensible in terms of the graph. A composite log-likelihood is a function of the form,
\begin{align*}
	l^c({\boldsymbol\theta}\ |\ \mathbf{x}_1,\ldots,\mathbf{x}_n) &= \sum^n_{i=1}\sum^m_{j=1}\ln\left(\mathcal{L}_j({\boldsymbol\theta}_j;\mathbf{x}_i)\right),
\end{align*}
where each term $\mathcal{L}_j$ is the likelihood function resulting from the conditional density of a subset of the model, given another subset. In composite likelihood learning, the function $l^c$ is maximized instead of the log-likelihood. A loss in efficiency is traded for the gain in computational ease of optimizing terms over smaller scope. With respect to the parameter graph, $m$ subproblems are formed over the sets of nodes $\{{\boldsymbol\theta}_j\}$, which are solved \emph{simultaneously}.

More generally, a continuum of algorithms is formed by changing the form of the subproblems---which parameters and information they incorporate---and whether they are solved simultaneously or iteratively (or in some combination). The continuum between these two aspects provides a trade-off between efficiency and speed. It is typically quicker to solve many small subproblems, albeit at the cost of some efficiency, rather than fewer large ones. Also, by iterating the solution of the subproblems, information flows around the parameter graph and the loss in efficiency is minimized, at the cost of computational time.

Our algorithm, \emph{piecewise composite likelihood learning}, is a generalization of the aforementioned learning methods, and we explored two variants. In the first, a subproblem was formed over each node by constructing the clique tree over the scope of the factor at that node. The resulting clique tree also contains the factors that have overlapping scope with the factor under consideration---the neighbours of the factor in the parameter graph. In this way, each subproblem incorporates information from each of the neighbouring subproblems, while only requiring optimization over a single variable. L-BFGS with restart was used to solve each subproblem.

The function that is optimized at each subproblem is the product of the factors in the clique tree, differentiated with respect to the scope of the factor at that node. It does not correspond exactly to the marginal likelihood or conditional likelihood of any subset of the model. (If this were required, one could, for example, marginalize out the variables not in the clique tree that are present in the factors.)

The second variant similarly constructs a subproblem over each adjacent pair of nodes in the parameter graph. For each subproblem, a clique tree is constructed over the union of the scopes of the corresponding two factors, which is again solved by L-BFGS with restart. We found, however, that this variant was slower to converge.

Solving each subproblem is termed an ``inner iteration,'' whereas one pass of solving every subproblem in some order is termed an ``outer iteration.'' To find the global solution, the parameters are randomly initialized. Then, the subproblems are shuffled and solved in succession. After each inner iteration, the subproblem is said to be active when either of the following changes with respect to the last outer iteration by less than some fixed tolerance $\epsilon$:
\begin{itemize}
	\item the mean squared error in the parameter or parameters optimized during that subproblem;
	\item the ``likelihood'' over the subproblem clique tree.
\end{itemize}
In other words, a subproblem becomes active when the update in the solution or change in likelihood is sufficiently small. At the end of each outer iteration, if all subproblems are active, or one iteration of L-BFGS was performed per subproblem, the algorithm terminates. The subproblems are shuffled before the start of each outer iteration, to improve information flow.

One advantage of this method is that in many cases the maximal treewidth over the subproblems is less than the treewidth over the full model clique tree. For example, in a square grid with bivariate factors, the treewidth is linear in the width of the grid. In the first variant of our method, however, every subproblem has treewidth one, as it consists of a single clique with two variables. Another advantage is that it obviates the need to calculate the full model likelihood and gradient. For very large models, the derivative-sum-product algorithm and its modification to calculate the gradient frequently overflow or underflow, even when they are run in log-space. Also, the algorithm is well suited to being parallelized; the subproblems can be solved simultaneously.

The main disadvantage is that the sum of the sizes of the subproblem clique trees is typically greater than the size of the full model clique tree, so more messages are passed during a single outer iteration, compared to calculating the full model likelihood or gradient. The gain in speed from reducing the treewidth must outweigh the decrease in speed from passing more messages for there to be a speedup. Also, there appears to be a loss of efficiency. This trade-off is explored in Experiment \ref{exp:piecewise}

A study of the theory of convergence of the algorithm is beyond the scope of this thesis.

\section{Summary}
\begin{itemize}
	\item The marginal parameters of a CDN are learnt by MLE.
	\item Subsequently, the copula parameters are learnt by minimizing the negative copula log-likelihood, termed the energy function.
	\item We discussed three gradient-based descent methods---gradient descent, L-BFGS with restart, and L-BFGS with barrier method---to perform the optimization.
	\item The gradient is efficiently calculated by dynamic programming and a modification to the derivative-sum-product algorithm.
	\item An algorithm for learning from MCAR data follows from the triviality of marginalization. Learning from censored data proceeds similarly.
	\item A novel algorithm termed \emph{piecewise composite likelihood learning} makes learning feasible over models with high-dimension or treewidth.
\end{itemize}

%% file: experiments.tex
\chapter{Experiments}\label{experiments}
In this chapter, we detail four simple experiments that test various properties of the models and algorithms. Our purpose was to determine:
\begin{itemize}
	\item the relationship between the speed of inference, the size of the model, its largest clique, and the copula type;
	\item how the performance of learning varies over model typology and learning method;
	\item the efficiency of learning with MCAR data as the proportion of missing data is varied;
	\item the relative efficiency of running times of piecewise learning compared to L-BFGS with restart, and whether piecewise learning is consistent.
\end{itemize}
The experiments were run on four classes of models supposed to represent the archetypal components of a graphical model's typology: chains, loops, trees and grid. The models are parameterized with bivariate factors, which are all either normal or Clayton. Refer to Figure \ref{fig:modelTypes} for an illustration of the graphs associated with these models, and Figure \ref{fig:modelSizes} for how the treewidth varies with the model size. Chains, loops, and trees have constant treewidth, whereas the treewidth of grids increases linearly with the width of the grid. (Actually, the minimal treewidth of an $n\times n$ grid is $n$, so the figure indicates that the min-fill heuristic is unable to find the optimal clique tree for grids.)

Due the copula representation, the marginals were irrelevant to any of our enquiries. Thus, they were simply set to have a standard normal distribution. All learning was performed with unknown margins, that is, the margins were learnt from the data prior to learning the copula distribution. Timed experiments were run on an Intel(R) Core(TM) i7-3770 CPU at 3.40GHz, utilizing a single core per process.

We conclude with a discussion of the limitations of CDNs.
\begin{figure}[p]
  \centering
	\renewcommand{\arraystretch}{4}
	{
	\begin{tabular}{c@{\hskip 1cm}cc}
	 \multicolumn{3}{c}{Chains} \\[-30pt]
		\raisebox{-.5\height}{\includegraphics[scale=0.4]{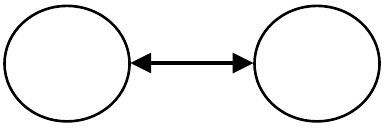}} & \raisebox{-.5\height}{\includegraphics[scale=0.4]{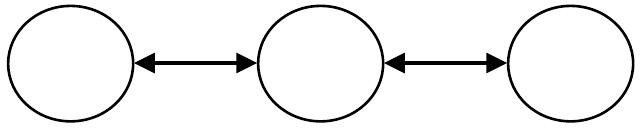}} & \raisebox{-.5\height}{\includegraphics[scale=0.4]{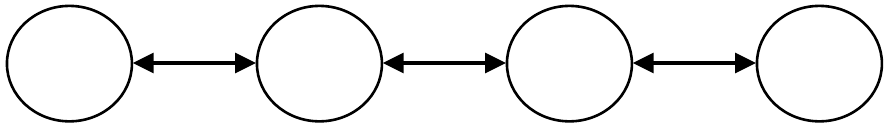}} \\[-24pt]
		$n=2$ & $n=3$ & $n=4$ \vspace{0.25cm} \\[0pt] \hline
		\multicolumn{3}{c}{Loops} \\[-5pt]
		\raisebox{-.5\height}{\includegraphics[scale=0.4]{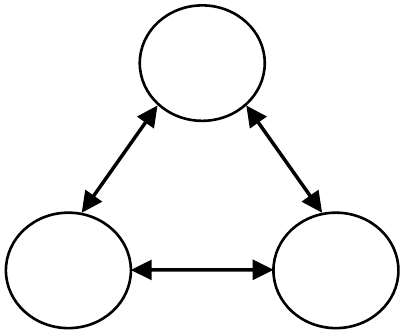}} & \raisebox{-.5\height}{\includegraphics[scale=0.4]{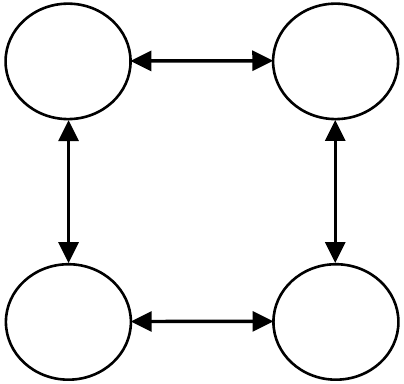}} & \raisebox{-.5\height}{\includegraphics[scale=0.4]{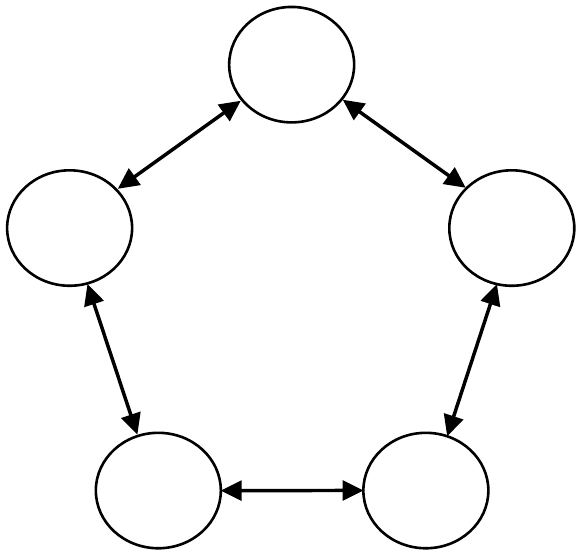}} \\[-15pt]
		$n=3$ & $n=4$ & $n=5$ \vspace{0.25cm} \\ \hline
		\multicolumn{3}{c}{Trees} \\
		\raisebox{-.5\height}{\includegraphics[scale=0.4]{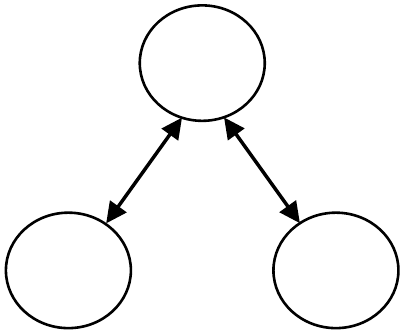}} & \raisebox{-.5\height}{\includegraphics[scale=0.4]{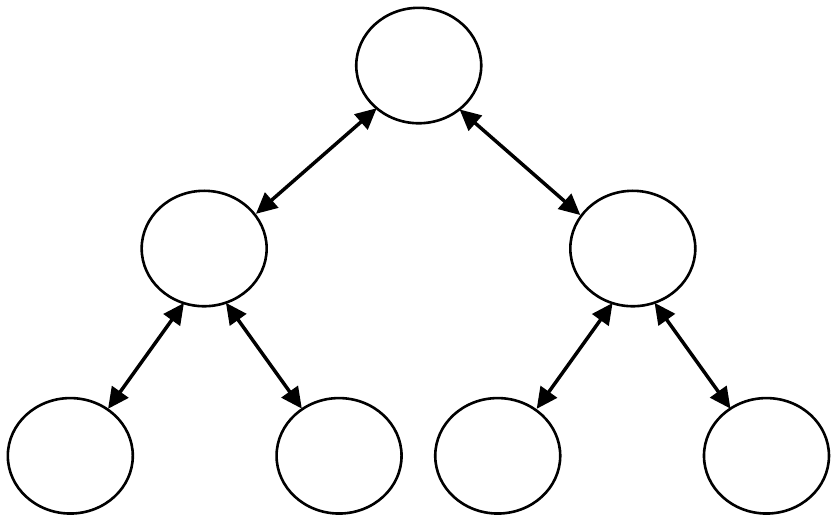}} & \raisebox{-.5\height}{\includegraphics[scale=0.4]{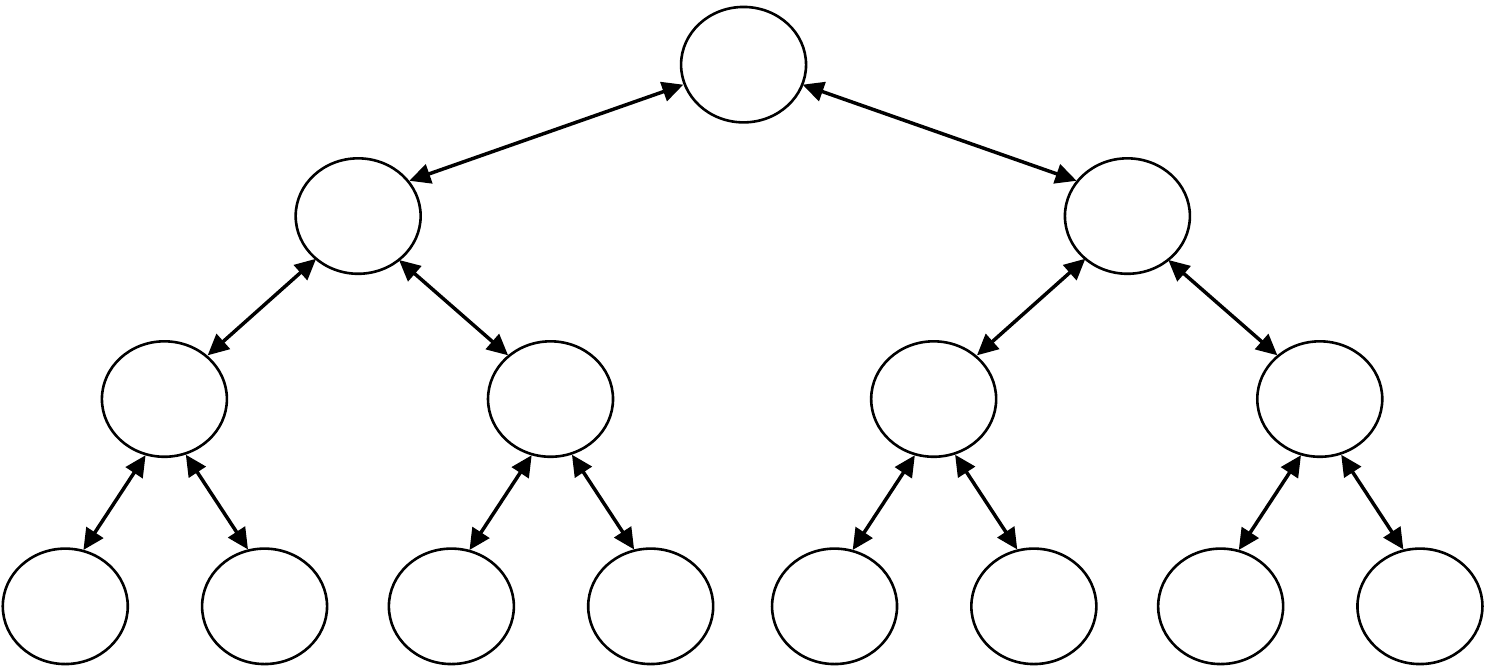}} \\[-15pt]
		$n=2$ & $n=3$ & $n=4$ \vspace{0.25cm} \\ \hline
		\multicolumn{3}{c}{Grids} \\
		\raisebox{-.5\height}{\includegraphics[scale=0.4]{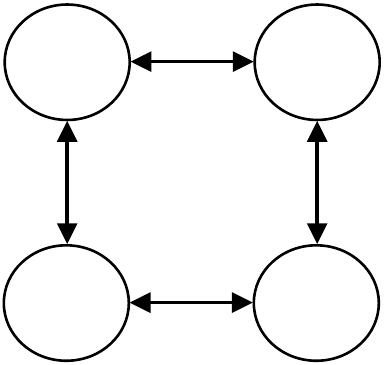}} & \raisebox{-.5\height}{\includegraphics[scale=0.4]{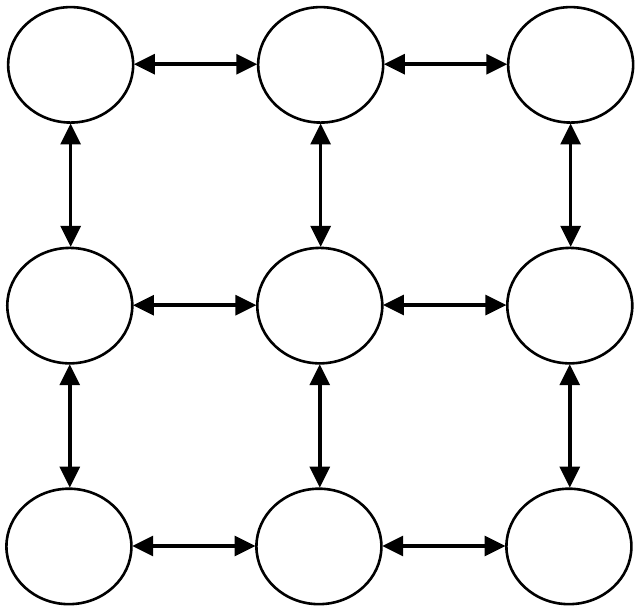}} & \raisebox{-.5\height}{\includegraphics[scale=0.4]{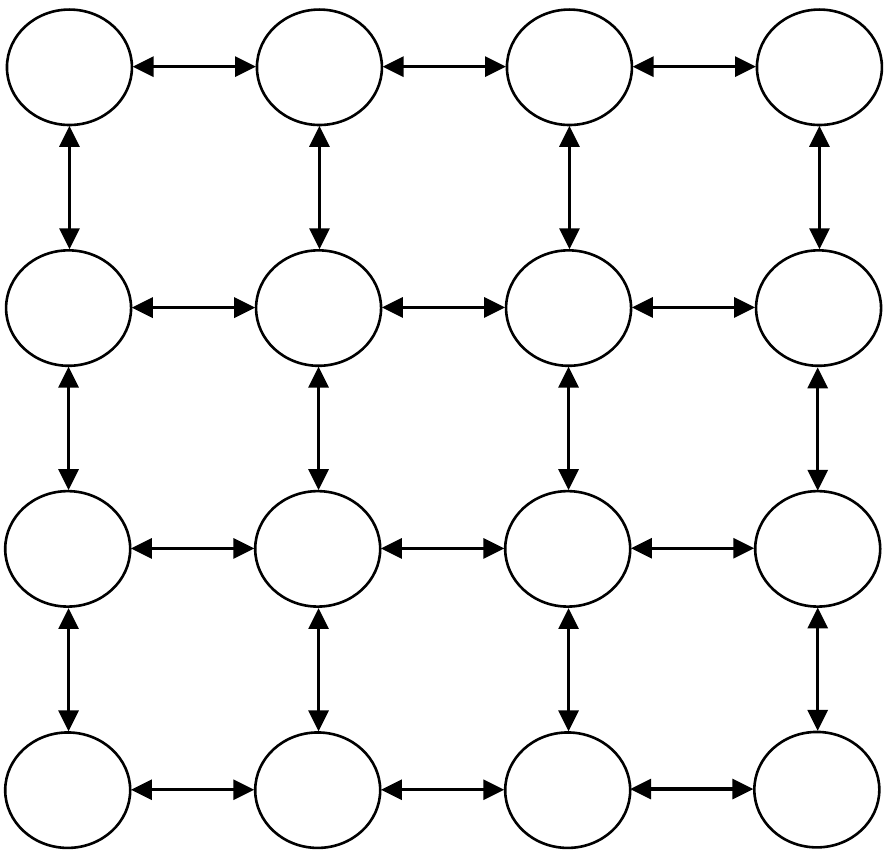}} \\[-15pt]
		$n=2$ & $n=3$ & $n=4$
		\vspace{6pt}
	\end{tabular}}
	\caption[Archetypal models for our experiments]{Four archetypal components of a graphical model's typology, whereon our experiments are run. For loops and chains, the dimension $n$ is the number of variables. For grids, $n$ is the number of variables across the width, and for trees $n$ is the number of levels.}
	\label{fig:modelTypes}
\end{figure}

\begin{figure}[p]
  \centering\begin{tabular}{@{\hskip 0.65cm}c@{\hskip 1cm}c}
	 \ \ \ \ \ Chains & \ \ \ \ \ Loops \\ 
	\raisebox{-.5\height}{\includegraphics[scale=0.3]{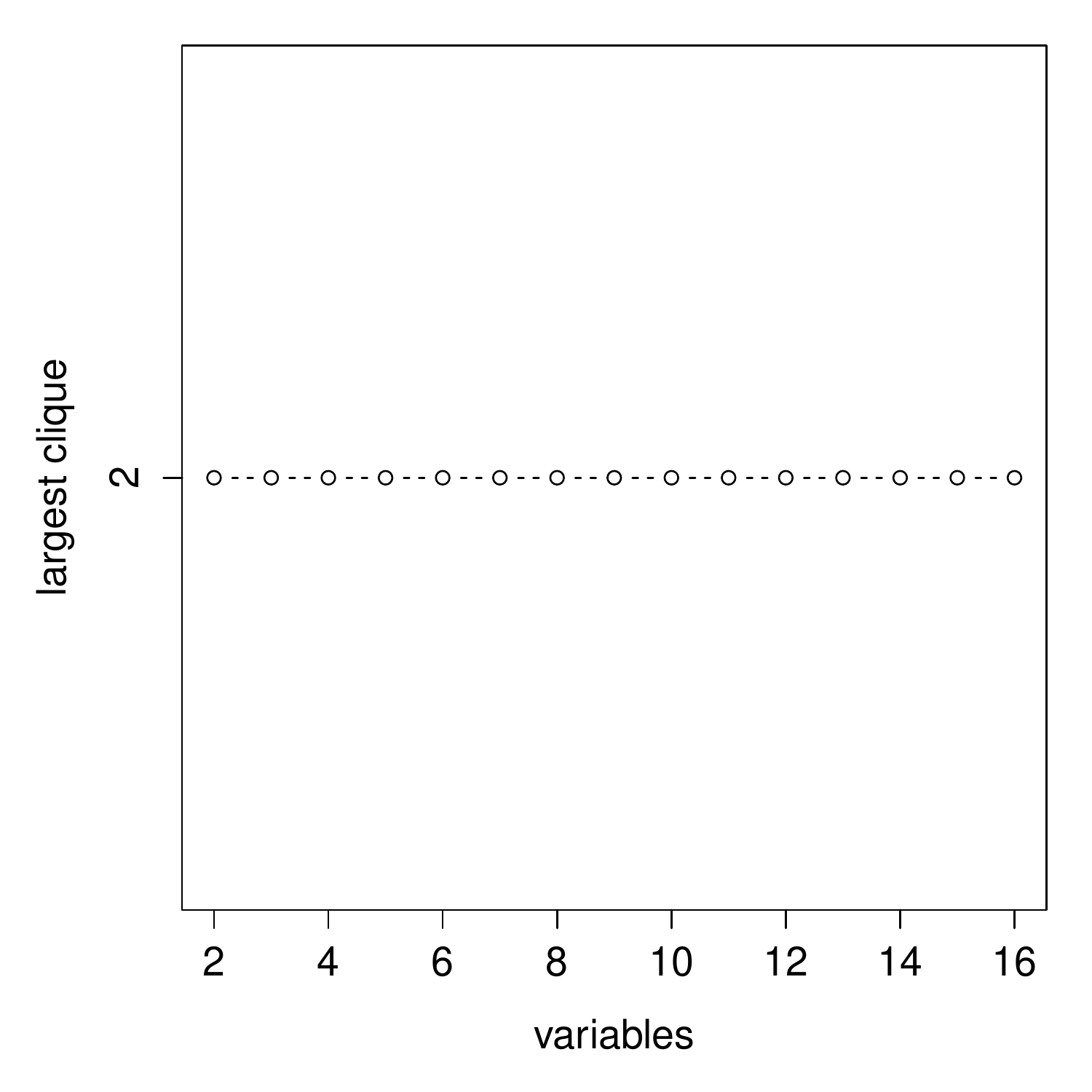}} & \raisebox{-.5\height}{\includegraphics[scale=0.3]{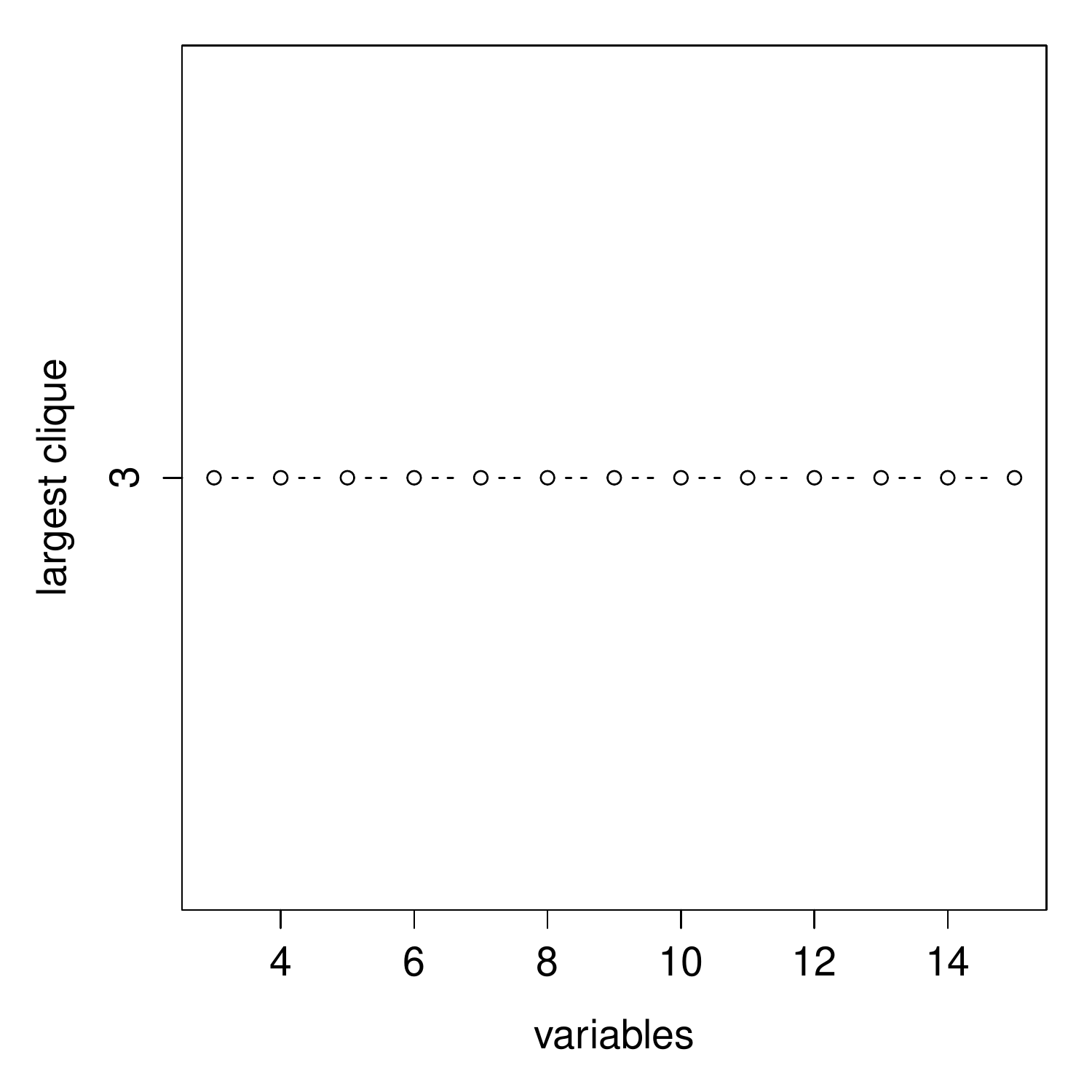}} \vspace{0.5cm} \\
	\ \ \ \ \ Trees & \ \ \ \ \ \ Grids \\ 
	\raisebox{-.5\height}{\includegraphics[scale=0.3]{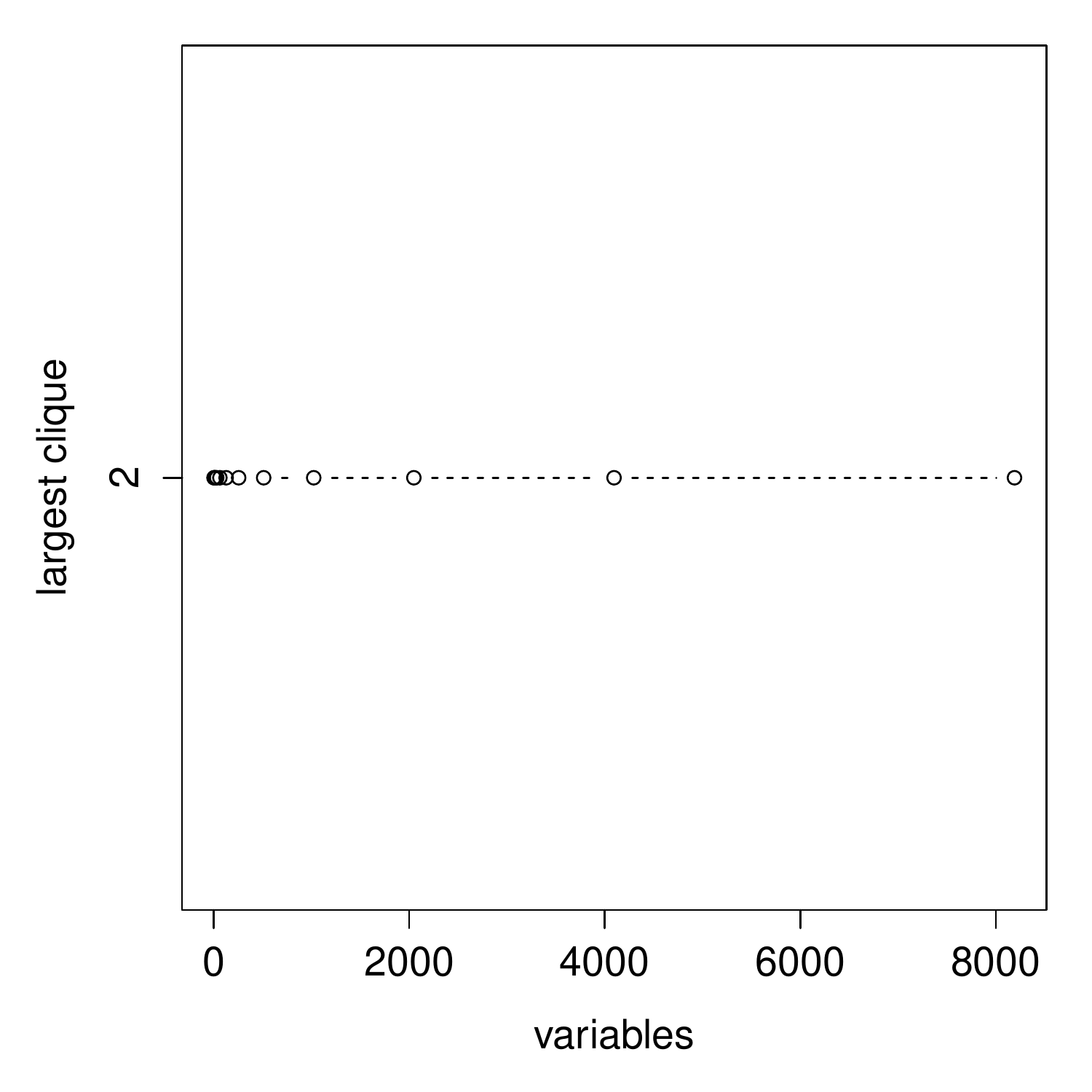}} & \raisebox{-.5\height}{\includegraphics[scale=0.3]{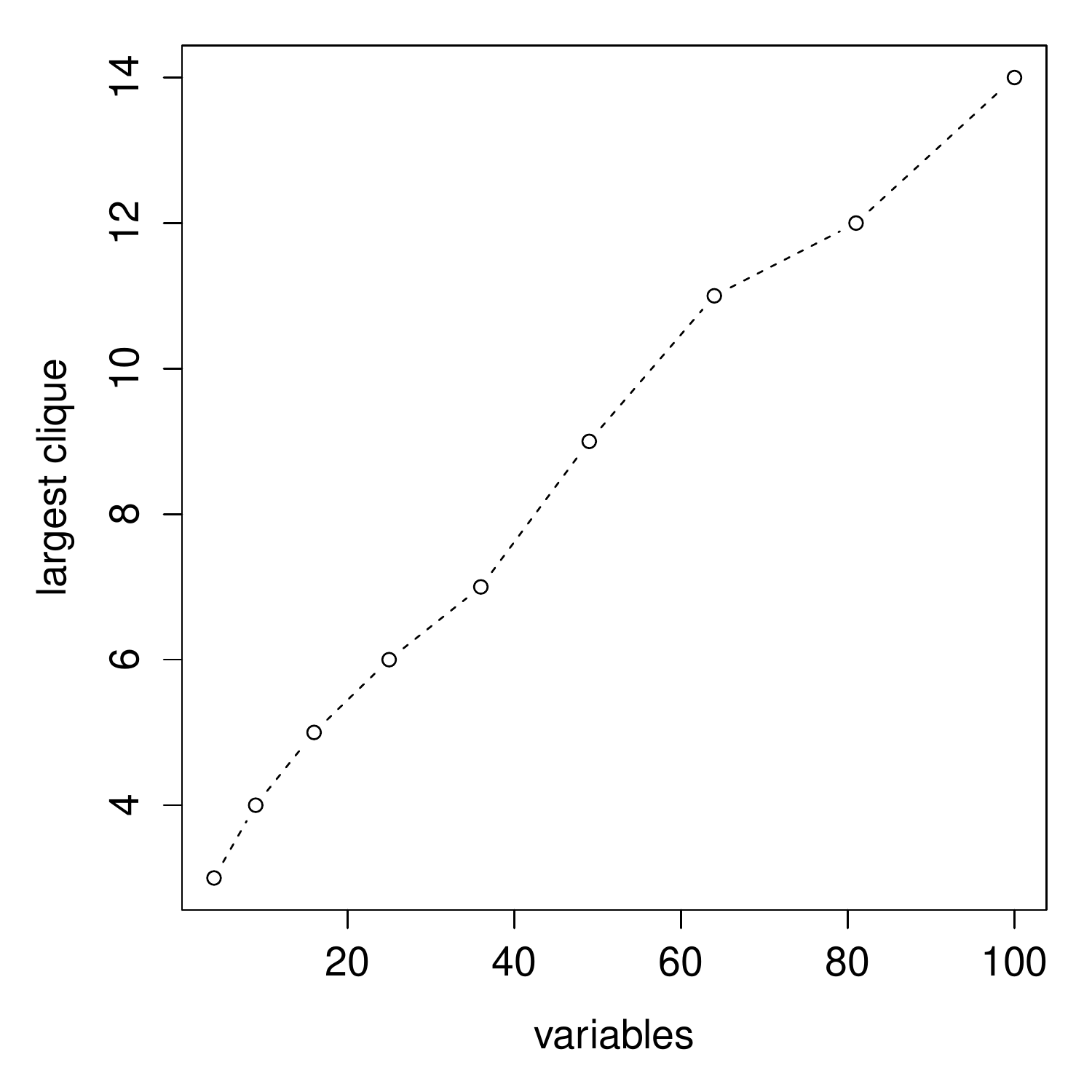}} \\
	\end{tabular}
	\caption[Size of the largest clique against the number of variables in test models.]{Size of largest clique in a clique tree constructed with the min-fill heuristic against the number of variables in test models.}
	\label{fig:modelSizes}
\end{figure}

\begin{figure}[p]
  \centering\begin{tabular}{@{\hskip 0.65cm}c@{\hskip 1cm}c}
	 \ \ Chains & \ \ Loops \vspace{-0.15cm} \\ 
	\raisebox{-.5\height}{\includegraphics[scale=0.3]{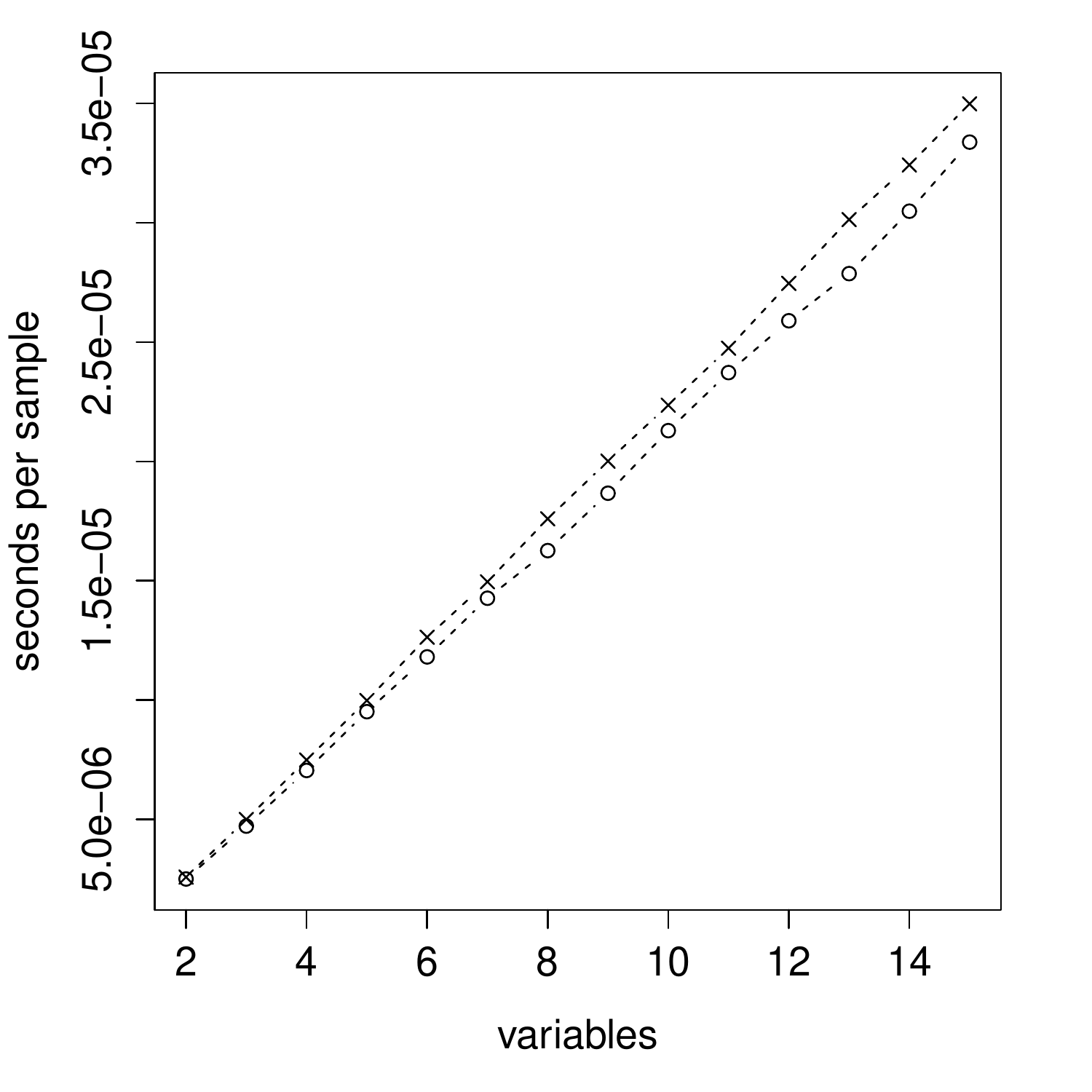}} & \raisebox{-.5\height}{\includegraphics[scale=0.3]{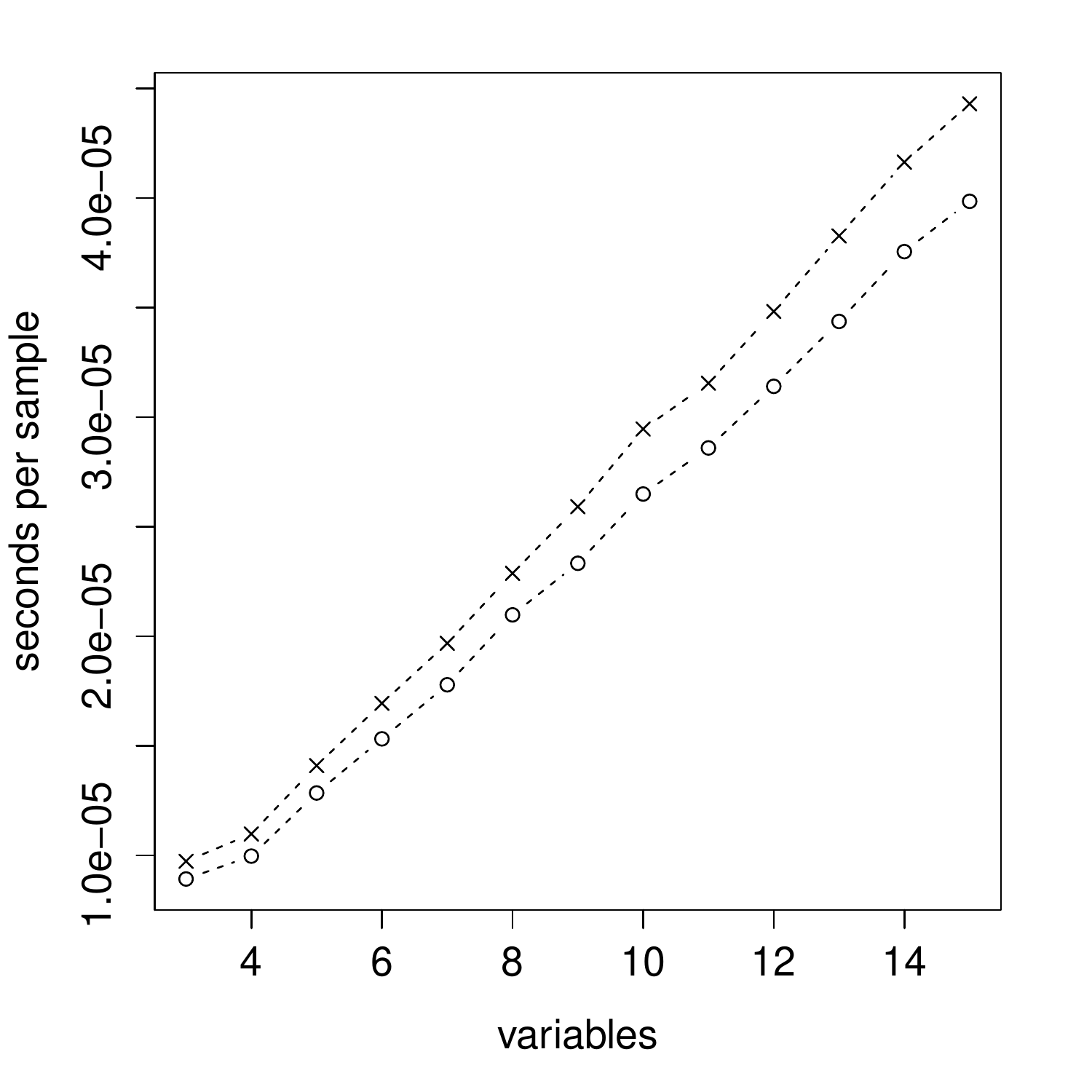}} \vspace{0.5cm} \\
	\ \ Trees & \ \ \ Grids \vspace{-0.15cm} \\ 
	\raisebox{-.5\height}{\includegraphics[scale=0.3]{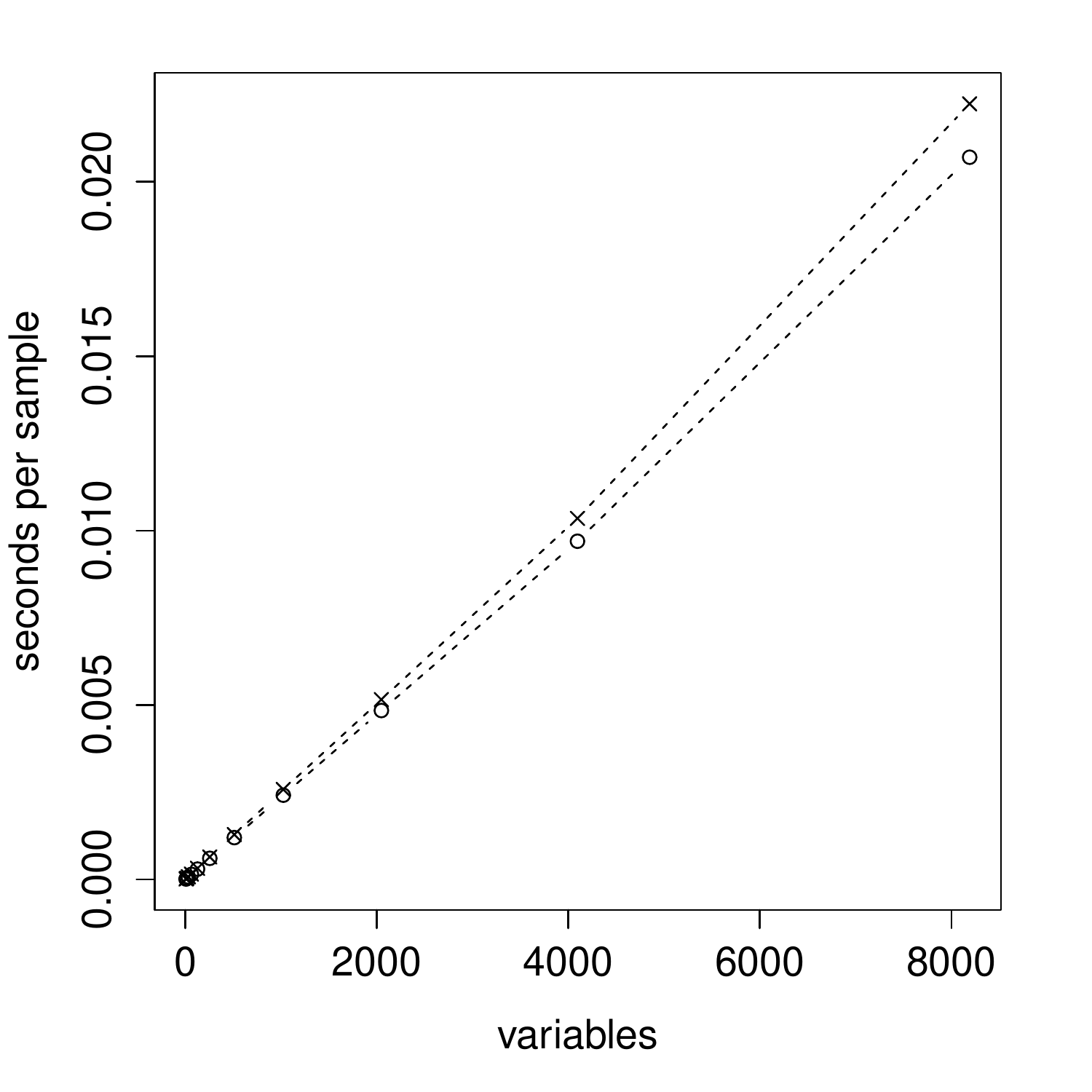}} & \raisebox{-.5\height}{\includegraphics[scale=0.3]{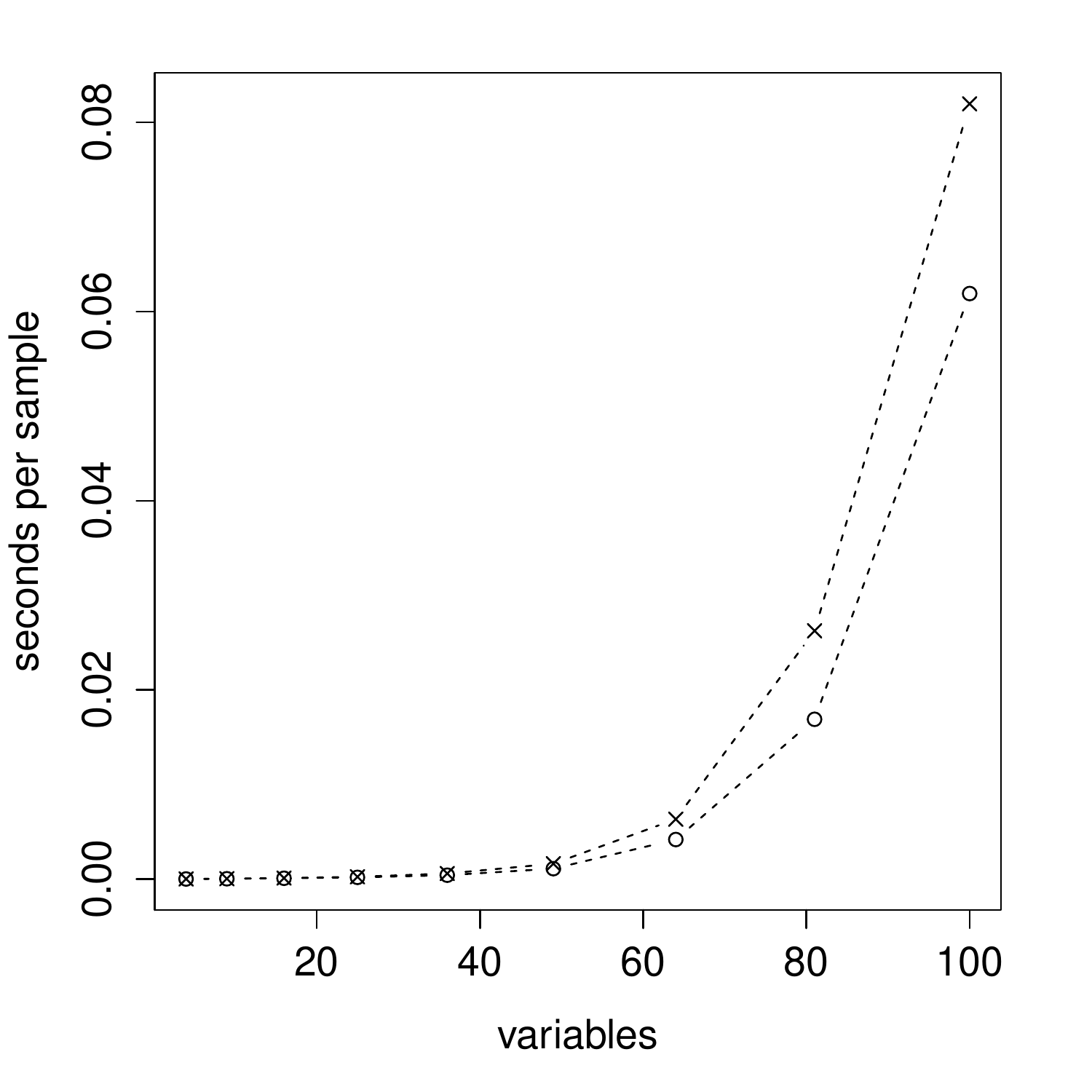}} \\
	\end{tabular}
	\caption[Speed of calculating copula density in archetypal models with normal copulae.]{Speed of calculating copula density in archetypal models with normal copulae. Crosses indicate message passing in log-space, circles in normal space.}
	\label{fig:infSpeeds}
\end{figure}

\section{Inference}
We ran the following experiment to investigate how the speed of inference varies over the archetypal models and the copula type.
\begin{experiment}\label{exp:inference}
	The models used were: chains of size $n=2,\ldots,15$, loops of size $n=3,\ldots,15$, trees of size $n=2,\ldots,13$, and grids of size $n=2,\ldots,10$ with normal and Clayton copulae. For each model, the parameters were randomized and a number of samples drawn. The time taken to calculate the copula density in both normal and log-space, which requires passing messages over the full clique tree, for all samples was recorded, and this process repeated thirty times. (Inference was performed over different values to reduce cache effects.)
		
	The timings were used to calculate the average time to calculate the copula density for a single sample, an indicator of the speed of inference. The number of samples used was varied for the model and copula types. That is, it had to be chosen to be large enough to avoid the error due to the finite precision of the machine clock (when it only took a fraction of a section to evaluate all samples), and small enough so that the experiment ran in a feasible time.
\end{experiment}
The results for normal models are presented in Figure \ref{fig:infSpeeds}; the corresponding Clayton models are very similar and thus omitted. The figures confirm the theoretical result that the running time of inference is exponential in the size of the treewidth, and, holding the treewidth constant, is linear in the model size. The speed of inference in chains, loops, and trees, which have fixed treewidth, appears to the be linear in the model size, whereas the speed of inference in grids appears to increase exponentially as the model size and hence the treewidth increases.

For the same model, the type of copula only appears to affect the scale and not the shape of the relationship between the speed of inference and the model size. In our implementation, for models with small treewidth, such the chains, loops, and trees, inference in Clayton copula models is roughly 2--3 times faster than the equivalent normal copula model. As the treewidth increases the difference in speed between the two copula types vanishes, as the time taken to incorporate the additional messages dominates the time taken in evaluating the copula factors. Inference in log-space is roughly 5--10\% slower relative to inference in normal space.

\section{Standard learning}
We performed the following experiment to examine how the performance of learning varies over models and copulae, and which learning method is preferred. 
\begin{experiment}\label{exp:learning}
	The models used were: chains of size $n=2,\ldots,15$, and loops of size $n=3,\ldots,15$ with normal and Clayton copulae. For each model, the parameters were randomized and a sample of size $100$, $1000$, or $10000$ was drawn. For the same sample, the three learning methods---gradient descent, L-BFGS with restart, and L-BFGS with barrier---were performed from different random starts with $\epsilon=10^{-8}$. This was repeated 100 times for each combination of model and sample size.
	
	A record was made of the mean squared error (MSE) of the learnt parameters from the true parameter, and how many iterations required until convergence. All learning algorithms were restricted to 100 iterations to expedite the experiment. The method of randomizing the parameters was: for the normal copula, the parameters were uniformly sampled from $(0,1)$; for the Clayton copula, Kendall's $\tau$ was uniformly sampled from $(0,0.5)$, which was transformed to the bivariate copula parameter using the relationship \cite{Nelson2010},
	\begin{align*}
		\theta &= \frac{2\tau}{1-\tau}.
	\end{align*}
	The range of Kendall's $\tau$ was restricted because it was found to reduce the number of iterations taken until convergence (it effects the shape of the sublevel sets).
\end{experiment}

\begin{figure}[p]
  \centering\begin{tabular}{cccc}
	& \multicolumn{3}{c}{\large\ \ \ \ \ Normal Loops, Errors} \vspace{0.25cm} \\
	 Samples & \ \ \ \ L-BFGS with Restart & \ \ \ \ L-BFGS with Barrier & \ \ \ \ \ \ Gradient Descent \vspace{0.15cm} \\ 
	100 & \raisebox{-.5\height}{\includegraphics[scale=0.27]{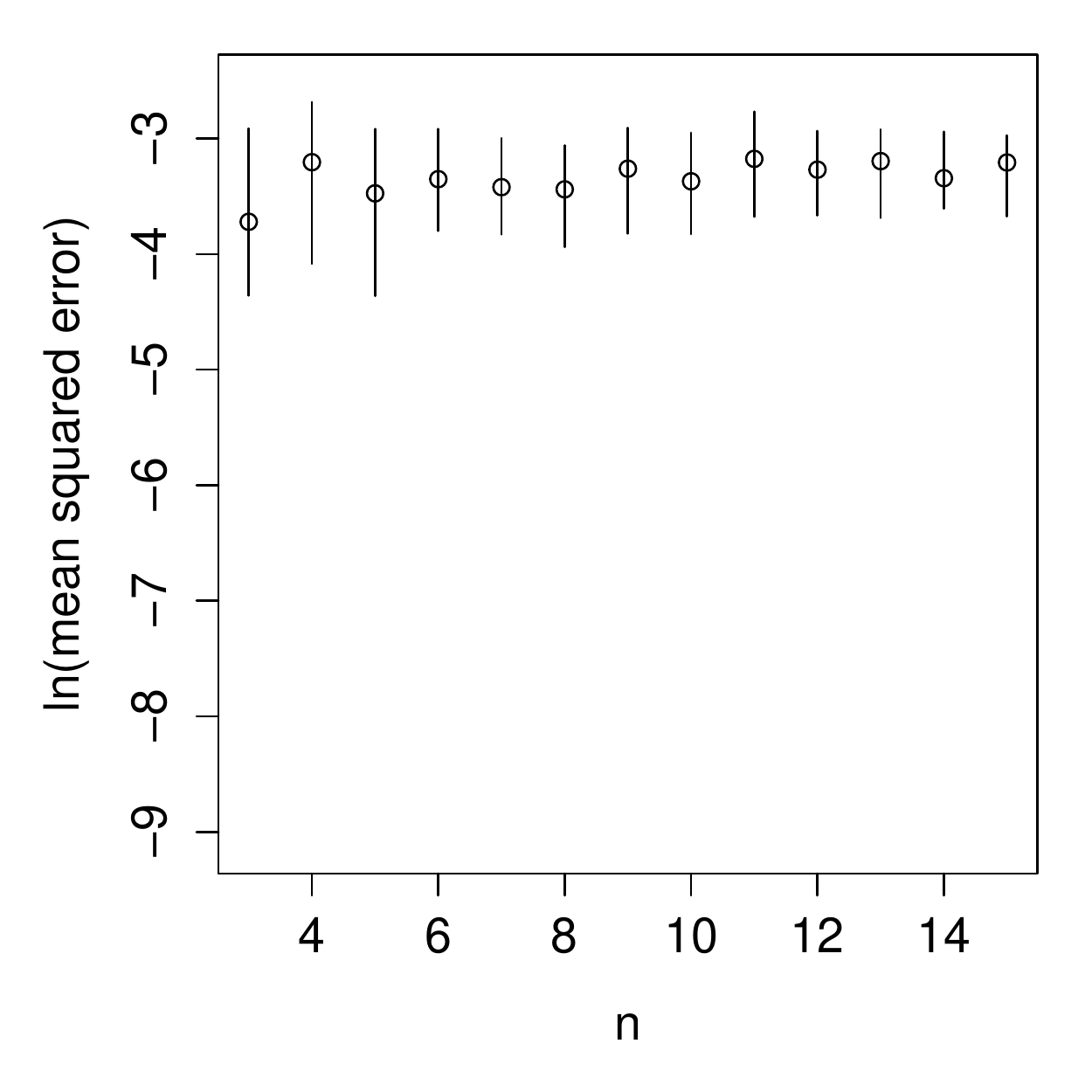}} & \raisebox{-.5\height}{\includegraphics[scale=0.27]{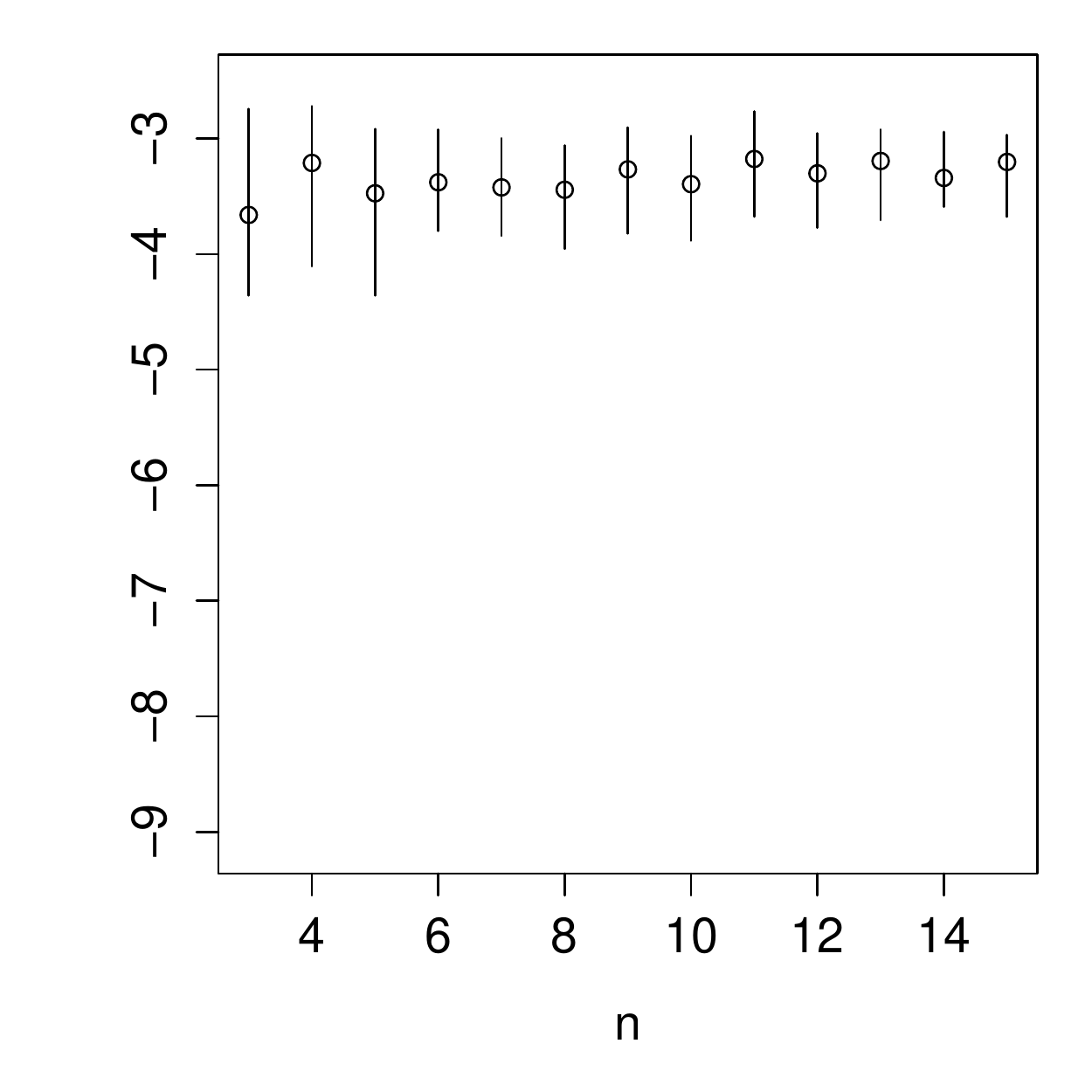}} & \raisebox{-.5\height}{\includegraphics[scale=0.27]{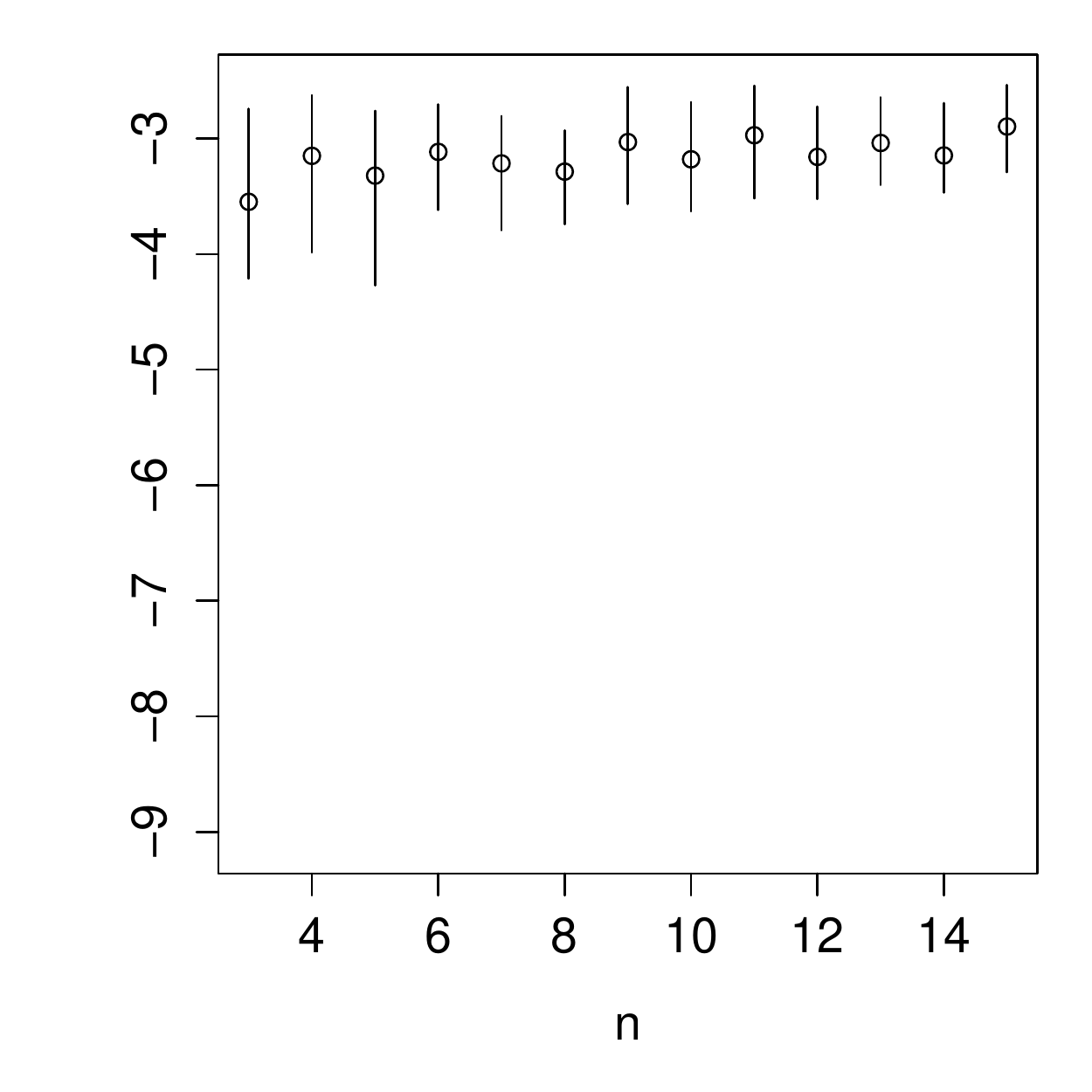}} \\
	1000 & \raisebox{-.5\height}{\includegraphics[scale=0.27]{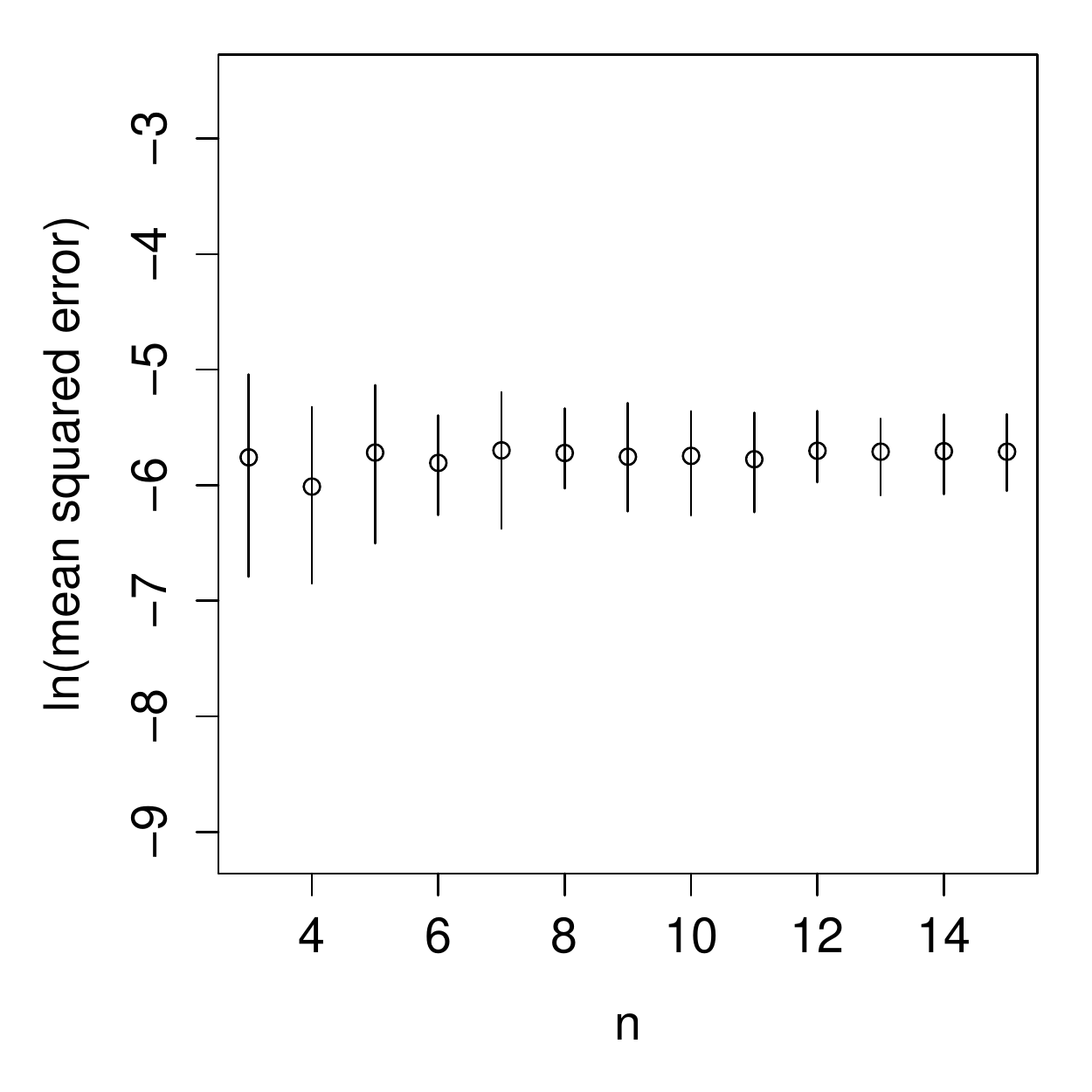}} & \raisebox{-.5\height}{\includegraphics[scale=0.27]{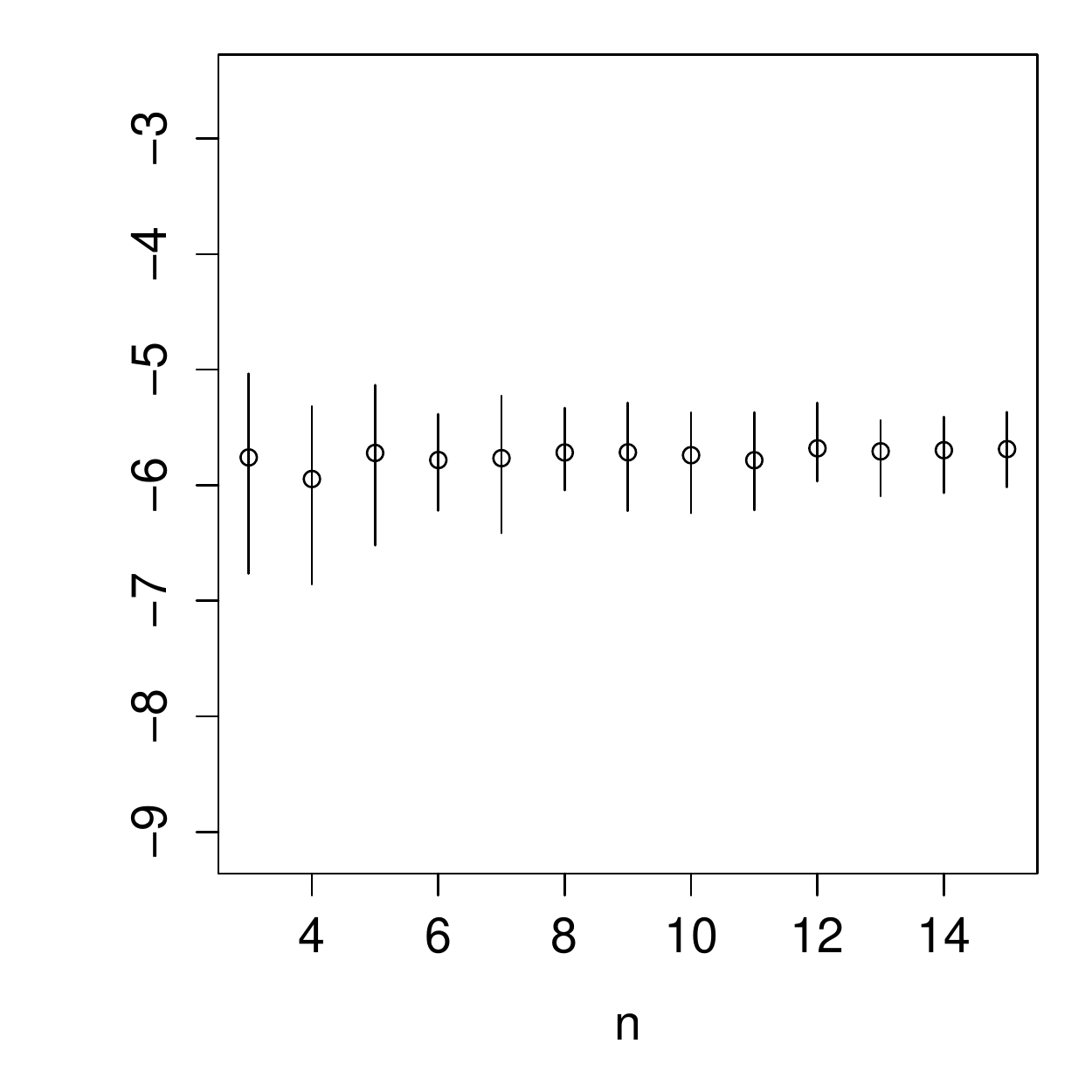}} & \raisebox{-.5\height}{\includegraphics[scale=0.27]{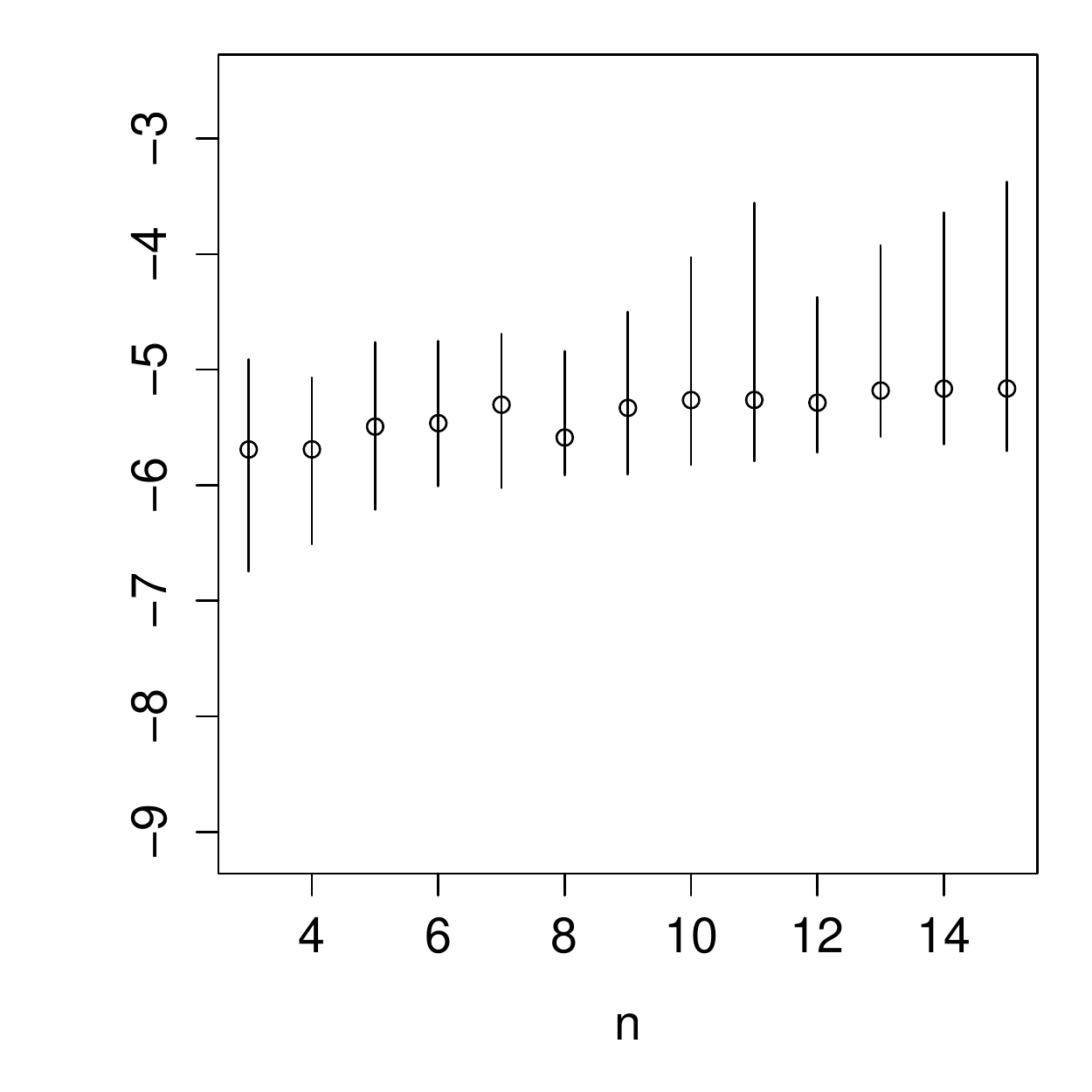}} \\
	10000 & \raisebox{-.5\height}{\includegraphics[scale=0.27]{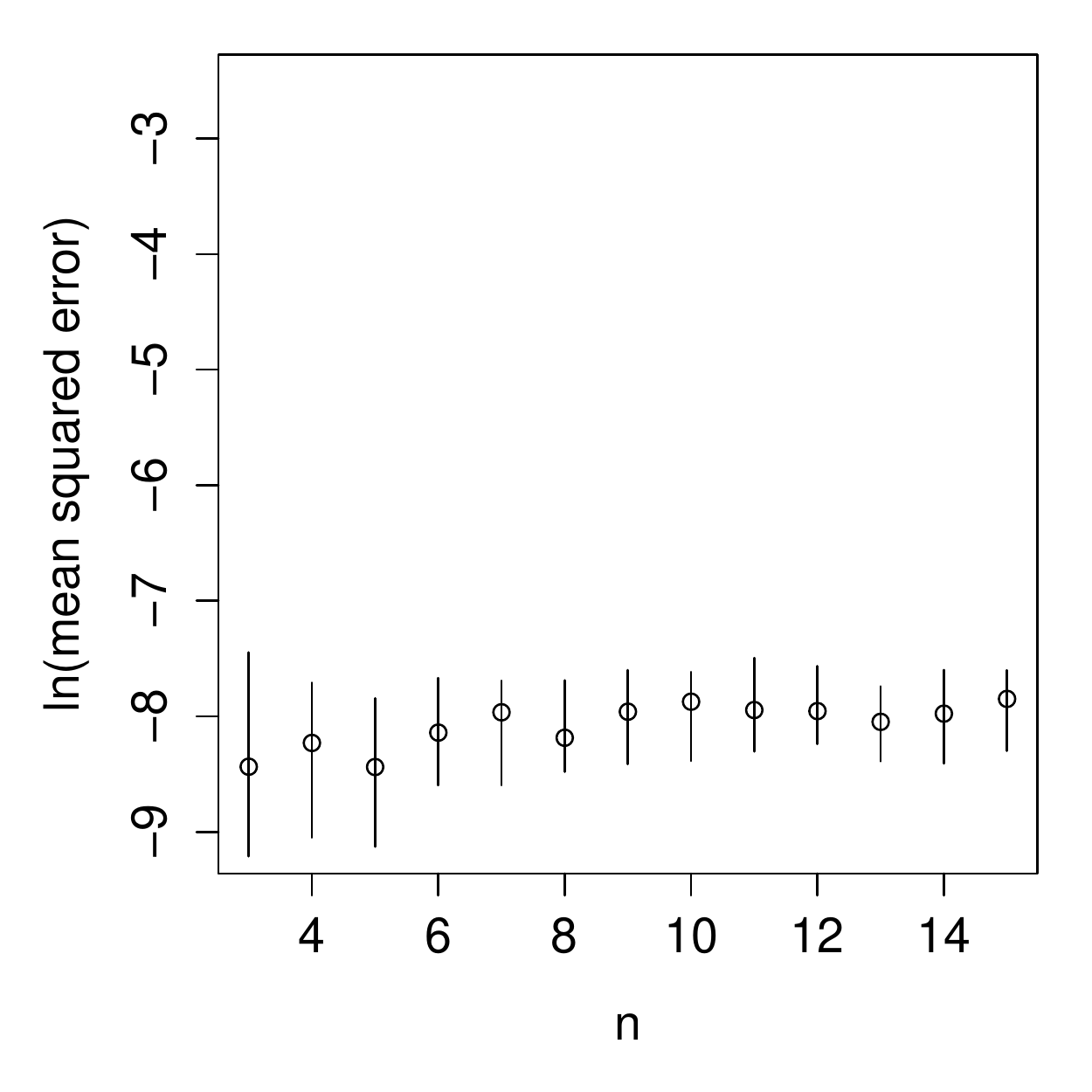}} & \raisebox{-.5\height}{\includegraphics[scale=0.27]{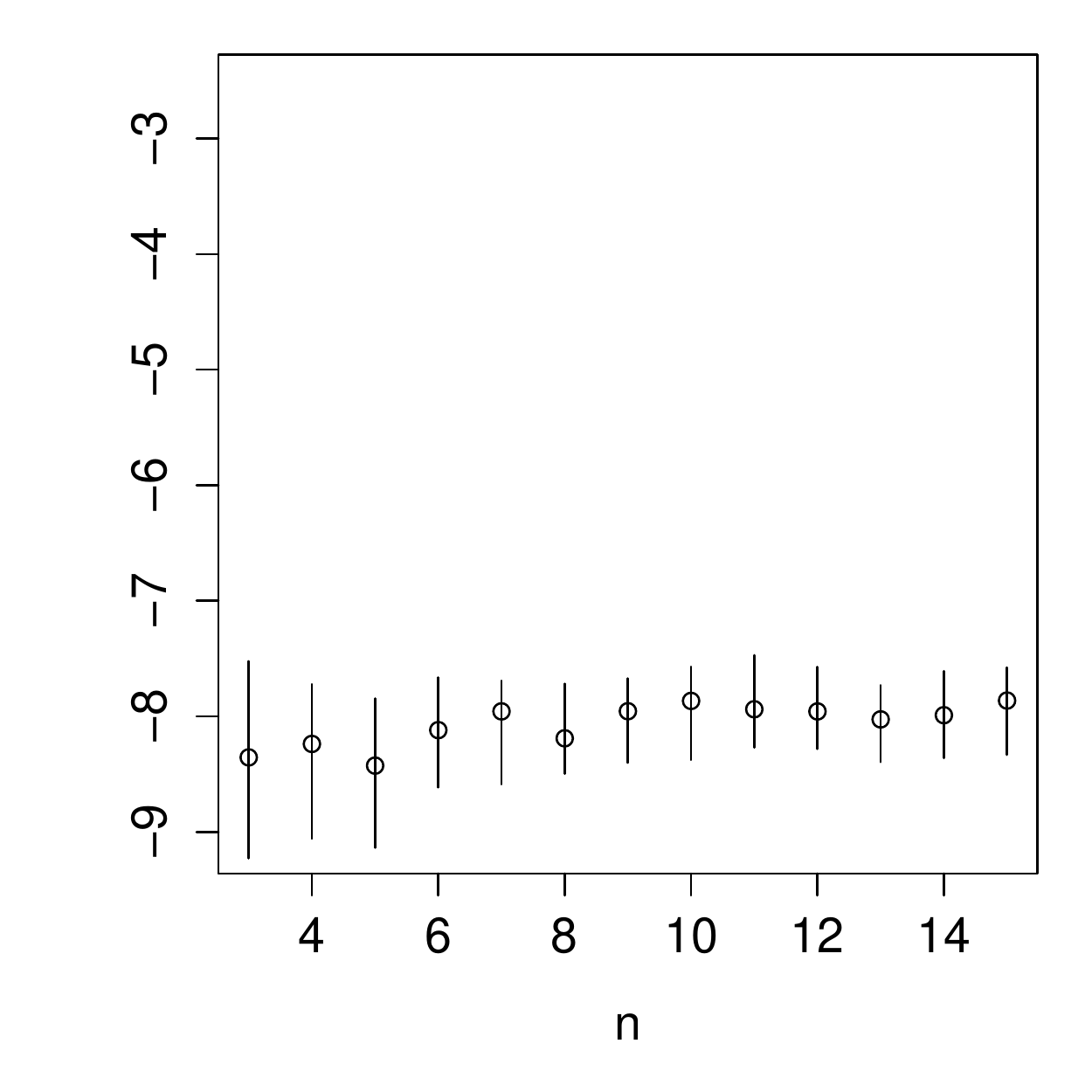}} & \raisebox{-.5\height}{\includegraphics[scale=0.27]{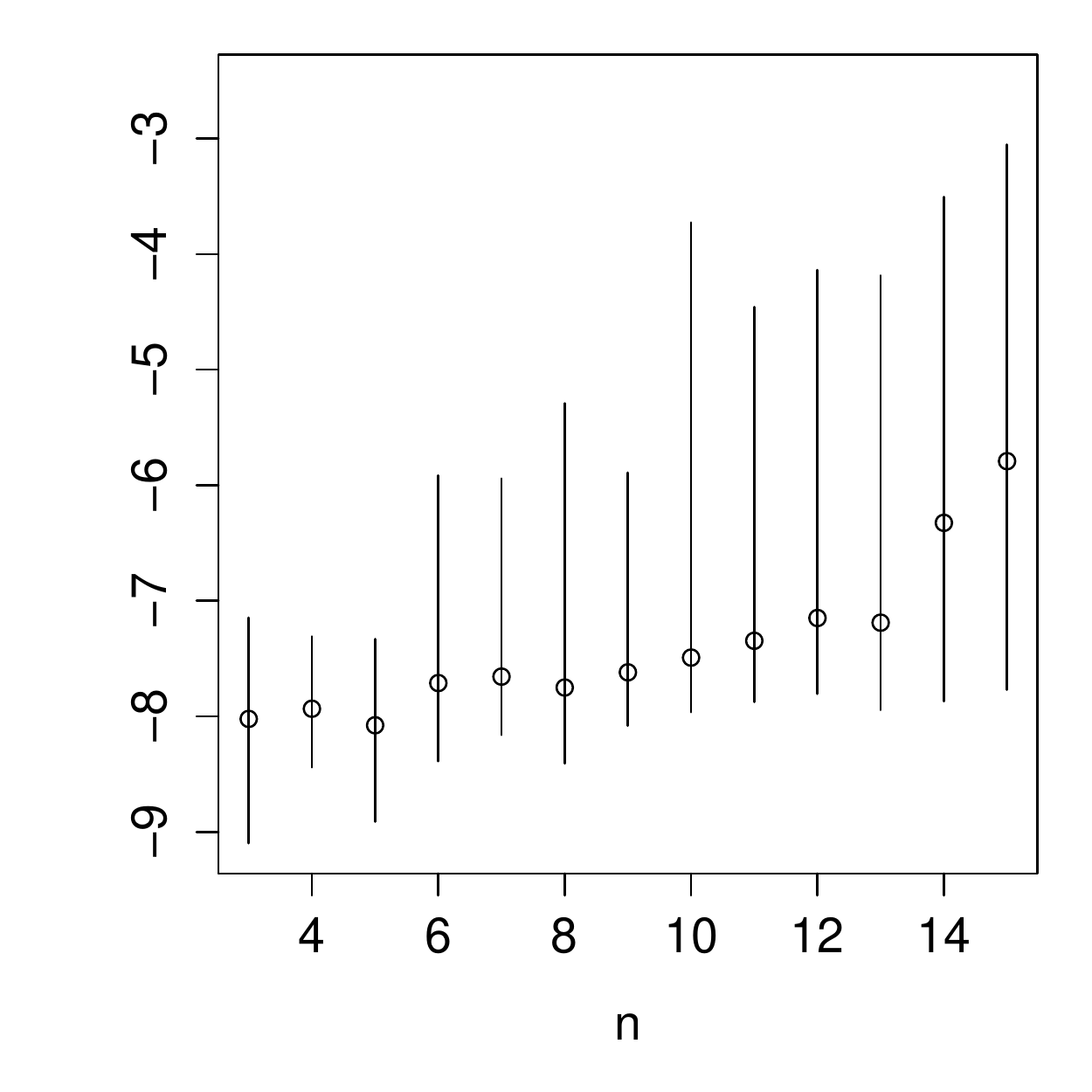}} \vspace{0.5cm} \\
	& \multicolumn{3}{c}{\large\ \ \ \ \ Normal Loops, Iterations} \vspace{0.25cm} \\
	 & \ \ \ \ L-BFGS with Restart & \ \ \ \ L-BFGS with Barrier & \ \ \ \ \ \ Gradient Descent \vspace{0.15cm} \\ 
	100 & \raisebox{-.5\height}{\includegraphics[scale=0.27]{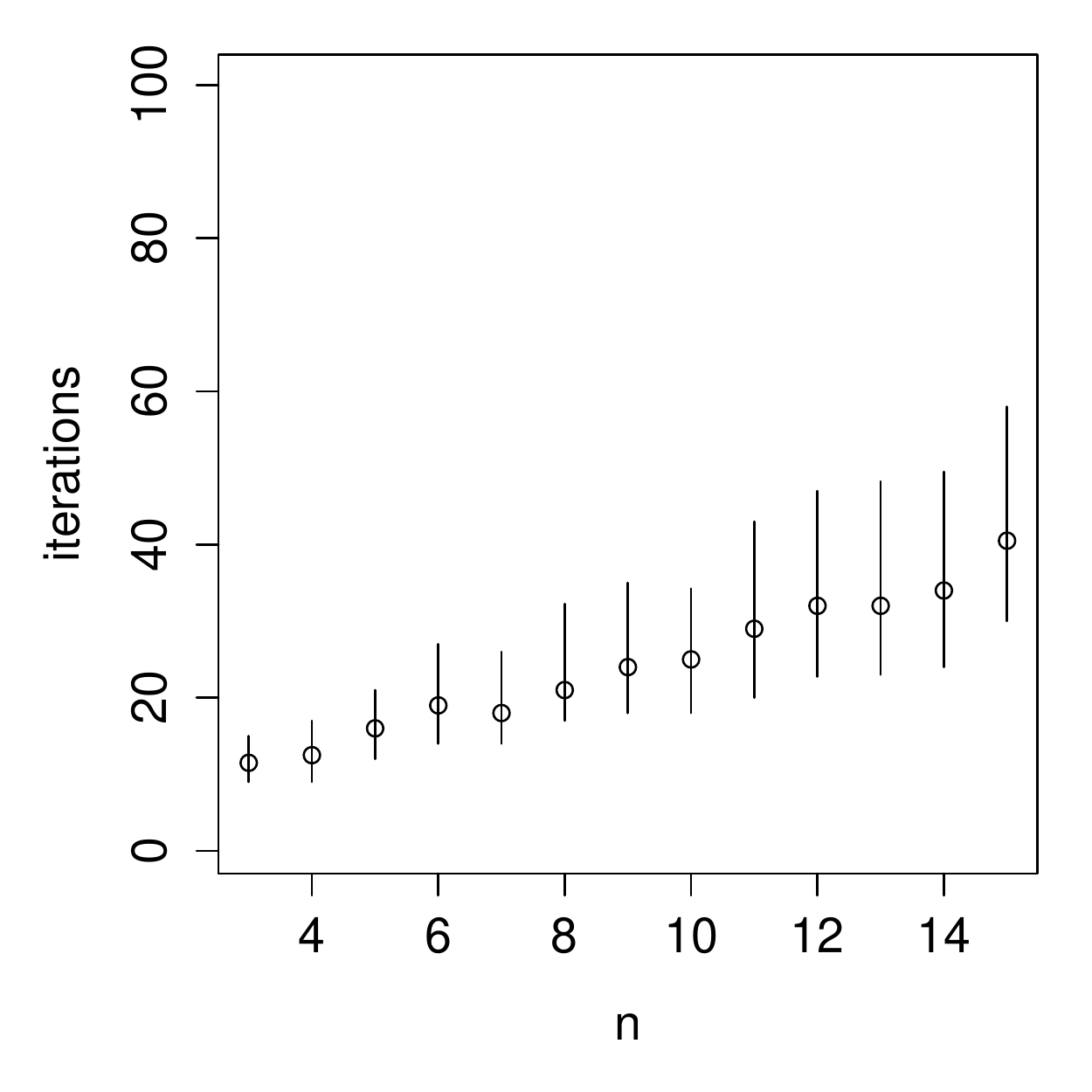}} & \raisebox{-.5\height}{\includegraphics[scale=0.27]{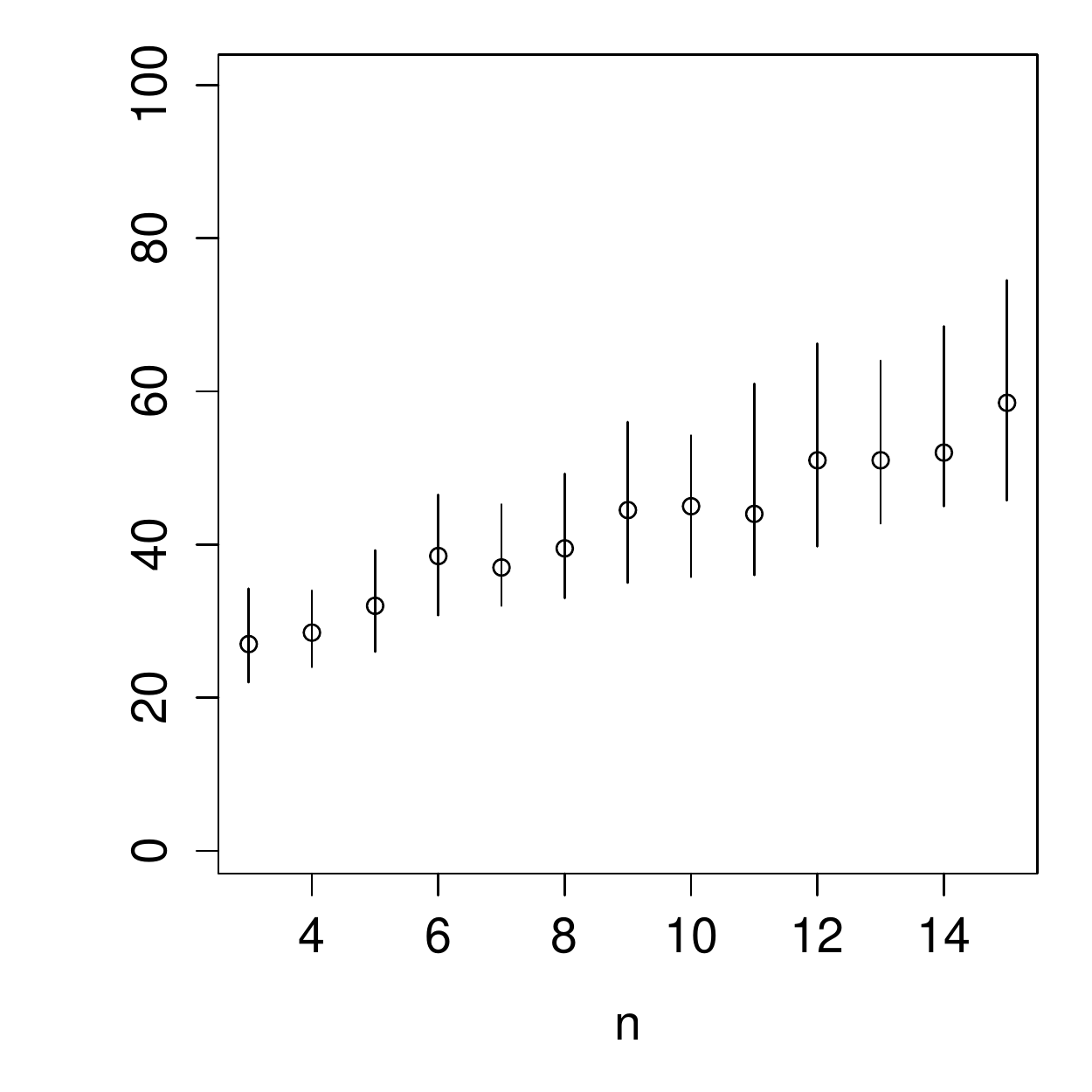}} & \raisebox{-.5\height}{\includegraphics[scale=0.27]{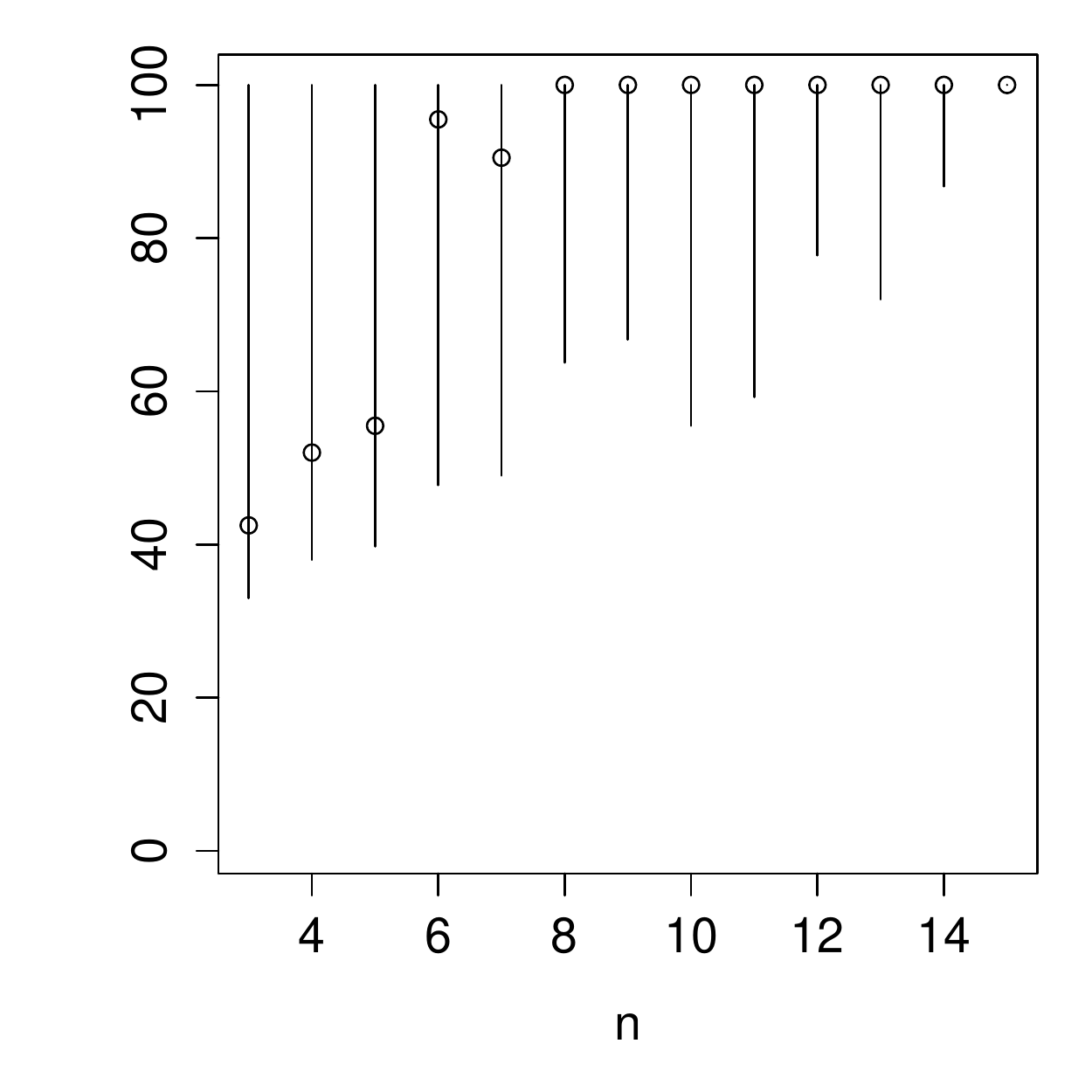}} \\
	1000 & \raisebox{-.5\height}{\includegraphics[scale=0.27]{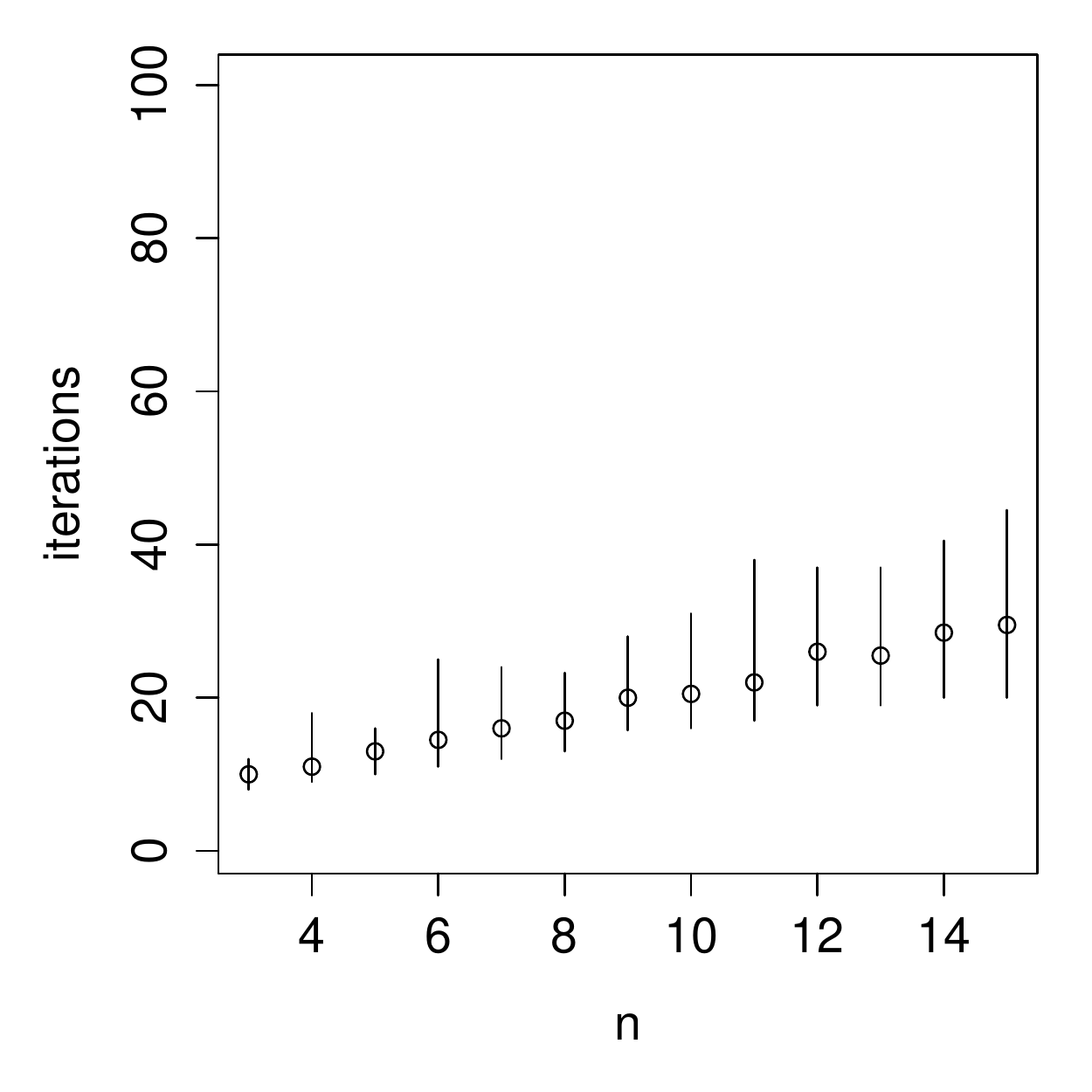}} & \raisebox{-.5\height}{\includegraphics[scale=0.27]{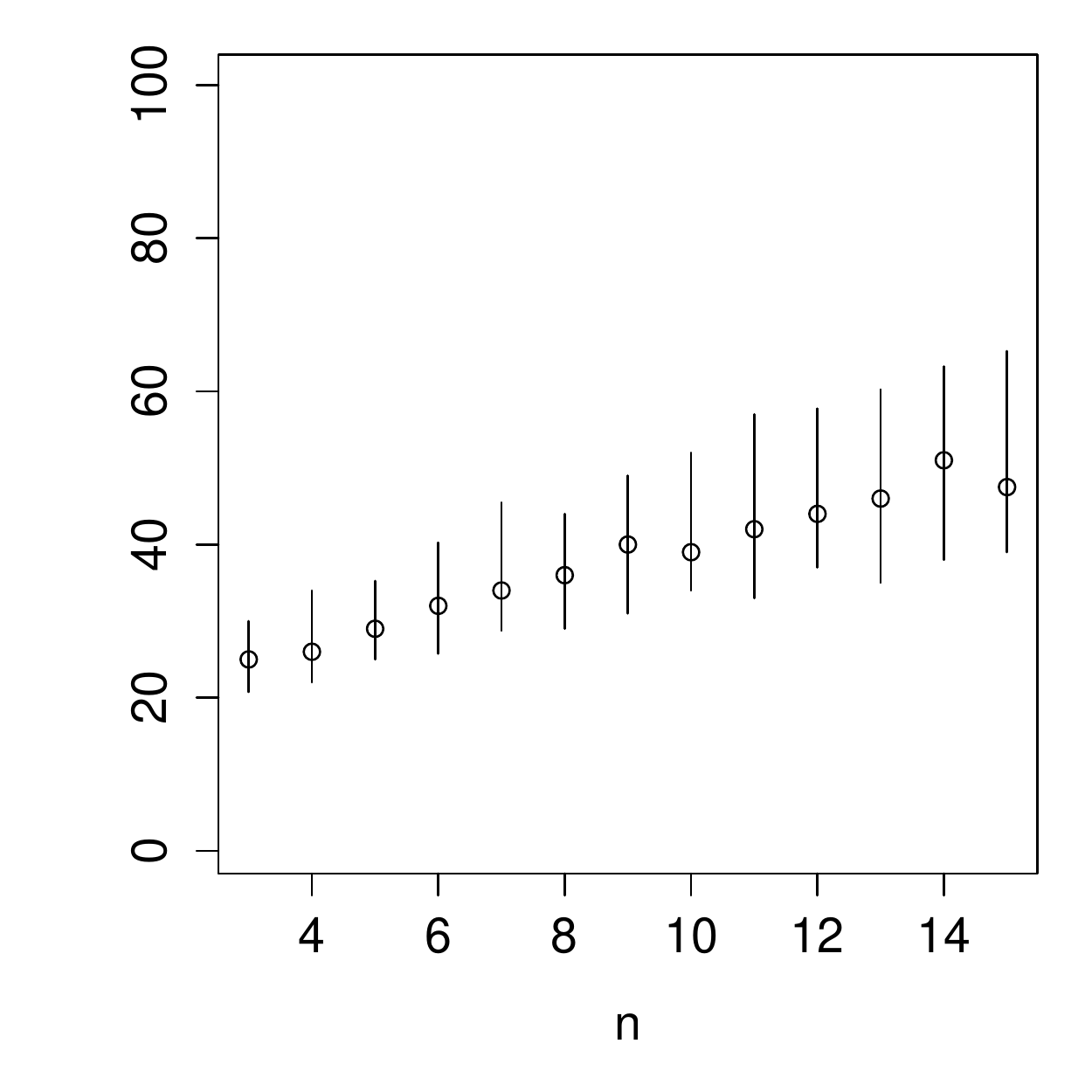}} & \raisebox{-.5\height}{\includegraphics[scale=0.27]{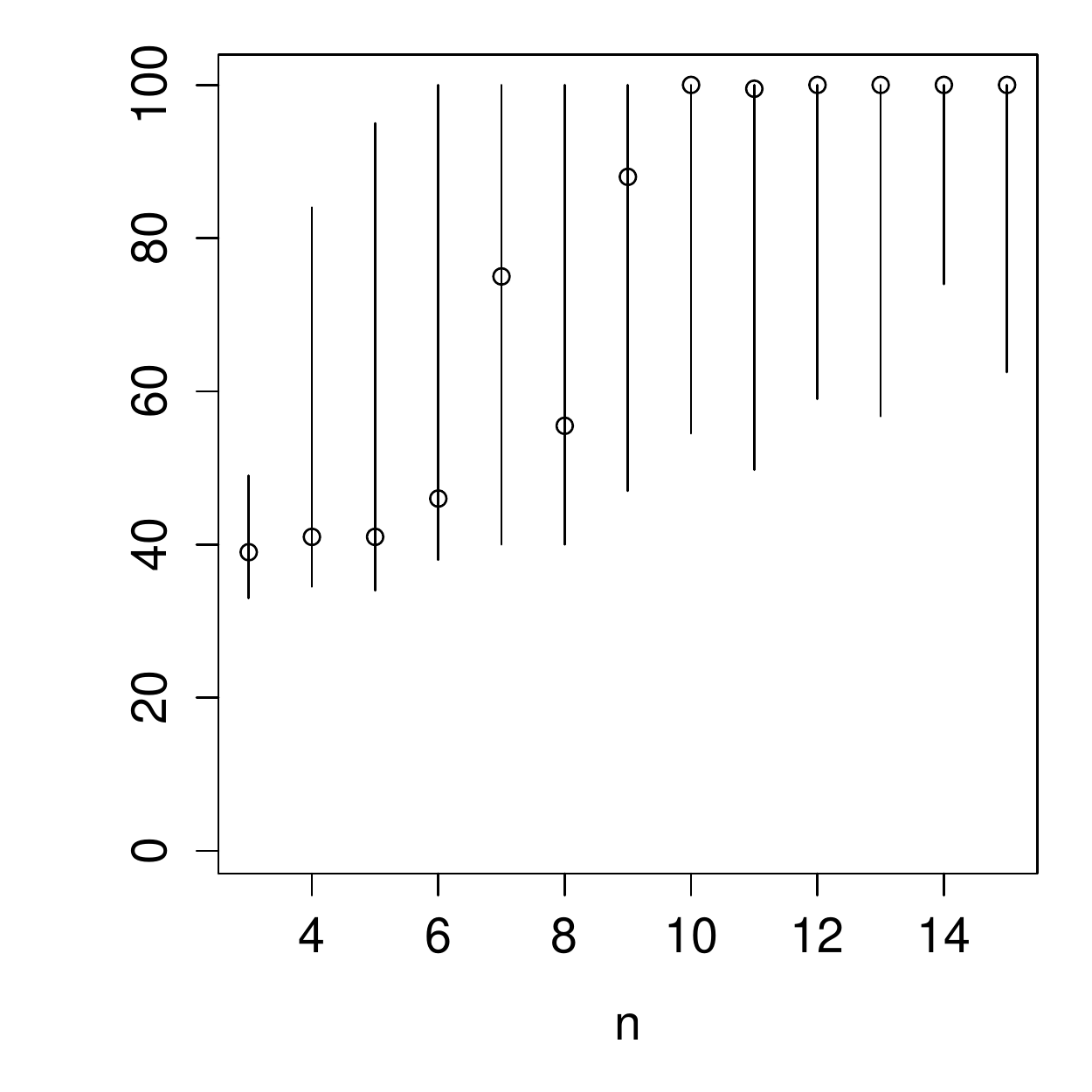}} \\
	10000 & \raisebox{-.5\height}{\includegraphics[scale=0.27]{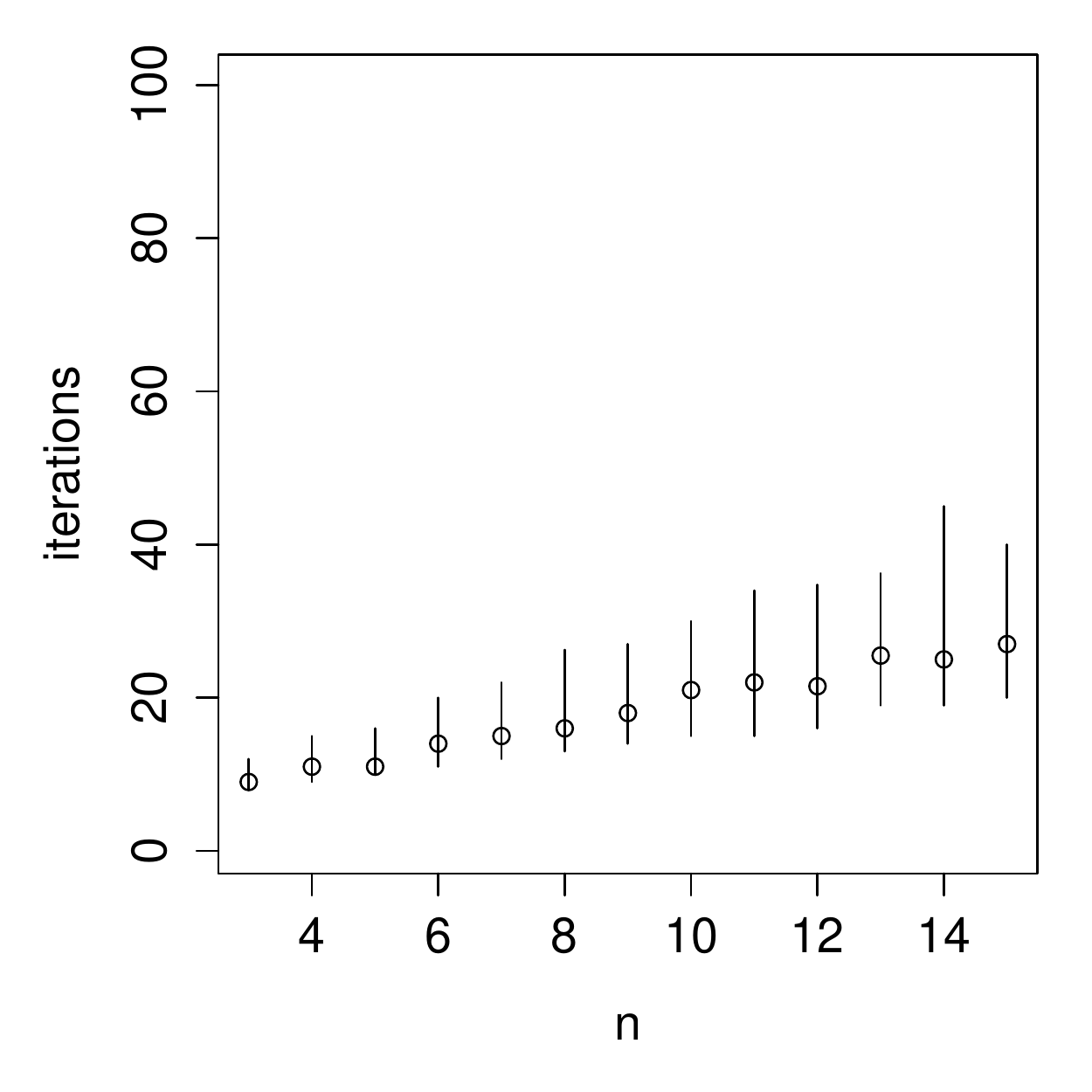}} & \raisebox{-.5\height}{\includegraphics[scale=0.27]{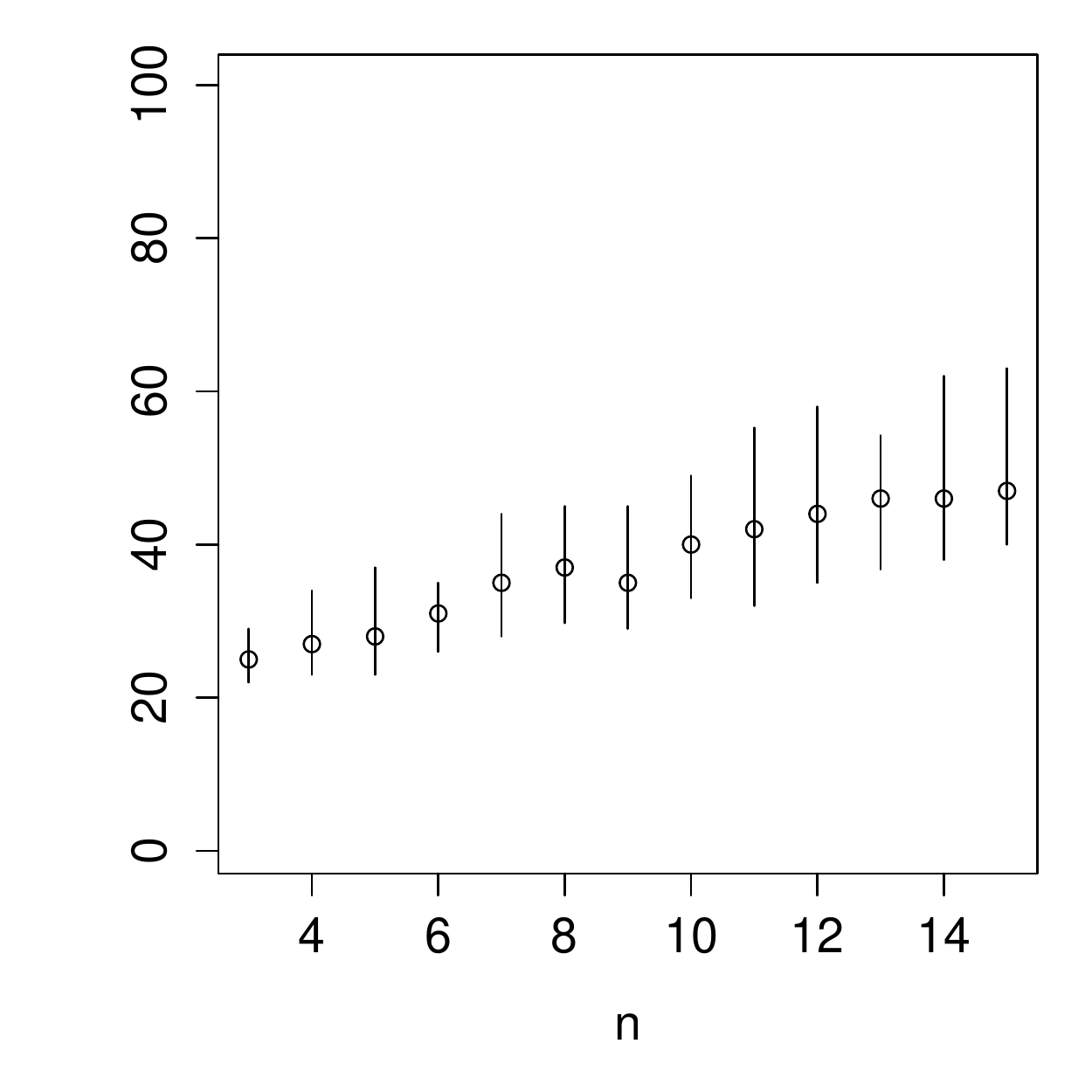}} & \raisebox{-.5\height}{\includegraphics[scale=0.27]{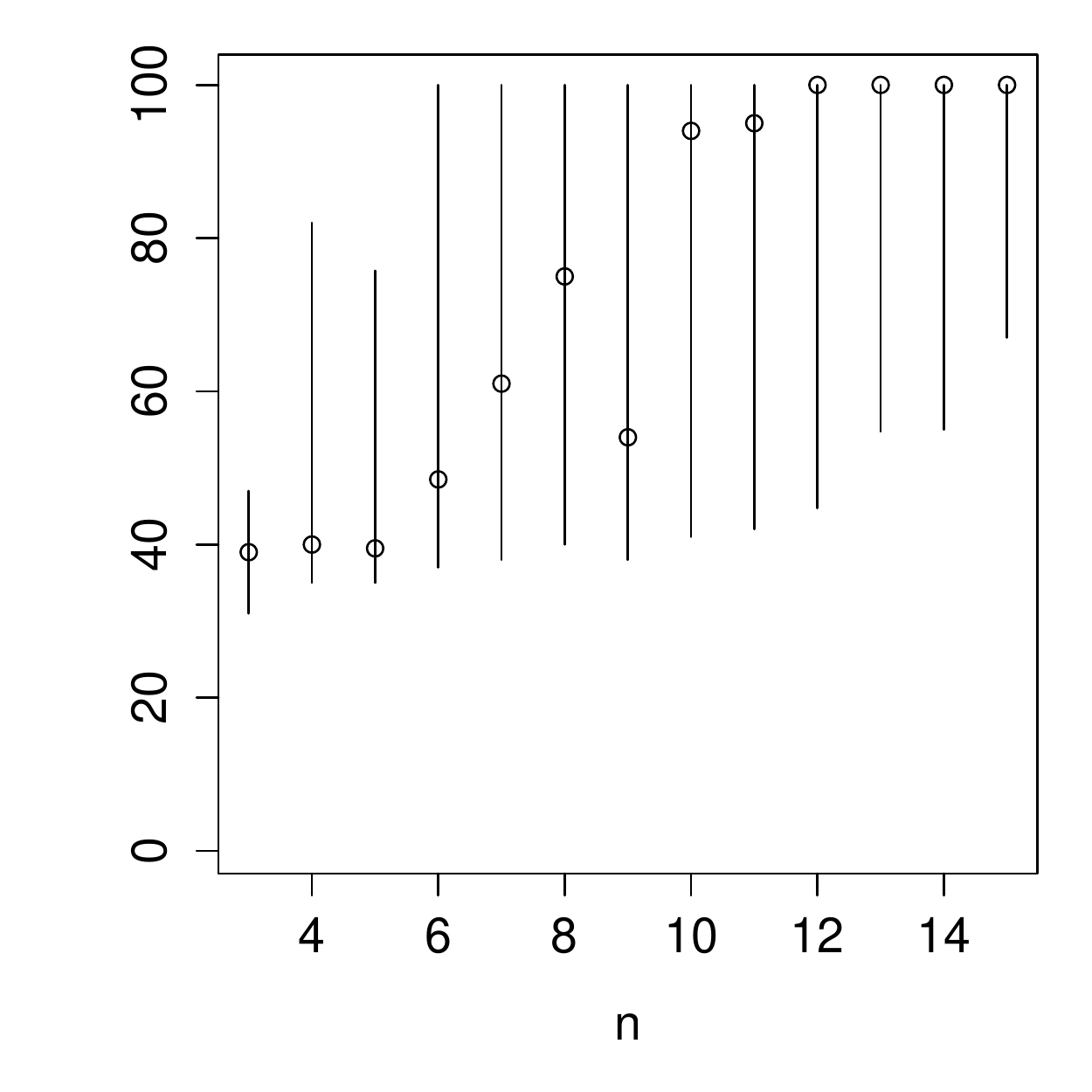}}
	\end{tabular}
	\caption[Comparing MSEs and number of iterations of learning methods for normal loops.]{Comparing MSEs on log scale and number of iterations of learning methods for normal loops. The circle indicates the median, and the error bars indicate the first and third quartiles---further graphs are to be read likewise.}
	\label{fig:learningMse}
\end{figure}

The results for normal loops are presented in Figure \ref{fig:learningMse}; the other models exhibit similar patterns and are thus omitted. The distributions for the MSEs were found to be heavy tailed, and so we chose to use the median to indicate the ``typical'' value. There are several sources of error contributing to the MSE, for example, that the marginals are learnt separately from the copula parameters, the finiteness of the sample, the imperfect nature of descent methods, and the true parameters, which effect the shape of the sublevel sets.

As the number of samples increases, the median MSE and its variation decrease, providing evidence that learning is consistent. The median MSE tends to increase as the model size increases, due to the fact that more parameters are being estimated with the same number of samples. The spike in MSE for the larger loops learnt with gradient descent is caused by the restriction on the number of iterations. Were we to let the algorithm run to convergence, this anomaly would disappear. Comparing the MSEs between model types and copulae, the MSEs for equivalent normal and Clayton models are almost identical, and likewise for chains and loops with the same number of parameters.

Since the three algorithms have similar MSEs, the choice between the three should be decided by which converges the fastest. A single iteration is comparable between the three algorithms because the time taken to evaluate the objective and gradient dominates the time taken to construct the search direction. There is a slight decrease, across all three methods, in the number of iterations until convergence as the number of samples increases. As expected (see the discussion of convergence rates in \S\ref{sec:newtonMethods}), the L-BFGS methods converge faster than gradient descent, and the simpler restart method appears to converge faster than the barrier method. Therefore, we conclude that L-BFGS with restart is the preferred learning method.

\section{MCAR learning}
We performed the following experiment to investigate how the efficiency of MCAR learning varies as the proportion of missing data is varied, or rather, the robustness of learning with missing data.
\begin{experiment}\label{exp:cmar}
	The models used were chains of size $n=2,\ldots,15$ with normal copulae. For each model, the parameters were randomized and $10000$ samples drawn. A given proportion of the data---$0$, $0.01$, $0.1$, $0.5$ or $0.75$---was erased at random so that many of the samples were incomplete. The parameters were learnt from a random restart by the procedure outlined in \S\ref{sec:cmarLearning} using L-BFGS with restart and $\epsilon=10^{-8}$, and a record made of the MSE of the learnt parameters from the true parameters. This was repeated $100$ times for each combination of model and proportion of missing data.
\end{experiment}
\begin{figure}[t]
  \centering\begin{tabular}{ccc}
		\multicolumn{3}{c}{\large\ \ \ \ \ \ Normal Chains, Errors} \vspace{0.25cm} \\
	  \ \ \ \ \ \ \ \ 0\% & \ \ \ \ \ \ \ \ 25\% & \ \ \ \ \ \ \ \ 50\%  \\
	 \raisebox{-.5\height}{\includegraphics[scale=0.3]{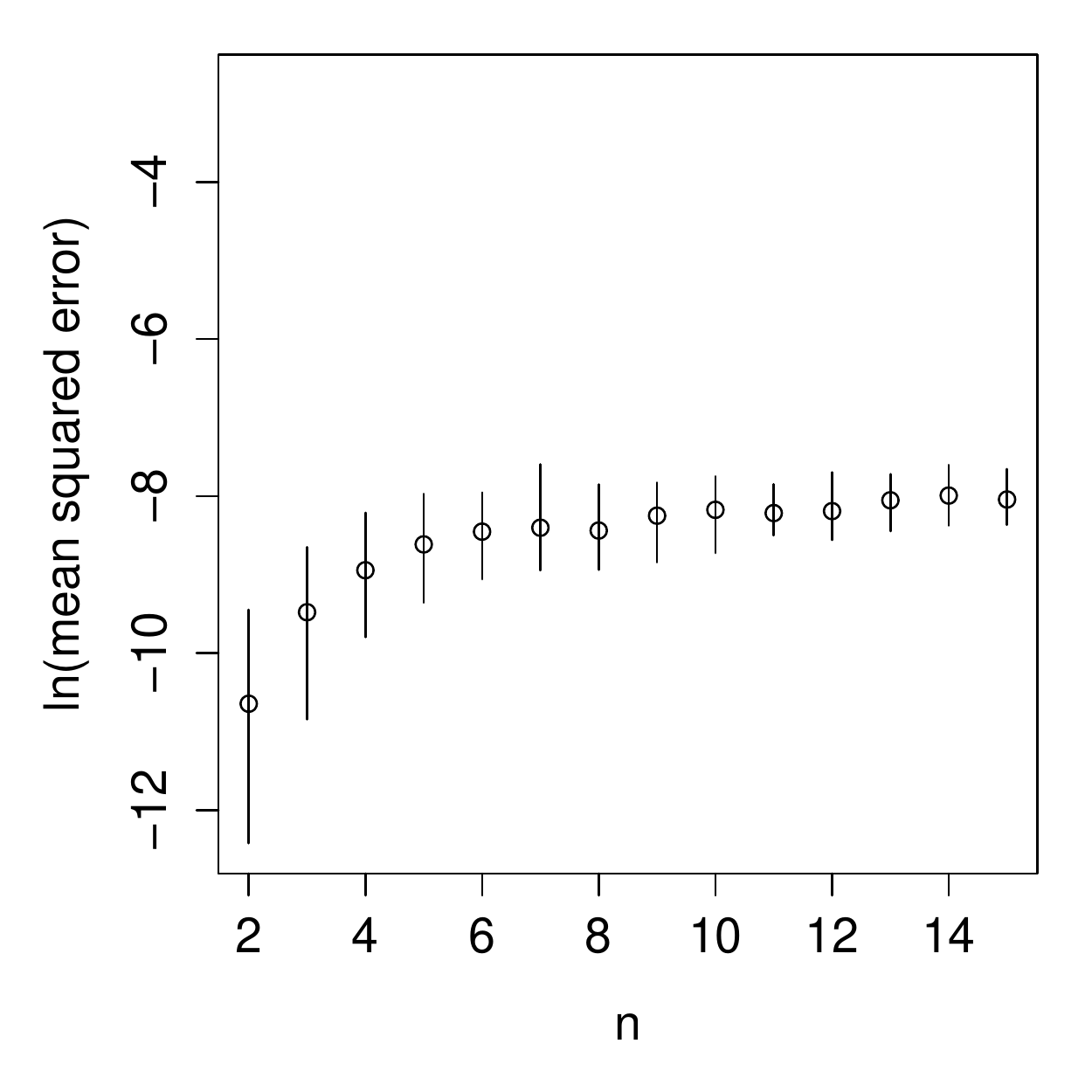}} & \raisebox{-.5\height}{\includegraphics[scale=0.3]{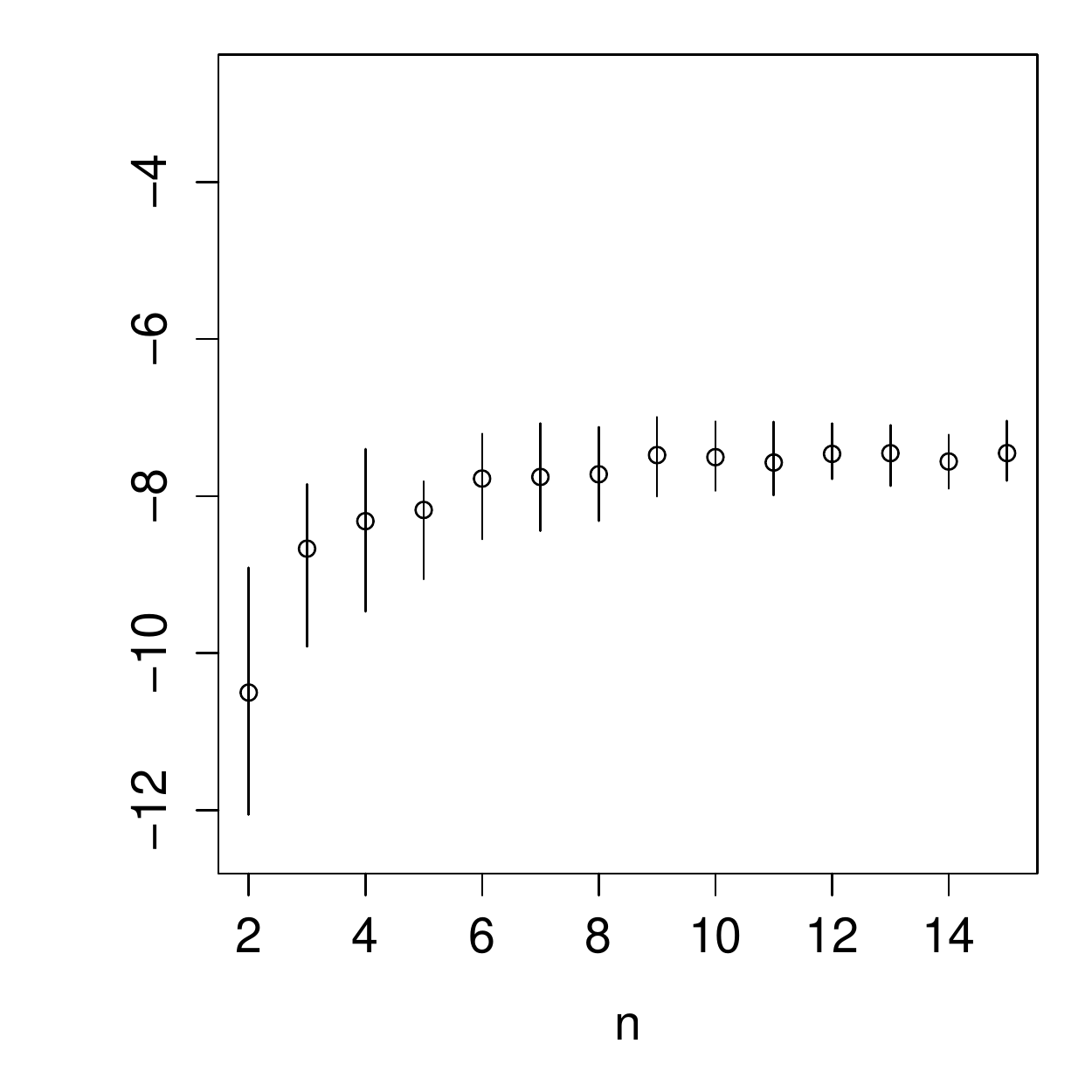}} & \raisebox{-.5\height}{\includegraphics[scale=0.3]{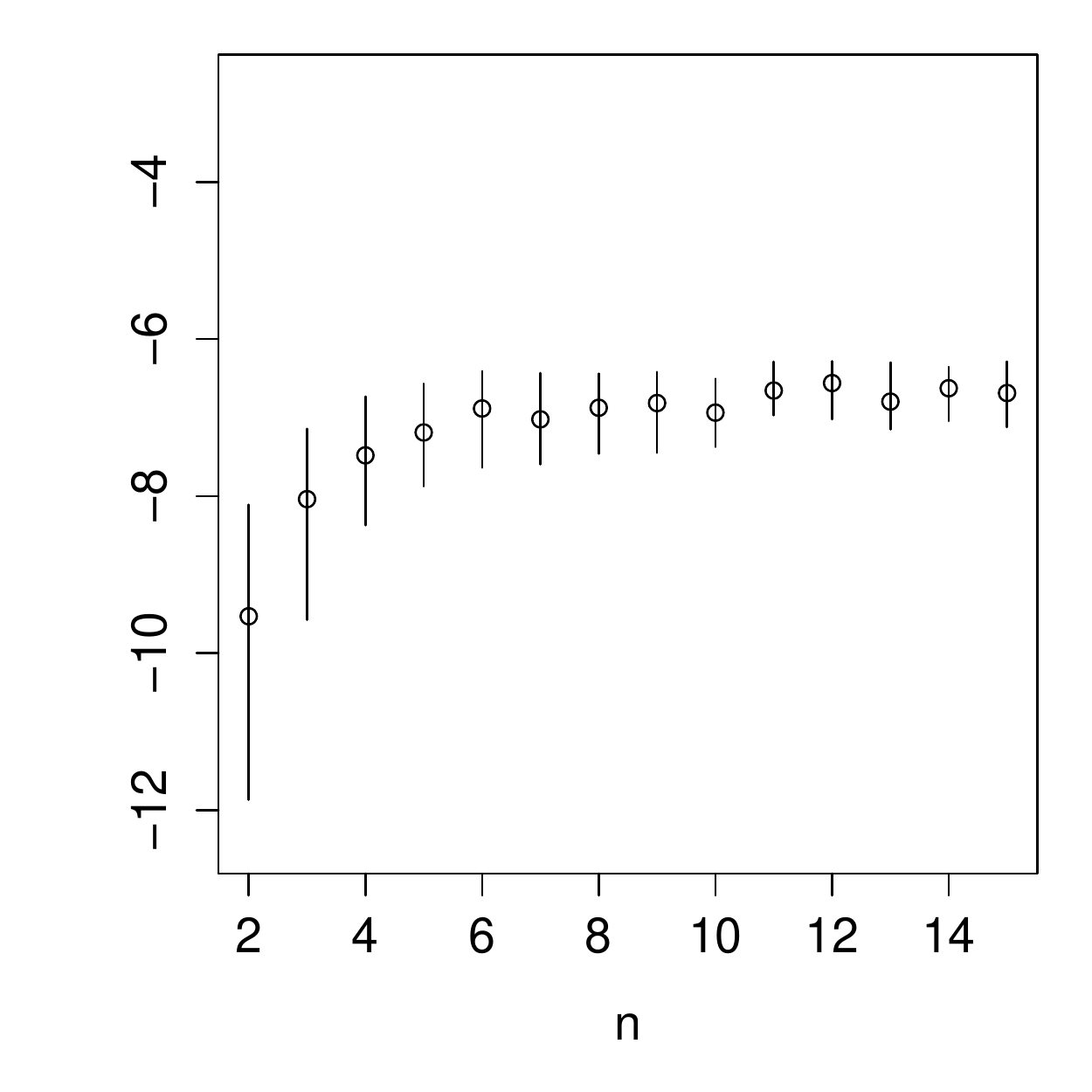}}\vspace{8pt} \\
		\ \ \ \ \ \ \ \ 75\% & \ \ \ \ \ \ \ \ 90\% & \ \ \ \ \ 100 Samples\\
	 \raisebox{-.5\height}{\includegraphics[scale=0.3]{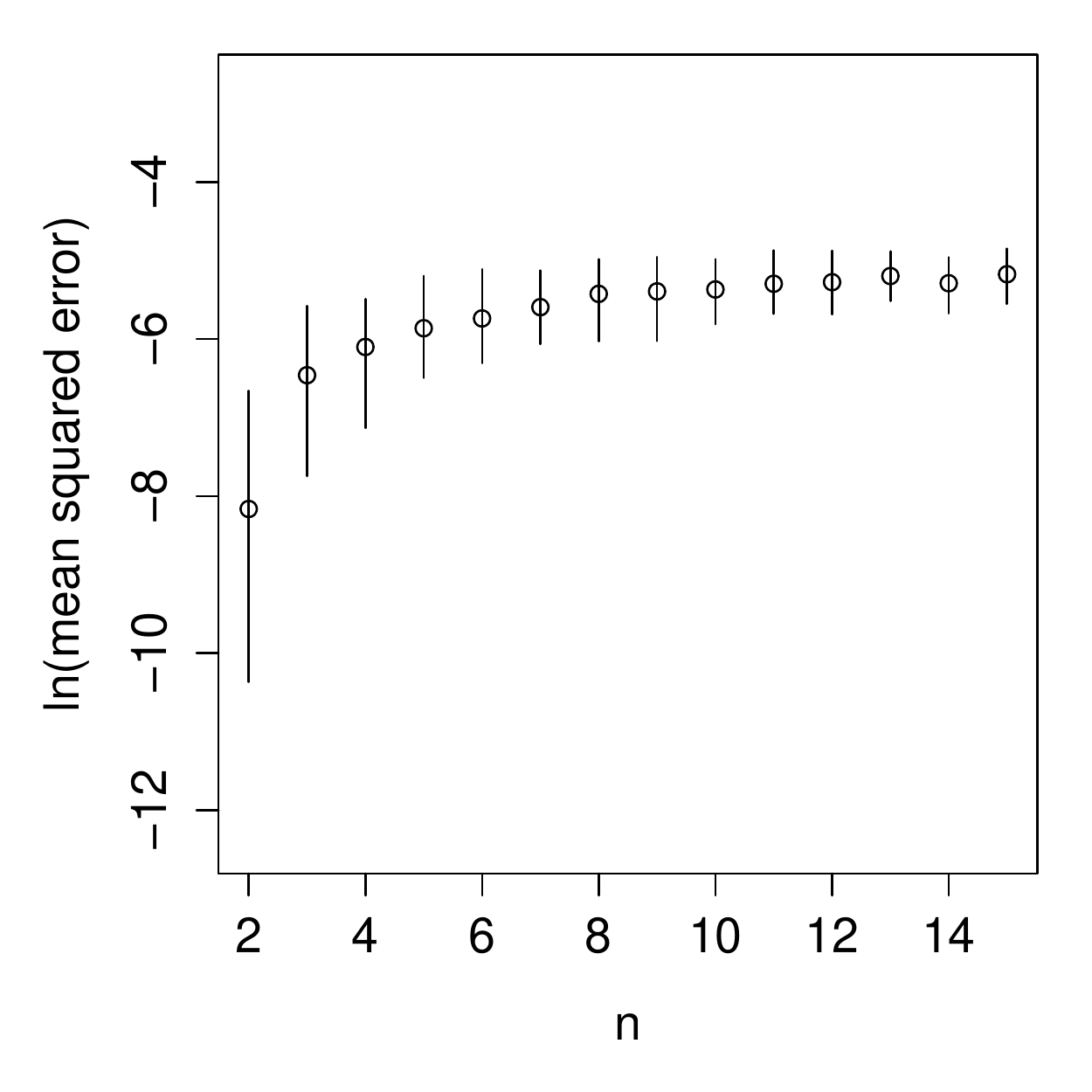}} & \raisebox{-.5\height}{\includegraphics[scale=0.3]{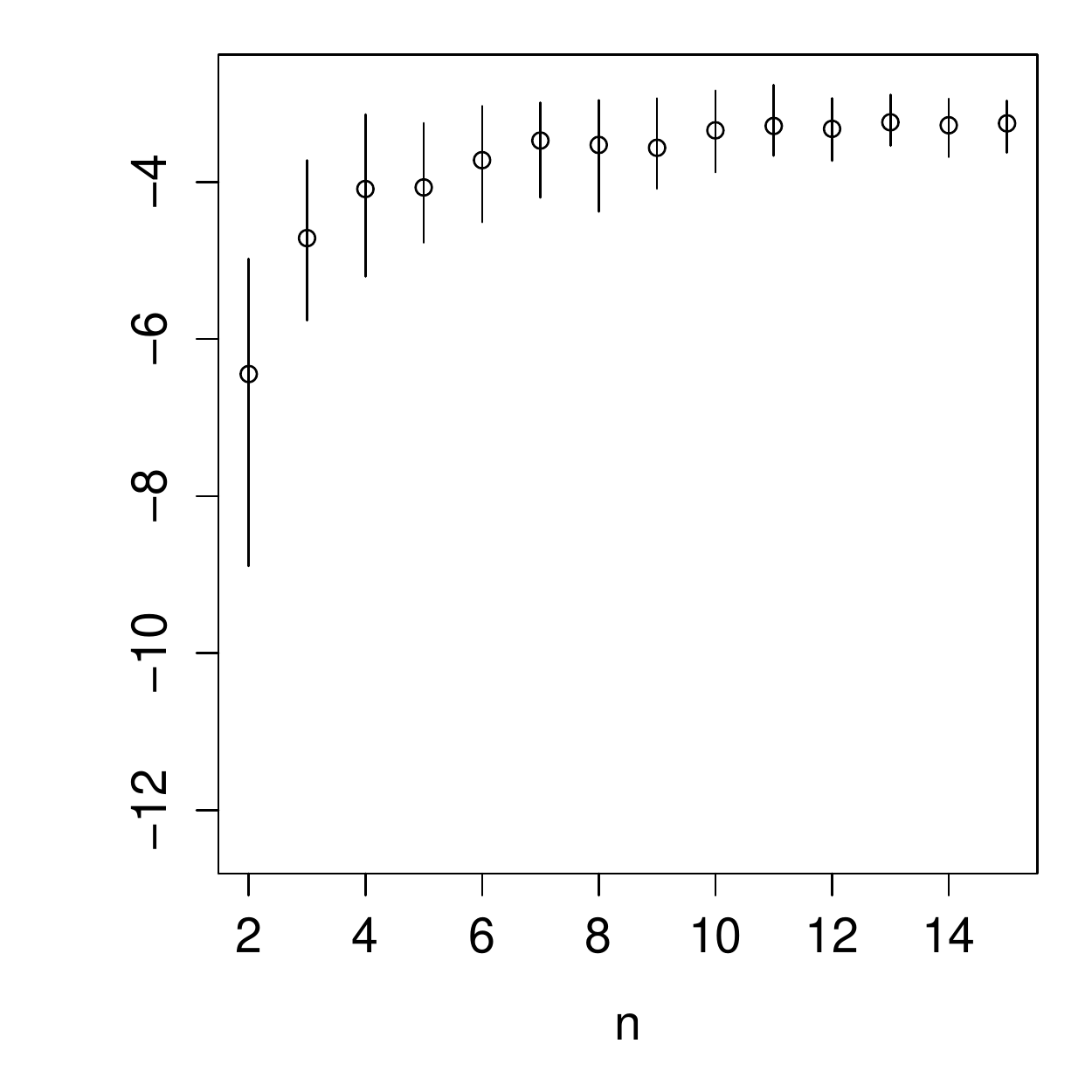}} & \raisebox{-.5\height}{\includegraphics[scale=0.3]{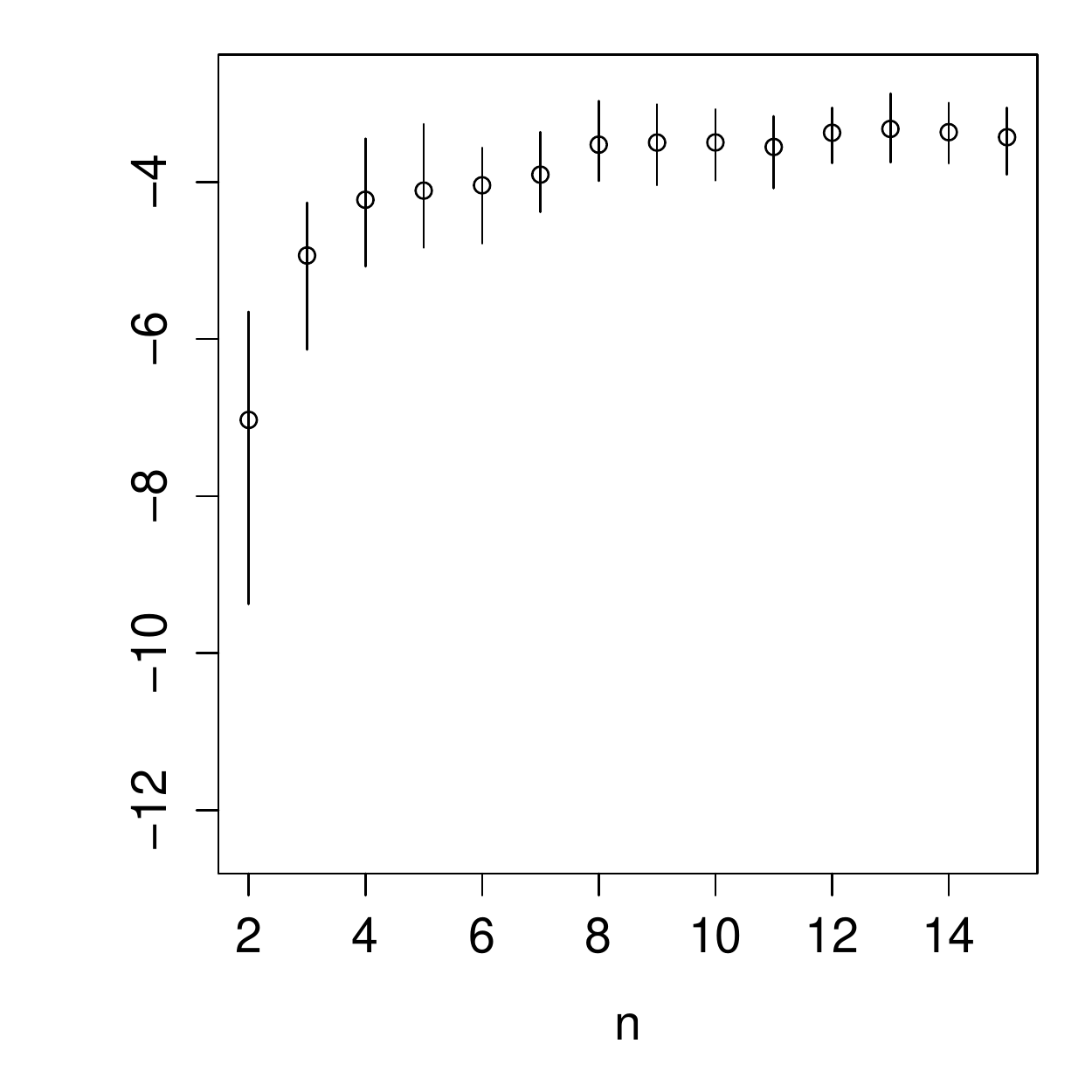}} \\
	\end{tabular}
	\caption[MSEs from learning with missing data for normal chains.]{MSEs on log scale from MCAR learning for normal chains. The numbers above the graphs indicate the percentage of missing data. The final plot shows learning by L-BFGS with restart with 100 samples.}
	\label{fig:cmarMse}
\end{figure}
The results are presented in Figure \ref{fig:cmarMse}. As expected, the median MSE and its variation increase as the proportion of missing data increases. From the regular patterns, we deduce that learning with a proportion of the data missing is equivalent to learning with a lesser number of complete samples. For example, when 90\% of the data is missing, the MSEs take similar values to learning with 100 complete samples. This was a surprise; we anticipated that learning with 10\% of the data would be equivalent to learning with 10\%, not 1\%, of the samples.

We provide the following intuition. Consider two situations: learning from a single complete sample, or learning from two incomplete samples that are formed by halving the complete one. By dividing the sample, the information that those variables occurred simultaneously is lost, thus reducing the accuracy of learning. In other words, a single complete sample is more informative than two half-complete samples.

\section{Piecewise composite likelihood learning}
We performed the following experiment to evaluate the relative efficiency of piecewise learning compared to L-BFGS with restart, and investigate whether the algorithm is consistent.
\begin{experiment}\label{exp:piecewise}
	The models used were trees of size $n=3,\ldots,10$ and grids of size $n=2,\ldots,9$ with Clayton copulae (it would have required too much time to learn larger models). For each model, the parameters were randomized and a sample of size $10$, $100$, or $1000$ was drawn. For the same sample, the two learning methods were performed from different random starts with $\epsilon=10^{-8}$. This was repeated 100 times for each combination of model and sample size. A record was made of the mean squared error (MSE) of the learnt parameters from the true parameter, and the time taken until convergence. 
\end{experiment}
The results are presented in Figures \ref{fig:piecewiseTrees} and \ref{fig:piecewiseGrids}. As the number of samples increases, the MSEs for piecewise learning appear to converge to those of L-BFGS learning, suggesting our algorithm is consistent. The variation in the MSEs for piecewise learning decreases on the log scale as the number of samples increases, in contrast to L-BFGS. There does appear to be a loss in efficiency; in general, for the same number of samples, learning with L-BFGS produces a lower median MSE than piecewise learning. The difference in efficiency, however, diminishes as the sample size and model increases. In some cases, piecewise learning even has a lower median MSE than L-BFGS with restart, for example, with trees larger than $n=8$ with 1000 samples.
\begin{figure}[p]
  \centering\begin{tabular}{c@{\hskip 1cm}c@{\hskip 1cm}c}
	& \multicolumn{2}{c}{\large\ Clayton Trees, Errors} \vspace{0.25cm} \\
	 Samples & \ \ \ \ L-BFGS with Restart & \ \ \ \ \ \ \ Piecewise \vspace{0.1cm} \\ 
	10 & \raisebox{-.5\height}{\includegraphics[scale=0.27]{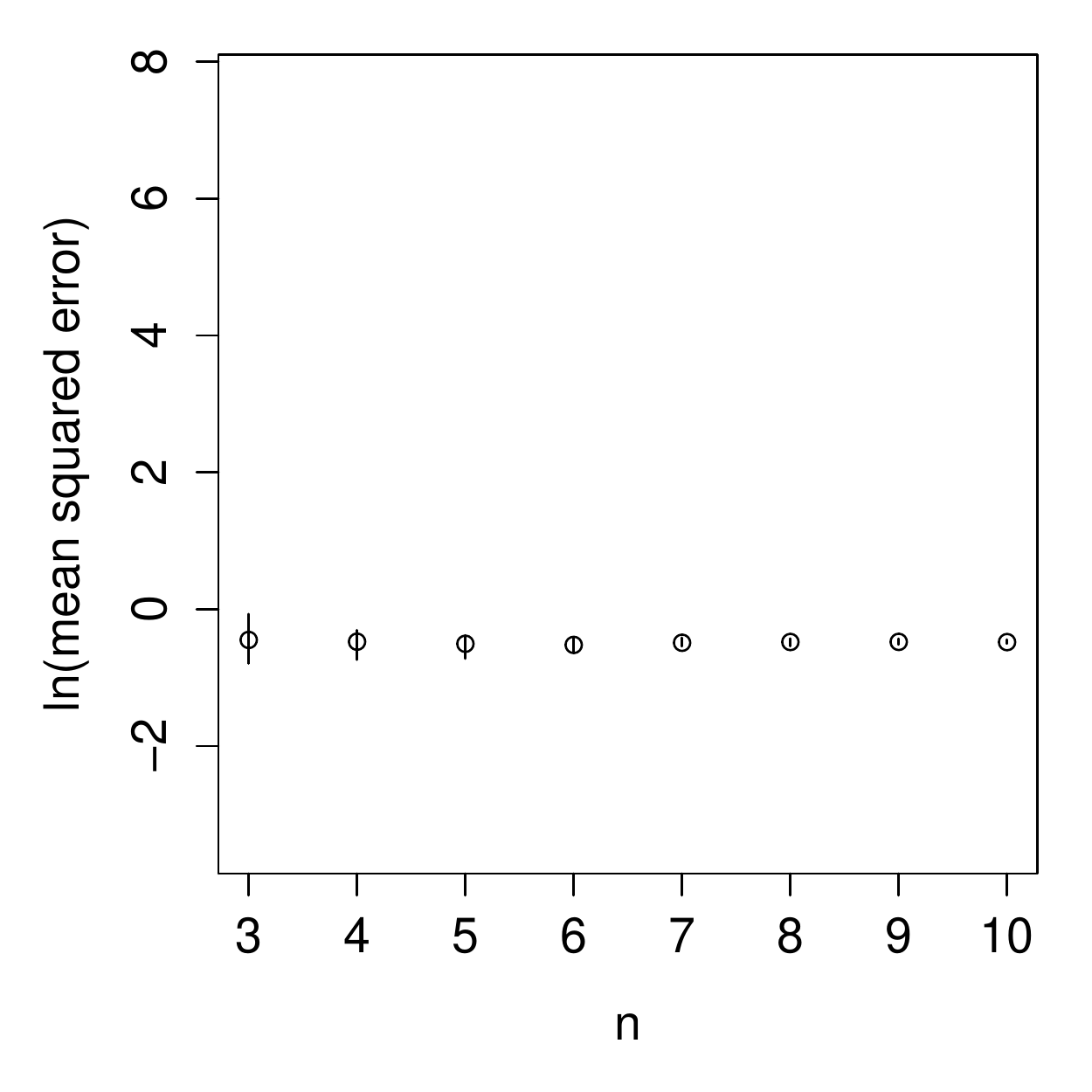}} & \raisebox{-.5\height}{\includegraphics[scale=0.27]{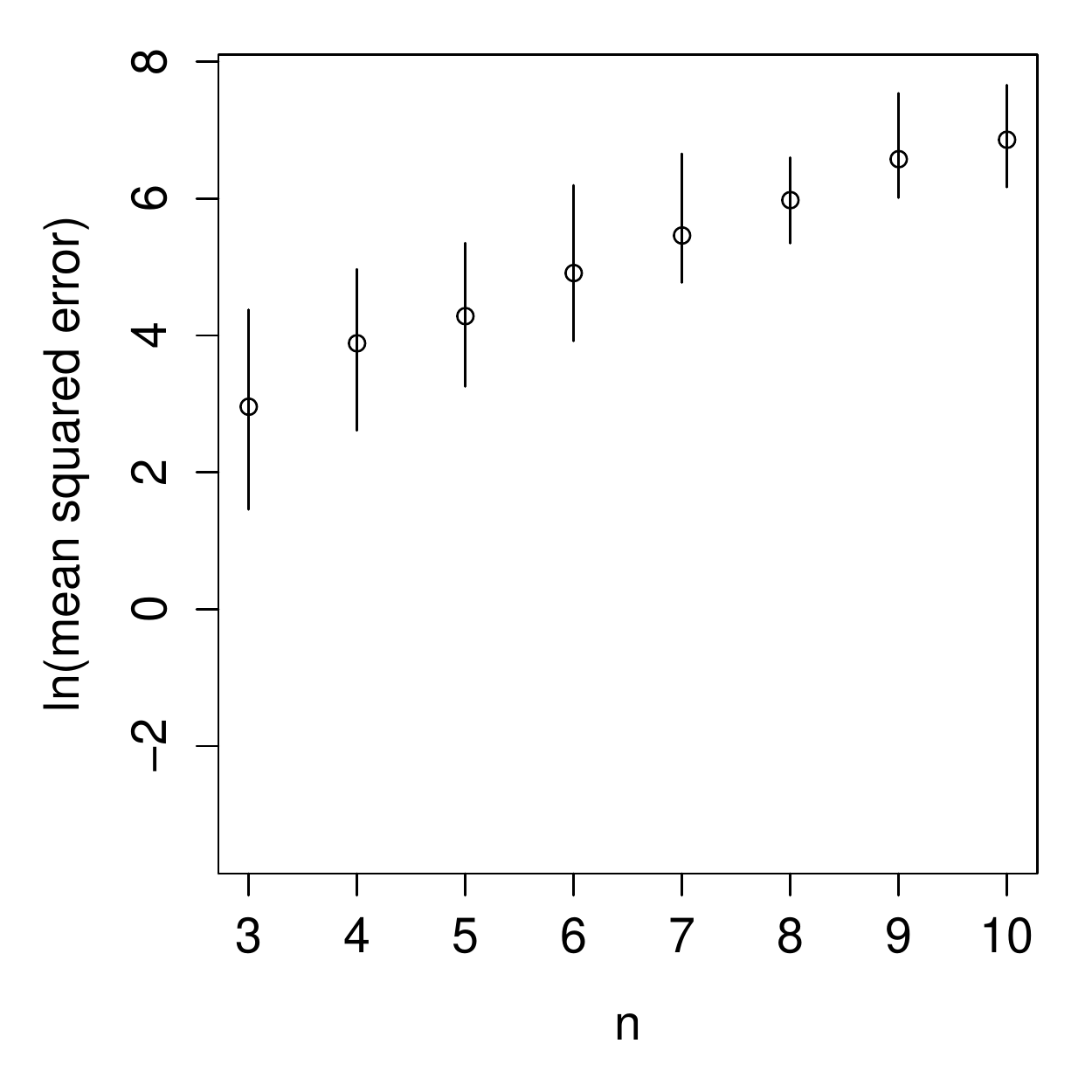}} \\
	100 & \raisebox{-.5\height}{\includegraphics[scale=0.27]{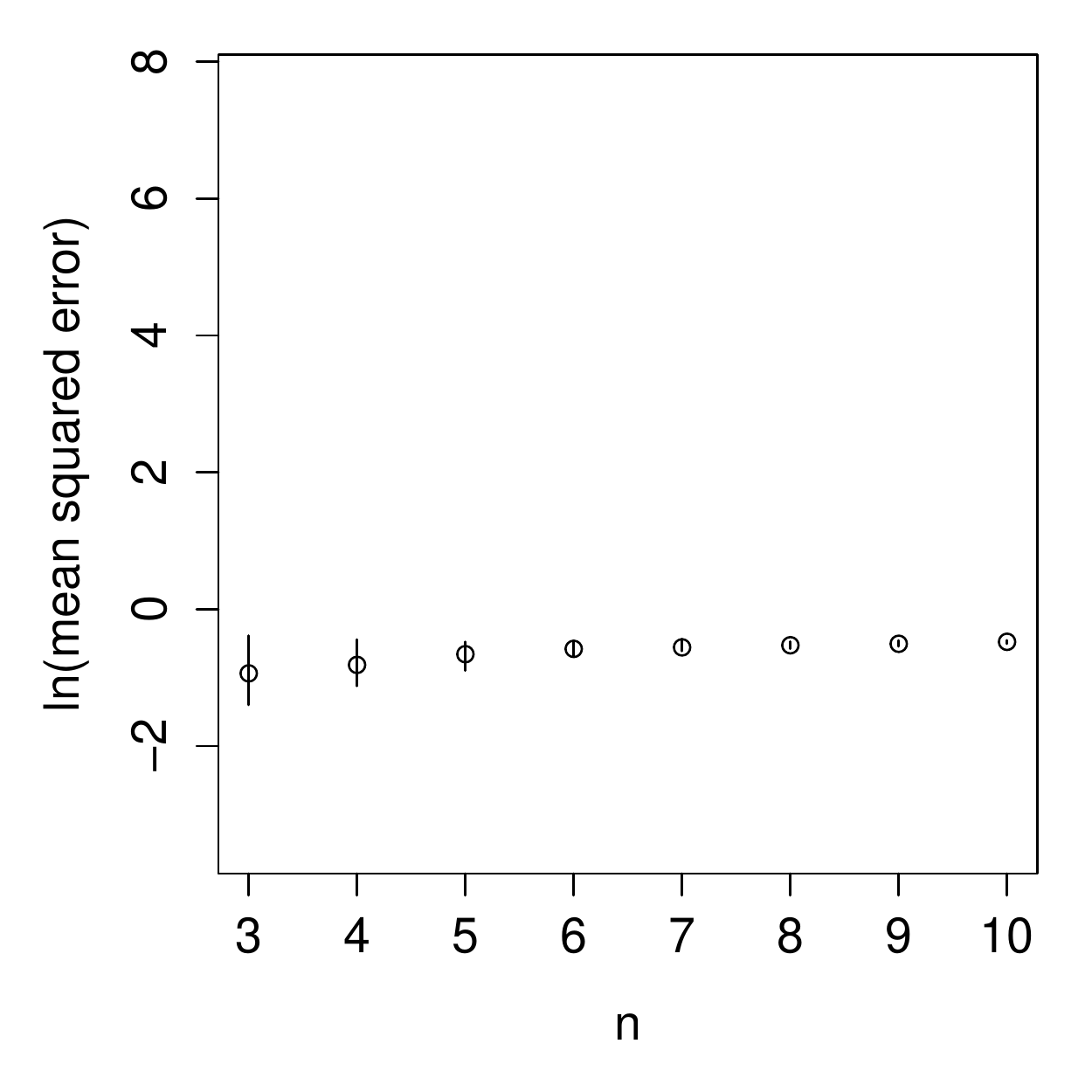}} & \raisebox{-.5\height}{\includegraphics[scale=0.27]{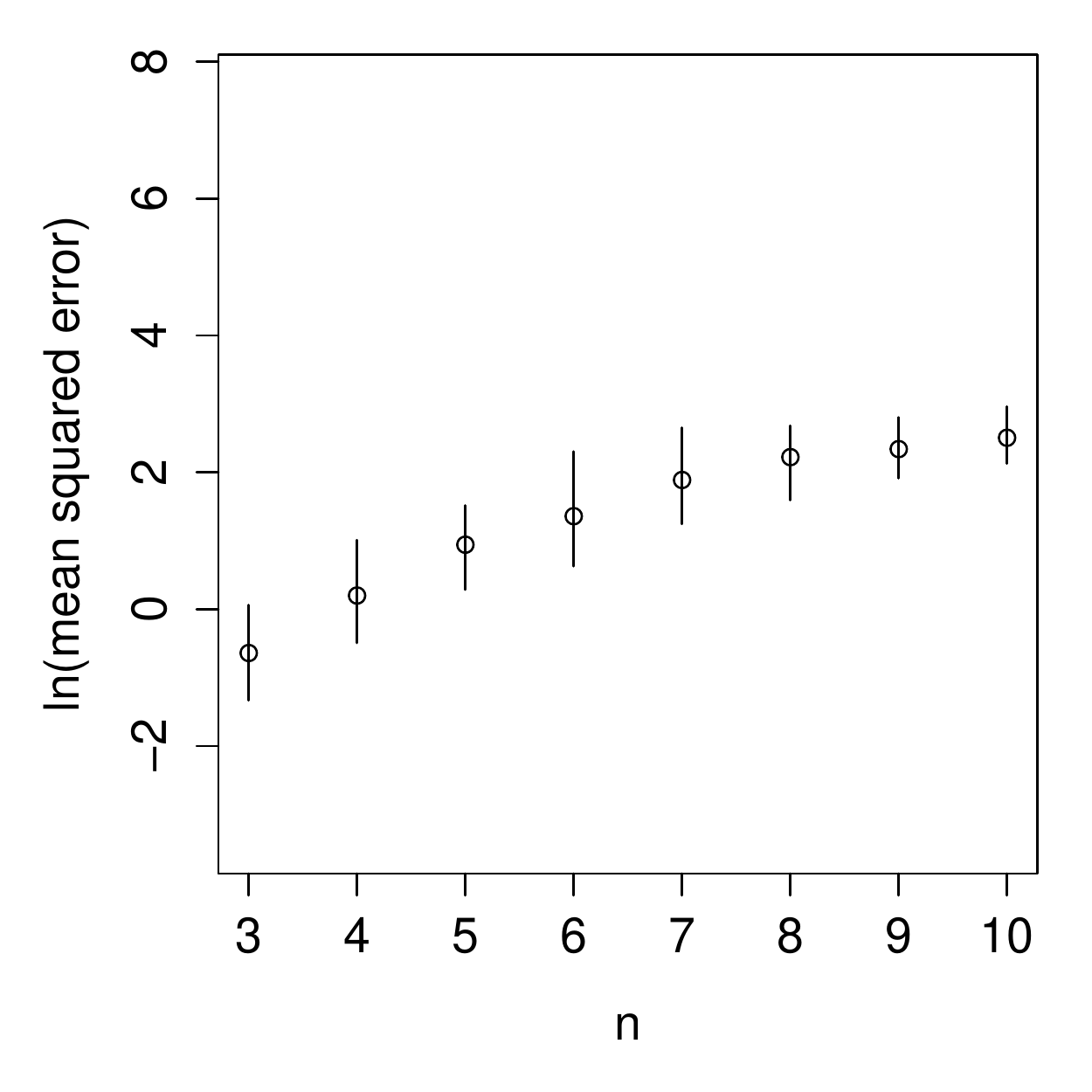}} \\
	1000 & \raisebox{-.5\height}{\includegraphics[scale=0.27]{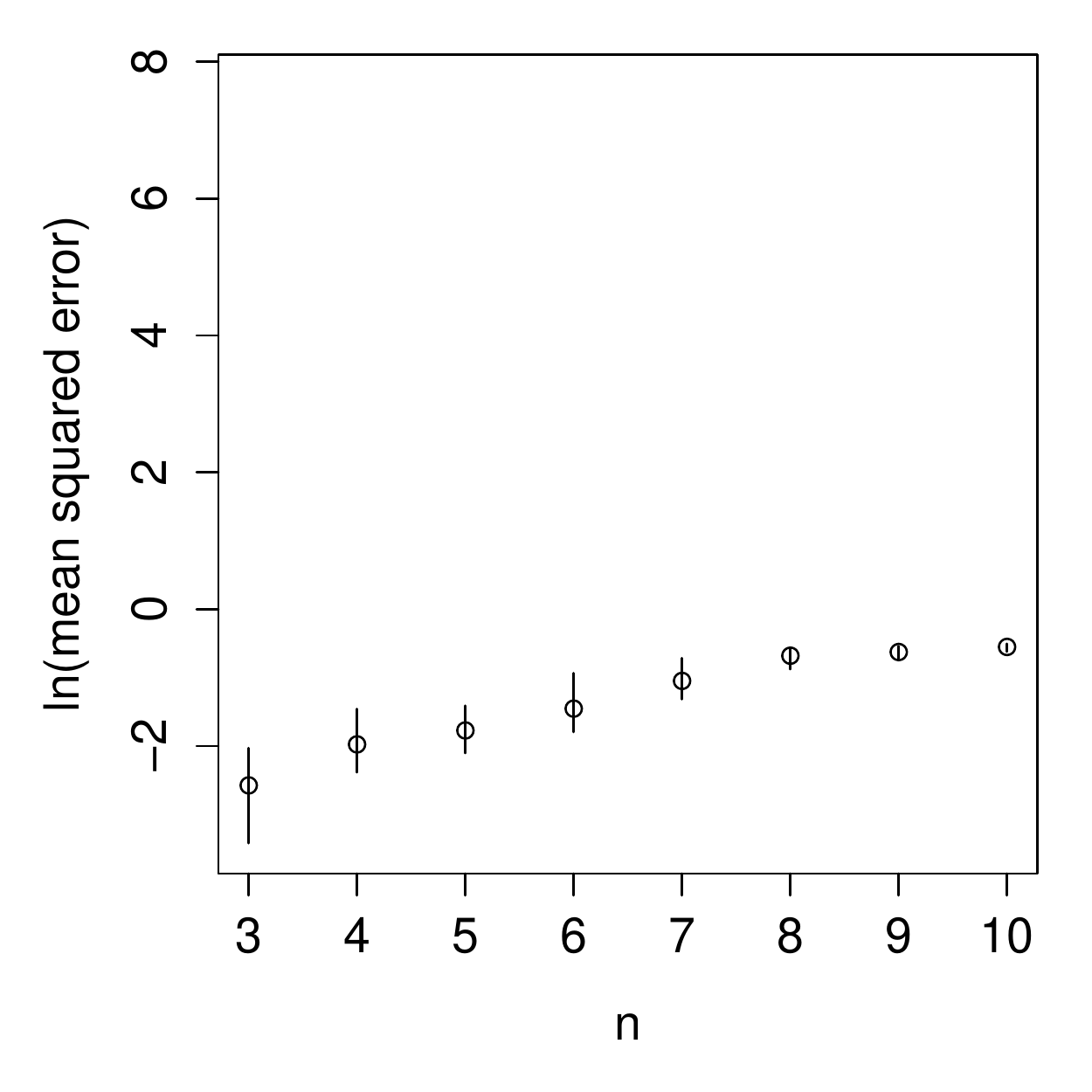}} & \raisebox{-.5\height}{\includegraphics[scale=0.27]{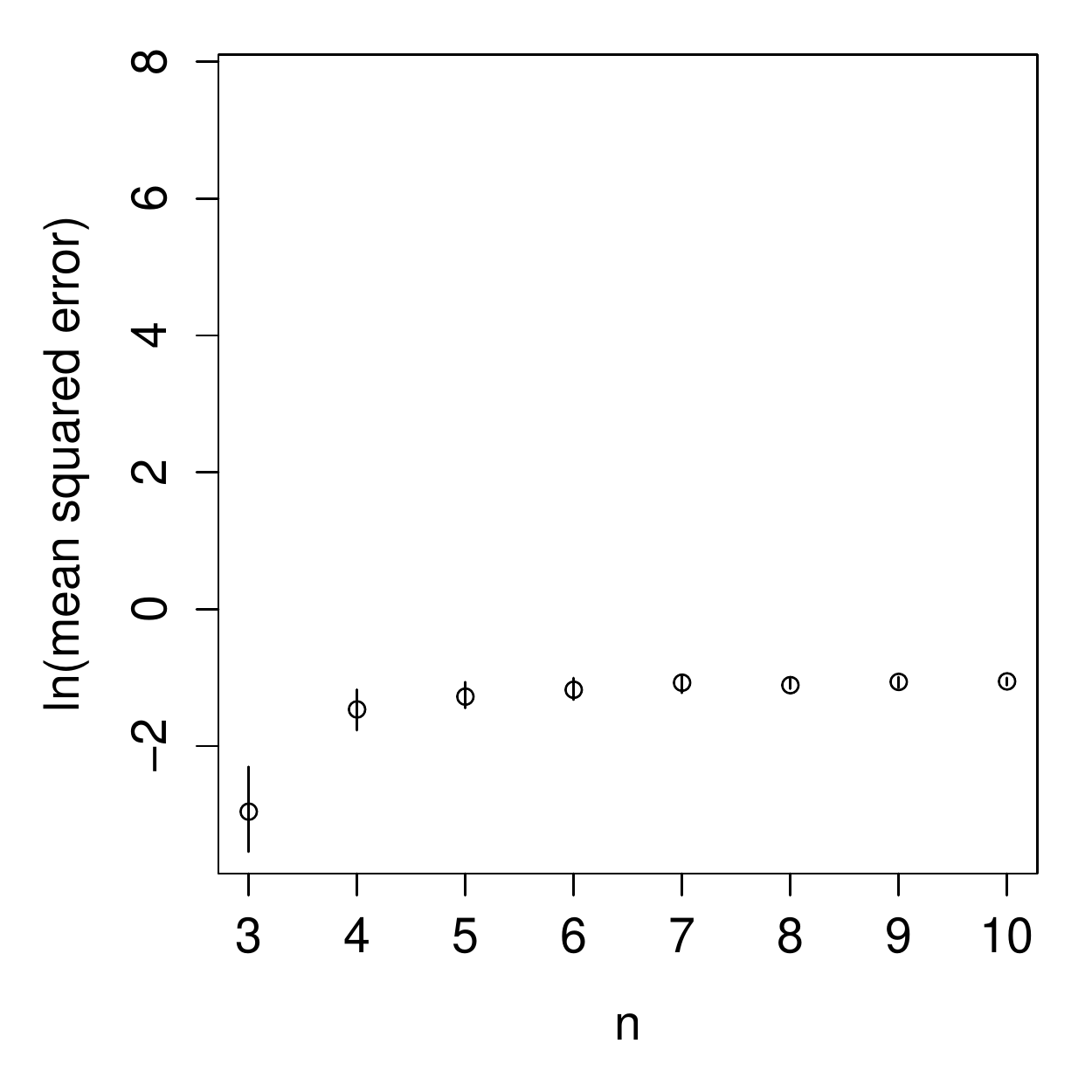}} \vspace{0.5cm} \\
	& \multicolumn{2}{c}{\large\ Clayton Trees, Running Times} \vspace{0.25cm} \\
	 & \ \ \ \ L-BFGS with Restart & \ \ \ \ \ \ \ Piecewise \vspace{0.1cm} \\ 
	10 & \raisebox{-.5\height}{\includegraphics[scale=0.27]{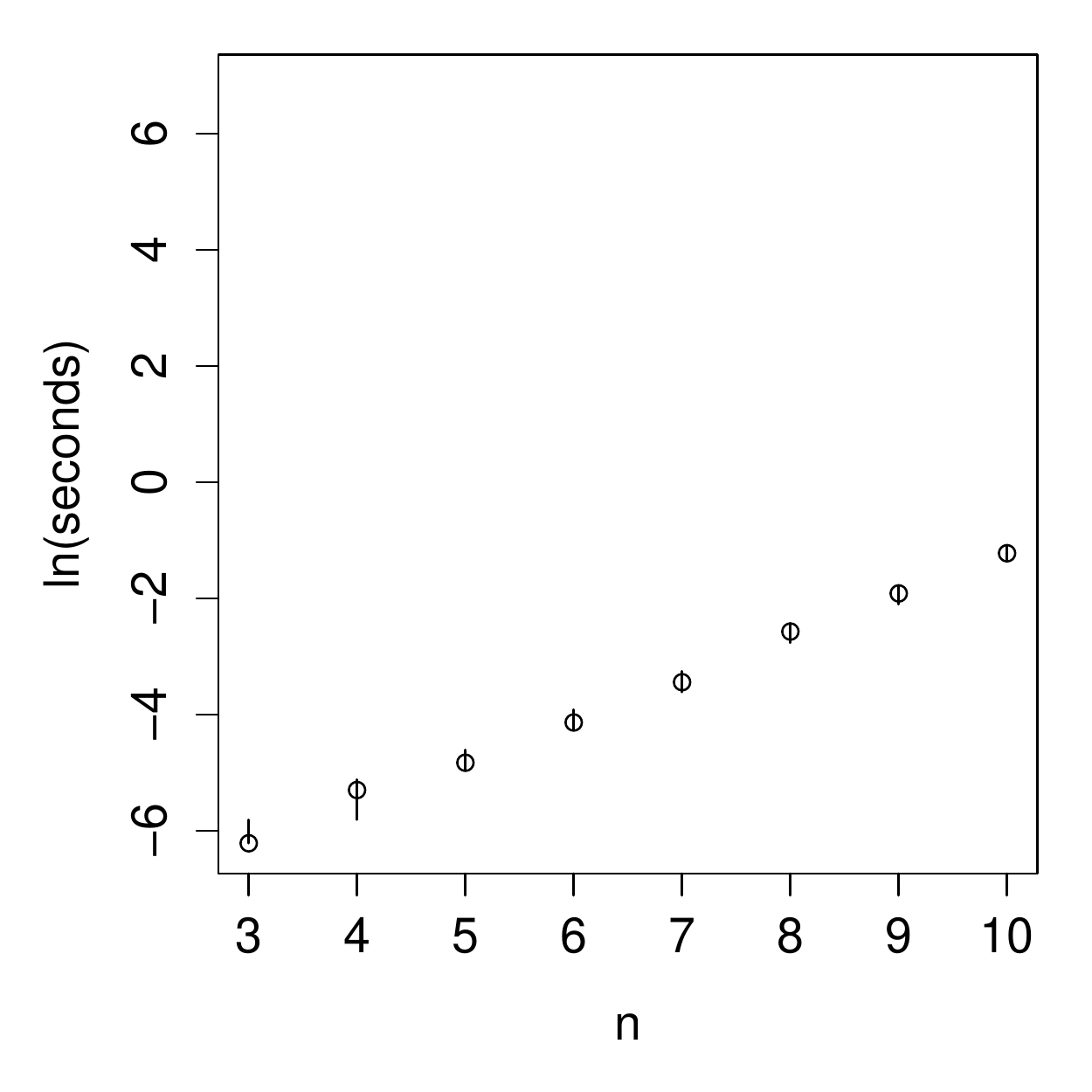}} & \raisebox{-.5\height}{\includegraphics[scale=0.27]{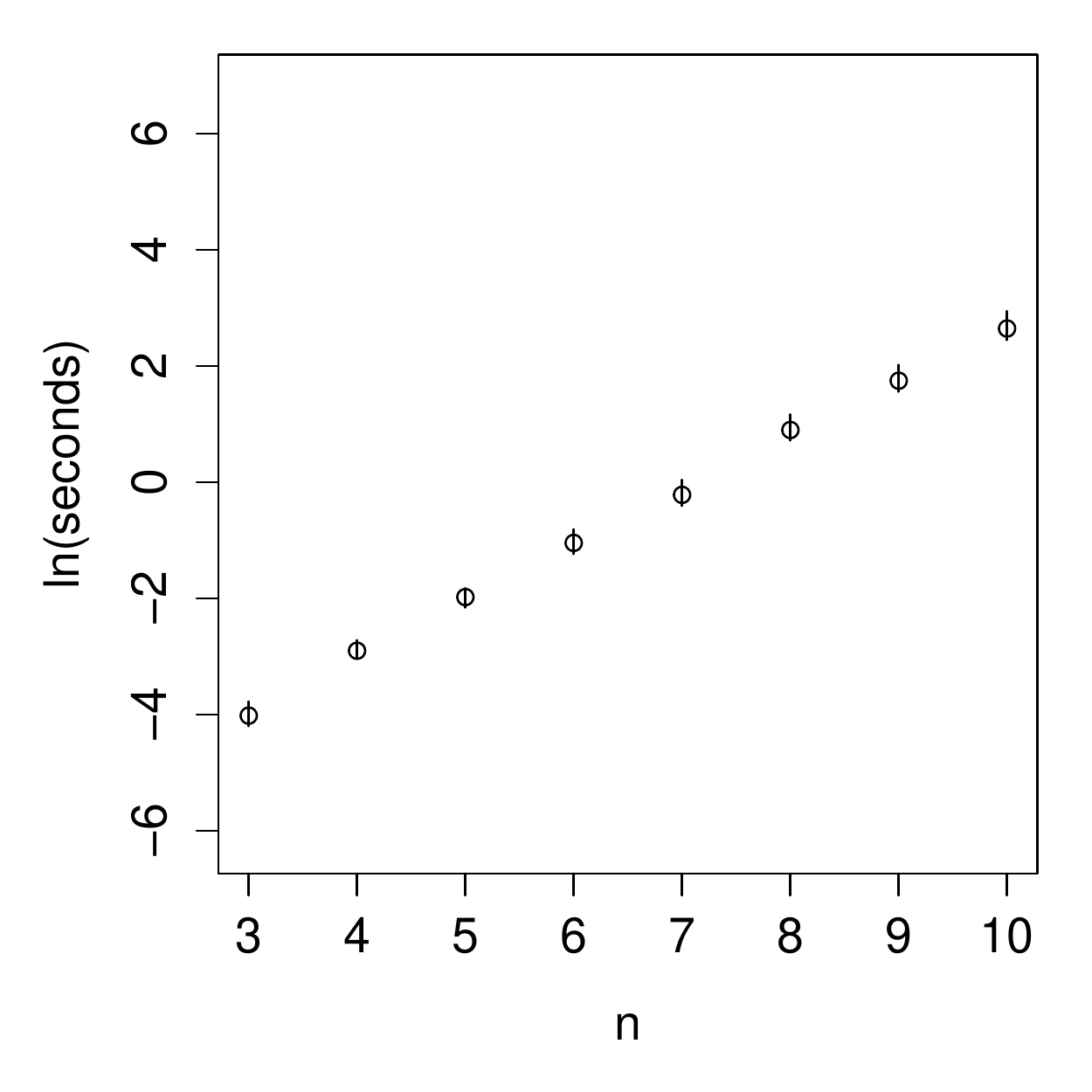}} \\
	100 & \raisebox{-.5\height}{\includegraphics[scale=0.27]{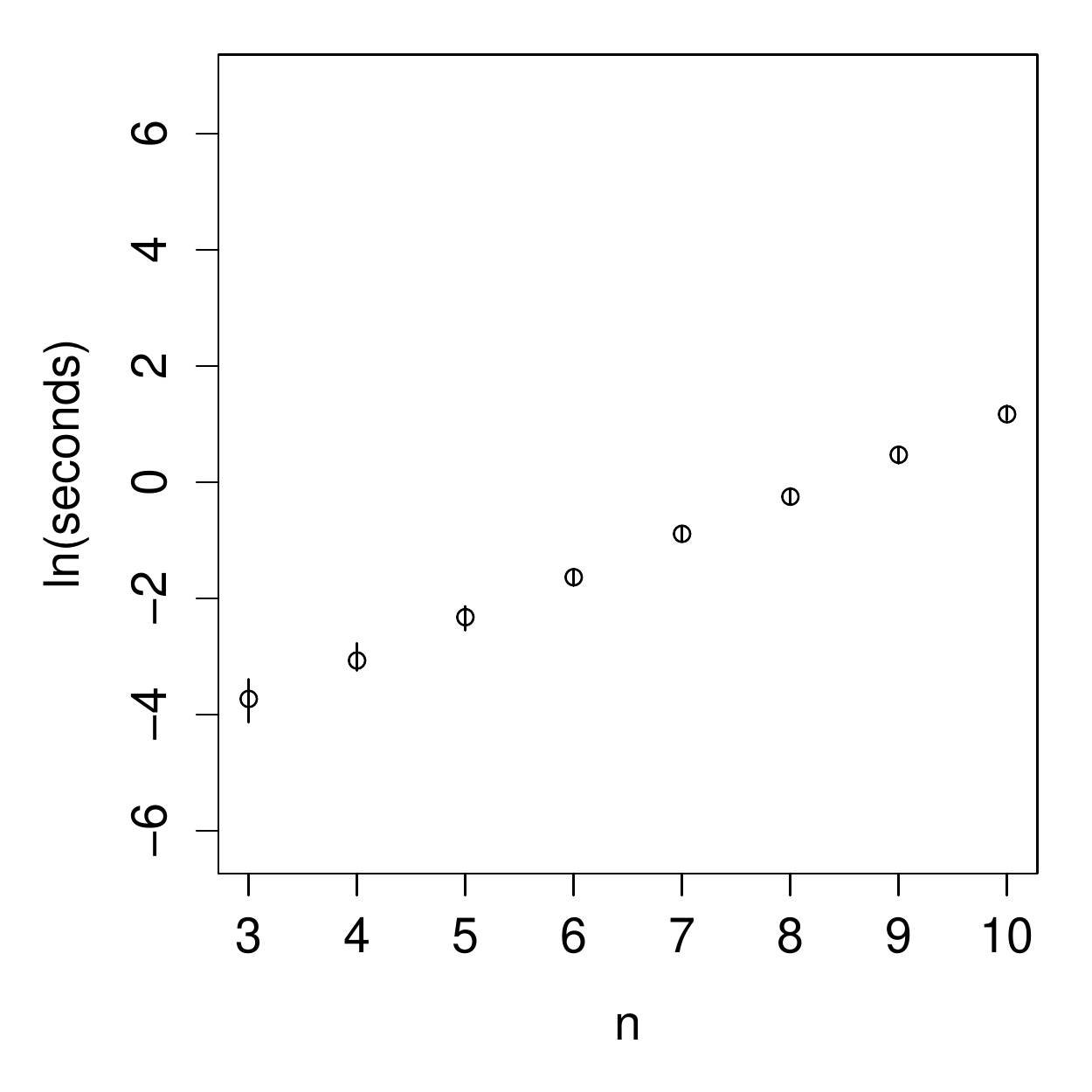}} & \raisebox{-.5\height}{\includegraphics[scale=0.27]{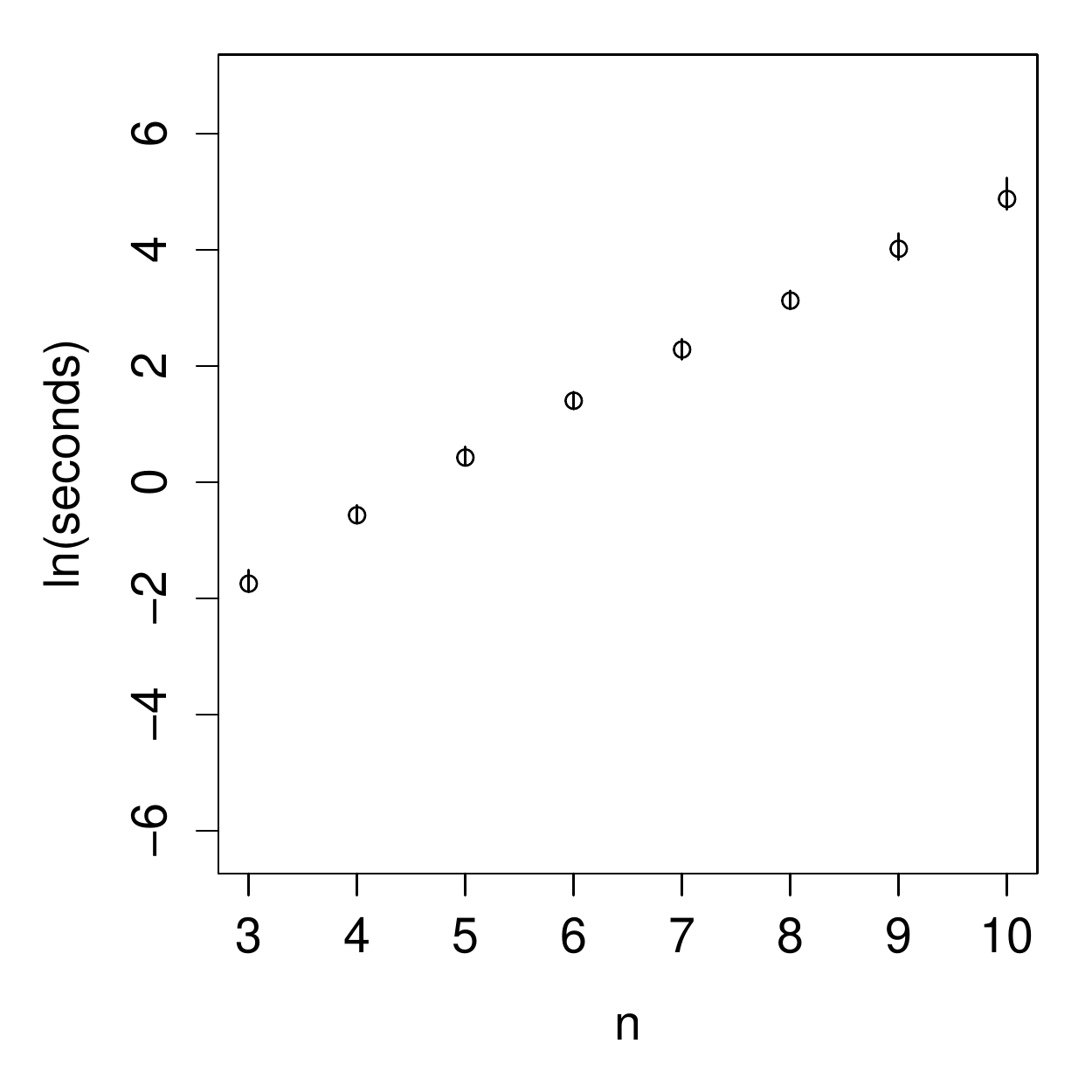}} \\
	1000 & \raisebox{-.5\height}{\includegraphics[scale=0.27]{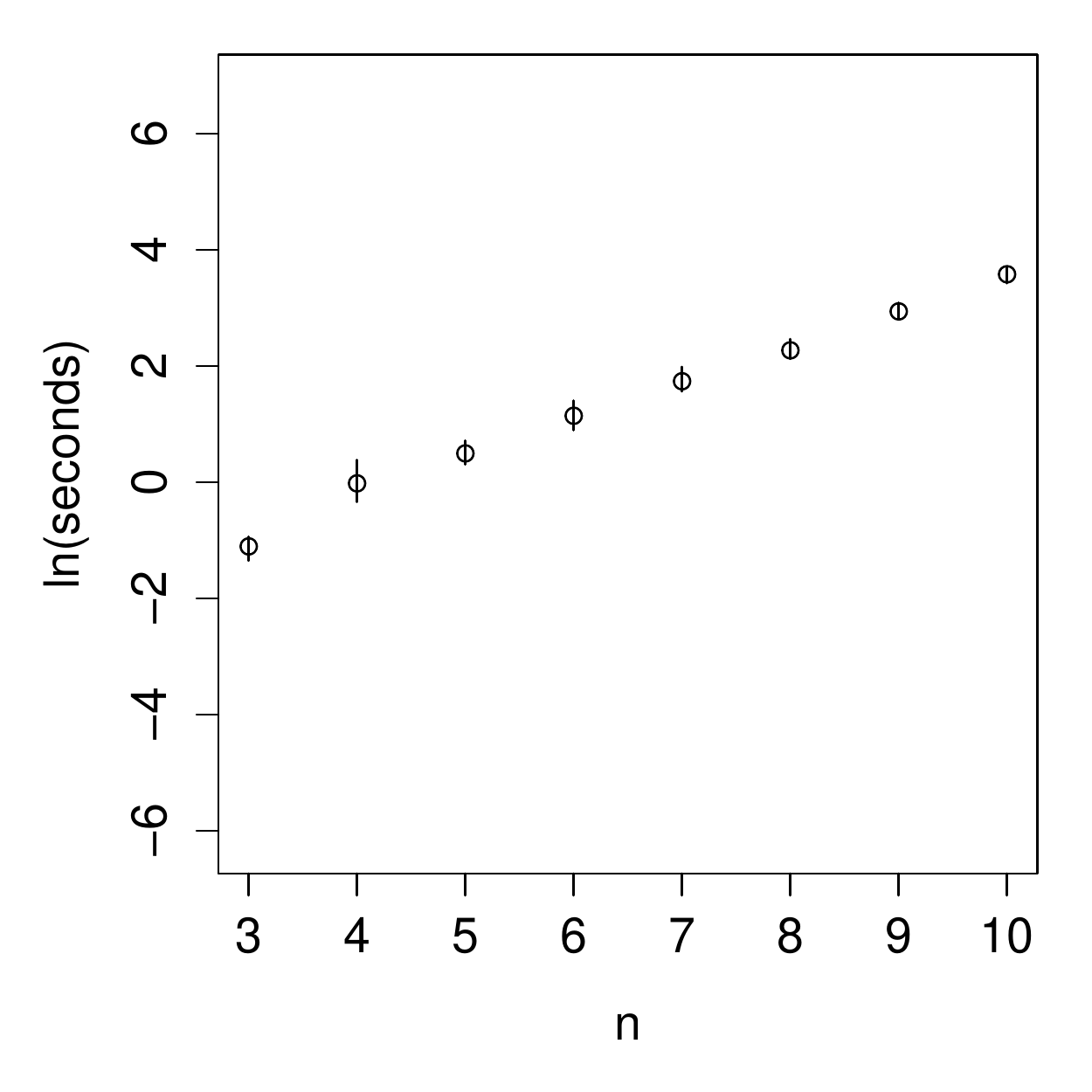}} & \raisebox{-.5\height}{\includegraphics[scale=0.27]{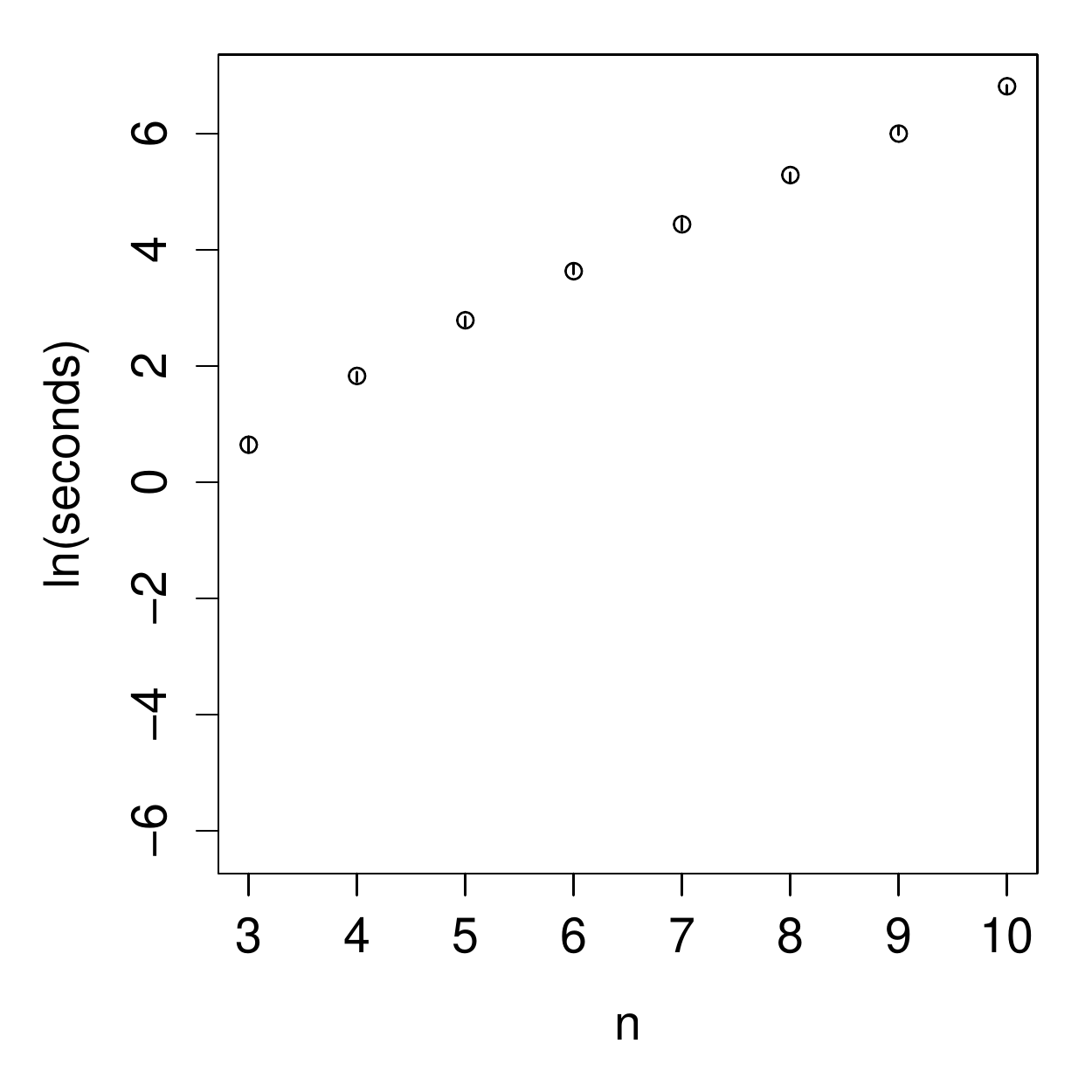}}
	\end{tabular}
	\caption[Comparing MSEs and running times from piecewise learning to L-BFGS with restart for Clayton trees.]{Comparing MSEs and running times on log scale from piecewise learning to L-BFGS with restart for Clayton trees.}
	\label{fig:piecewiseTrees}
\end{figure}

\begin{figure}[p]
  \centering\begin{tabular}{c@{\hskip 1cm}c@{\hskip 1cm}c}
	& \multicolumn{2}{c}{\large\ Clayton Grids, Errors} \vspace{0.25cm} \\
	 Samples & \ \ \ \ L-BFGS with Restart & \ \ \ \ \ \ \ Piecewise \vspace{0.1cm} \\ 
	10 & \raisebox{-.5\height}{\includegraphics[scale=0.27]{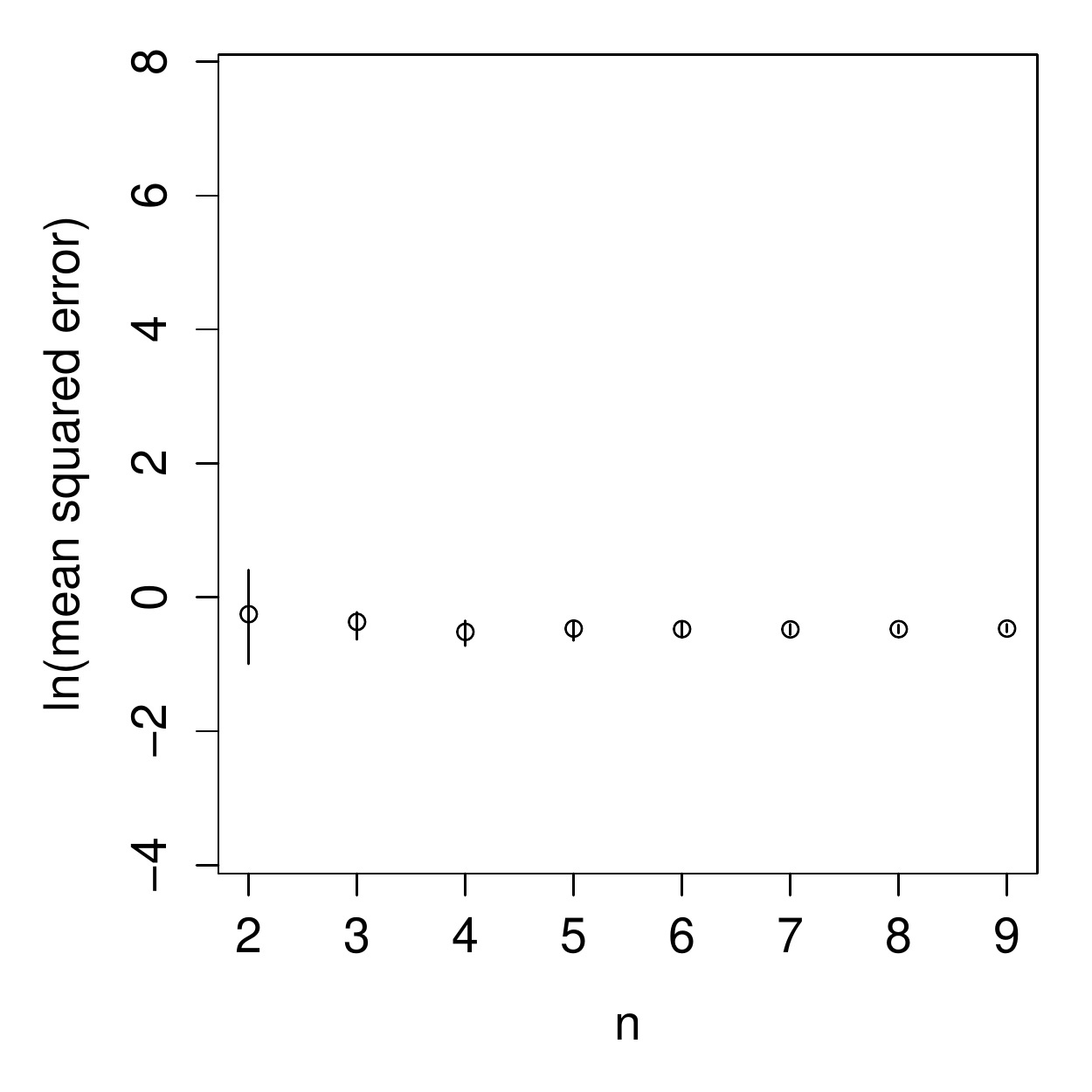}} & \raisebox{-.5\height}{\includegraphics[scale=0.27]{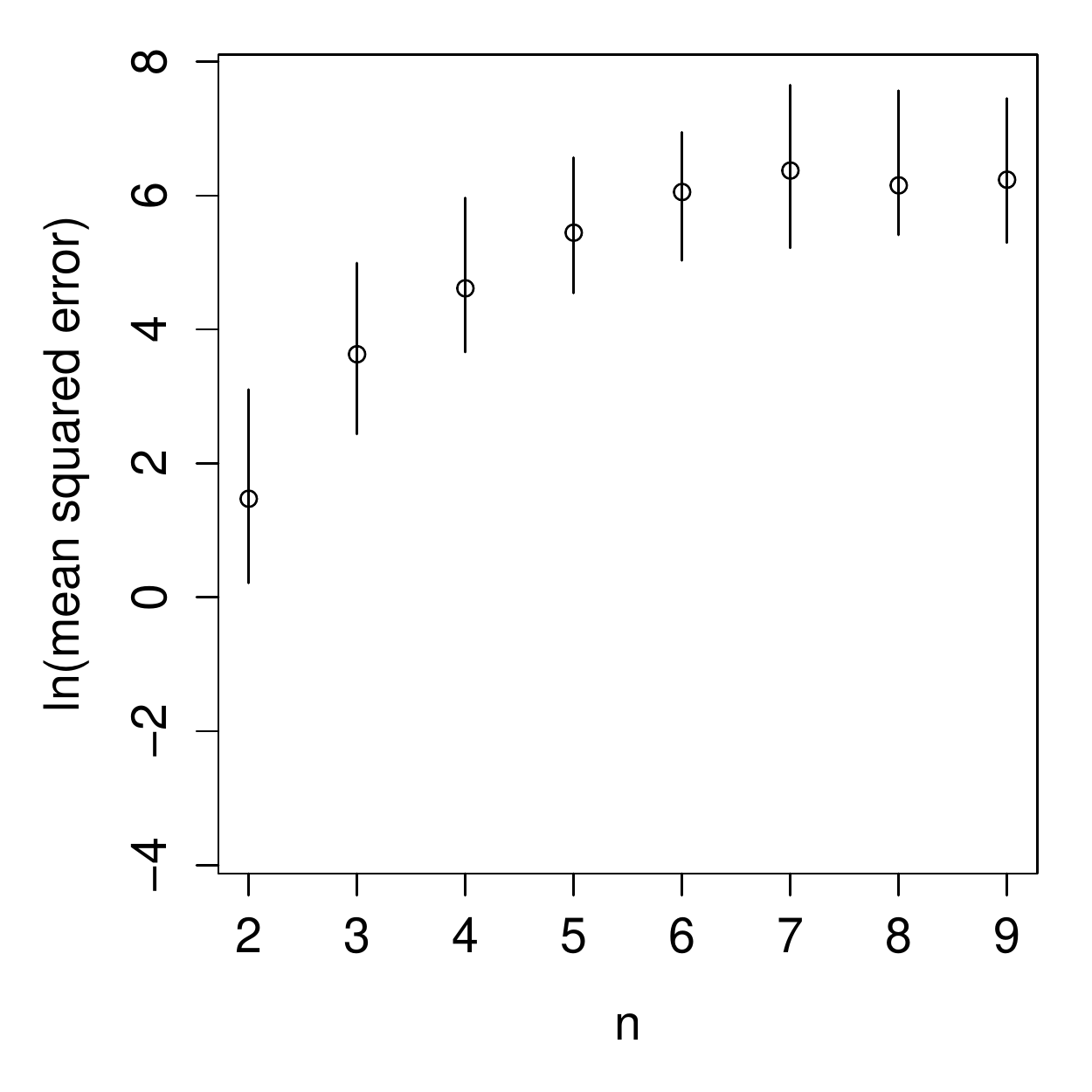}} \\
	100 & \raisebox{-.5\height}{\includegraphics[scale=0.27]{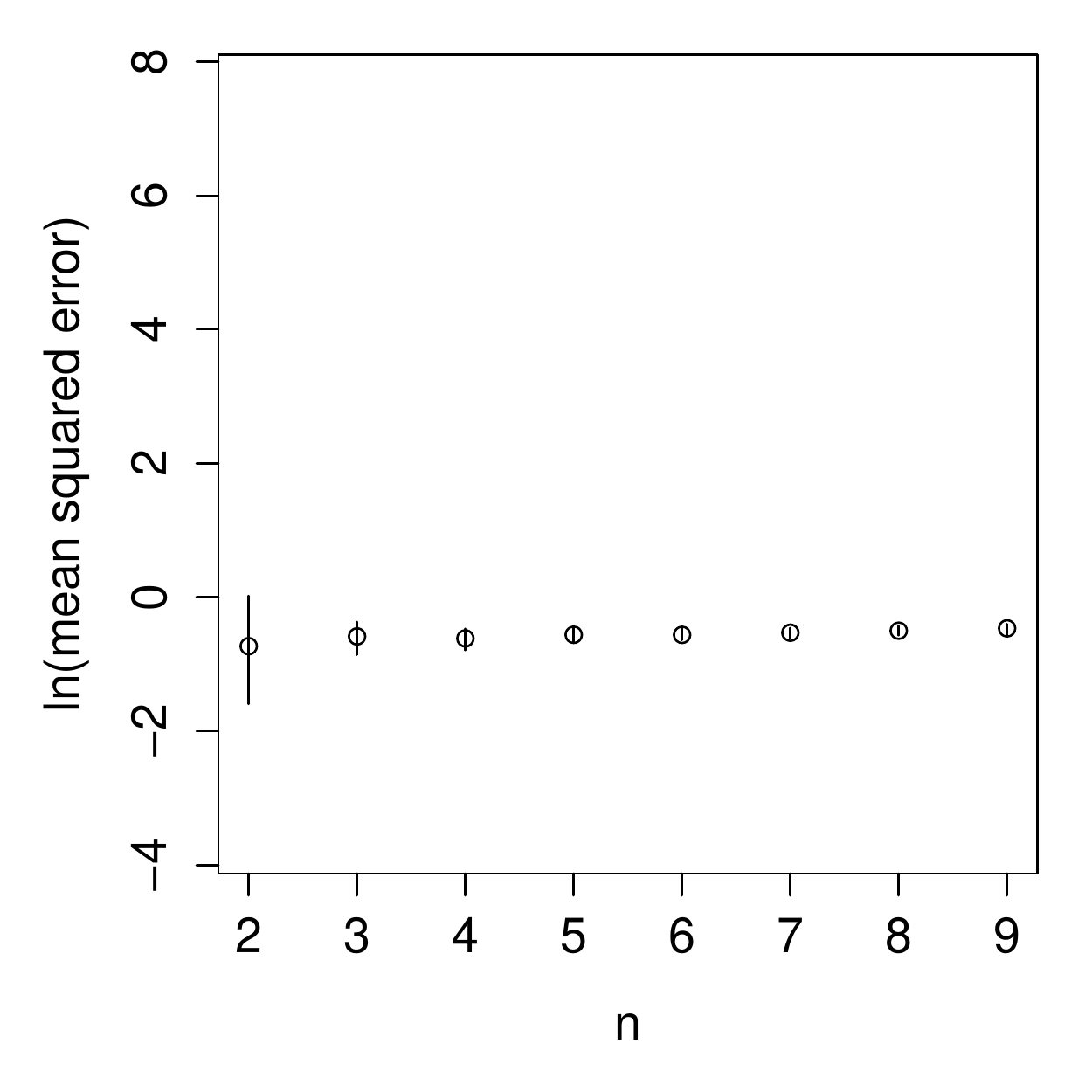}} & \raisebox{-.5\height}{\includegraphics[scale=0.27]{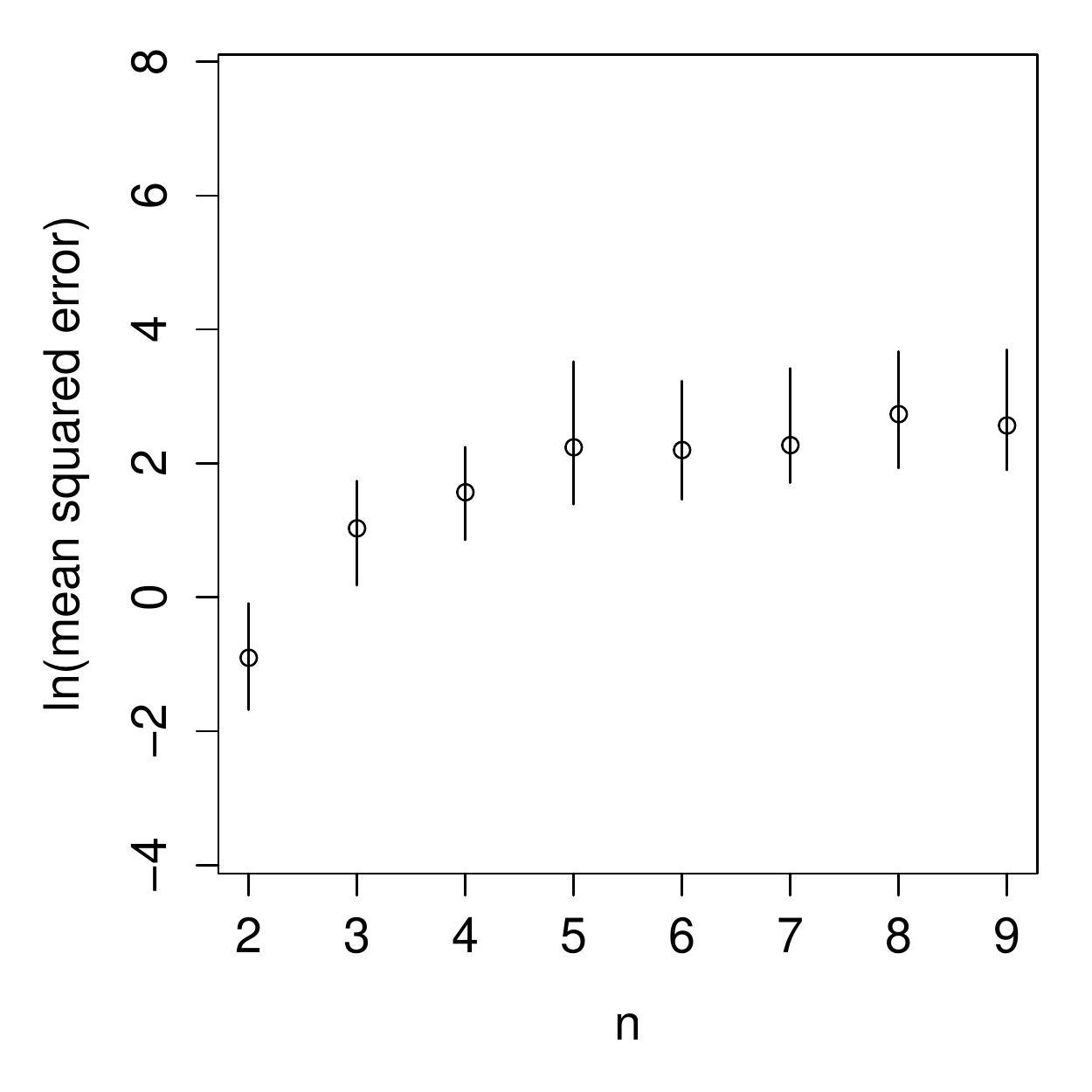}} \\
	1000 & \raisebox{-.5\height}{\includegraphics[scale=0.27]{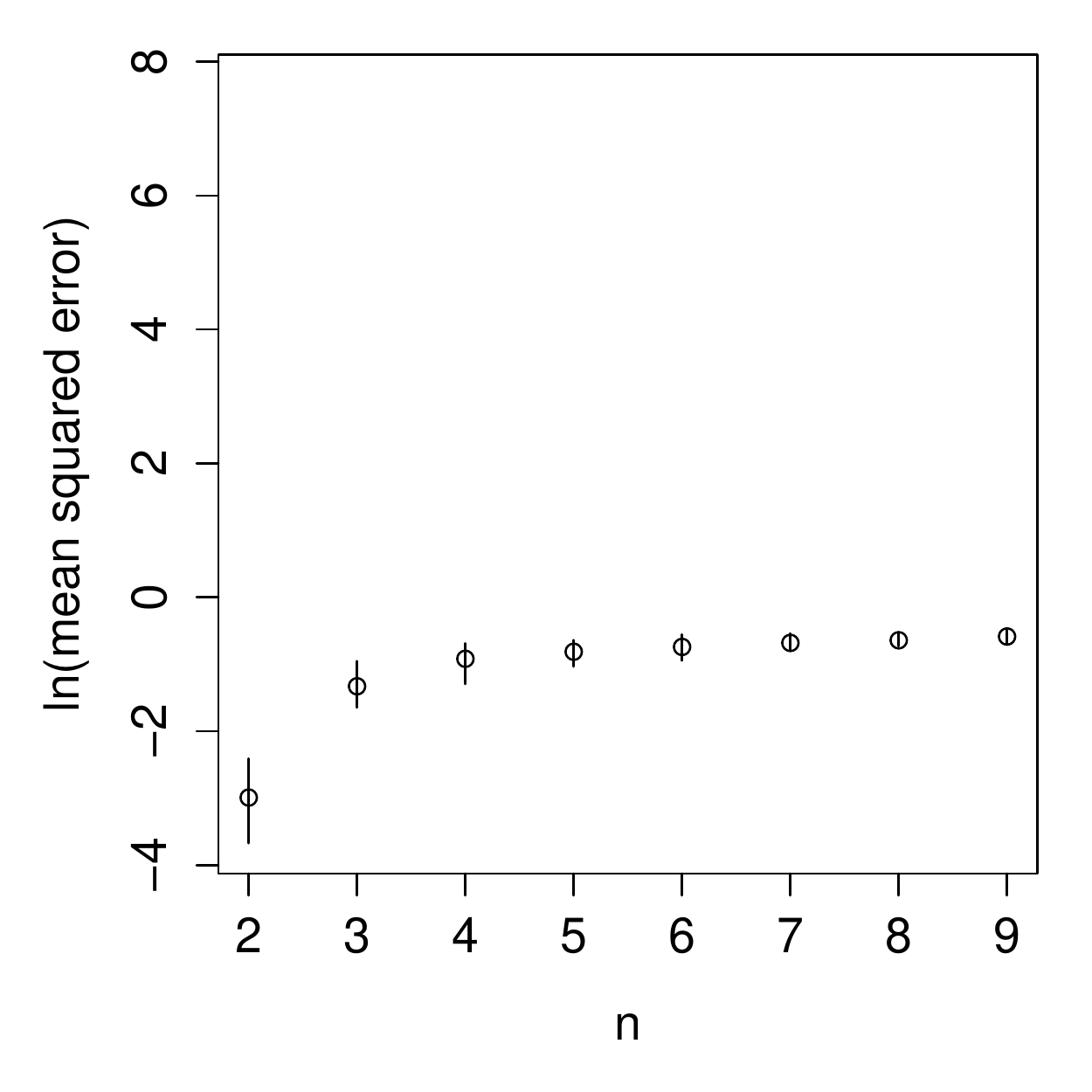}} & \raisebox{-.5\height}{\includegraphics[scale=0.27]{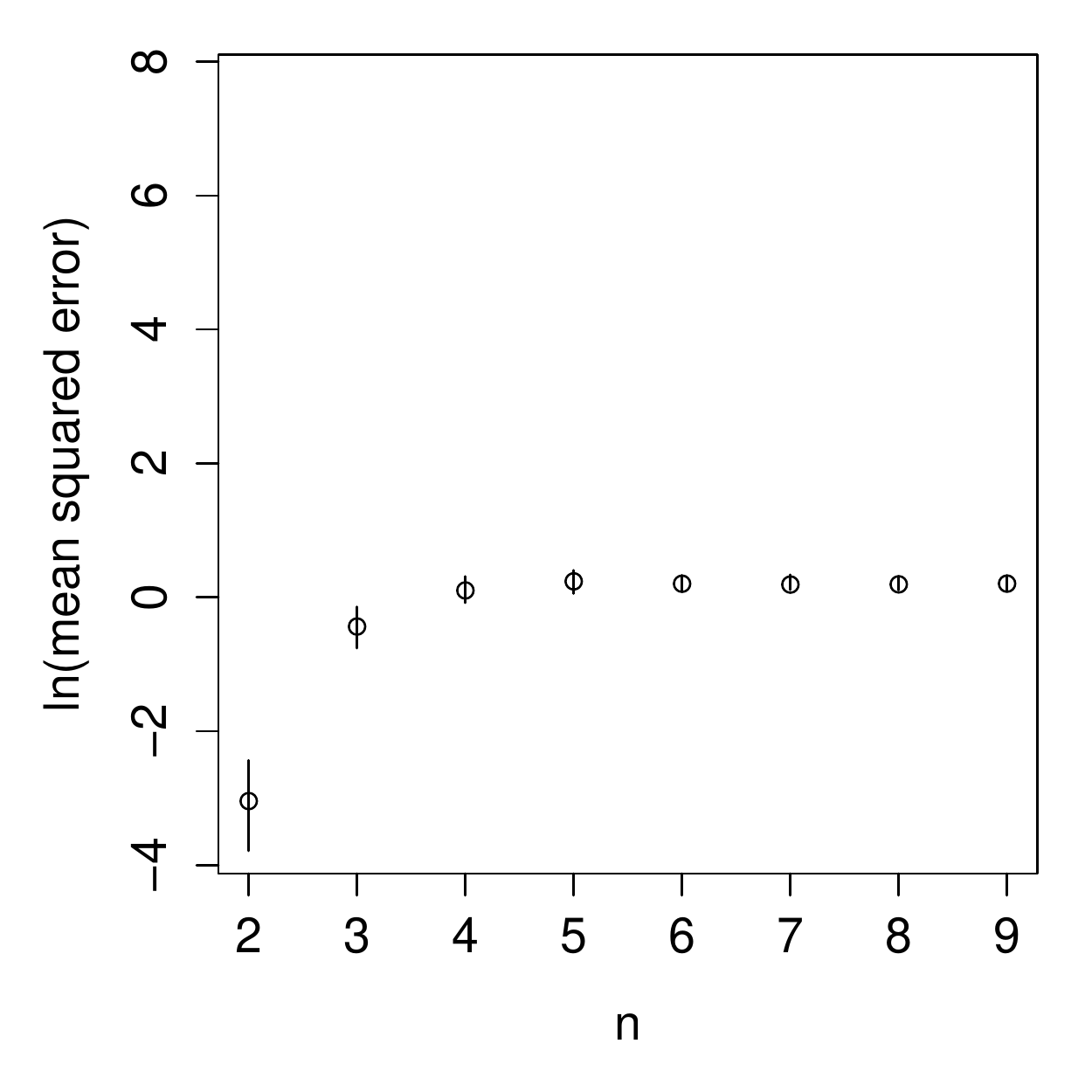}} \vspace{0.5cm} \\
	& \multicolumn{2}{c}{\large\ Clayton Grids, Running Times} \vspace{0.25cm} \\
	 & \ \ \ \ L-BFGS with Restart & \ \ \ \ \ \ \ Piecewise \vspace{0.1cm} \\ 
	10 & \raisebox{-.5\height}{\includegraphics[scale=0.27]{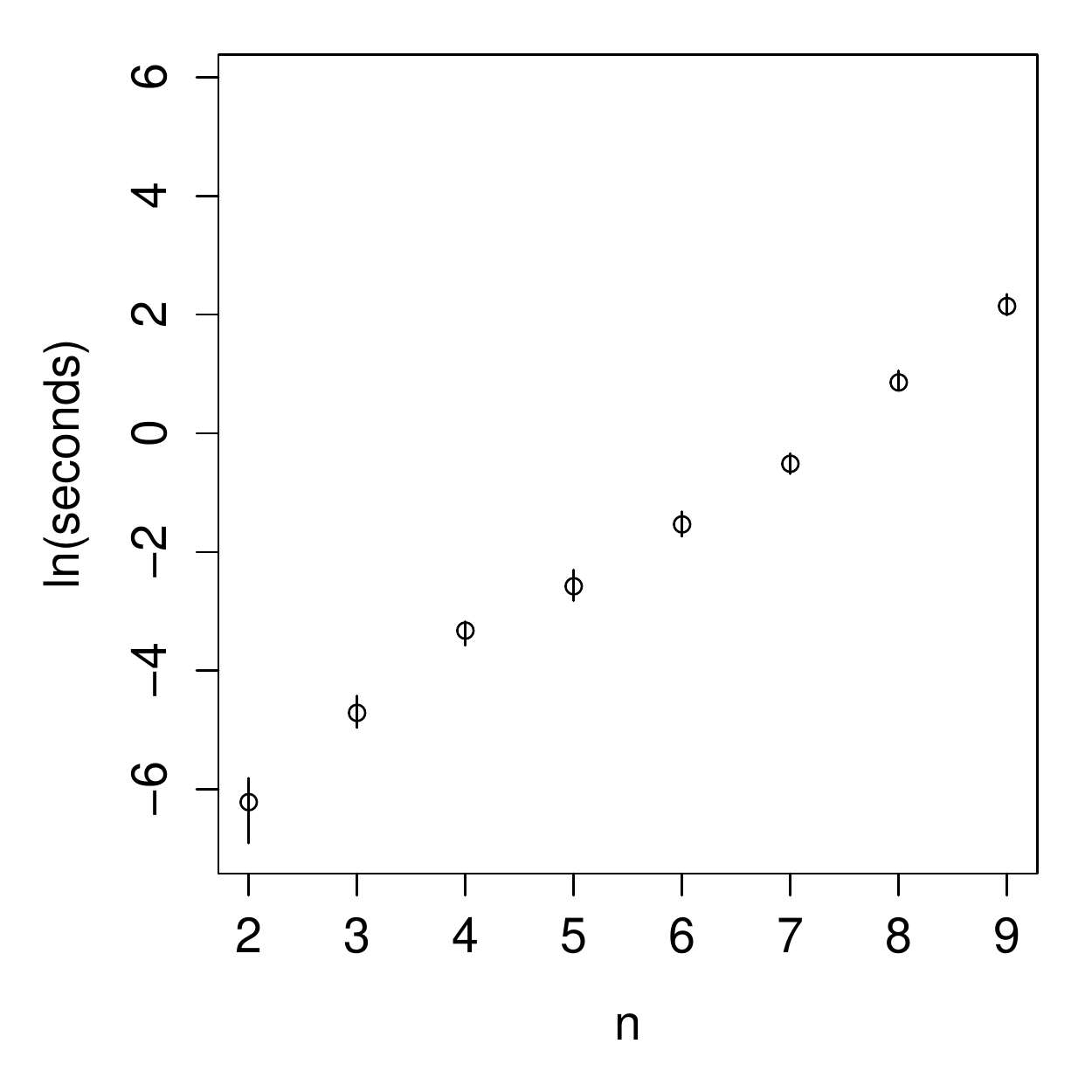}} & \raisebox{-.5\height}{\includegraphics[scale=0.27]{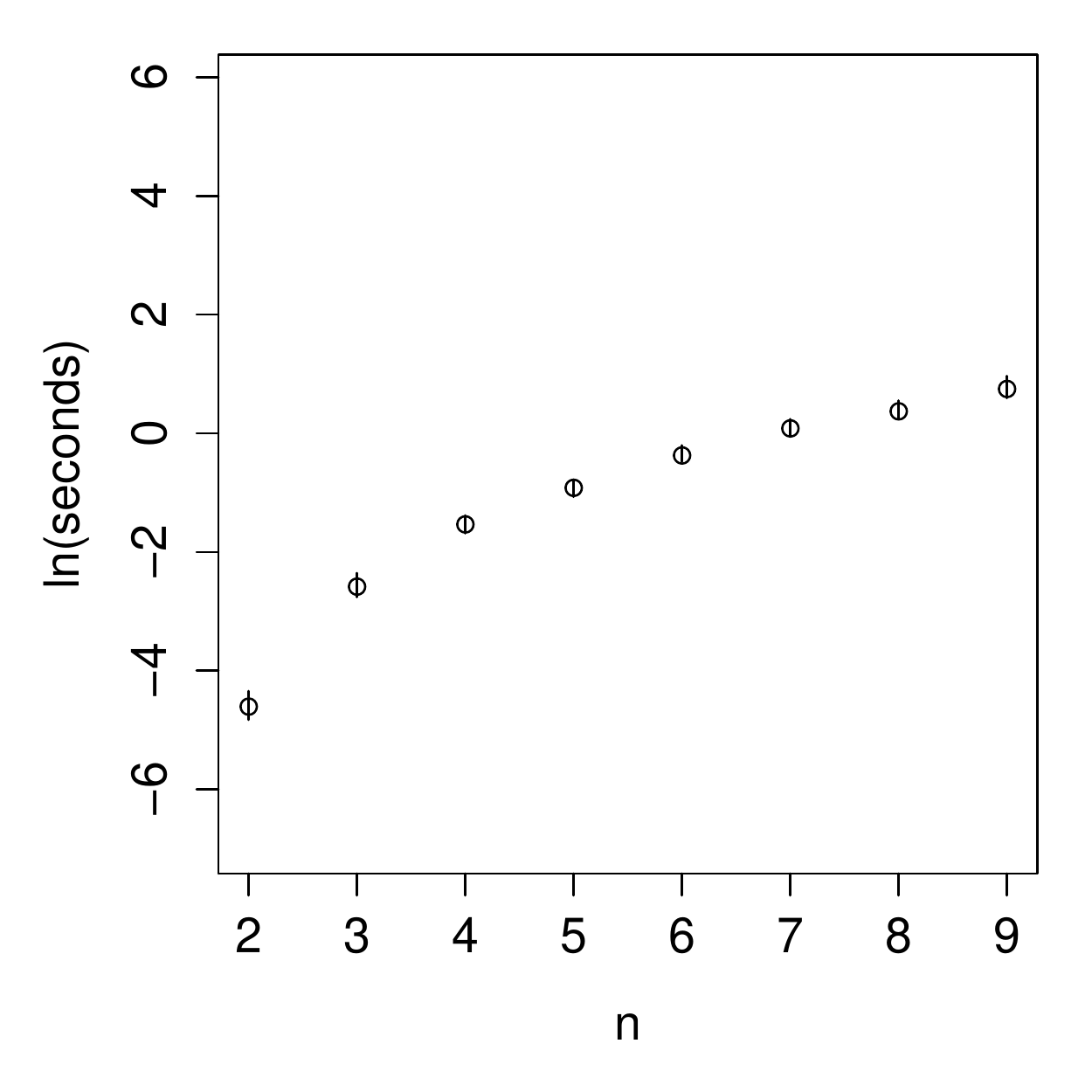}} \\
	100 & \raisebox{-.5\height}{\includegraphics[scale=0.27]{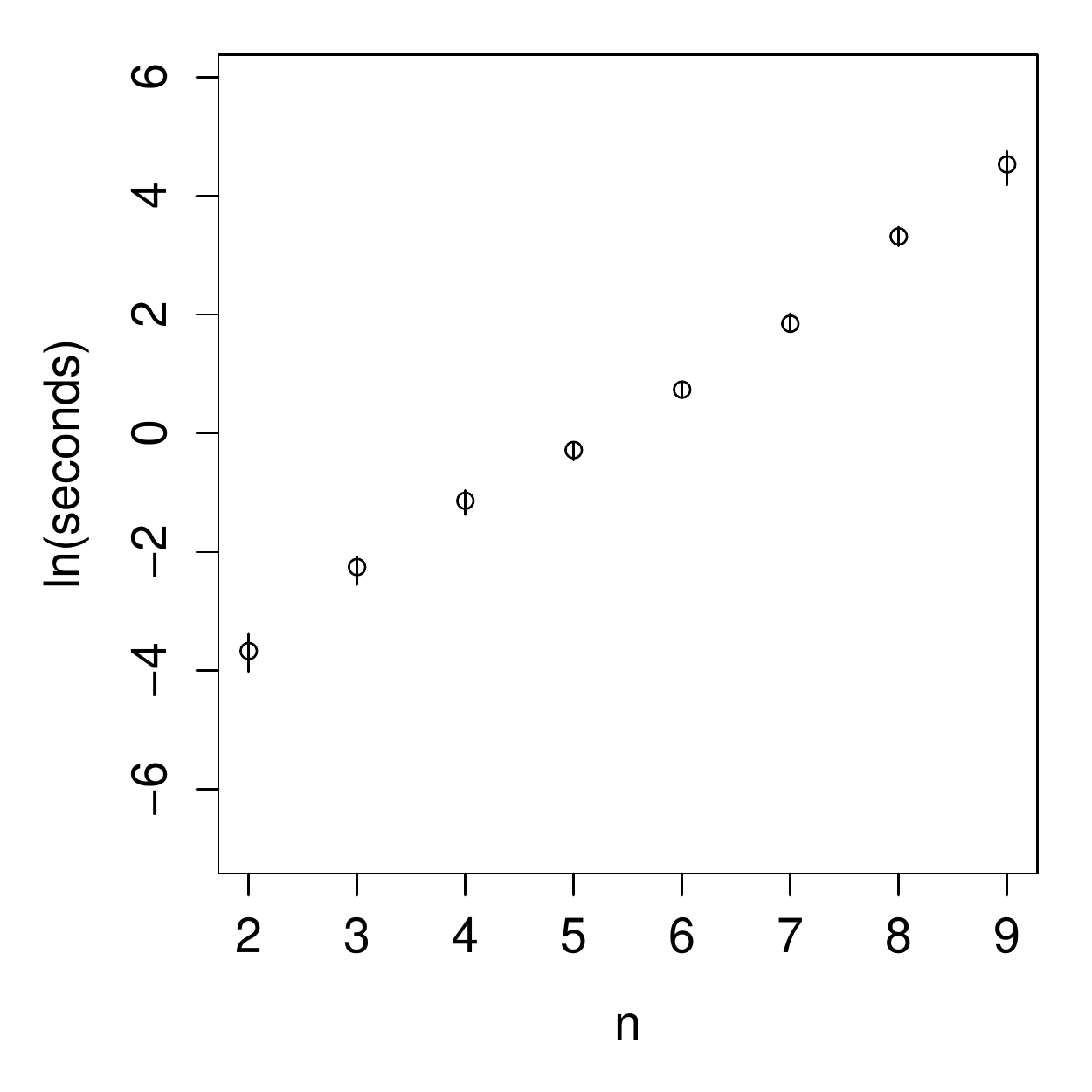}} & \raisebox{-.5\height}{\includegraphics[scale=0.27]{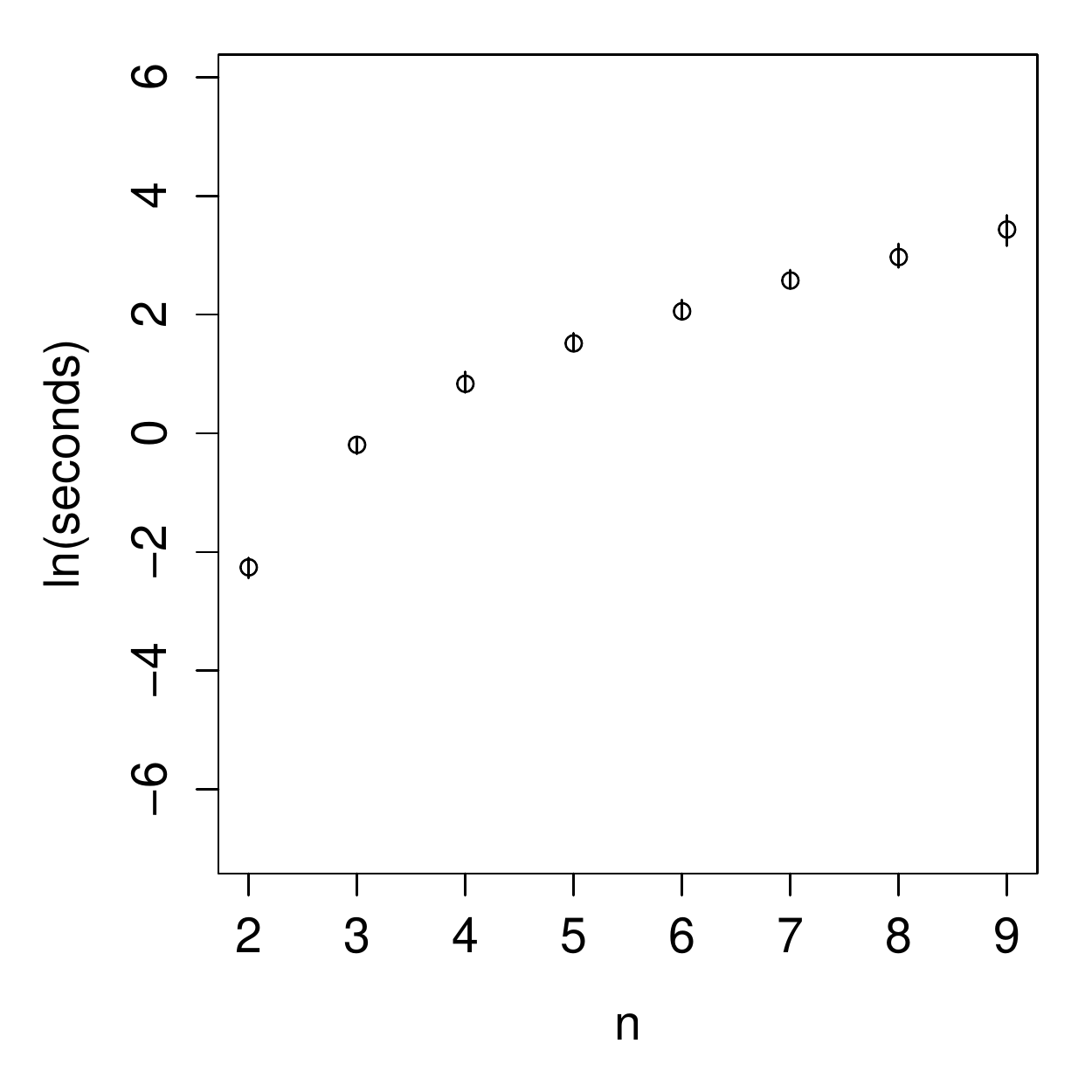}} \\
	1000 & \raisebox{-.5\height}{\includegraphics[scale=0.27]{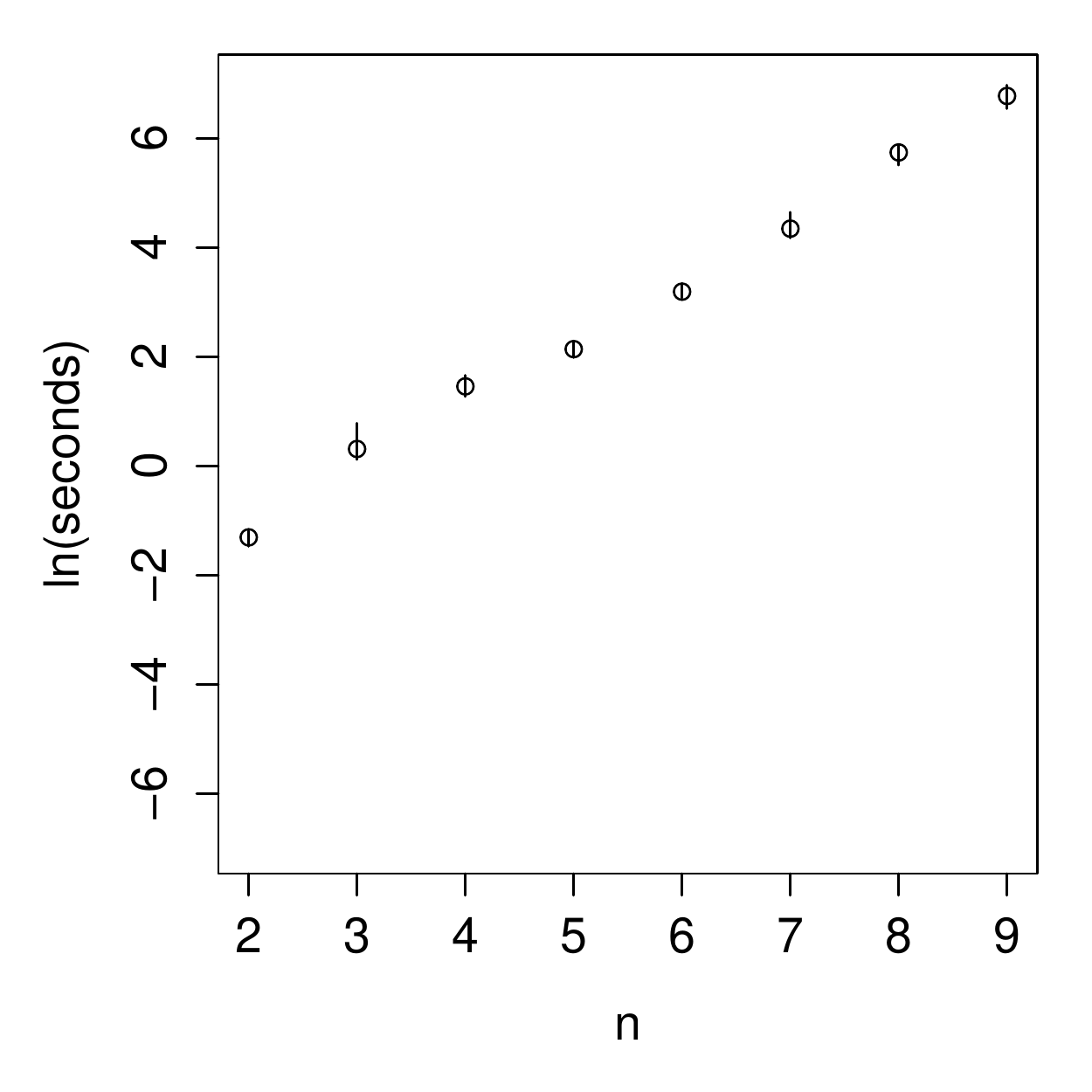}} & \raisebox{-.5\height}{\includegraphics[scale=0.27]{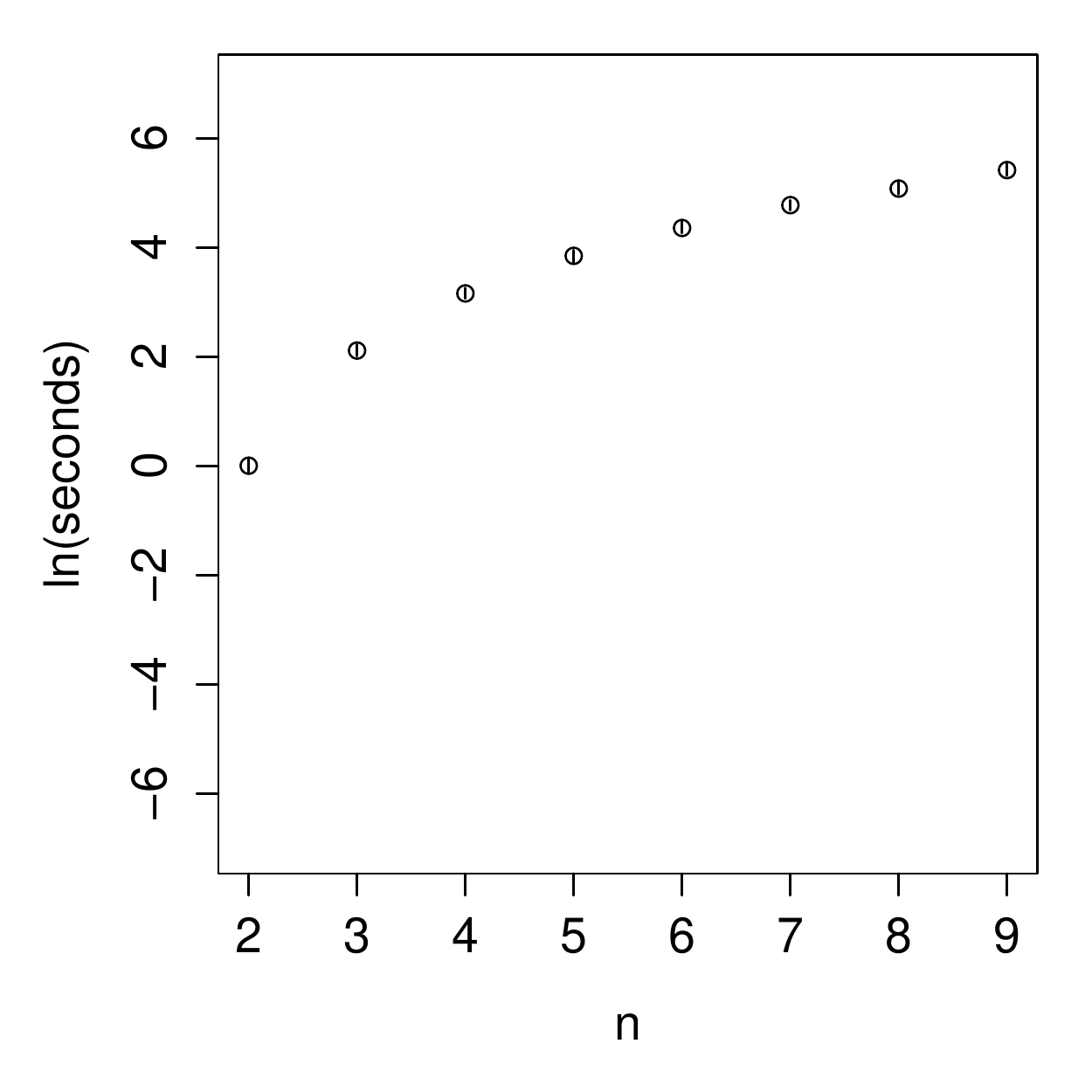}}
	\end{tabular}
	\caption[Comparing MSEs and running times from piecewise learning to L-BFGS with restart for Clayton grids.]{Comparing MSEs and running times on log scale from piecewise learning to L-BFGS with restart for Clayton grids.}
	\label{fig:piecewiseGrids}
\end{figure}

For the trees, the running times for both algorithms increases linearly in the model size. As $n$ increases, the number of variables increases exponentially, so a linear relationship on the log scale indicates that the running time increases linearly in the number of variables. Thus, for trees, L-BFGS with restart converges faster than piecewise learning for all model sizes.

For the grids, the running time of L-BFGS with restart increases exponentially in the model size (linearly in log-space), whereas piecewise learning increases only linearly; the treewidth of the model increases linearly in $n$, whereas piecewise learning operates on subproblems of unity treewidth. Thus, we have reduced the asymptotic complexity of learning grid structures. Whilst L-BFGS with restart is faster for small grids, piecewise learning is faster for grids of size $n\ge7$. 

Furthermore, piecewise learning is currently the only algorithm that can learn models of arbitrary treewidth and size.

\section{Limitations of the model}\label{sec:limitations}
The property of marginal independence severely limits the ability of CDNs to capture nonlocal structure, as we shall demonstrate with an example.
\begin{ex}
 Consider a chain of size $n=3$, with variables $\{X_1,X_2,X_3\}$ in the order of the graph. In this model, the variables represent the binary states of three adjacent pixels. As our models are continuous, we represent the binary variables with normally distributed marginals, and use the convention that a pixel $i$ is set when $X_i>0$, and is clear otherwise. There is a bivariate normal factor between each pair of connected variables, and the parameters are learnt over two samples, $(-1,-1,-1)$ and $(1,1,1)$, representing two classes. If our model captures the nonlocal structure, samples should have either all pixels set, or all pixels clear.
\end{ex}
One might incorrectly reason, as follows, that samples from the model belong to either class. First, we sample $X_1$. Then, we sample $X_2$, which is dependent on $X_1$, and, as we learnt the parameters from the sample, agrees in value with $X_1$. Finally, we sample $X_3$, which is dependent on $X_2$ and conditionally dependent on $X_1$ given $X_2$. Therefore, it should agree with the other variables as well.

In the conditional method, however, the order of sampling is irrelevant, so that sampling in any order produces a sample from the same distribution. Suppose we sample $X_1$ and then $X_3$. These variables are marginally independent, since they are not connected by an edge. Therefore, we expect them to take different values half of the time, or rather, half of the samples belong to neither class on average. This was confirmed experimentally.

To capture the nonlocal structure, an edge is required between the end variables (to form a loop of size $n=3$). With this extra edge, the end variables $X_1$ and $X_3$ are no longer marginally independent. Regardless of the sampling order, the second sampled variable is dependent on the first sampled variable and must agree with it, and similarly for the third variable. Again, we confirmed experimentally that samples drawn from this model belonged to either class.

At the start of this research project, it was hoped that the CDN could form a strong model of binary shape. On the contrary, from this discussion and some further experiments omitted here, it is clear that the model is not suitable for this purpose; to capture nonlocal structure an edge and factor is required between every variable belonging to the structure, and the resulting increase in treewidth makes such a model intractable.

%% file: conclusions.tex
\chapter{Conclusions}
\label{conclusions}
In this thesis we have presented a general class of models, the copula CDN, that is capable of representing and manipulating high-dimensional probability distributions. We explained how to equip the model with a copula parameterization so that the marginals and dependence structure can be specified separately. We presented how to perform inference, sampling, and learning in the models, and studied properties of these algorithms using Monte Carlo simulations. We developed the first sampling algorithm for CDNs, an algorithm for efficiently learning from MCAR data, a novel learning method for large treewidth and high-dimensional models, and a method for performing gradient based optimization on CDNs parameterized with normal copulae.

In section \ref{sec:limitations}, we discussed the limitations of CDNs in representing nonlocal dependencies caused by the model's property of marginal independence. We believe this substantially diminishes the utility of pure CDN models. Instead, we hold the conditional independencies encodable by CDNs to \emph{complement} rather than supplant those of standard PGMs. The most promising research direction is, we believe, the approach of Silva et al. \cite{SilvaEtAl2011}, combining the CDN and BN frameworks to form a convenient parameterization of acyclic directed mixed graphs---the \emph{mixed CDN}.

As the authors comment, the addition of bidirected edges allows the inclusion of implicit latent variables, whose parameterization and connectivity does not have to be specified. These can complement or replace the modelling of explicit latent variables. Indeed, additional bidirected edges in mixed CDNs do not complicate inference and learning, in contrast to the inclusion of latent variables in standard PGMs. One possible application is to augment a BN model for medical diagnosis with bidirected edges representing confounding factors between symptoms. Another possible application is to the modelling of gene regulatory networks; a bidirected edge would represent the regulation of two genes by an unknown molecule, the discovery of such edges during structure learning motivating further biological research. An application is given in \cite{Silva2012} to modelling the patterns in a national survey of NHS staff, where a combination of standard latent variables and bidirected edges are used to represent the ``known unknowns'' and ``unknown unknowns,'' respectively. We suggest that for domains without latent or confounding variables, the standard BN framework is preferable.


A further application of mixed CDNs is to form a general classifier. For example, one could start with a standard naive Bayes classifier and learn the structure of a bidirected model over the observed variables. In other words, the framework of mixed CDNs allows one to weaken the assumption of independence between the observed variables in a naive Bayes classifier \emph{without modelling a direct dependence between the variables or introducing explicit latent variables}. The marginals could be represented non-parametrically with their kernel density estimate, as in \cite{Elidan2012b}. It would be interesting to compare the performance of such a model to the support vector machine, the current state-of-the-art, and the newly developed copula network classifier \cite{Elidan2012b}.

Despite the fact that in this thesis our novel algorithms were developed for pure CDNs , variants are applicable and useful for mixed CDN models. For example, piecewise composite likelihood learning could be incorporated into an algorithm for structure learning over very high-dimensional spaces. Unanswered questions include how to apply MCMC sampling methods to efficiently sample from models with high tree width, and how to perform approximate, MAP, and marginal MAP inference. Since marginalization is trivial in mixed CDNs, marginal MAP inference should have comparable tractability to MAP inference, in contrast to standard PGMs.